\documentclass{article} %
\usepackage{times}
\usepackage{eso-pic} %
\leftmargini\leftmargin \leftmarginii 2em
\newif\ificlrfinal
\iclrfinalfalse
\def\iclrfinalcopy{\iclrfinaltrue}

\usepackage{fancyhdr}
\makeatletter
\def\maketitle{\par
\begingroup
   \def\thefootnote{\fnsymbol{footnote}}
   \def\@makefnmark{\hbox to 0pt{$^{\@thefnmark}$\hss}} %
   \long\def\@makefntext##1{\parindent 1em\noindent
                            \hbox to1.8em{\hss $\m@th ^{\@thefnmark}$}##1}
   \@maketitle \@thanks
\endgroup
\setcounter{footnote}{0}
\let\maketitle\relax \let\@maketitle\relax
\gdef\@thanks{}\gdef\@author{}\gdef\@title{}\let\thanks\relax}
\makeatother

\makeatletter
\def\@maketitle{\vbox{\hsize\textwidth
{\LARGE\sc \@title\par}
\ificlrfinal
    \lhead{Published as a conference paper at ICLR 2022}
    \def\And{\end{tabular}\hfil\linebreak[0]\hfil
            \begin{tabular}[t]{l}\bf\rule{\z@}{24pt}\ignorespaces}%
  \def\AND{\end{tabular}\hfil\linebreak[4]\hfil
            \begin{tabular}[t]{l}\bf\rule{\z@}{24pt}\ignorespaces}%
    \begin{tabular}[t]{l}\bf\rule{\z@}{24pt}\@author\end{tabular}%
\else
       \lhead{Under review as a conference paper at ICLR 2022}
   \def\And{\end{tabular}\hfil\linebreak[0]\hfil
            \begin{tabular}[t]{l}\bf\rule{\z@}{24pt}\ignorespaces}%
  \def\AND{\end{tabular}\hfil\linebreak[4]\hfil
            \begin{tabular}[t]{l}\bf\rule{\z@}{24pt}\ignorespaces}%
    \begin{tabular}[t]{l}\bf\rule{\z@}{24pt}Anonymous authors\\Paper under double-blind review\end{tabular}%
\fi
\vskip 0.3in minus 0.1in}}
\makeatother

\renewenvironment{abstract}{\vskip.075in\centerline{\large\sc
Abstract}\vspace{0.5ex}\begin{quote}}{\par\end{quote}\vskip 1ex}

\makeatletter
\def\section{\@startsection {section}{1}{\z@}{-2.0ex plus
    -0.5ex minus -.2ex}{1.5ex plus 0.3ex
minus0.2ex}{\large\sc\raggedright}}

\def\subsection{\@startsection{subsection}{2}{\z@}{-1.8ex plus
-0.5ex minus -.2ex}{0.8ex plus .2ex}{\normalsize\sc\raggedright}}
\def\subsubsection{\@startsection{subsubsection}{3}{\z@}{-1.5ex
plus      -0.5ex minus -.2ex}{0.5ex plus
.2ex}{\normalsize\sc\raggedright}}
\def\paragraph{\@startsection{paragraph}{4}{\z@}{1.5ex plus
0.5ex minus .2ex}{-1em}{\normalsize\bf}}
\def\subparagraph{\@startsection{subparagraph}{5}{\z@}{1.5ex plus
  0.5ex minus .2ex}{-1em}{\normalsize\sc}}

\makeatother

\skip\footins 9pt plus 4pt minus 2pt
\def\footnoterule{\kern-3pt \hrule width 12pc \kern 2.6pt }
\leftmargini\leftmargin \leftmarginii 2em
\def\@listi{\leftmargin\leftmargini}
\def\@listii{\leftmargin\leftmarginii
   \labelwidth\leftmarginii\advance\labelwidth-\labelsep
   \topsep 2pt plus 1pt minus 0.5pt
   \parsep 1pt plus 0.5pt minus 0.5pt
   \itemsep \parsep}
\def\@listiii{\leftmargin\leftmarginiii
    \labelwidth\leftmarginiii\advance\labelwidth-\labelsep
    \topsep 1pt plus 0.5pt minus 0.5pt
    \parsep \z@ \partopsep 0.5pt plus 0pt minus 0.5pt
    \itemsep \topsep}
\def\@listiv{\leftmargin\leftmarginiv
     \labelwidth\leftmarginiv\advance\labelwidth-\labelsep}
\def\@listv{\leftmargin\leftmarginv
     \labelwidth\leftmarginv\advance\labelwidth-\labelsep}
\def\@listvi{\leftmargin\leftmarginvi
     \labelwidth\leftmarginvi\advance\labelwidth-\labelsep}

\belowdisplayskip \abovedisplayskip
\makeatletter
\def\normalsize{\@setsize\normalsize{11pt}\xpt\@xpt}
\def\small{\@setsize\small{10pt}\ixpt\@ixpt}
\def\footnotesize{\@setsize\footnotesize{10pt}\ixpt\@ixpt}
\def\scriptsize{\@setsize\scriptsize{8pt}\viipt\@viipt}
\def\tiny{\@setsize\tiny{7pt}\vipt\@vipt}
\def\large{\@setsize\large{14pt}\xiipt\@xiipt}
\def\Large{\@setsize\Large{16pt}\xivpt\@xivpt}
\def\LARGE{\@setsize\LARGE{20pt}\xviipt\@xviipt}
\def\huge{\@setsize\huge{23pt}\xxpt\@xxpt}
\def\Huge{\@setsize\Huge{28pt}\xxvpt\@xxvpt}
\makeatother

\usepackage[utf8]{inputenc} %
\usepackage[T1]{fontenc}    %
\usepackage{hyperref}       %
\usepackage{url}            %
\usepackage{booktabs}       %
\usepackage{amsfonts}       %
\usepackage{nicefrac}       %
\usepackage{microtype}      %
\usepackage{xcolor}         %
\usepackage{natbib}
\usepackage{cancel}
\usepackage{graphicx}       
\usepackage{graphics}
\usepackage{subcaption}
\usepackage{float}

\usepackage[acronym,smallcaps,nowarn,section,nogroupskip,nonumberlist]{glossaries}
\usepackage[nameinlink,capitalise]{cleveref}

\usepackage{amsmath,amsfonts,bm}
\usepackage{graphicx}       
\usepackage{graphics}
\usepackage{float}
\usepackage{hyperref}
\usepackage{xcolor}
\usepackage{wrapfig}
\usepackage[export]{adjustbox}
\usepackage{float}
\usepackage{amssymb}
\usepackage{enumitem}
\usepackage{mathtools}
\usepackage{slashbox}
\usepackage{minitoc}

\usepackage{enumitem}
\usepackage[titletoc,title]{appendix}

\usepackage{amsthm}
\usepackage{thmtools}
\usepackage{thm-restate}
\usepackage{booktabs}
\usepackage{nicefrac}       %
\usepackage{microtype}      %
\usepackage{amsfonts}       %
\usepackage{harpoon}
\usepackage{tabularx}
\usepackage{multirow}
\usepackage{minitoc}

\newcommand{\nocontentsline}[3]{}
\newcommand{\tocless}[2]{\bgroup\let\addcontentsline=\nocontentsline#1{#2}\egroup}

\newcommand{\vheader}{\vspace*{-.165cm}}

\newcommand{\ext}{\text{ext}}

\newcommand{\myoverset}[2]{{\substack{#1 \\ #2}}}

\def\eqref#1{equation~\ref{#1}}

\def\1{\bm{1}}

\DeclareMathAlphabet{\mathsfit}{\encodingdefault}{\sfdefault}{m}{sl}
\SetMathAlphabet{\mathsfit}{bold}{\encodingdefault}{\sfdefault}{bx}{n}

\hypersetup{
    colorlinks,
    linkcolor={blue!50!black},
    citecolor={blue!50!black},
    urlcolor={blue!50!black}
}

\newcommand{\myunderbrace}[2]{\underbrace{#1}_{\text{\small $#2$}}}

\newcommand{\argsalt}{\vx, \vz_{\text{ext}}}
\newcommand{\argsaltcond}{\vz_{\text{ext}}|\vx}
\newcommand{\argsaltZ}{\vz_{\text{ext}}}
\newcommand{\args}{\argsalt}
\newcommand{\argscond}{\argsaltcond}
\newcommand{\argsZ}{\argsaltZ}
\newcommand{\giwaeT}{T_{\phi}}
\newcommand{\idx}{{s}}
\newcommand{\qprop}[1]{q_{\textsc{prop}}^{\textsc{#1}}}

\newcommand{\ptgt}[1]{p_{\textsc{tgt}}^{\textsc{#1}}}

 \newcommand{\udist}{\mathcal{U}}

\newcommand{\mycor}[1]{\hyperref[cor:#1]{Cor.~\ref*{cor:#1}}}
\newcommand{\mycora}[2]{\hyperref[cor:#1]{Cor.~\ref*{cor:#1}(#2)}}
\newcommand{\myeq}[1]{\hyperref[eq:#1]{Eq.~(\ref*{eq:#1})}}
\newcommand{\mysecondeq}[1]{\hyperref[eq:#1]{(\ref*{eq:#1})}}
\newcommand{\mysec}[1]{\hyperref[sec:#1]{Sec.~\ref*{sec:#1}}}
\newcommand{\mytable}[1]{\hyperref[table:#1]{Table~\ref*{table:#1}}}
\newcommand{\myfig}[1]{\hyperref[fig:#1]{Fig.~\ref*{fig:#1}}}
\newcommand{\myfiga}[2]{\hyperref[fig:#1]{Fig.~\ref*{fig:#1}#2}}
\newcommand{\myfigab}[3]{\hyperref[fig:#1]{Fig.~\ref*{fig:#1}#2,#3}}
\newcommand{\myfigabc}[4]{\hyperref[fig:#1]{Fig.~\ref*{fig:#1}#2,#3,#4}}
\newcommand{\myfigabcd}[5]{\hyperref[fig:#1]{Fig.~\ref*{fig:#1}#2,#3,#4,#5}}
\newcommand{\myapp}[1]{\hyperref[app:#1]{App.~\ref*{app:#1}}}
\newcommand{\mylemma}[1]{\hyperref[lemma:#1]{Lemma~\ref*{lemma:#1}}}
\newcommand{\myprop}[1]{\hyperref[prop:#1]{Prop.~\ref*{prop:#1}}}
\newcommand{\mypropa}[2]{\hyperref[prop:#1]{Prop.~\ref*{prop:#1}(#2)}}
\newcommand{\tqzxshort}{\tilde{\pi}} %
\newcommand{\tqzxnoparam}{\tqzxshort(\vz|\vx)}
\newcommand{\baseminef}{\qzx}
\newcommand{\tqzx}{\tqzxshort_{\theta,\phi}(\vz|\vx)}
\newcommand{\tqzxnorm}{\pi_{\theta,\phi}(\vz|\vx)}

\newcommand{\normq}{\mathcal{Z}_{\tqzxshort}(\vx)}

\newcommand{\GKL}{D_{\textsc{GKL}}}
\newcommand{\DKL}{D_{\textsc{KL}}}
\DeclareRobustCommand{\parhead}[1]{\textbf{#1}~}

\newcommand{\indexposterior}{\ptgt{giwae}}

\newcommand{\bb}[1]{\mathbf{#1}}
\newcommand{\bbb}{\bb{b}}
\newcommand{\bx}{\bb{x}}

\newcommand{\bz}{\bb{z}}
\newcommand{\bSigma}{\boldsymbol{\Sigma}}
\newcommand{\bmu}{\boldsymbol{\mu}}

\newcommand{\bzero}{\bb{0}}

\newcommand{\bW}{\bb{W}}

\newcommand{\bWT}{\bb{W}^\intercal}

\newcommand{\bI}{\bb{I}}

\newcommand{\Exp}[2]{\mathbb{E}_{#1}\left[#2\right]}

\newcommand{\be}{\begin{eqnarray} \begin{aligned}}
\newcommand{\ee}{\end{aligned} \end{eqnarray} }
\newcommand{\benn}{\begin{eqnarray*} \begin{aligned}}
\newcommand{\eenn}{\end{aligned} \end{eqnarray*} }

\newcommand{\vx}{\mathbf{x}}
\newcommand{\vz}{\mathbf{z}}

\newcommand{\vy}{\mathbf{y}}
\newcommand{\qzx}{q_\theta(\vz|\vx)}

\newcommand{\qzxi}[1]{\prod \limits_{k=2}^K q_{\theta}(\vz^{(k)}|\vx)}
\newcommand{\qzxiall}[1]{\prod \limits_{k=1}^K q_{\theta}(\vz^{(k)}|\vx)}

\newcommand{\pzx}{p(\vz|\vx)}
\newcommand{\pxandz}{p(\vx,\vz)}

\newcommand{\Ixz}{{I}(\vx;\vz)}

\newcommand{\IBAL}{\textsc{ibal}} %
\newcommand{\ibal}{\IBAL}

\newcommand{\px}{p(\vx)}
\newcommand{\KL}[2]{D_{\text{KL}}[{#1}\|{#2}]}
\newcommand{\Txz}{\giwaeT(\vx,\vz)}%
\newcommand{\fpartial}{\frac{\partial}{\partial \theta}}
\newcommand{\fpartialf}[1]{\frac{\partial}{\partial {#1}}}

\newcommand{\logZmine}{\log \mathcal{Z}_{\theta,\phi}(\vx)}
\newcommand{\logZgiwae}{\log {\mathcal{Z}}_{\gls{GIWAE}}(\vx, K) }
\newcommand{\logZgiwaelb}{{\textsc{lb}(\vx;K)}}%
\newcommand{\logZgiwaeub}{{\textsc{ub}(\vx;K) }} %

\newcommand{\Hxz}{H(\vx|\vz)}
\newcommand{\Hx}{H(\vx)}

\newcommand{\h}{t}

\newcommand{\K}{{K}}
\newcommand{\kk}{{s}}

\newcommand{\jj}{{k}}
\newcommand{\p}{{p}} %
\newcommand{\prop}{{\pi_0}} %
\newcommand{\tpi}{{\tilde{\pi}}}
\newcommand{\dist}{{\pi}} %

\newcommand{\minevar}{\pi}
\newcommand{\minetparam}{\phi}

\newcommand{\tfwd}{{\mathcal{T}_t}}
\newcommand{\trev}{{\tilde{\mathcal{T}}_t}}

\newcommand{\tfwdphi}{\mathcal{T}_t}%
\newcommand{\trevphi}{{\tilde{\mathcal{T}}_t}}%
\newcommand{\tqrev}{\tilde{\mathcal{T}}_t^{q}}
\newcommand{\ptgtpost}{p^{\textsc{approx}}_{\textsc{tgt}}} %
\newcommand{\ptgtenergy}{p^{\textsc{ais}, \pi}_{\textsc{tgt}}}
\newcommand{\qpropenergy}{q^{\textsc{ais},\pi}_{\textsc{prop}}}
\newcommand{\ptgtgiwae}{p_{\textsc{tgt}}^{\textsc{giwae}, {\pi}}}
\newcommand{\argsais}{\vz_{0:T}}

\newcommand{\pz}{{p(\vz)}}

\newcommand{\pxgz}{{p(\vx|\vz)}}

\newcommand{\setofz}{{\vz^{(1:K)}}} %

\newcommand{\scriptveryshortarrow}[1][3pt]{{%
    \hbox{\rule[\scriptratio\dimexpr\fontdimen22\textfont2-.2pt\relax]
               {\scriptratio\dimexpr#1\relax}{\scriptratio\dimexpr.4pt\relax}}%
   \mkern-4mu\hbox{\let\f@size\sf@size\usefont{U}{lasy}{m}{n}\symbol{41}}}}

\definecolor{ashgrey}{rgb}{0.7, 0.75, 0.71}

\newcommand{\kl}{\textsc{kl} }
\newacronym{AIS}{ais}{annealed importance sampling}
\newacronym{BA}{ba}{Barber-Agakov}
\newacronym{BQ}{bq}{Bayesian Quadrature}
\newacronym{AUC}{auc}{area under the curve}
\newacronym{BAR}{bar}{Bennett's Acceptance Ratio}
\newacronym{GAN}{gan}{generative adversarial network}
\newacronym{BDMC}{bdmc}{Bidirectional Monte Carlo}
\newacronym{CLUB}{club}{Contrastive Log-Ratio Upper Bound}
\newacronym{JS}{js}{Jensen-Shannon}
\newacronym{CFT}{cft}{Crooks's Fluctuation Theorem}
\newacronym{ELBO}{elbo}{Evidence Lower Bound}
\newacronym{EUBO}{eubo}{Evidence Upper Bound}
\newacronym{HMC}{hmc}{Hamiltonian Monte Carlo}
\newacronym{IB}{ib}{Information Bottleneck}
\newacronym{IBAL}{ibal}{\textit{Implicit Barber-Agakov Lower bound}}
\newacronym{MI}{mi}{Mutual information}
\newacronym{MINE}{mine}{Mutual Information Neural Estimation}
\newacronym{JE}{je}{Jarzynksi equality}
\newacronym{IS}{is}{importance sampling}
\newacronym{IWAE}{iwae}{importance-weighted autoencoder}
\newacronym{GIWAE}{giwae}{\textit{Generalized} \textsc{iwae}}
\newacronym{RAISE}{raise}{Reverse \gls{AIS} Estimator}
\newacronym{MCMC}{mcmc}{Markov Chain Monte Carlo}
\newacronym{RD}{rd}{rate-distortion}
\newacronym{RWS}{rws}{reweighted wake-sleep}
\newacronym{RBM}{rbm}{Restricted Boltzmann Machines}
\newacronym{SGD}{sgd}{stochastic gradient descent}
\newacronym{SNIS}{snis}{self-normalized importance sampling}
\newacronym{TI}{ti}{thermodynamic integration}
\newacronym{TVI}{tvi}{thermodynamic variational inference}
\newacronym{TVO}{tvo}{thermodynamic variational objective}
\newacronym{VAE}{vae}{variational autoencoders}
\newacronym{VAEc}{vae}{Variational Autoencoders}
\newacronym{VI}{vi}{variational inference}
\newacronym{VIMCO}{vimco}{variational inference for Monte Carlo objectives}
\newacronym{WS}{ws}{wake-sleep}

\setlist[itemize]{leftmargin=*}

\makeatletter
\def\thm@space@setup{%
  \thm@preskip=\parskip \thm@postskip=0pt
}
\makeatother

\newtheorem{theorem}{Theorem}[section]

\newtheorem{proposition}[theorem]{Proposition}
\newtheorem{corollary}[theorem]{Corollary}

\title{Improving Mutual Information Estimation \\with Annealed and Energy-Based Bounds
}

\author{Rob Brekelmans\thanks{\small Equal Contribution.  
Correspondence to
 \href{mailto:brekelma@usc.edu; huang@cs.toronto.edu; makhzani@cs.toronto.edu}{brekelma@usc.edu; \{huang, makhzani\}@cs.toronto.edu. 
{A shorter version appeared in the International Conference on Learning Representations (ICLR) 2022, available \href{https://openreview.net/forum?id=T0B9AoM_bFg}{here}.}}
} $^{\, , 1}$ \\
\And
 \hspace{1cm}  Sicong Huang$^{*,2,3}$  \\
\And
 \hspace{1cm}  Marzyeh Ghassemi$^{2,4}$  \\
\And
Greg Ver Steeg$^1$ \\
\And
Roger Grosse$^{2,3}$ \\
\And
Alireza Makhzani$^{2,3}$  \\ \And
  $^1$ \normalfont{Information Sciences Institute, University of Southern California} \\
  $^2$ \normalfont{Vector Institute} \quad $^3$ University of Toronto \quad $^4$ \normalfont{MIT EECS / IMES / CSAIL}
}

\glsdisablehyper{}

\iclrfinalcopy %
\begin{document}
\doparttoc %
\faketableofcontents %
\maketitle

\begin{abstract}
Mutual information (\textsc{mi}) is a fundamental quantity in information theory and machine learning. However, direct estimation of \textsc{mi} is intractable, even if the true joint probability density for the variables of interest is known, as it involves estimating a potentially high-dimensional log partition function. In this work, we present a unifying view of existing \textsc{mi} bounds from the perspective of importance sampling, and propose three novel bounds based on this approach. Since accurate estimation of \textsc{mi} without density information requires a sample size exponential in the true \textsc{mi}, we assume either a single marginal or the full joint density information is known.  In settings where the full joint density is available, we propose Multi-Sample Annealed Importance Sampling (\textsc{ais}) bounds on \textsc{mi}, which we demonstrate can tightly estimate large values of \textsc{mi} in our experiments. In settings where only a single marginal distribution is known, we propose \textit{Generalized} \textsc{iwae} (\textsc{giwae}) and \textsc{mine-ais} bounds. Our \textsc{giwae} bound unifies variational and contrastive bounds in a single framework that generalizes \textsc{InfoNCE}, \textsc{iwae}, and Barber-Agakov bounds. Our \textsc{mine-ais} method improves upon existing energy-based methods such as \textsc{mine-dv} and \textsc{mine-f} by directly optimizing a tighter lower bound on \textsc{mi}. \textsc{mine-ais} uses \textsc{mcmc} sampling to estimate gradients for training and Multi-Sample \textsc{ais} for evaluating the bound. Our methods are particularly suitable for evaluating \textsc{mi} in deep generative models, since explicit forms of the marginal or joint densities are often available. We evaluate our bounds on estimating the \textsc{mi} of \textsc{vae}s and \textsc{gan}s trained on the \textsc{mnist} and \textsc{cifar} datasets, and showcase significant gains over existing bounds in these challenging settings with high ground truth \textsc{mi}.

\end{abstract}
\glsresetall

\section{Introduction}\label{sec:intro}
\gls{MI} is 
among the most general measures of dependence between two random variables.  
Among many other applications in machine learning, mutual information has been used for both training \citep{alemi2016deep, alemi2018fixing, chen2016infogan, zhao2018information} %
and evaluating \citep{alemi2018gilbo, huang2020evaluating} generative models.  
 Furthermore, recent success in neural network function approximation has encouraged a wave of variational or contrastive methods for \gls{MI} estimation from samples alone \citep{belghazi2018mutual, oord2018representation,  poole2019variational}. 
{
However, \citet{mcallester2020formal} have shown that for any estimator that uses direct sampling from the product of marginals and does not have access to the analytical form of at least one marginal distribution, exponential sample complexity in the true \gls{MI} is required to obtain a high confidence lower bound.} This is particularly concerning in applications such as representation learning and generative modeling where we expect the \gls{MI} to be large.  In light of these limitations, we primarily consider settings where a single marginal or the full joint distribution are available.  Even in these cases, evaluating \gls{MI} can be challenging due to the need to estimate a potentially high-dimensional log partition function.

In this work, we view \gls{MI} estimation from the perspective of importance sampling, which sheds light on the limitations of existing estimators and motivates our search for improved proposal distributions 
using \gls{MCMC}.
Using a general approach for constructing extended state space bounds on \gls{MI}, we combine insights from the \gls{IWAE} \citep{burda2016importance, sobolev2019hierarchical} and \gls{AIS} \citep{neal2001annealed} to propose \textit{Multi-Sample} \gls{AIS} bounds in \mysec{ais_estimation}.    We empirically show that this approach can 
tightly estimate large values of \gls{MI} when the full joint distribution is known. 
Our importance sampling perspective also suggests improvements upon several existing energy-based lower bounds on \gls{MI}, where our proposed methods only assume access to joint samples for optimization, but require a single marginal distribution for evaluation.
In \mysec{gen_iwae}, we propose \textit{Generalized} \textsc{iwae} (\textsc{giwae}), which generalizes both \gls{IWAE} and \textsc{InfoNCE} \citep{oord2018representation, poole2019variational} and highlights how variational learning can complement multi-sample contrastive estimation to improve \gls{MI} lower bounds.

Finally, in \mysec{mine-ais}, we propose \textit{\textsc{mine}-\gls{AIS}},
which optimizes a tighter lower bound than \gls{MINE} \citep{belghazi2018mutual} 
using a stable energy-based training procedure.
We denote this bound as the \textit{Implicit Barber Agakov Lower bound} (\textsc{ibal}), and demonstrate that it corresponds to the infinite-sample limit of the \textsc{giwae} lower bound.
However, our training scheme involves only a single `negative' contrastive sample obtained using \gls{MCMC}.
\textsc{mine-ais} then uses Multi-Sample \gls{AIS} to evaluate a lower bound on \gls{MI} for a given energy function and known marginal, and shows notable improvement over existing variational bounds in the challenging setting of \gls{MI} estimation for deep generative models.

\begin{figure}[t]
\centering
\vspace*{-1cm}
\includegraphics[scale=.75]{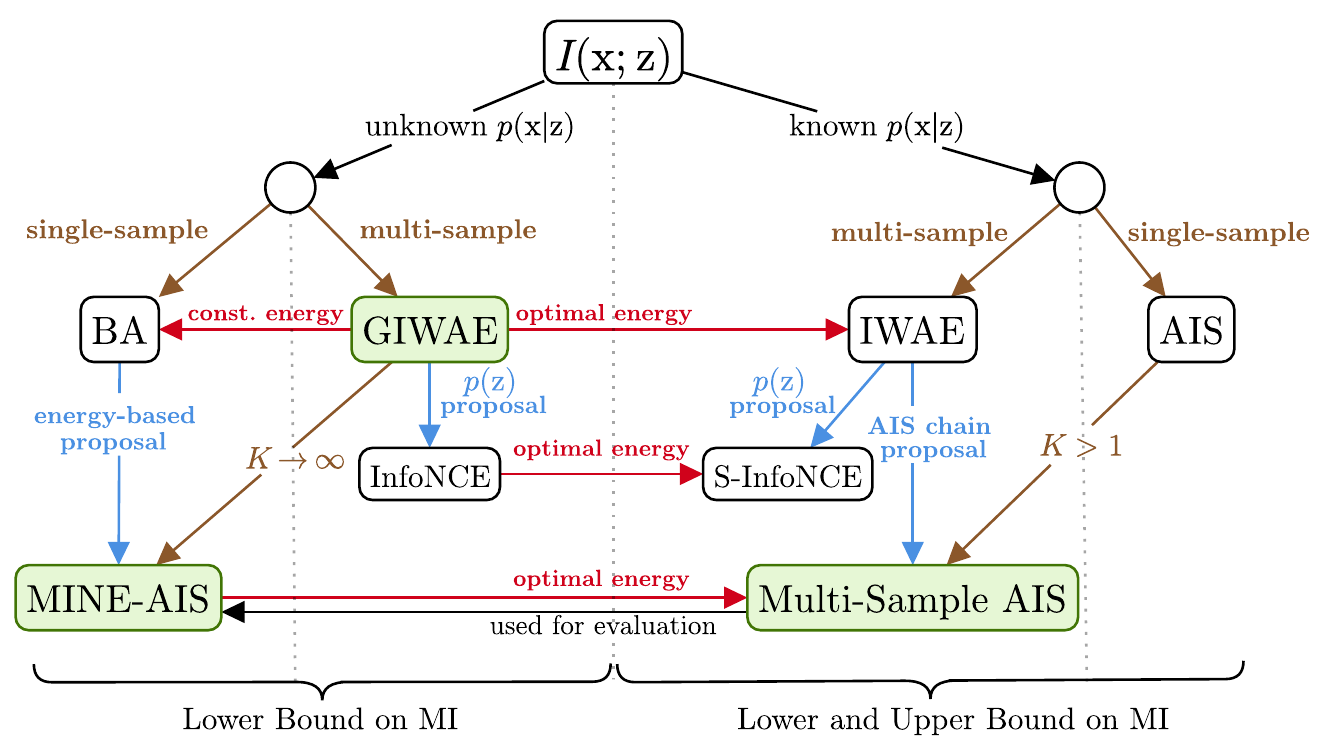} \caption{\label{fig:all_bounds} Schematic of various \gls{MI} bounds discussed in this paper.  Green shading indicates our contributions, while columns and gold labels indicate single- or multi-sample bounds.  Blue arrows indicate special cases using the indicated proposal distribution.  
Several bounds with unknown $p(\vx|\vz)$ use learned energy or critic functions, where the optimal critic function reflects the true $p(\vx|\vz)$.   Relationships based on critic functions are indicated by red arrows.   
Bounds with unknown $p(\vx|\vz)$ provide only lower bounds on \gls{MI}, while we obtain both upper and lower bounds with known $p(\vx|\vz)$.  All bounds require a single known marginal $p(\vz)$ for evaluation, apart from (Structured) \textsc{Info-NCE}.}
\vspace*{-.2cm} 
\end{figure}%

\subsection{Problem Setting}\label{sec:setting}
The mutual information between two random variables $\vx$ and $\vz$ with joint distribution $p(\vx,\vz)$ is %
\begin{align}
 \hspace*{-.2cm} \Ixz = \Exp{p(\vx, \vz)}{\log \frac{p(\vx, \vz)}{p(\vx)p(\vz)}} = \Hx - \Hxz %
    = \mathbb{E}_{p(\vx,\vz)}[\log p(\vx|\vz)] - \mathbb{E}_{p(\vx)}[\log p(\vx)], \label{eq:mi_x_decomposition}
\end{align}
\normalsize
where $\Hxz$ denotes the conditional entropy $-\mathbb{E}_{p(\vx,\vz)}\log p(\vx|\vz)$.
We primarily focus on bounds which assume either a single marginal distribution or the full joint distribution are available.
A natural setting where the full joint distribution is available is estimating \gls{MI} in deep generative models between the latent variables, with a known prior $\vz \sim p(\vz)$, and data $\vx \sim p(\vx)$ simulated from the model \citep{alemi2018gilbo}.\footnote{An alternative, ``encoding'' \gls{MI} between the real data and the latent code is often of interest (see \myapp{representation}), but cannot be directly estimated using our methods due to the unavailability of $p_d(\vx)$ or $q(\vz) =\int p_d(\vx)q(\vz|\vx)d\vx$.} 
Settings where only a single marginal 
is available appear, for example, in simulation-based inference \citep{cranmer2020frontier}, where information about input parameters $\theta$ is known and a simulator can generate $\vx$ for a given $\theta$, but the likelihood $p(\vx|\theta)$ is intractable.

While sampling from the posterior $p(\vz|\vx)$ for an arbitrary $\vx$ is often intractable, we can obtain a single posterior sample for $\vx \sim p(\vx)$ 
in cases where samples from the joint distribution 
$p(\vx)p(\vz|\vx)$ 
are available.  
Throughout this paper, we will refer to bounds which involve only a single posterior sample as \textit{practical}, and those involving multiple posterior samples as \textit{impractical}.

When the conditional $p(\vx|\vz)$ is tractable to sample and evaluate, simple Monte Carlo sampling provides an unbiased, low variance estimate of the conditional entropy term in \myeq{mi_x_decomposition}.   In this case, the difficulty of \gls{MI} estimation reduces to estimating the log partition function, for which \gls{IS} based methods are among the most well studied and successful solutions.

\section{Unifying Mutual Information Bounds via Importance %
Sampling}\label{sec:background}%

In this section, we 
present a unified view of mutual information estimation from the perspective of importance sampling, which yields new insights into existing bounds and will provide the foundation for our contributions in \mysec{ais_estimation} and \mysec{mine-ais}.   

\subsection{A General Approach for Extended State Space Importance Sampling Bounds}\label{sec:general}
Throughout this paper, we use extended state space importance sampling
\citep{finke2015extended, domke2018importance} to derive lower and upper bounds on the log partition function, which translate to bounds on \gls{MI} with known $p(\vx|\vz)$ as in \mysec{setting}.   
This general approach provides a probabilistic interpretation of existing \gls{MI} bounds and will suggest novel extensions in later sections (see \myapp{general}).

In particular, we construct proposal $\qprop{}(\argsaltcond)$ and target $\ptgt{}(\argsalt)$ distributions over an extended state space, such that 
the normalization constant of $\ptgt{}(\argsalt)$ is $\mathcal{Z}_\textsc{tgt} = \int \ptgt{}(\argsalt) d\argsaltZ = p(\vx)$ and the normalization constant of $\qprop{}(\argsalt)$ is 
$\mathcal{Z}_\textsc{prop}=1$. 
Taking expectations of the log importance weight 
$\log \ptgt{}(\argsalt)/\qprop{}(\argsaltcond)$
under the proposal and target, respectively, we obtain lower and upper bounds on the log partition function
\begin{equation}
 \myunderbrace{ \Exp{q_{\textsc{prop}}(\argsaltcond)}{\log \frac{p_\textsc{tgt}(\argsalt)}{q_{\textsc{prop}}(\argsaltcond)} } }{\textsc{elbo}\big(\vx; \qprop{}, \ptgt{}\big)} \leq  
  \log p(\vx)
  \leq \myunderbrace{ \Exp{p_\textsc{tgt}(\argsaltcond)}{\log \frac{p_\textsc{tgt}(\argsalt)}{q_{\textsc{prop}}(\argsaltcond) } } }{ \textsc{eubo}\big(\vx;\qprop{}, \ptgt{}\big)} \,. 
  \label{eq:general_approach_main2}
\end{equation}
\normalsize
These bounds correspond to extended state space versions of the \gls{ELBO} and \gls{EUBO}, respectively.   In particular, the gap in the lower bound is the forward \kl divergence, with $\log \px = \gls{ELBO}\big(\vx;\qprop{}, \ptgt{}\big) + \DKL[\qprop{}(\argsaltcond)\|\ptgt{}(\argsaltcond) ]$ and the gap in the upper bound equal to the reverse \kl divergence $\log \px = \gls{EUBO}\big(\vx;\qprop{}, \ptgt{}\big) - \DKL[\ptgt{}(\argsaltcond) \| \qprop{}(\argsaltcond) ]$.

\subsection{Barber-Agakov Lower and Upper Bounds}\label{sec:ba}
As a first example, consider the standard $\gls{ELBO}(\vx; q_\theta)$ and $\gls{EUBO}(\vx; q_\theta)$ bounds, which are derived from simple importance sampling using a variational distribution $\qzx$ and $\argsaltZ = \vz$ in \cref{eq:general_approach_main2}.
Plugging these 
lower and upper bounds on $\log p(\vx)$ into \myeq{mi_x_decomposition}, we obtain upper and lower bounds on \gls{MI} as
\small
\begin{align}
\hspace*{-.2cm} I_{{\textsc{ba}}_L}(q_\theta) \coloneqq \Exp{\pxandz}{\log \frac{\qzx}{\pz}} \leq \Ixz \leq \Exp{\px \qzx}{\log \frac{\qzx}{\pxandz}} - \Hxz \eqqcolon I_{{\textsc{ba}}_U}(q_\theta) . \label{eq:simple_importance_sampling2}
\end{align}
\normalsize
The left hand side of \myeq{simple_importance_sampling2} is the well-known \gls{BA} bound \citep{barber2003algorithm}, which has a gap of $\mathbb{E}_{p(\vx)}[\KL{\pzx}{\qzx}]$.
We refer to the right hand side as the \gls{BA} upper bound $I_{{\textsc{ba}}_U}(q_\theta)$, with a gap of $\mathbb{E}_{p(\vx)}[\KL{\qzx}{\pzx}]$.  
In contrast to $I_{\gls{BA}_U}(q_\theta)$, note that $I_{\gls{BA}_L}(q_\theta)$ does not require access to the conditional density $p(\vx|\vz)$ to evaluate the bound.   
\subsection{Importance Weighted Autoencoder Lower and Upper Bounds}\label{sec:iwae}\label{sec:partition}\label{sec:std_iwae} 
The \gls{IWAE} lower and upper bounds on $\log p(\vx)$ \citep{burda2016importance, sobolev2019hierarchical} improve upon simple importance sampling by extending the state space using multiple samples $\argsaltZ = \vz^{(1:K)}$ \citep{domke2018importance}. %
Consider a proposal $\qprop{iwae}
(\vz^{(1:K)}|\vx)$ 
with $K$ independent samples from a given variational distribution $\qzx$.   The extended state space target $\ptgt{iwae}
(\vz^{(1:K)}|\vx)$ 
is a mixture distribution involving a single sample from the posterior $p(\vz|\vx)$ or joint $p(\vx, \vz)$ distribution and $K-1$ samples from $\qzx$
\small
\begin{align}
\begin{aligned}
    \qprop{iwae}(\setofz|\vx) 
    \coloneqq 
    \prod \limits_{\kk=1}^{\K} q_{\theta}(\vz^{(\kk)}|\vx) \vphantom{\prod \limits_{\myoverset{k =1}{k \neq \idx}}^{\K}}, \label{eq:iwae_fwdrev2}
\end{aligned}\hspace*{0.5cm}
\begin{aligned}
    \ptgt{iwae}(\vx,\setofz) 
    \coloneqq 
    \frac{1}{\K} \sum \limits_{s=1}^{\K} \, \p(\vx, \vz^{(\idx)}) \, \prod  \limits_{\myoverset{k =1}{k \neq \idx}}^{\K} q_{\theta}(\vz^{(k)}|\vx)  \, . %
\end{aligned}
\end{align}
\normalsize
As in \mysec{general}, taking the expectation of the log importance weight $\log \frac{ \ptgt{iwae}(\vx,\setofz)}{ \qprop{iwae}(\setofz|\vx)}$ under the proposal and target, respectively, yields the \textsc{iwae} lower \citep{burda2016importance} and upper \citep{sobolev2019hierarchical} bounds on $\log \px$,
\normalsize
\begin{equation}
\resizebox{.93\textwidth}{!}{$\myunderbrace{\Exp{\qzxiall{}}{\log \frac{1}{K} \sum_{i=1}^K \frac{p(\vx,\vz^{(k)})}{q_{\theta}(\vz^{(k)}|\vx)}}}{\gls{ELBO}_{\gls{IWAE}}(\vx; q_\theta, K)}
\leq \log\px \leq
\myunderbrace{\Exp{p(\vz^{(1)}|\vx)\qzxi{2:K}}{\log \frac{1}{K} \sum_{i=1}^K \frac{p(\vx,\vz^{(k)})}{q_{\theta}(\vz^{(k)}|\vx)}}}{\gls{EUBO}_{\gls{IWAE}}(\vx; q_\theta, K)}. \,   
$}\label{eq:iwae_ublb}
\end{equation}
See \myapp{iwae} for derivations.  
For simplicity of notation,
we assume
$s=1$ and $\vz^{(1)} \sim \pzx$ when writing the expectation in $\gls{EUBO}_{\gls{IWAE}}(q_\theta, K)$, 
due to invariance of \cref{eq:iwae_fwdrev2} to permutation of the indices.

As for the standard \gls{ELBO} and \gls{EUBO}, the gaps in the lower and upper bounds are $D_{KL}[\qprop{iwae}(\setofz|\vx) \|\ptgt{iwae}(\setofz|\vx) ]$ and $D_{KL}[\ptgt{iwae}(\setofz|\vx) \|\qprop{iwae}(\setofz|\vx) ]$, respectively. 
With known $p(\vx|\vz)$,
the lower and upper bounds on $\log p(\vx)$
translate to upper and lower bounds on \gls{MI}, $I_{\text{IWAE}_U}(q_\theta, K)$ and $I_{\text{IWAE}_L}(q_{\theta},K)$, as in \mysec{setting}. %

\paragraph{Complexity in K}
While it is well-known that increasing $K$ leads to tighter \gls{IWAE} bounds \citep{burda2016importance, sobolev2019hierarchical}, we explicitly characterize the improvement of multi-sample \gls{IWAE} bounds over the single-sample \gls{ELBO} or \gls{EUBO} in the following proposition.  

This proposition lays the foundation 
for similar results throughout the paper.   In particular, any bound which involves expectations under a mixture of one `positive' sample and $K-1$ `negative' samples, such as $\gls{EUBO}_{\gls{IWAE}}(\vx; q_\theta, K)$ in \cref{eq:iwae_ublb} or (\ref{eq:iwae_over_eubo}), will be limited to logarithmic improvement in $K$.
\begin{restatable}[{Improvement of \gls{IWAE} with Increasing $K$}]{proposition}{iwaeelboeubo}\label{prop:iwae_elbo_eubo}
Let $\ptgt{iwae}( \idx |\vx, \vz^{(1:K)})$
$= \frac{p(\vx,\vz^{(\idx)})}{q_{\theta}(\vz^{(\idx)}|\vx)}/\sum_{k=1}^K \frac{p(\vx,\vz^{(k)})}{q_{\theta}(\vz^{(k)}|\vx)}$
denote the normalized importance weights and $\mathcal{U}(s)$ indicate the uniform distribution over $K$ discrete values.  
Then, we can characterize the improvement of $\gls{ELBO}_{\textsc{IWAE}}(\vx; q_\theta,K)$ and $\gls{EUBO}_{\textsc{IWAE}}(\vx; q_\theta,K)$ over $\gls{ELBO}(\vx; q_\theta)$ and $\gls{EUBO}(\vx; q_\theta)$ using \kl divergences, as follows
\begin{align}%
\hspace*{-.15cm} \gls{ELBO}_{\textsc{IWAE}}(\vx; q_\theta,K)
&= \gls{ELBO}(\vx; q_\theta) +  \underbrace{\mathbb{E}_{\qprop{iwae}(\vz^{(1:K)}|\vx)} \bigg[\DKL\big[\, \mathcal{U}(s) \big\| \ptgt{iwae}(\idx |\vz^{(1:K)}, \vx  ) \big] \vphantom{\frac{1}{2} }  \,  \bigg]}
_{\mathclap{\text{\footnotesize $0 \leq  \textsc{kl}  \text{ of uniform from  \gls{SNIS} weights} \leq \DKL[\qzx\|\pzx ]$}}} \, , \label{eq:iwae_over_elbo} \\
\gls{EUBO}_{\textsc{IWAE}}(\vx; q_\theta,K) 
&= \gls{EUBO}(\vx; q_\theta) - 
 \underbrace{\mathbb{E}_{\ptgt{iwae}(\vz^{(1:K)}|\vx)}\bigg[\DKL\big[\ptgt{iwae}(\idx |\vz^{(1:K)}, \vx  ) \, \big\| \, \mathcal{U}(s)\big] \vphantom{\frac{1}{2} }  \,  \bigg]}_{\mathclap{\text{\footnotesize $0 \leq  \textsc{kl}  \text{ of \gls{SNIS} weights from uniform} \leq \log K$}}} \, . \label{eq:iwae_over_eubo}
\end{align}
\normalsize
\end{restatable}
\myprop{iwae_elbo_eubo} demonstrates that the improvement of the \gls{IWAE} log partition function bounds over its single-sample counterparts is larger for more non-uniform \gls{SNIS} weights.   
Notably, the improvement of $\gls{EUBO}_{\textsc{IWAE}}(\vx;q_\theta,K)$ over the single-sample $\gls{EUBO}(\vx;q_\theta)$ is limited by $\log K$.   
Translating \myprop{iwae_elbo_eubo} to the \gls{IWAE} bounds on \gls{MI} yields the following corollary.

\begin{restatable}{corollary}{iwaemishort}
\label{cor:iwae_mi_short} 
\gls{IWAE} bounds on \gls{MI} improve upon the \gls{BA} bounds with the following relationships:
\begin{align}
\begin{aligned}
   I_{\textsc{BA}_L}(q_\theta) \leq I_{\textsc{IWAE}_L}(q_\theta, K) & \leq I_{\textsc{BA}_L}(q_\theta) + \log K, \label{eq:prop_iwae_lb}
\end{aligned}\qquad
\begin{aligned}
     I_{\textsc{IWAE}_U}(q_\theta, K) & \leq I_{\textsc{BA}_U}(q_\theta).
\end{aligned}
\end{align}
\end{restatable}

\mycor{iwae_mi_short}
shows that, in order to obtain a tight bound on \gls{MI}, the \gls{IWAE} lower bound requires exponential sample complexity in $\Exp{p(\vx)}{\DKL[\pzx\|\qzx]}$, which is the gap of either $\gls{EUBO}(\vx; q_\theta)$ or $I_{\gls{BA}_L}(q_\theta)$.
{Although $\gls{ELBO}_{\gls{IWAE}}(\vx;q_\theta, K)$ and $I_{\text{IWAE}_U}(q_\theta, K)$ are not limited to 
logarithmic improvement 
with increasing $K$, 
it has been argued that the same exponential sample complexity,  
$K \propto \exp\big(
\DKL[\pzx \| \qzx]
\big)$,
is required to 
achieve tight importance sampling bounds} 
(see \myapp{iwae_importance}, \citet{chatterjee2018sample}).
These observations motivate our improved \gls{AIS} proposals in \mysec{ais_estimation}, which achieve \textit{linear} bias reduction in the number of intermediate distributions $T$ used to bridge between $\qzx$ and $p(\vz|\vx)$.

\paragraph{Relationship with Structured InfoNCE} 
We can recognize the Structured \textsc{InfoNCE} upper and lower bounds (\myapp{structured}, \citet{poole2019variational}) as corresponding to the standard \gls{IWAE} bounds, with known $\pxgz$ and the marginal $p(\vz)$ used in place of the variational distribution $\qzx$.
In other words, we have
$I_{\textsc{S-InfoNCE}_L}(K) = I_{\textsc{iwae}_L}(p(\vz),K) $ and 
$I_{\textsc{S-InfoNCE}_U}(K) = I_{\textsc{iwae}_U}(p(\vz),K)$.

\subsection{Generalized IWAE}\label{sec:gen_iwae}

In the previous section, we have seen that 
\gls{IWAE}
improves upon 
\gls{BA}
using multiple samples. 
However, \gls{IWAE} leverages knowledge of the full joint density $p(\vx,\vz)$, 
whereas evaluating the \gls{BA} lower bound only requires access to the marginal $p(\vz)$.  
In this section, we consider a family of \gls{GIWAE} lower bounds, which use a contrastive critic function $\giwaeT(\vx,\vz)$ to achieve similar multi-sample improvement as \gls{IWAE} without access to $p(\vx,\vz)$.
While similar bounds appear in \citep{lawson2019energy,sobolevblog}, we 
provide 
empirical
validation 
for \gls{MI} estimation 
(\mysec{exp_giwae}), and discuss theoretical connections with 
\gls{IWAE} and
our proposed \textsc{mine-ais} method (\mysec{mine-ais}).  %

To derive a probabilistic interpretation for \gls{GIWAE}, 
our starting point is to further extend the state space of the \gls{IWAE} target distribution in \myeq{iwae_fwdrev2}, using a uniform index variable $\mathcal{U}(s) = \frac1K \, \forall s$ that specifies which sample $\vz^{(k)}$ is drawn from the posterior $\pzx$.  
This leads to the following joint distribution over $(\vx, \vz^{(1:K)}, s)$ and posterior over the index variable $s$,
\footnotesize
\begin{equation}
  \hspace*{-.2cm} 
  \ptgt{giwae}(\vx, \vz^{(1:K)}, \idx) 
  \coloneqq 
  \dfrac1K \, p(\vx, \vz^{(\idx)}) \prod \limits_{
\myoverset{k=1}{k\neq s}}
^K q_{\theta}(\vz^{(k)}|\vx),  \, \,  \text{resulting in} \, \, \,  \indexposterior( \idx |\vx, \vz^{(1:K)}) =  \dfrac{\frac{p(\vx,\vz^{(\idx)})}{q_{\theta}(\vz^{(\idx)}|\vx)}}{\sum \limits_{k=1}^K \frac{p(\vx,\vz^{(k)})}{q_{\theta}(\vz^{(k)}|\vx)}} .
\label{eq:giwae_reverse_main} %
\end{equation}
\normalsize
Note that marginalization over $s$ in $\ptgt{giwae}(\vz^{(1:K)},s|\vx)$ leads to the \gls{IWAE} target distribution $\ptgt{iwae}(\vz^{(1:K)}|\vx)$ in \myeq{iwae_fwdrev2}. 
The posterior $\indexposterior( \idx |\vx, \vz^{(1:K)})$ over the index variable $s$, which infers the positive sample drawn from $\pzx$ given a set of samples $\vz^{(1:K)}$, corresponds to the normalized importance weights.

For the \gls{GIWAE} extended state space proposal distribution, we consider a categorical index variable $\qprop{giwae}(\idx | \vz^{(1:K)}, \vx)$ 
drawn according to \textit{variational} \gls{SNIS} weights, which are calculated using a learned critic or negative energy function $\giwaeT(\vx,\vz)$
\footnotesize
\begin{align}
\qprop{giwae}(\vz^{(1:K)},\idx|\vx) 
   \coloneqq
   \bigg( \prod \limits_{k=1}^K q_{\theta}(\vz^{(k)}|\vx) \bigg) \qprop{giwae}(\idx |\vz^{(1:K)}, \vx ) ,
  \, \text{where} \, \, 
   \qprop{giwae}(\idx |\vz^{(1:K)}, \vx) \coloneqq \frac{e^{\giwaeT(\vx,\vz^{(\idx)})}}{ \sum \limits_{k=1}^K e^{\giwaeT(\vx,\vz^{(k)})}} .
   \label{eq:giwae_forward}
\end{align}
\normalsize
We can view the \gls{SNIS} distribution $\qprop{giwae}(\idx |\vz^{(1:K)}, \vx)$
as performing variational inference of the posterior $\indexposterior( \idx |\vx, \vz^{(1:K)})$.
We will show in \myprop{iwae_giwae} that the optimal \gls{GIWAE} critic function is $T^{*}(\vx,\vz) = \log \frac{p(\vx,\vz)}{\qzx} + c(\vx)$, in which case \cref{eq:giwae_reverse_main}-(\ref{eq:giwae_forward}) recover the \gls{IWAE} probabilistic interpretation from \citet{domke2018importance} (see \myapp{iwae_prob}). 

We focus on the upper bound on $\log p(\vx)$, obtained by taking the 
expected log importance ratio 
under $\ptgt{giwae}(\vz^{(1:K)},s|\vx)$ as in \mysec{general}.  \footnote{We consider $\gls{ELBO}_{\gls{GIWAE}}(\vx; q_\theta, \giwaeT, K)$ in \myapp{giwae_ub}.  
For \gls{MI} estimation, the corresponding upper bound is always inferior to $I_{\gls{IWAE}_U}(q_\theta, K)$, 
since known $p(\vx|\vz)$ is needed to evaluate $\Hxz$. See \mycor{giwae} and \myapp{giwae_ub}.}   We write the corresponding \gls{GIWAE} lower bound on \gls{MI} as
\begin{equation}%
\resizebox{.93\textwidth}{!}{$I_{\textsc{GIWAE}_L}(q_{\theta},\giwaeT,K)=
\underbrace{\Exp{p(\vx,\vz) \vphantom{\qzxi{2:K}} }{\log \frac{\qzx}{p(\vz)}} \vphantom{\frac{1}{K} \sum_{i=1}^K e^{\giwaeT(\vx,\vz^{(k)})}} }_{\footnotesize I_{\textsc{BA}_L}(q_\theta)} +
\underbrace{\Exp{p(\vx) p(\vz^{(1)}|\vx)\qzxi{2:K}}{\log
\frac{e^{\giwaeT(\vx,\vz^{(1)})}}{\frac{1}{K} \sum_{i=1}^K e^{\giwaeT(\vx,\vz^{(k)})}}}}_{\text{\small contrastive term $\leq \log K$}}. \, 
$}
\label{eq:giwae_lb_mi_main}
\end{equation}
\normalsize
We see that the \gls{GIWAE} lower bound decomposes into the sum of two terms, where the first is the  \gls{BA} variational lower bound for $\qzx$ and the second is a contrastive term which distinguishes negative samples from $\qzx$ and positive samples from $\pzx$.   
Similarly to \textsc{ba}, \gls{GIWAE} requires access to the analytical form of $p(\vz)$ to \emph{evaluate} the bound on \gls{MI}. However, if the goal is to \textit{optimize} mutual information, 
both the \textsc{ba} and \gls{GIWAE} lower bounds can be used even if no marginal distribution is available.   
See \myapp{representation} for more detailed discussion.

\paragraph{Relationship with BA} Choosing constant $\giwaeT(\vx,\vz)=c$ means that the second term in \gls{GIWAE} vanishes, so that we have $I_{\textsc{GIWAE}_L}(q, \giwaeT=c ,K) =I_{\textsc{BA}_L}(q_\theta)$ for all $K$.   Similarly, the single sample $I_{\textsc{GIWAE}_L}(q, \giwaeT , K=1)$ equals the \gls{BA} lower bound for all $\giwaeT(\vx,\vz)$. 

\paragraph{Relationship with InfoNCE} 
When the prior $p(\vz)$ is used in place of $\qzx$, we can recognize the second term in \myeq{giwae_lb_mi_main} as the \textsc{InfoNCE} contrastive lower bound \citep{oord2018representation, poole2019variational}, with $I_{\textsc{InfoNCE}_L}(\giwaeT,K) = I_{\textsc{giwae}_L}(p(\vz), \giwaeT, K)$. 
From this perspective, the \gls{GIWAE} lower bound highlights how variational learning can complement contrastive bounds to improve \gls{MI} estimation beyond the known $\log K$ limitations of \textsc{InfoNCE} (\citet{oord2018representation}, \cref{eq:infonce_k}). 
However, using the prior as the proposal in \textsc{InfoNCE} does allow the critic function to admit a bi-linear implementation 
{$\giwaeT(\vx,\vz)= f_{\phi_\vx}(\vx)^{\mathsf{T}}  f_{\phi_\vz}(\vz)$} ,
which requires only $N+K$ forward passes instead of $NK$ for \gls{GIWAE}, where $N$ is the batch size and $K$ is the total number of positive and negative samples.

\paragraph{Relationship with IWAE} 
The following proposition characterizes the relationship between the \gls{GIWAE} lower bound in \myeq{giwae_lb_mi_main} and the \gls{IWAE} lower bound on \gls{MI} from \mysec{std_iwae}.   
\begin{restatable}[Improvement of \gls{IWAE} over \gls{GIWAE}]{proposition}{iwaegiwae}\label{prop:iwae_giwae}
For a given $\qzx$, the \gls{IWAE} lower bound on \textsc{mi} is tighter than the \gls{GIWAE} lower bound for any $\giwaeT(\vx,\vz)$.
  Their difference %
is the average \textsc{kl} divergence between the normalized importance weights $\indexposterior(\idx |\vz^{(1:K)}, \vx )$ and the variational distribution $\qprop{giwae}(\idx |\vz^{(1:K)}, \vx )$ in \myeq{giwae_forward},
\small
\begin{align}%
I_{\textsc{IWAE}_L}(q_\theta,K) = I_{\textsc{GIWAE}_L}(q_\theta,\giwaeT,K) + 
\Exp{p(\vx)\ptgt{iwae}(\vz^{(1:K)}|\vx)}{\KL{\indexposterior(\idx |\vz^{(1:K)}, \vx )}{\qprop{giwae}(\idx |\vz^{(1:K)}, \vx )}}. 
\nonumber
\end{align}
\normalsize
\end{restatable}
\begin{restatable}[Optimal \gls{GIWAE} Critic Function yields \gls{IWAE}]{corollary}{giwae}
\label{cor:giwae} %
For a given $\qzx$ and $K>1$, the optimal \gls{GIWAE} critic function is equal to the true log importance weight up to an arbitrary constant: $T^*(\vx,\vz) = \log \frac{p(\vx,\vz)}{\qzx} + c(\vx)$.   With this choice of $T^*(\vx,\vz)$, we have
\begin{align}
    I_{\textsc{GIWAE}_L}(q_\theta, T^{*}, K) &=I_{\textsc{IWAE}_L}(q_\theta, K) \, . 
\end{align}
\normalsize
\end{restatable}

\begin{restatable}[{Logarithmic Improvement of \gls{GIWAE}}]{corollary}{giwaelogk}\label{cor:giwae_logk}
Suppose the critic function $T_{\phi}(\vx,\vz)$ is parameterized by $\phi$, and
$\exists \, \phi_0 \, s.t. \,  \forall \, (\vx,\vz), \,  T_{\phi_0}(\vx,\vz) = \text{const}$.  
For a given $\qzx$, let $T_{\phi^*}(\vx,\vz)$ denote the critic function that maximizes the \gls{GIWAE} lower bound.  Using \mycor{iwae_mi_short} and \mycor{giwae}, we have
\begin{align}
    I_{\textsc{BA}_L}(q_\theta)  \leq I_{\textsc{GIWAE}_L}(q_\theta,T_{\phi^*},K) \leq
     I_{\textsc{IWAE}_L}(q_\theta,K) \leq  I_{\textsc{BA}_L}(q_\theta)  +\log K \, .
\end{align}
\end{restatable}
See \myapp{pf_iwaegiwae}-\ref{app:pf_giwae} for proofs. 
Note that \myprop{iwae_elbo_eubo}, which relates the \gls{IWAE} and \gls{BA} lower bounds, can be seen as a special case of \myprop{iwae_giwae}, since $I_{\gls{BA}_L}(q_\theta)$ is a special case of \gls{GIWAE}. 
\mycor{giwae} shows that when the full joint density $p(\vx,\vz)$ is available, \gls{IWAE} is always preferable to \gls{GIWAE}.
 Finally, \cref{eq:giwae_lb_mi_main} suggests a similar decomposition of $I_{\textsc{IWAE}_L}(q_\theta,K)$ into a \gls{BA} term and a contrastive term (see \myapp{iwae_log_improvement})
\small
\begin{align}%
I_{\textsc{IWAE}_{L}}(q_{\theta},K)=
\underbrace{\Exp{p(\vx,\vz)}{\log \frac{\qzx}{p(\vz)}}
\vphantom{\Exp{p(\vx) p(\vz^{(1)}|\vx)\qzxi{2:K}}{\log
\frac{\frac{p(\vx,\vz^{(1)})}{q_{\theta}(\vz^{(1)}|\vx)}}{\frac{1}{K} \sum_{i=1}^K \frac{p(\vx,\vz^{(k)})}{q_{\theta}(\vz^{(k)}|\vx)}}}}
}_{I_\textsc{BA}(q)} +
\underbrace{\Exp{p(\vx) p(\vz^{(1)}|\vx)\qzxi{2:K}}{\log
\frac{\frac{p(\vx,\vz^{(1)})}{q_{\theta}(\vz^{(1)}|\vx)}}{\frac{1}{K} \sum_{i=1}^K \frac{p(\vx,\vz^{(k)})}{q_{\theta}(\vz^{(k)}|\vx)}}}}_{\hspace{.5cm} 0\leq \, \text{contrastive term} \, \leq \log K} .
 \label{eq:iwae_two_terms_main}
\end{align}
\normalsize

\paragraph{Relationship with Structured InfoNCE}
\textsc{InfoNCE} and Structured \textsc{InfoNCE} 
are special cases of \gls{GIWAE} and \gls{IWAE}, respectively, which use $p(\vz)$ as the variational distribution.   Since $I_{\gls{BA}_L}(p(\vz))=0$, \mycor{giwae_logk} suggests the following
relationship
\begin{align}
    0 \leq I_{\textsc{InfoNCE}_L}(T_{\phi},K) \leq  I_{\textsc{S-InfoNCE}_L}(K)  \leq \log K \, . 
    \label{eq:infonce_k}
\end{align}
From the \gls{GIWAE} perspective, we can interpret the $\log K$ limitations of (Structured) \textsc{InfoNCE} as arising from 
 the $\log K$ improvement 
results 
for 
$I_{\gls{GIWAE}_L}(q_\theta, \giwaeT, K)$ and $I_{\gls{IWAE}_L}(q_\theta, K)$ in
\mycor{giwae_logk}.
\section{Multi-Sample AIS Bounds for Estimating Mutual Information}\label{sec:ais_estimation}
\gls{AIS} \citep{neal2001annealed} is considered the gold standard for obtaining unbiased and low variance estimates of the partition function, while \gls{BDMC} \citep{grosse2015sandwiching, grosse2016measuring} provides lower and upper bounds on the log partition function using forward and reverse \gls{AIS} chains. 
In this section, we propose various \textit{Multi-Sample} \gls{AIS} bounds, which highlight that extending the state space over multiple samples, as in \gls{IWAE}, is complementary to extending the state space over intermediate distributions as in \gls{AIS}.  
Our approach includes \gls{BDMC} bounds as special cases, and we obtain a novel upper bound which can match or improve upon the performance of \gls{BDMC}.
We derive novel probabilistic interpretations of Multi-Sample \gls{AIS} bounds
which, perhaps surprisingly, suggest that the practical sampling schemes used in \gls{BDMC} do not correspond to the same probabilistic interpretation. 

We present our Multi-Sample \gls{AIS} bounds of $\log \px$ within the context of estimating the generative mutual information \citep{alemi2018gilbo}, as in \mysec{setting}.  {However, our methods are equally applicable for other applications, including evaluating the marginal likelihood \citep{wu2016quantitative, grosse2015sandwiching} or rate-distortion curve \citep{huang2020evaluating} of generative models. }

We first review background on \gls{AIS}, before describing Multi-Sample \gls{AIS} log partition function bounds in Sec. \ref{sec:multi-sample-ais} and \ref{sec:coupled_rev}. 
For all bounds in this section, we assume that \textit{both} the true marginal $p(\vz)$ and conditional  $p(\vx|\vz)$ densities are known.

\subsection{Annealed Importance Sampling Background} \label{sec:ais}
\gls{AIS} \citep{neal2001annealed} constructs a sequence of intermediate distributions $\{ \pi_t(\vz) \}_{t=0}^T$, which bridge between a normalized initial distribution $\pi_0(\vz|\vx)$ and target distribution $\pi_T(\vz|\vx) = p(\vz|\vx)$.   The target has an unnormalized density $\pi_T(\vx,\vz) = p(\vx,\vz)$ and normalizing constant $\mathcal{Z}_T(\vx) = p(\vx)$.  A common choice for intermediate distributions is the geometric path parameterized by $\{ \beta_t \}_{t=0}^T$:
\begin{align}
    \dist_t (\vz|\vx) \coloneqq \frac{{\pi}_0(\vz|\vx)^{1-\beta_t} \, {\pi}_T( \vx,\vz)^{\beta_t}}{\mathcal{Z}_{t}(\vx)}\,,  \quad \text{where} \quad \mathcal{Z}_{t}(\vx) \coloneqq \int {\pi}_0(\vz|\vx)^{1-\beta_t} \, {\pi}_T(\vx,\vz)^{\beta_t} d\vz \,. \label{eq:geopath0}
\end{align}
In the probabilistic interpretation of \gls{AIS}, we consider an extended state space proposal $\qprop{ais}(\vz_{0:T}|\vx)$, obtained by sampling from the initial $\prop(\vz|\vx)$ and constructing transitions 
$\tfwd(\vz_{t}|\vz_{t-1})$ which leave $\dist_{t-1}(\vz|\vx)$ invariant. The target distribution $\ptgt{ais}(\vz_{0:T}|\vx)$ is given by running the reverse transitions $\trev(\vz_{t-1}|\vz_{t})$ starting from a target or posterior sample $\pi_T(\vz|\vx)$, as shown in \myfig{ais},
\small
\begin{align}
    \hspace*{-.2cm} \qprop{ais}(\vz_{0:T}|\vx)  \coloneqq  \prop(\vz_0|\vx) \prod \limits_{t=1}^{T} \tfwd(\vz_{t} | \vz_{t-1}) \,, \quad \, \ptgt{ais}(\vx, \vz_{0:T}) \coloneqq  \pi_T( \vx, \vz_T) \prod \limits_{t=1}^{T} \trev(\vz_{t-1} | \vz_{t}) \,.  \label{eq:ais_fwd_rev}
\end{align}
\normalsize

As in \mysec{general}, 
taking expectations of the log importance weights under the proposal and target yields a lower and upper bound on the log partition function $\log \px$   
\begin{align}
    \myunderbrace{\Exp{\vz_{0:T} \sim \qprop{ais}}{ \log  \frac{\ptgt{ais}(\vx,\vz_{0:T})}{\qprop{ais}(\vz_{0:T}|\vx)} }}{\gls{ELBO}_{\textsc{ais}}(\vx; \pi_0,T)}  \leq \log \px \,   \leq \myunderbrace{\Exp{\vz_{0:T} \sim \ptgt{ais}}{  \log  \frac{\ptgt{ais}(\vx,\vz_{0:T})} {\qprop{ais}(\vz_{0:T}|\vx)}}}{\gls{EUBO}_{\textsc{ais}}(\vx; \pi_0,T)}. \label{eq:elbo_ais}
\end{align} 
These single-chain lower and upper bounds translate to upper and lower bounds on \gls{MI}, $I_{\gls{AIS}_U}(\pi_0,T)$ and $I_{\gls{AIS}_L}(\pi_0,T)$, which were suggested 
for \gls{MI} estimation in the blog post of \citet{sobolevblog}.  
See \myapp{ind-multi-sample-ais} for detailed derivations.

To characterize the bias reduction for \gls{AIS} with increasing $T$, we prove the following proposition.
\begin{restatable}[{Complexity in $T$}]{proposition}{aisprop}
\label{prop:ais} 
Assuming perfect transitions and a 
geometric annealing path with linearly-spaced $\{\beta_t\}_{t=1}^T$, 
the gap of the
\gls{AIS} upper and lower bounds (\cref{eq:elbo_ais})
reduces \textit{linearly} with 
increasing 
$T$,
\small
\begin{align}
\gls{EUBO}_{\textsc{ais}}(\vx; \pi_0,T) - \gls{ELBO}_{\textsc{ais}}(\vx; \pi_0,T) = \frac{1}{T} \Big( \DKL[\pi_T(\vz|\vx)\|\pi_0(\vz|\vx)] + \DKL[\pi_0(\vz|\vx)\|{\pi_T(\vz|\vx)}]  \Big). \label{eq:symm_kl_ineq}  %
\end{align}
\normalsize
\end{restatable}
See \myapp{ais_pf} for a proof.
Our proposition generalizes Thm. 1 of \citet{grosse2013annealing}, which holds for the case of $T \rightarrow \infty$ instead of finite $T$ as above.  
In our experiments in \mysec{experiment}, we will find that this linear bias reduction in $T$ is crucial for achieving tight \gls{MI} estimation when both $p(\vz)$ and $p(\vx|\vz)$ are known. 
We can further tighten the single-sample \gls{AIS} bounds with multiple 
annealing chains ($K > 1$) using two different approaches, which we present in the following sections.

\begin{figure}[t]
\vspace*{-1cm}
\centering
\includegraphics[scale=0.48]{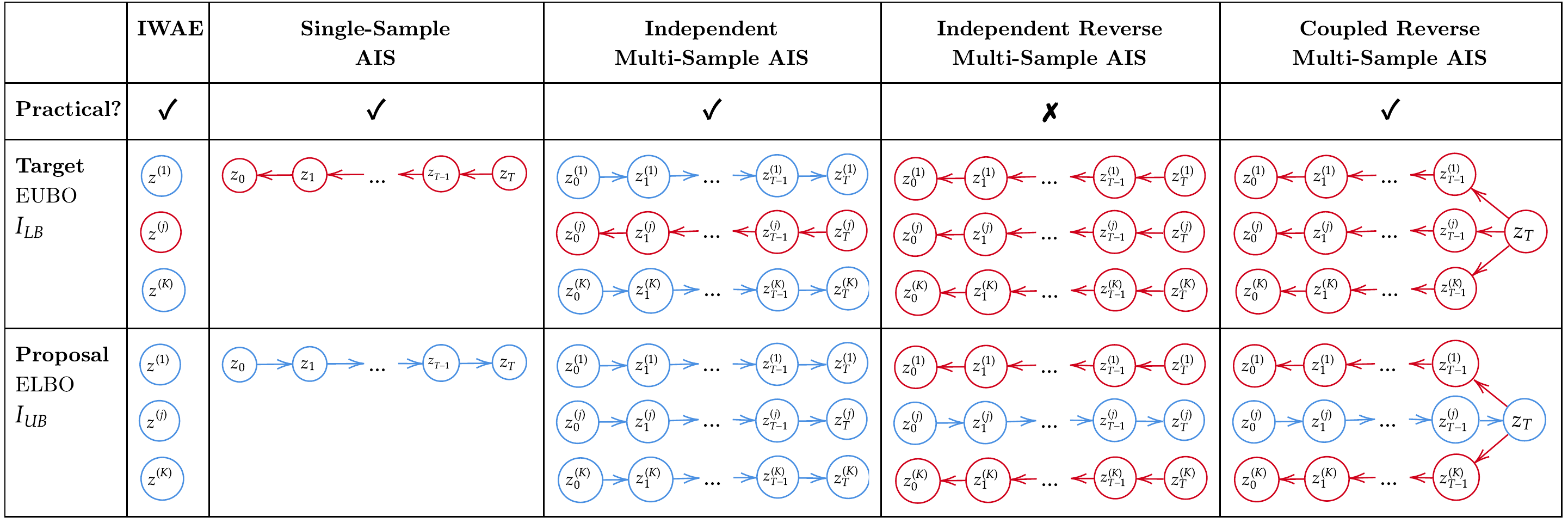}
\vspace{-.2cm}
\caption{\label{fig:ais} Extended state-space probabilistic interpretations of Multi-Sample \gls{AIS} bounds. Forward chains are colored in blue, and backward chains are colored in red. Note that \gls{ELBO}s and \gls{EUBO}s are obtained by taking the expectation of the 
log unnormalized importance weights $\log \ptgt{}(\cdot)/\qprop{}(\cdot)$ 
under either the proposal or target distribution, and can then be translated to \gls{MI} bounds.}
\end{figure}

\subsection{Independent Multi-Sample AIS Bounds}\label{sec:multi-sample-ais}
In our first approach, \textit{Independent} Multi-Sample \gls{AIS} (\textsc{im-ais}), we construct an extended state space proposal by running $K$ independent \gls{AIS} forward chains $\vz_{0:T}^{(k)} \sim \qprop{ais}$ in parallel.  Similarly to the \gls{IWAE} upper bound (\myeq{iwae_ublb}), the extended state space target involves selecting an index $s$ uniformly at random, and running a backward \gls{AIS} chain $\vz_{0:T}^{(s)} \sim \ptgt{ais}$ starting from a true posterior sample $\vz_T \sim p(\vz|\vx)$.  The remaining $K-1$ samples are obtained by running forward \gls{AIS} chains, as visualized in \myfig{ais}
\footnotesize
\begin{align}
        \hspace*{-.2cm} \qprop{im-ais}(\vz^{(1:K)}_{0:T}| \vx) \coloneqq  \prod \limits_{k=1}^{K} \qprop{ais}(\vz^{(k)}_{0:T}|\vx)\,, %
    \quad  \ptgt{im-ais}(\vx, \vz^{(1:K)}_{0:T}) \coloneqq  \dfrac{1}{K} \sum \limits_{s=1}^K  \ptgt{ais}(\vx, \vz^{(s)}_{0:T}) \prod \limits_{\myoverset{k=1}{k\neq s}}^{K}  \qprop{ais}(\vz^{(k)}_{0:T}| \vx)\,, 
\end{align}
\normalsize
where  $\qprop{ais}$ and $\ptgt{ais}$ were defined in \myeq{ais_fwd_rev}.
Note that sampling from the extended state space target distribution is \textit{practical}, as it only requires one sample from the true posterior distribution.

Taking the expectation of the log unnormalized density ratio under the proposal and target yields lower and upper bounds on $\log \px$, respectively,
\small
\begin{align}
      \myunderbrace{\mathbb{E}_{\myoverset{\vz^{(1:K)}_{0:T} \sim \qprop{ais}}{\phantom{\vz^{(1:K)}_{0:T} \sim \qprop{ais}}} } \bigg[ \log  \frac{1}{K} \sum \limits_{k =1}^{K}  \frac{\ptgt{ais}(\vx,\vz^{(k)}_{0:T})}{\qprop{ais}(\vz^{(k)}_{0:T}|\vx)} \bigg] }{ \gls{ELBO}_{\textsc{im-ais}}(\vx;\pi_0,K,T)}  \leq  \log p(\vx)\leq
 \myunderbrace{ \mathbb{E}_{\myoverset{\vz^{(1)}_{0:T} \sim \ptgt{ais}}{\vz^{(2:K)}_{0:T} \sim \qprop{ais} }} \bigg[ \log  \frac{1}{K} \sum \limits_{k=1}^{K}    \frac{\ptgt{ais}(\vx,\vz^{(k)}_{0:T})}{\qprop{ais}(\vz^{(k)}_{0:T}|\vx)} \bigg]}{\gls{EUBO}_{\textsc{im-ais}}(\vx;\pi_0,K,T)}. \label{eq:im_ais_ublb}
\end{align}
\normalsize
which again have \textsc{kl} divergences as the gap in their bounds.  
Independent Multi-Sample \gls{AIS} reduces to \gls{IWAE} for $T=1$, and reduces to single-sample \gls{AIS} for $K=1$.  Both the upper and lower bounds become tight as $K\rightarrow \infty$ or $T \rightarrow \infty$.
The \textsc{im-ais} upper and lower bounds on $\log p(\vx)$ translate to lower and upper bounds on \gls{MI}, respectively, as in \mysec{setting}.

Since Independent Multi-Sample \gls{AIS} can be interpreted as \gls{IWAE} with an \gls{AIS} proposal, we can use similar arguments as \citep{burda2016importance, sobolev2019hierarchical} to show that the multi-sample bounds are tighter than their single-sample counterparts.
In \myapp{ais_logarithmic_pf}, we show a similar result to \myprop{iwae_elbo_eubo}, which characterizes the improvement of Independent Multi-Sample \gls{AIS} with increasing $K$. In particular, the lower bound on \gls{MI} is limited to logarithmic improvement over single-sample \gls{AIS}, with $I_{\textsc{im-ais}_L}(\pi_0, K, T) \leq  I_{\textsc{ais}_L}(\pi_0, T) + \log K$.

\subsection{Coupled Reverse Multi-Sample AIS Bounds}\label{sec:coupled_rev}
 We can exchange the role of the forward and backward annealing chains in Independent Multi-Sample \gls{AIS} to obtain alternative bounds on the log partition function.  
We define \emph{Independent Reverse Multi-Sample \gls{AIS}} (\textsc{ir-ais}) using the following proposal and target distribution, as shown in \myfig{ais}
\small
 \begin{align*}
    \qprop{ir-ais}(\vx, \vz^{(1:K)}_{0:T}) &\coloneqq   
    \frac{1}{K} \sum \limits_{s=1}^K \qprop{ais}(\vz^{(s)}_{0:T}|\vx) \prod \limits_{\myoverset{k=1}{k\neq \idx}}^{K}  \ptgt{ais}( \vx, \vz^{(k)}_{0:T}) \,,
    \qquad \, \, \ptgt{ir-ais}(\vx,\vz^{(1:K)}_{0:T}) \coloneqq 
     \prod \limits_{k=1}^{K}  \ptgt{ais}(\vx, \vz^{(k)}_{0:T}) \,.
\end{align*}
\normalsize
We use a similar approach as in \mysec{general} to 
derive lower and upper bounds on $\log p(\vx)$ and \gls{MI} from these proposal and target distributions in \myapp{reverse_ms_ais}.    
Just as Independent Multi-Sample \gls{AIS} extends \gls{IWAE} using an \gls{AIS} proposal, we can view Independent Reverse Multi-Sample \gls{AIS} as an extension of \textit{Reverse} \gls{IWAE}, which we introduce in \myapp{rev_iwae}. 
However, these bounds will be impractical in most settings since both the proposal and target  distributions require multiple true posterior samples.  %

To address this impracticality, we propose \textit{Coupled} Reverse Multi-Sample \gls{AIS} (\textsc{cr-ais}), which uses backward annealing chains but requires only a \textit{single} posterior sample.   In particular, the extended state space target distribution in \cref{fig:ais} initializes $K$ backward chains from a single $\vz_T \sim \pi_T(\vz|\vx)$, with  
remaining transitions denoted by $\ptgt{ais}(\vz_{0:T-1}|\vz_T)$, since they are identical to standard \gls{AIS},
\begin{align}
     \ptgt{cr-ais}(\vx, \vz^{(1:K)}_{0:T-1}, \vz_T) 
     &\coloneqq 
     \pi_T(\vx, \vz_T) \prod \limits_{k=1}^{K}  \ptgt{ais}(\vz^{(k)}_{0:T-1}| \vz_T, \vx)\,.
\end{align}
 The extended state space proposal is obtained by selecting an index $s$ uniformly at random and running a single forward \gls{AIS} chain.   We then run $K-1$ backward chains, all starting from the last state of the selected forward chain, as visualized in \myfig{ais},
\begin{align}
    \qprop{cr-ais}(\vz^{(1:K)}_{0:T-1},\vz_T|\vx) 
    &\coloneqq 
    \frac{1}{K} \sum \limits_{s=1}^K \qprop{ais}(\vz^{(s)}_{0:T-1}, \vz_{T}|\vx) \prod \limits_
    {\myoverset{k=1}{k\neq s}}^{K}  
    \ptgt{ais}( \vz^{(k)}_{0:T-1} | \vz_T, \vx) \,.
\end{align}
Taking the expected log ratio under the proposal and target yields lower and upper bounds on $\log p(\vx)$,
\scriptsize
\begin{align*}
   \myunderbrace{- \mathbb{E}_{\myoverset{\vz^{(1)}_{0:T-1}, \vz_T \sim \qprop{ais}(\vz_{0:T}|\vx)}{\vz^{(2:K)}_{0:T-1} \sim \ptgt{ais}(\vz_{0:T-1}|\vz_T, \vx)}} \bigg[ \log  \frac{1}{K} \sum \limits_{k =1}^{K}  \frac{\qprop{ais}(\cdot)%
   }{\ptgt{ais}(\cdot)
   } \bigg]}{\gls{ELBO}_{\textsc{cr-ais}}(\vx; \pi_0,K,T)} %
   \leq  \log p(\vx)  \leq
     \myunderbrace{-\mathbb{E}_{\myoverset{ \vphantom{\qprop{ais}(\vz_{0:T-1}^{(1)})} \vz_{T} \sim \pi_T(\vz_T|\vx)}{\vz^{(1:K)}_{0:T-1} \sim \ptgt{ais}(\vz_{0:T-1}|\vz_T, \vx)}}
     \bigg[ \log  \frac{1}{K} \sum \limits_{k=1}^{K}  \frac{\qprop{ais}(\cdot)%
     }{\ptgt{ais}(\cdot)%
     } \bigg]}{\gls{EUBO}_{\textsc{cr-ais}}(\vx;\pi_0,K,T)}.
\end{align*}
\normalsize
We show in \myapp{coupled_k_pf} that Coupled Reverse Multi-Sample \gls{AIS} bounds get tighter with increasing $K$.   In \myapp{coupled_pf}, we also show that
the Coupled Reverse Multi-Sample \gls{AIS} \textit{upper} bound on \gls{MI} is limited to logarithmic improvement over single-sample \gls{AIS}, with $I_{\textsc{cr-ais}_U}(\pi_0, K, T) \geq  I_{\textsc{ais}_U}(\pi_0, T) - \log K$.

\subsection{Discussion}\label{sec:discussion}

\paragraph{Relationship with BDMC}

Bidirectional Monte Carlo (\gls{BDMC}) \citep{grosse2015sandwiching, grosse2016measuring} was the first to propose multi-sample log partition function bounds using \gls{AIS} chains.   In particular, the
multi-sample \gls{BDMC} lower and upper bounds \citep{grosse2015sandwiching, grosse2016measuring} on the log partition function correspond to the lower bound of Independent Multi-Sample \gls{AIS} and upper bound of Coupled Reverse Multi-Sample \gls{AIS}, respectively. 
Our probabilistic interpretations provide novel perspective on these existing \gls{BDMC} bounds. 
Perhaps surprisingly, we found that the multi-sample \gls{BDMC} lower bound (\myfig{ais} Col. 4,  Row 4) and upper bound  (\myfig{ais} Col. 6, Row 3) do not arise from the same probabilistic interpretation (i.e. are not in the same row in \myfig{ais}).   In other words, the gap in their log partition function bounds do not correspond to the forward and reverse \textsc{kl} divergences between the same pair of extended state space proposal and target distributions.
We experimentally compare all Multi-Sample \gls{AIS} bounds, including \gls{BDMC}, in \mysec{experiment_multi_comparison} and provide recommendations for which bounds to use in practice.
\paragraph{Effect of $K$ and $T$} 
Our Independent and Coupled Reverse Multi-Sample \gls{AIS} approaches both inherit the linear bias reduction of single-sample \gls{AIS} in \myprop{ais}, 
although this computation must be done in serial fashion.  %
While increasing $K$ involves parallel computation, its bias reduction is often only \textit{logarithmic}, as we have shown for $I_{\gls{IWAE}_L}(q_\theta, K)$ (\mycor{iwae_mi_short}), $I_{\textsc{im-ais}_L}(\pi_0, K, T)$ (\myapp{ais_logarithmic_pf} \myprop{ais_elbo_eubo}), and $I_{\textsc{cr-ais}_U}(\pi_0, K, T)$ (\myapp{coupled_pf} 
\myprop{c_r_multi_sample_ais}).
Based on these arguments, we recommend increasing $K$ until computation can no longer be parallelized on given hardware and allocating all remaining resources to increasing $T$.

\section{MINE-AIS Estimation of Mutual Information}\label{sec:mine-ais}

In this section, we present 
energy-based bounds
which are designed for settings where the true conditional density $p(\vx|\vz)$ is unknown.    We first review \gls{MINE} \citep{belghazi2018mutual} in \mysec{gmine}, and present probabilistic interpretations which allow us to derive \textit{Generalized} \gls{MINE} lower bounds with a variational $\qzx$ as the base distribution instead of the marginal $p(\vz)$.  
Our main contribution in this section is the \textsc{mine-ais} method, which optimizes a tighter lower bound on \gls{MI} than the Generalized \gls{MINE} lower bound and
extends Multi-Sample \gls{AIS} evaluation to the case where $p(\vx|\vz)$ is unknown.

\subsection{Generalized Mutual Information Neural Estimation}\label{sec:gmine}
To derive a probabilistic interpretation for \gls{MINE} \citep{belghazi2018mutual}, consider an energy based approximation to the joint distribution $p(\vx,\vz)$.  {To extend \gls{MINE}, we consider a general base variational distribution $\qzx$ in place of the marginal $p(\vz)$} 
\begin{align}
    \pi_{\theta,\phi}(\vx,\vz) \coloneqq 
    \frac{1}{\mathcal{Z}_{\theta,\phi}} p(\vx) \qzx e^{\giwaeT(\vx,\vz)}, \quad \text{where} \quad \mathcal{Z}_{\theta,\phi} = \mathbb{E}_{p(\vx) \qzx} \left[ e^{\giwaeT(\vx, \vz)} \right] \,,  \label{eq:energy_pi_gmine}
\end{align}
\normalsize
where $\giwaeT(\vx,\vz)$ is the negative energy function and the partition function $\mathcal{Z}_{\theta,\phi}$ integrates over both $\vx$ and $\vz$.  
Subtracting a joint \kl divergence $\DKL[p(\vx,\vz) \| \pi_{\theta, \minetparam}(\vx,\vz)]$ from $\Ixz$, we obtain the \textit{Generalized} \textsc{mine-dv} lower bound on \gls{MI}
\begin{align}
     \hspace*{-.25cm}  
   \resizebox{.9\textwidth}{!}{$
     \Ixz \geq I_{\textsc{gmine-dv}}(q_{\theta}, T_\phi) \coloneqq
     \myunderbrace{ \Exp{p(\vx,\vz)}{\log \frac{\qzx}{p(\vz)}} }{\text{\small $I_{\textsc{BA}_L}(q_{\theta})$}} + \myunderbrace{ \Exp{p(\vx,\vz)}{\giwaeT(\vx,\vz)} - \log \Exp{p(\vx)\qzx}{ e^{T_{\minetparam}(\vx, \vz)} } \vphantom{\frac{\qzx}{p(\vz)}}}
     { \text{\small contrastive term}}.
    $}
     \label{eq:gen_minedv_main}
\end{align}
\normalsize
Note that the standard \textsc{mine-dv} lower bound \citep{belghazi2018mutual} corresponds to Generalized \textsc{mine-dv} using the prior $p(\vz)$ as the proposal, $I_{\textsc{mine-dv}}(T_\phi)  = I_{\textsc{gmine-dv}}(p(\vz), T_\phi)$.  
In \myapp{mine_dv_dual}, we interpret the contrastive term in \cref{eq:gen_minedv_main} 
as arising from 
the dual representation of the \textsc{kl} divergence $\DKL[p(\vx,\vz)\|p(\vx)\qzx]$.

The Generalized \textsc{mine-f} bound 
is a looser lower bound than Generalized \textsc{mine-dv}, which
can be derived from \cref{eq:gen_minedv_main} using the inequality $\log u \leq \frac{u}{e}$,
\begin{align}
  \resizebox{.99\textwidth}{!}{
  $
    I_{\textsc{gmine-dv}}(q_\theta, \giwaeT) \geq I_{\textsc{gmine-f}}(q_\theta, \giwaeT) \coloneqq      \Exp{p(\vx,\vz)}{\log \frac{\qzx}{p(\vz)}} + \Exp{p(\vx,\vz)}{\giwaeT(\vx,\vz)} -  \Exp{p(\vx)\qzx}{e^{\giwaeT(\vx, \vz)-1}} .  %
    $
    }\nonumber
\end{align}
\normalsize
Generalized \textsc{mine-f} reduces to standard \textsc{mine-f} when the base distribution is the prior, with $I_{\textsc{mine-f}}(T_\phi) = I_{\textsc{gmine-f}}(p(\vz), T_\phi)$.
 We provide a probabilistic interpretation 
 of Generalized \textsc{mine-f} 
 in \myapp{mine_f_prob}, and a conjugate duality interpretation in 
\myapp{mine_f_dual}.

Despite the intractability of the log partition function term $\log \mathcal{Z}_{\theta,\phi}$ in \cref{eq:gen_minedv_main},  \citet{belghazi2018mutual} use direct sampling inside the logarithm for training $\giwaeT$.
Our \textsc{mine-ais} method will both improve the training and evaluation schemes of (Generalized) \textsc{mine} and optimize a tighter lower bound. 

\subsection{MINE-AIS Estimation of Mutual Information}\label{sec:mine-ais-sub}

In this section, we present \textsc{mine-ais}, which is inspired by Generalized \gls{MINE} but optimizes a tighter lower bound on \gls{MI} that involves an intractable log partition function.
We show in \myprop{mine-ais-limit} that this bound corresponds to the limiting behavior of the \gls{GIWAE} lower bound as $K\rightarrow \infty$.   However, we present a qualitatively different training scheme inspired by contrastive divergence learning of energy-based models \citep{hinton2002training} in \mysec{mine_ais_training}.
In \mysec{mine_ais_eval}, we use Multi-Sample \gls{AIS} to evaluate the intractable log partition function, and show that \textsc{mine-ais} reduces to the methods in \mysec{ais_estimation} for the optimal critic function or known $p(\vx|\vz)$.

To begin, we 
consider a 
flexible energy-based distribution $\pi_{\theta,\phi}(\vz|\vx)$ as 
an approximation to the posterior $\pzx$
\citep{poole2019variational, arbel2020generalized}
\begin{align}
    \pi_{\theta,\phi}(\vz|\vx) \coloneqq 
    \frac{1}{\mathcal{Z}_{\theta,\phi}(\vx)} \qzx e^{\giwaeT(\vx,\vz)}, \quad \text{where} \quad \mathcal{Z}_{\theta,\phi}(\vx) = \mathbb{E}_{\qzx} \left[ e^{\giwaeT(\vx, \vz)} \right]  , \label{eq:energy_pi}
\end{align}
where $\qzx$ is a base variational distribution which can be tractably sampled from and evaluated and $\giwaeT(\vx,\vz)$ is the negative energy or critic function.

Plugging $\pi_{\theta,\phi}(\vz|\vx)$ into the \gls{BA} lower bound, we denote the resulting bound as the \gls{IBAL}.   We use the term ``implicit'', since it is often difficult to explicitly evaluate the \gls{IBAL} due to the intractable log partition function term.
After simplifying in \myapp{mine-ais_deriv}, 
\small
\begin{align}
   \hspace*{-.15cm} \Ixz \geq \ibal(q_\theta,T_\phi) \coloneqq I_{\gls{BA}_L}(\pi_{\theta, \phi}) 
   &= \Exp{\pxandz}{\log \frac{\qzx}{p(\vz)}} + \Exp{\pxandz}{\Txz} - \Exp{\px}{\log \mathcal{Z}_{\phi, \theta}(\vx)}  \nonumber \\ %
    &= \myunderbrace{\Exp{p(\vx,\vz)}{\log \frac{\qzx}{p(\vz)}} \vphantom{{\log \frac{e^{\giwaeT(\vx,\vz)}}{\Exp{\qzx}{e^{\giwaeT(\vx,\vz)}}}}}}{I_{\gls{BA}_L}(q_\theta)} + \myunderbrace{\Exp{p(\vx,\vz)}{\log \frac{e^{\giwaeT(\vx,\vz)}}{\Exp{\qzx}{e^{\giwaeT(\vx,\vz)}}}} \vphantom{{\log \frac{e^{\giwaeT(\vx,\vz)}}{\Exp{\qzx}{e^{\giwaeT(\vx,\vz)}}}}}}{ \mathclap{\text{ \scriptsize contrastive term $\leq \Exp{\px}{\DKL[p(\vz|\vx)\|\qzx]}$}  }}, \label{eq:mine-ais}
\end{align}
\normalsize
where the gap of the \gls{IBAL} is 
$\mathbb{E}_{\px}[\KL{\pzx)}{\pi_{\theta,\phi}(\vz|\vx)}]$. 
Note that the \gls{IBAL} generalizes the Unnormalized Barber-Agakov bound from \citet{poole2019variational}, with $I_{\textsc{uba}}(\giwaeT) = \gls{IBAL}(p(\vz), \giwaeT)$.

\begin{restatable}{proposition}{mineaisopt}
\label{prop:mine-ais-optimal}
For a given $\qzx$, the optimal \gls{IBAL} critic function equals the log importance weights up to a
constant $T^*(\vx,\vz) = \log \frac{p(\vx,\vz)}{\qzx} + c(\vx)$.   
For this $T^*$, we have $\ibal(q_\theta,T^{*}) = \Ixz$.
\end{restatable}

\paragraph{Relationship with MINE} $\gls{IBAL}(q_{\theta}, \giwaeT)$ provides a tighter lower bound than Generalized \gls{MINE}, as we discuss in detail in \myapp{mine} and summarize in 
\myfig{gmine_bounds}.
In particular, we have
\begin{align}
    \Ixz \geq \gls{IBAL}(q_{\theta}, \giwaeT) \geq I_{\textsc{gmine-dv}}(q_{\theta}, T_\phi) \geq I_{\textsc{gmine-f}}(q_{\theta}, T_\phi) .
\end{align}
Our \textsc{mine-ais} method will optimize and evaluate the intractable $\gls{IBAL}$ directly.

\paragraph{Relationship with GIWAE} 
The \gls{IBAL} lower bound resembles the \gls{GIWAE} lower bound in that both improve upon the \gls{BA} variational bound using a contrastive term.   Further, \myprop{mine-ais-optimal} and \mycor{giwae} show that, in both cases, the optimal critic function equals the true log importance weights plus a constant.
The following proposition shows that $\IBAL(q_\theta, \giwaeT)$
can be viewed the limiting behavior of the \gls{GIWAE} lower bound as $K\rightarrow \infty$.
\begin{restatable}[\gls{IBAL} as Limiting Behavior of \gls{GIWAE}]{proposition}{mineaislimit}
\label{prop:mine-ais-limit}
For given $\qzx$ and $T_{\phi}(\vx,\vz)$, we have
\begin{align}
    \lim \limits_{K\rightarrow \infty} I_{\textsc{GIWAE}_L}( q_{\theta}, T_{\phi}, K) =  \ibal(q_\theta,T_{\phi}) \, .
\end{align}
\end{restatable}
See \myapp{mine-ais-limit-pf} for proof.
To gain intuition for \myprop{mine-ais-limit}, we consider a closely related result involving the \textit{marginal \textsc{\textit{SNIS}} distribution} of \gls{GIWAE}  $\qprop{giwae}(\vz|\vx; K)$, which results from 
sampling $K$ times from $\qzx$ and returning the sample in index $s \sim \qprop{giwae}(s|\vx, \vz^{(1:K)}) \propto e^{\giwaeT(\vx,\vz^{(s)})}$.
In \myapp{convergence}, we show that in the limit as $K \rightarrow \infty$,
the marginal \gls{SNIS} distribution of \gls{GIWAE}
matches the \textsc{mine-ais} energy-based posterior $\pi_{\theta,\phi}(\vz|\vx) \propto \qzx e^{\giwaeT(\vx,\vz)}$.  
In both cases, the energy or critic function `modulates' the base distribution $\qzx$ to better approximate the true posterior.   

The contrastive term in the \gls{GIWAE} bound is tractable, as it only requires a single posterior sample and $K-1$ samples from the base distribution.
As shown in \mycor{giwae}, this limits the improvement of $I_{\gls{GIWAE}_L}(q_{\theta}, T)$  over $I_{\gls{BA}_L}(q_{\theta})$ to $\log K$ nats, even for the optimal critic function.
By contrast, in the following proposition we show that the contrastive term in $\ibal(q_{\theta}, \giwaeT)$ 
is expressive enough to potentially close the gap in the \gls{BA} bound.   This improved contrastive term comes at the cost of tractability, as $\gls{IBAL}(q_\theta, \giwaeT)$ involves an intractable log partition function $\log \mathbb{E}_{\qzx}[e^{\giwaeT(\vx,\vz)}]$. 
\begin{restatable}{proposition}{mineaisproperties}
\label{prop:mineaisproperties}
Suppose the critic function $T_{\phi}(\vx,\vz)$ is parameterized by $\phi$, and that $\exists \, \phi_0 \, \, s.t. \,    T_{\phi_0}(\vx,\vz) = \text{const}$. For a given $\qzx$,
let $T_{\phi^*}(\vx,\vz)$ denote the critic function that maximizes $\ibal(q_\theta,T_{\phi})$. %
Then,
\begin{align}
    I_{\textsc{BA}_L}(q_{\theta})  \leq \ibal(q_\theta,T_{\phi^*}) \leq \Ixz =  I_{\textsc{BA}_L}(q_{\theta})  + \Exp{\px}{\DKL[p(\vz|\vx)\|\qzx]}.
\end{align}
In particular, the contrastive term in \myeq{mine-ais} is upper bounded by $\Exp{\px}{\DKL[p(\vz|\vx)\|\qzx]}$.  
\end{restatable}
See  \myapp{mine-ais-properties} for proof.
In the following sections, we present our \textsc{mine-ais} method, which overcomes the intractability of $\gls{IBAL}(q_\theta, \giwaeT)$ using an \gls{MCMC}-based training scheme and \gls{AIS}-based evaluation.

\subsubsection{Energy-Based Training of the IBAL}\label{sec:mine_ais_training}
Although the log partition function $\logZmine$ in the \gls{IBAL} is intractable to evaluate,
we only require an unbiased estimator of its gradient for training.   Differentiating \cref{eq:mine-ais} with respect to the parameters $\theta$ and $\phi$ of the variational and energy function, respectively, we obtain 
\begin{align}
\fpartialf{\theta} \ibal(q_\theta,T_\phi) &= \Exp{\pxandz}{\fpartialf{\theta} \log \qzx} - \Exp{\px \pi_{\theta,\phi}(\vz|\vx) }{\fpartial \log \qzx}.  \label{eq:mine-asi-derv}\\
\fpartialf{\phi}  \ibal(q_\theta,T_\phi) &=  \Exp{\pxandz}{\fpartialf{\phi} \Txz} - \Exp{\px \pi_{\theta,\phi}(\vz|\vx)}{\fpartialf{\phi} \Txz}.  
\label{eq:mine-asi-derv2}
\end{align}

\myeq{mine-asi-derv} indicates that in order to maximize the \gls{IBAL} as a function of $\theta$ and $\phi$,
we need to increase the value of the energy function $\giwaeT$ or score function $\log \qzx$ on the \emph{real} samples of the true joint $\pxandz$ and decrease the value on \emph{fake} samples from the approximate $\px \pi_{\theta,\phi}(\vz|\vx)$.   However, 
as is common in training energy-based models, it is difficult to draw samples from $\pi_{\theta,\phi}$. 

In order to sample from $\pi_{\theta,\phi}(\vz|\vx)$, a natural approach is to initialize \gls{MCMC} chains from a sample of the base distribution $\qzx$, using \gls{HMC} transition kernels~\citep{neal2011mcmc}, for example.   However, we may require infeasibly long \gls{MCMC} chains when the base distribution is far from desired energy-based model $\pi_{\theta,\phi}(\vz|\vx)$.
Instead, we can choose to initialize the chains from the true posterior sample $\vz \sim p(\vz|\vx)$ for a simulated data point $\vx \sim p(\vx)$.
Letting $\mathcal{T}_{1:M}$ indicate the composition of $M$ transition steps, the approximate energy function gradient becomes
\begin{align}
\fpartialf{\phi}  \ibal(q_\theta,T_\phi) &\approx \Exp{\pxandz}{\fpartialf{\phi} \Txz} - \Exp{p(\vx,\vz_0) \mathcal{T}_{1:M}(\vz|\vz_0,\vx)}{\fpartialf{\phi} T_\phi(\vx,\vz)} ,  
\label{eq:mine-asi-derv3}
\end{align}
We can use an identical modification for the gradient with respect to $\theta$.

This initialization greatly reduces the computational cost and variance of the estimated gradient and enables training energy functions 
in high dimensional latent spaces, as shown in our experiments.
This approach is in spirit similar to Contrastive Divergence learning of energy-based models \citep{hinton2002training}, where 
one starts an \textsc{mcmc} chain from the true data distribution to obtain a lower variance gradient estimate.

\subsubsection{Multi-Sample AIS Evaluation of the IBAL}\label{sec:mine_ais_eval}
After learning an critic function using our improved \textsc{mine-ais} training scheme, we still need to evaluate the $\ibal$ lower bound on \gls{MI}.  %
Although we have an unbiased, low variance estimate for $\Exp{\pxandz}{\Txz}$, the log partition function $\logZmine$ is intractable.   We use our Multi-Sample \gls{AIS} bounds from \mysec{ais_estimation} to provide both an upper bound and approximate lower bound on $\gls{IBAL}(q_\theta, \giwaeT).$

\parhead{Multi-Sample \textsc{ais} upper bound on $\ibal$.}
We first consider estimating a lower bound on the log partition function $\logZmine$ of $\pi_{\theta, \phi}(\vz|\vx) = \frac{1}{\mathcal{Z}_{\theta, \phi}(\vx)} \qzx e^{\Txz}$. In order to do so, 
we can use a Multi-Sample \gls{AIS} lower bound with a base distribution $\qzx$, intermediate distributions $\pi_t(\vz|\vx)\propto \qzx e^{  \beta_t \, \giwaeT(\vx,\vz)}$, and final distribution $\pi_{\theta, \phi}(\vz|\vx)$ at $\beta_T=1$.
This yields a stochastic lower bound on $\logZmine$, which translates to a stochastic upper bound for estimating $\gls{IBAL}(q_\theta, \giwaeT)$.    

\parhead{Multi-Sample \textsc{ais} approximate lower bound on $\ibal$.} %
In order to obtain a lower bound on $\ibal$ and preserve a lower bound on $\Ixz$, we would like to estimate an upper bound on $\logZmine$ using Multi-Sample \gls{AIS}.   However, note that this requires true samples from $\pi_{\theta, \phi}(\vz|\vx)$ to initialize backward chains.  Since these samples are unavailable, we argue in
\myapp{reverse_annealing_ais} \myprop{mine-ais-sampling-ais} that, under mild conditions,
Multi-Sample \gls{AIS} 
can preserve an approximate upper bound on $\logZmine$ and lower bound on $\gls{IBAL}(q_\theta, \giwaeT)$ by initializing reverse chains from the true posterior $\pzx$ instead of $\pi_{\theta, \phi}(\vz|\vx)$.  %
With the optimal critic function,  we have $\pi_{\theta, \phi}(\vz|\vx) = p(\vz|\vx)$, and 
our approximate lower bound becomes a true stochastic lower bound on $\ibal(q_{\theta}, T^{*}) = \Ixz$.
If the full joint density $p(\vx,\vz)$ is known, 
the \textsc{mine-ais} upper bound and approximate lower bound on $\ibal(q_{\theta}, \giwaeT)$ 
reduce to the Multi-Sample \gls{AIS} upper and lower bounds on $\Ixz$.
\section{Experiments}\label{sec:experiment}
\newcommand{\scalefig}{.28}

In this section, we evaluate our proposed Multi-Sample \gls{AIS}, \gls{MINE}-\gls{AIS}, and \textsc{giwae} bounds for estimating the mutual information of deep generative models such as \textsc{vae}s and \textsc{gan}s trained on \textsc{mnist} and \textsc{cifar} datasets.
We chose deep generative models as they reflect complex relationships between latent variables and high-dimensional images, and allow comparison across methods with different assumptions. 

We first compare our various Multi-Sample \gls{AIS} bounds in \mysec{experiment_multi_comparison} with the goal of making recommendations for which bound to use in practice.    We then compare Multi-Sample \gls{AIS} and \gls{IWAE} bounds in \mysec{exp_ais} 
for estimating large ground truth values of \gls{MI}.   
We will show that Multi-Sample \gls{AIS} can tightly bound the ground truth \gls{MI}, assuming access to the true $p(\vx|\vz)$.  
This ground truth information is valuable for evaluating the bias of our energy-based lower bounds, including \gls{GIWAE} and \textsc{mine-ais}, which are applicable even when $p(\vx|\vz)$ is unknown.
We provide a detailed comparison of these energy-based \gls{MI} bounds in \mysec{exp_giwae}.
We describe experimental details and link to our publicly available code in \myapp{experiment_details}.

\begin{figure}[t]
\vspace*{-.5cm}
\centering
\subcaptionbox{Linear VAE (K=10000)}{
\includegraphics[scale = \scalefig]{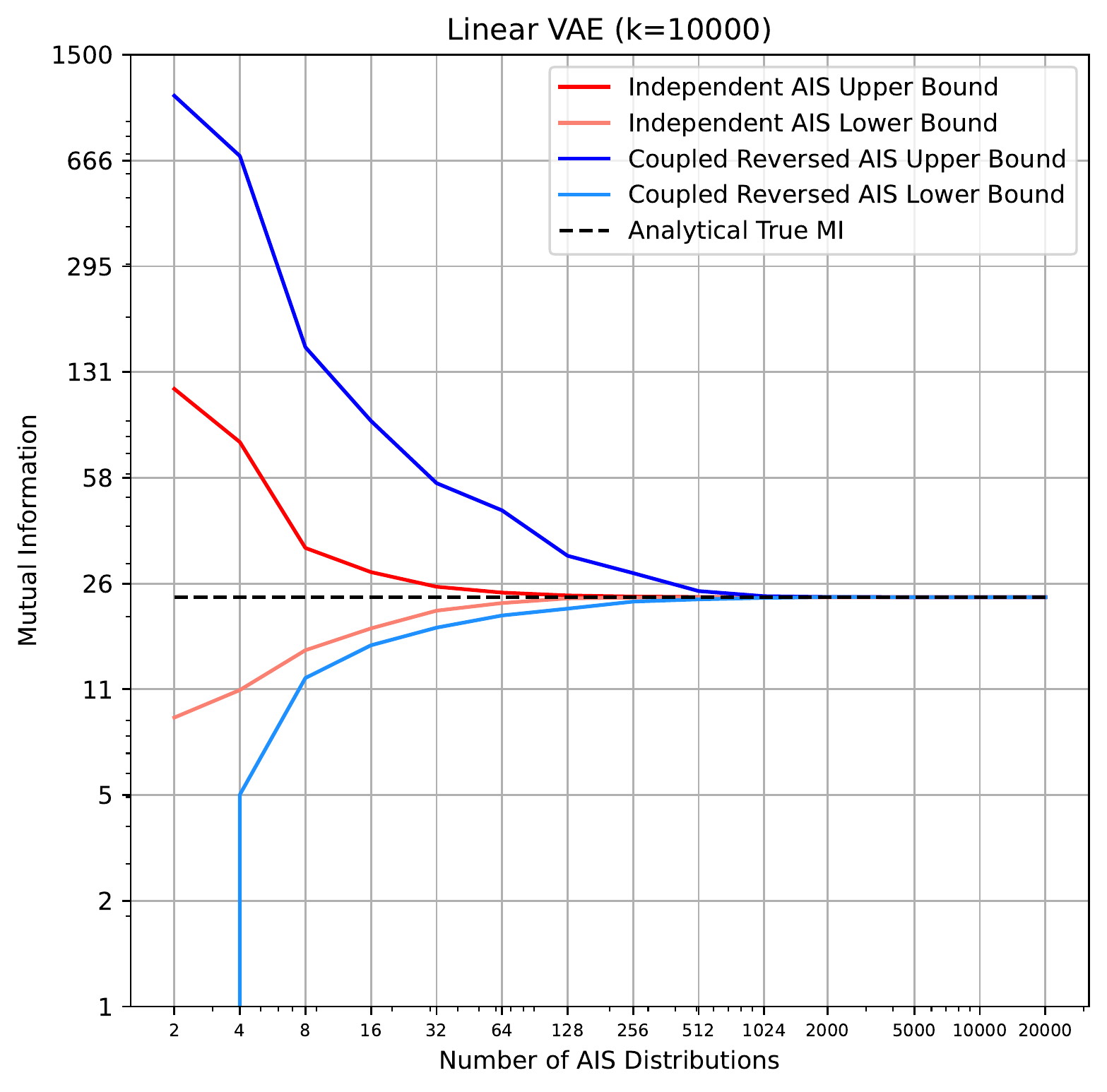}} %
\subcaptionbox{MNIST VAE20 (K=1000)}{
\includegraphics[scale=\scalefig]{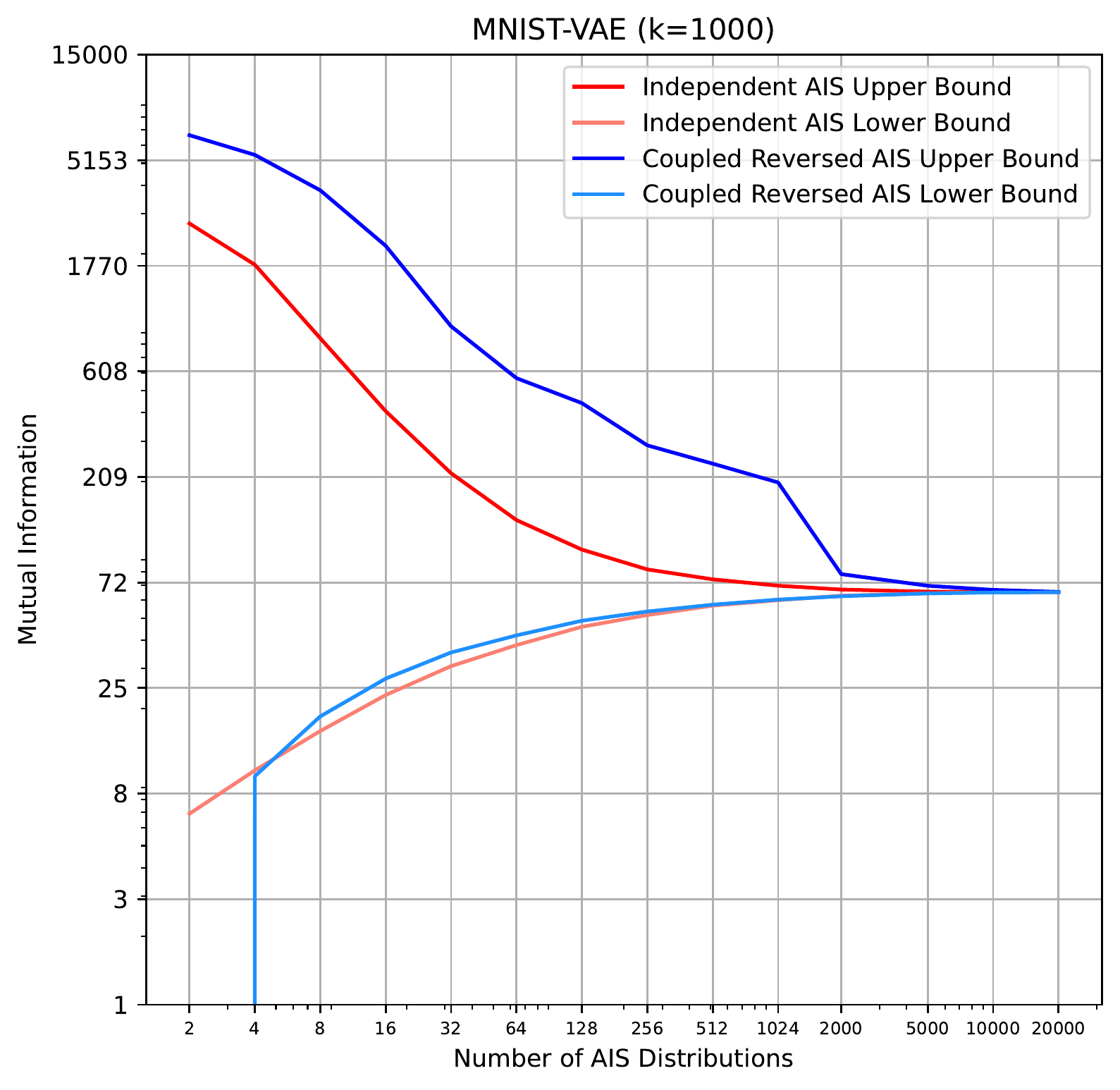}}
\subcaptionbox{MNIST GAN20 (K=1000)}{
\includegraphics[scale=\scalefig]{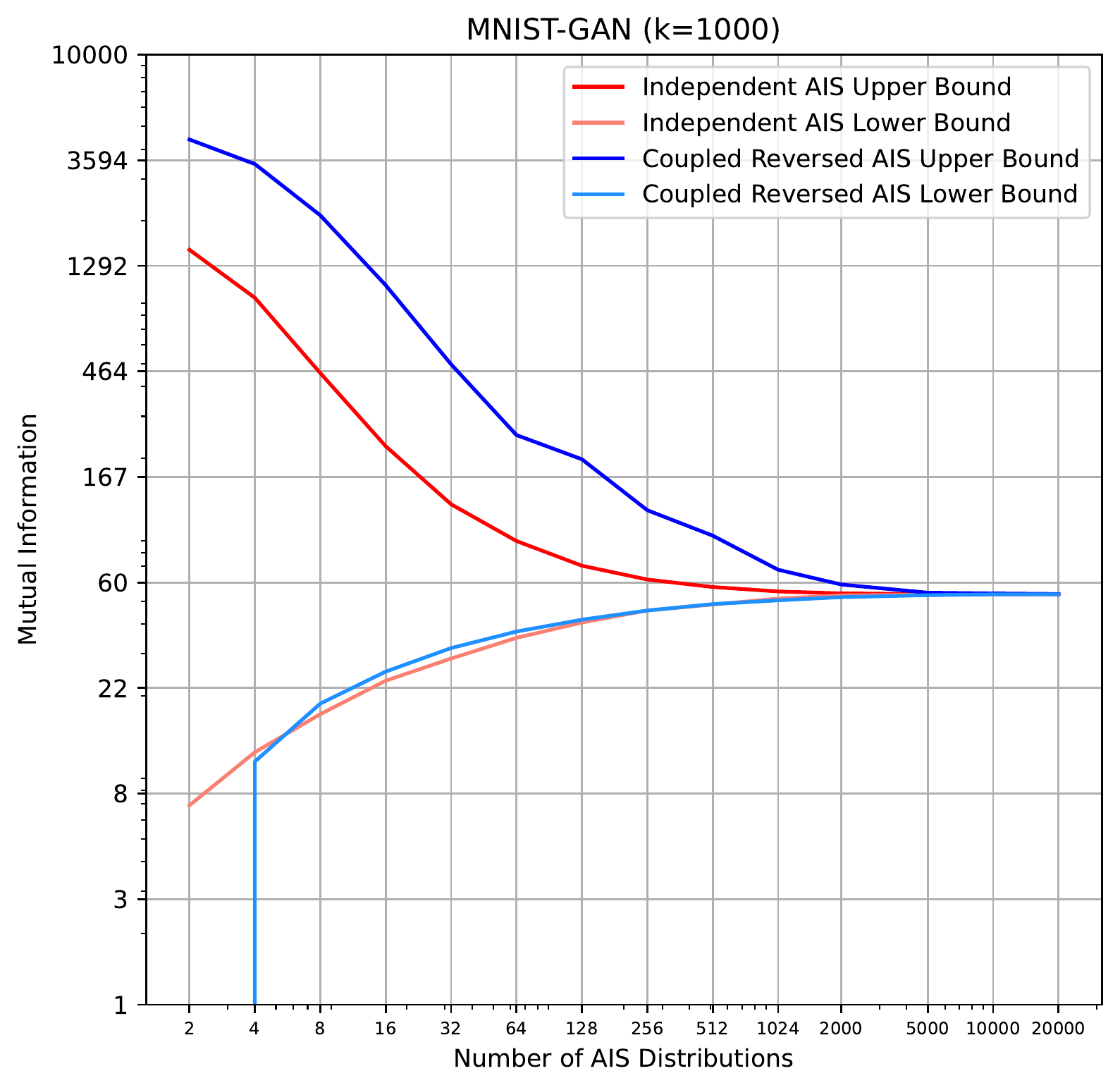}}
\vspace*{-.2cm} \caption{\label{fig:multiple-ais-small} Comparing 
Multi-Sample \gls{AIS} sandwich bounds for 
varying number of \textsc{ais} distributions.}
\end{figure}

\vspace*{-.1cm} 
\subsection{Comparison of Multi-Sample AIS Bounds}\label{sec:experiment_multi_comparison}
\vspace*{-.1cm}
In \myfig{multiple-ais-small}, we compare the performance of our various Multi-Sample \gls{AIS} bounds
for \gls{MI} estimation of a Linear \textsc{vae} with 10 latent variables and random weights, and \textsc{vae} and \textsc{gan} models with 20 latent variables trained on \textsc{mnist}.
We provide more extensive results in \myapp{multi-sample-ais-comparison}.

To obtain an upper bound on \gls{MI}, we recommend using the Independent Multi-Sample \gls{AIS} \textsc{elbo} for log partition function estimation.   This corresponds to the forward direction of \gls{BDMC} and achieves the best performance in all cases.   
This upper bound uses independent samples and is not limited to $\log K$ improvement, 
in contrast to 
the Coupled Reverse Multi-Sample \gls{AIS} upper bound on \gls{MI} (\myprop{c_r_multi_sample_ais}).

The results are less conclusive for the Multi-Sample \gls{AIS} lower bounds on \gls{MI}, where either the Independent Multi-Sample \gls{AIS} \gls{EUBO} or Coupled Reverse \gls{AIS} \gls{EUBO} may be preferable for log partition function estimation.  
Recall that these bounds have different sources of stochasticity that provide improvement over single-chain \gls{AIS}.  
The stochasticity in the Independent Multi-Sample \gls{AIS} lower bound on \gls{MI}
comes from $K-1$ independent negative forward chains which, by 
\myapp{ais_logarithmic_pf} \myprop{ais_elbo_eubo},
can only lead to $\log K$ improvement over single-sample \gls{AIS}.  However, these gains are easily attained for low values of $T$. For example, 
with 
two total \textsc{ais} distributions, which corresponds to simple importance sampling with $\pi_0(\vz)=p(\vz)$,
 the Independent Multi-Sample \gls{AIS} 
 lower bound on \gls{MI} 
reduces to Structured \textsc{Info-NCE} and saturates to $\log K$.
This may be useful for quickly estimating values of \gls{MI} at a similar order of magnitude as $\log K$.

The stochasticity in the Coupled Reverse \gls{AIS} lower bound on \gls{MI} is induced by \gls{MCMC} transitions in $K$ coupled backward chains.   
While this does not formally limit the improvement over single-sample \gls{AIS}, we see in \myfig{multiple-ais-small} that 
at least moderate values of $T$ may be needed to match or marginally improve upon Independent Multi-Sample \gls{AIS}.  These observations suggest that the
preferred lower bounds on \gls{MI} may vary based on the scale of the true \gls{MI} and the amount of computation available.
\captionsetup[table]{position=top}
\begin{figure*}[t]
\begin{center}
\vspace*{-1cm}
\captionof{table}{\gls{MI} Estimation on \textsc{mnist} and \textsc{cifar-10} with \gls{IWAE} (with varying number of samples $K$), and Multi-Sample \gls{AIS} (with varying number of intermediate distributions $T$).  Bounds with a gap of less than 2 nats from the ground truth \gls{MI} are in bold.}\label{table:mnist-cifar}
 \vspace*{-.1cm}
\resizebox{14cm}{!}{
\begin{tabular}{c|c|ccc|ccc}%
\toprule
\textbf{\small Method}&\small \small \textbf{Proposal}&\textbf{MNIST VAE2}&\textbf{MNIST VAE10}&\textbf{MNIST VAE100}&\textbf{MNIST GAN2}&\textbf{MNIST GAN10}&\textbf{MNIST GAN100}\\%
\cmidrule{1-8}%
\multirow[c]{2}{*}{\shortstack[c]{\textbf{AIS} \\ (T=1)}}&\textbf{$ p(\vz)$}&$(0.00, 249.82)$&$(0.00, 1929.84)$&$(0.00, 5830.52)$&$(0.00, 726.27)$&$(0.00, 786.12)$&$(0.00, 861.38)$\\%

\cmidrule{2-8}%
&\textbf{$q(\vz|\vx)$}&$\textbf{(7.59, 9.24)}$&$(21.06, 63.00)$&$(34.49, 362.13)$&$(7.21, 19.13)$&$(3.67, 314.72)$&$(2.61, 513.33)$\\%

\cmidrule{1-8}
\multirow[c]{2}{*}{\shortstack[c]{\textbf{AIS} \\ (T=500)}}&\textbf{$ p(\vz)$}&$\textbf{(8.63, 9.12)}$&$(34.05, 39.09)$&$(\textbf{79.90}, 95.17)$&$\textbf{(9.21, 10.83)}$&$\textbf{(21.57, 22.47)}$&$\textbf{(25.86, 27.55)}$\\%

\cmidrule{2-8}%
&\textbf{$q(\vz|\vx)$}&$\textbf{(9.09, 9.09)}$&$\textbf{(34.16, 34.29)}$&$\textbf{(80.19, 82.34)}$&$\textbf{(10.69, \textbf{11.06})}$&$\textbf{(21.60, 23.06)}$&$(\textbf{25.58}, 29.53)$\\%
 
\cmidrule{1-8}%
\multirow[c]{2}{*}{\shortstack[c]{\textbf{AIS} \\ (T=30K)}}&\textbf{$ p(\vz)$}&$\textbf{(8.98, 9.09)}$&$\textbf{(34.21, 34.21)}$&$\textbf{(80.78, 80.84)}$&$\textbf{(10.56, 10.81)}$&$\textbf{(21.97, 22.02)}$&$\textbf{(26.47, 26.52)}$\\%

\cmidrule{2-8}
&\textbf{$q(\vz|\vx)$}&$\textbf{(9.09, 9.09)}$&$\textbf{(34.21, 34.21)}$&$\textbf{(80.77, 80.80)}$&$\textbf{(10.80, 10.81)}$&$\textbf{(22.01, 22.01)}$&$\textbf{(26.53, 26.54)}$\\%

\cmidrule{1-8}%
\morecmidrules
\cmidrule{1-8}%

\multirow[c]{2}{*}{\shortstack[c]{\textbf{IWAE} \\ (K=1)}}&\textbf{$ p(\vz)$}&$(0.00, 799.55)$&$(0.00, 3827.58)$&$(0.00, 11501.92)$&$(0.00, 1638.10)$&$(0.00, 1630.00)$&$(0.00, 1740.39)$\\%

\cmidrule{2-8}%
&\textbf{$q(\vz|\vx)$}&$\textbf{(8.63, 9.19)}$&$(25.20, \textbf{35.34})$&$(44.54, 95.63)$&$(8.83, 17.58)$&$(4.23, 57.47)$&$(3.23, 260.87)$\\%

\cmidrule{1-8}%
\multirow[c]{2}{*}{\shortstack[c]{\textbf{IWAE} \\ (K=1K)}}&\textbf{$ p(\vz)$}&$(6.81, 29.40)$&$(6.91, 1197.75)$&$(6.91, 4234.19)$&$(6.88, 121.89)$&$(6.91, 446.80)$&$(6.91, 494.73)$\\%

\cmidrule{2-8}%
&\textbf{$q(\vz|\vx)$}&$\textbf{(9.09, 9.10)}$&$(31.69, \textbf{34.24})$&$(51.44, 85.30)$&$\textbf{(10.74, 11.40)}$&$(11.14, 52.73)$&$(10.14, 201.18)$\\%

\cmidrule{1-8}%
\multirow[c]{2}{*}{\shortstack[c]{\textbf{IWAE} \\ (K=1M)}}&\textbf{$ p(\vz)$}&$\textbf{(9.09, 9.09)}$&$(13.82, 376.89)$&$(13.82, 2247.73)$&$\textbf{(10.76, 10.99)}$&$(13.81, 81.51)$&$(13.82, 114.01)$\\%

\cmidrule{2-8}%
&\textbf{$q(\vz|\vx)$}&$\textbf{(9.09, 9.09)}$&$\textbf{(34.10, 34.22)}$&$(58.35, 83.39)$&$\textbf{(10.81, 10.81)}$&$(17.76, 30.88)$&$(16.98, 58.04)$\\%
\bottomrule 
\end{tabular}
}
\label{table:mnist}
\end{center}
\vspace*{.2cm}
\centering
\centering
\resizebox{9.5cm}{!}{

\begin{tabular}{c|c|ccc}%
\toprule
\textbf{\small Method}&\textbf{\small \small Proposal}&\textbf{CIFAR GAN5}&\textbf{CIFAR GAN10}&\textbf{CIFAR GAN100}\\%
\cmidrule{1-5}%
\multirow[c]{2}{*}{\shortstack[c]{\textbf{AIS} \\T=1}}&{\textbf{$ p(\vz)$}}&$(0.00
, 3601256.00)$&$(0.00, 4035635.75)$&$(0.00, 4853410.50)$\\%
\cmidrule{2-5}%
&{\textbf{$q(\vz|\vx)$}}&$(10.80, 157378.25)$&$(13.02, 758075.75)$&$(12.47, 2724562.25)$\\%

\cmidrule{1-5}%
\multirow[c]{2}{*}{\shortstack[c]{\textbf{AIS} \\T=500}}&{\textbf{$ p(\vz)$}}&$(18.37, 259.27)$&$(29.52, 33089.90)$&$(104.51, 63290.40)$\\%
\cmidrule{2-5}%
&{\textbf{$q(\vz|\vx)$}}&$(32.47, 69.54)$&$(48.16, 136.15)$&$(145.19, 2786.53)$\\%

\cmidrule{1-5}%
\multirow[c]{2}{*}{\shortstack[c]{\textbf{AIS} \\T=100K}}&{\textbf{$ p(\vz)$}}&$\textbf{(39.58, 41.06)}$&$\textbf{(71.87, 73.98)}$&$(480.26, 488.07)$\\%
\cmidrule{2-5}%
&{\textbf{$q(\vz|\vx)$}}&$\textbf{(39.22, 40.05)}$&$\textbf{(72.85, 73.54)}$&$(479.27, 484.84)$\\%

\cmidrule{1-5}\morecmidrules\cmidrule{1-5}%
\multirow[c]{2}{*}{\shortstack[c]{\textbf{IWAE} \\K=1}}&{\textbf{$ p(\vz)$}}&$(0.00, 7095384.00)$&$(0.00, 7765695.50)$&$(0.00, 9916102.00)$\\%
\cmidrule{2-5}%
&{\textbf{$q(\vz|\vx)$}}&$(14.53, \textbf{40.31})$&$(17.45, 77.52)$&$(20.00, 5346.85)$\\%

\cmidrule{1-5}%
\multirow[c]{2}{*}{\shortstack[c]{\textbf{IWAE} \\K=1K}}&{\textbf{$ p(\vz)$}}&$(6.91, 1065552.25)$&$(6.91, 2044170.75)$&$(6.91, 2856714.50)$\\%
\cmidrule{2-5}%
&{\textbf{$q(\vz|\vx)$}}&$(21.43, \textbf{39.73})$&$(24.06, \textbf{74.00})$&$(26.98, 5283.13)$\\%

\cmidrule{1-5}%
\multirow[c]{2}{*}{\shortstack[c]{\textbf{IWAE} \\K=1M}}&{\textbf{$ p(\vz)$}}&$(13.82, 96698.10)$&$(13.82, 710511.63)$&$(13.82, 1903854.50)$\\%
\cmidrule{2-5}%
&{\textbf{$q(\vz|\vx)$}}&$(28.34, \textbf{39.71})$&$(30.73, \textbf{73.36})$&$(33.81, 5271.56)$\\%
\bottomrule

\end{tabular}
}
 \vspace*{-.4cm}
  \label{table:cifar}
\end{figure*}

\vheader
\subsection{Multi-Sample AIS Estimation of Mutual Information}\label{sec:exp_ais}
\vheader
We compare Multi-Sample \gls{AIS} \gls{MI} estimation against \gls{IWAE}, since both methods assume the full joint distribution is available.
For the initial distribution of \gls{AIS} or variational distribution of \gls{IWAE}, we can use any distribution that is tractable to sample and evaluate.
We experiment using both the prior $\pz$ and a learned Gaussian $\qzx$. 
\mytable{mnist-cifar}
summarizes our results.

\parhead{IWAE}
As described in \mysec{partition}, \gls{IWAE} bounds encompass a wide range of \gls{MI} estimators. 
The $K=1$ bounds with learned $\qzx$ correspond to \gls{BA} bounds, 
while for $K>1$ and $p(\vz)$ as the proposal, we obtain Structured \textsc{InfoNCE}.
While the \gls{IWAE} upper bound on \gls{MI}, which uses the $\log p(\vx)$ lower bound with independent sampling from $\qzx$, is tight for certain models, we can see that the improvement of the \gls{IWAE} lower bound on \gls{MI} is limited to $\log K$, as expected from \myprop{iwae_elbo_eubo}.
In particular,  exponentially large sample size 
is required to 
to close the gap from the \gls{BA} lower bound ($K=1$) to the true \gls{MI}.
For example, 
on  \textsc{cifar} \textsc{gan100}, at least  $e^{460}$
total samples are required
to match the lower bound estimated by \gls{AIS}.  In \myapp{iwae_log_improvement}, we explicitly decompose $I_{\gls{IWAE}_L}(q_\theta, K)$ into an $I_{\gls{BA}_L}(q_\theta)$ term and a contrastive term to validate this logarithmic improvement with increasing $K$.

\parhead{Multi-Sample AIS} 
We evaluate Multi-Sample \gls{AIS} with $K=48$ chains 
on \textsc{mnist} and $K=12$ on \textsc{cifar},
and a varying number of intermediate distributions $T$.  
We show results for 
the Independent Multi-Sample \gls{AIS} upper bound on \gls{MI} and Coupled Reverse Multi-Sample \gls{AIS} lower bound on \gls{MI} in Table \ref{table:mnist}. 
Using large enough values of $T$, Multi-Sample \gls{AIS} can tightly sandwich large values of ground truth \gls{MI} for all models and datasets considered. 
This is in stark contrast to the exponential sample complexity required for the \gls{IWAE} \gls{MI} lower bound, and highlights that increasing $T$ in Multi-Sample \gls{AIS} is a practical way to reduce bias using additional computation.  We provide runtime details in \myapp{runtime}. 

\vheader
\subsection{Energy-Based Estimation of Mutual Information}\label{sec:mine-ais-experiments}\label{sec:exp_giwae}
\vheader

In this section, we evaluate the performance of our \gls{GIWAE} and \textsc{mine}-\gls{AIS} methods, which assume access to a known marginal $p(\vz)$ but not the conditional $p(\vx|\vz)$.  %
In \mycor{giwae_logk}, we have shown that for a \textit{fixed} $\qzx$, we have $I_{\textsc{BA}_L}(q_{\theta})  \leq I_{\textsc{GIWAE}_L}(q_{\theta},T_{\phi^*},K) \leq I_{\textsc{IWAE}_L}(q_{\theta},K) \leq  I_{\textsc{BA}_L}(q_{\theta})  +\log K$.
Although we perform separate optimizations and obtain different $\qzx$ for each entry in \cref{fig:a}, we find that these relationships hold in almost all cases.
We summarize the gaps in these bounds and their relationships in \cref{fig:b}.

\parhead{BA, IWAE, and GIWAE Bounds}
Recall that $I_{\textsc{GIWAE}_L}(q_\theta,\giwaeT,K)$ (\myeq{giwae_lb_mi_main}) and $I_{\textsc{IWAE}_L}(q_\theta,K)$ (\myeq{iwae_two_terms_main}) can be decomposed into the sum of a \gls{BA} lower bound and a contrastive term.   
We report the contribution of each term along with the overall lower bound in \myfiga{different-bounds}{a}. 
Despite the fact that \gls{GIWAE} uses a learned $\giwaeT(\vx,\vz)$ instead of the optimal critic in \gls{IWAE} (\mycor{giwae}), we observe that \gls{GIWAE} can approach the performance of \gls{IWAE}.   We can also confirm that both bounds improve upon the \gls{BA} lower bound.
In all cases, (Structured) \textsc{InfoNCE} bounds, which use $\qzx = p(\vz)$, saturate to $\log K$.

\begin{figure}[!t]
\vspace*{-1.3cm}
\centering
\subcaptionbox{Linear VAE}{
\includegraphics[scale=.33]{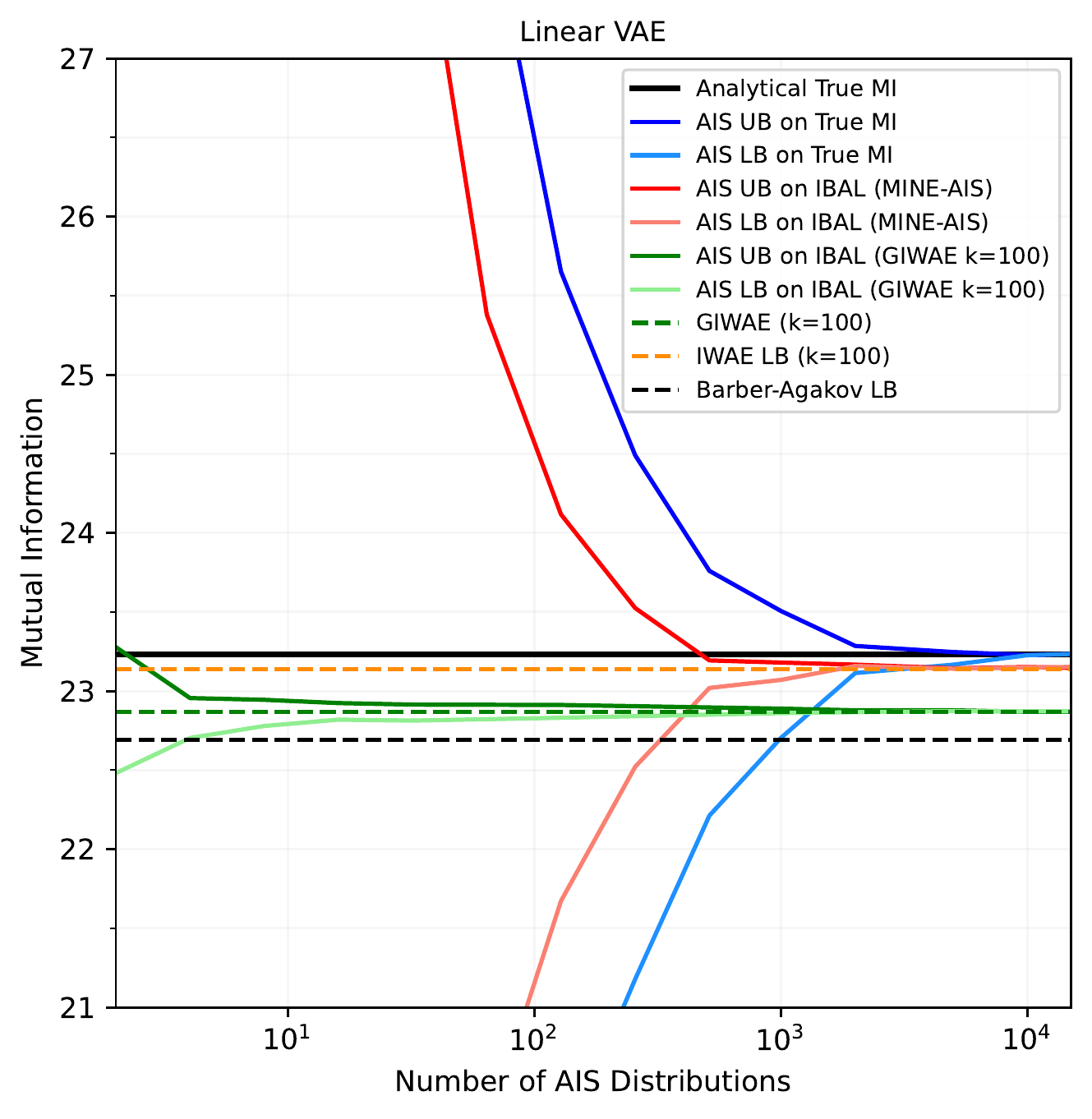}}
\subcaptionbox{MNIST VAE}{
\includegraphics[scale=.33]{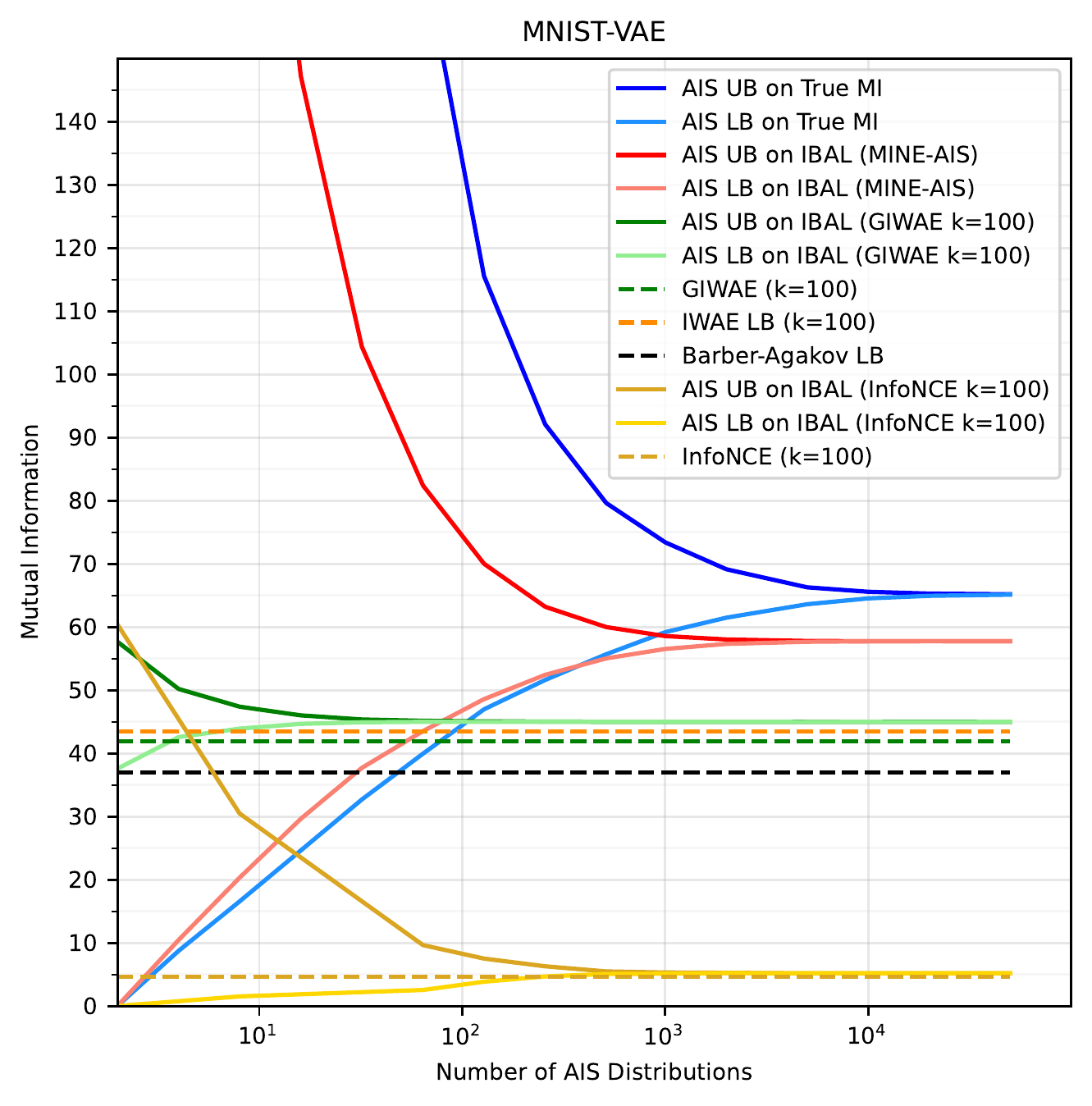}}
\subcaptionbox{MNIST GAN}{
\includegraphics[scale=.33]{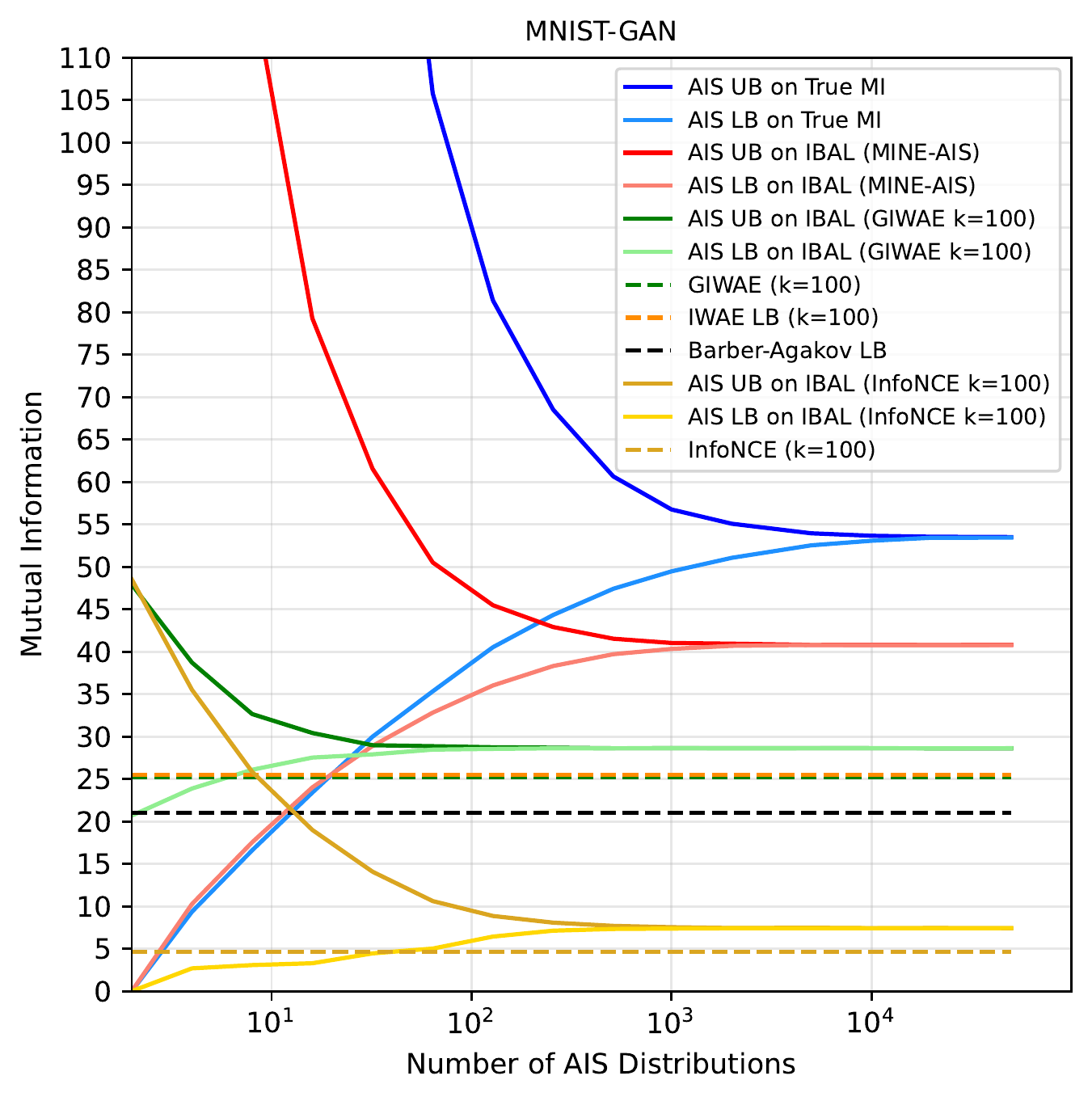}}
\vspace*{-.2cm} \caption{\label{fig:mine-ais} Estimating 
\textsc{ibal} using Multi-Sample \gls{AIS} for various methods of critic function training.}
\vspace*{.2cm} 
\end{figure}
\begin{figure*}[t]
\begin{minipage}{0.775\textwidth}
\subcaptionbox{\label{fig:a}}{
\centering
\resizebox{\textwidth}{!}{
\begin{tabular}{p{0.12\textwidth}|l|c|c|c}
\toprule 
\multirow{2}{2cm}{\small Input \\[.25ex]
Used} & \hspace{1.75cm}\backslashbox{Bound}{Model}  & 
\multirow{2}{2cm}{Linear \\[.25ex] VAE10} &  \multirow{2}{2cm}{MNIST \\[.25ex] VAE20} &  \multirow{2}{2cm}{MNIST \\[.25ex] GAN20} \\
\\
\midrule 
 $p(\vx,\vz)$ & Analytical & 23.23 & N/A & N/A \\
\midrule 
 \multirow{5}{2cm}{$p(\vx,\vz)$ \\[.25ex] \text{\small Joint} \\ \text{\small Samples} } & AIS Bound on True MI & $(23.23, 23.23)$ & $(65.11,65.17)$  & $(53.43,53.50)$ \\
\cmidrule{2-5} 
& IWAE LB ($K=1000$)&  $20.53+2.66=23.19$ & $38.21+6.90=45.11$ & $20.97+6.91=27.88$ \\
\cmidrule{2-5} 
& IWAE LB ($K=100$)  & $21.64+1.50=23.14$  &  $38.86+4.61=43.47$  &   $20.86+4.60=25.46$ \\
\cmidrule{2-5} 
  & Structured InfoNCE LB ($K=1000$)  & $6.91$ & $6.91$ & $6.91$  \\
\cmidrule{2-5} 
    & Structured InfoNCE LB ($K=100$) & $4.61$ & $4.61$ & $4.61$ \\
\midrule 
 \multirow{3}{3cm}{$\pz$ \\[.25ex] \text{\small Joint} \\ \text{\small Samples}} & AIS Bound on IBAL (MINE-AIS) & $(23.15,23.15)$  & $(57.72,57.74)$ &  $(40.79,40.79)$ \\
\cmidrule{2-5} 
  & AIS Bound on IBAL (GIWAE $K=100$)  &  $(22.87,22.87)$ & $(44.97,44.97)$ &  $(28.61,28.62)$ \\ \cmidrule{2-5} 
    & AIS Bound on IBAL (InfoNCE $K=100$)  &  $(11.38,11.39)$ & $(5.18,5.18)$ &  $(7.42,7.42)$ \\ \cmidrule{2-5} 
& Generalized IWAE LB ($K=1000$)  & $22.31+0.38=22.69$ & $37.23+6.55=43.78$ &  $20.50+6.72=27.22$ \\
\cmidrule{2-5} 
  & Generalized IWAE LB ($K=100$)  &  $22.48+0.39=22.87$ & $37.56+4.34=41.90$ &  $20.68+4.57=25.25$ \\
 \cmidrule{2-5} 
& Barber-Agakov LB ($K=1$) & $22.69$ &  $37.92$ &  $21.42$\\
\midrule 
 \multirow{2}{3cm}{\text{\small Joint} \\ \text{\small Samples} }
  & InfoNCE LB ($K=1000$)  & $6.91$ & $6.91$ & $6.91$  \\
\cmidrule{2-5} 
  & InfoNCE LB ($K = 100$) & $4.61$ & $4.61$ & $4.61$ \\
\bottomrule 
\end{tabular}
}
}
\vspace*{-.2cm}
\end{minipage}
\begin{minipage}{.22\textwidth}
\centering
\subcaptionbox{\label{fig:b} }{
\includegraphics[scale=1.26]{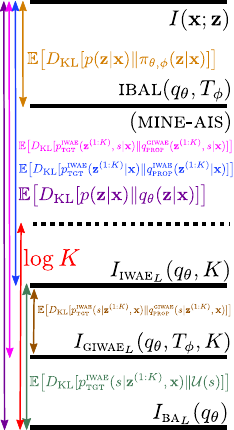}%
}
\end{minipage}
\caption{\label{fig:different-bounds} %
(a) Comparison of energy-based bounds (\gls{GIWAE} and \textsc{mine-ais}) with other \gls{MI} bounds. \\
(b) Visualizing the gaps of various energy-based lower bounds and their relationships.} 
\vspace*{-.2cm}
\end{figure*}

We can further analyze the contribution of the contrastive term across different models.
For the Linear \textsc{vae}, the true posterior is in the Gaussian variational family and $I_{\textsc{BA}_L}(q_{\theta})$ is close to the analytical \gls{MI}.  In this case, the contrastive term provides much less than $\log K$ improvement for \gls{GIWAE} and \gls{IWAE}, since even the optimal critic function cannot distinguish between $\qzx$ and $\pzx$. 
As $K$ increases, we learn a worse $q_{\theta}$ in almost all cases, as measured by a lower \gls{BA} term.  This allows 
the contrastive term to achieve closer to its full potential $\log K$, resulting in a higher overall bound.\footnote{
A similar observation can be made in training \textsc{vae}s with the $\gls{ELBO}_{\gls{IWAE}}$ objective and a restricted variational family, where increasing $K$ results in a worse $\qzx$ (in forward \textsc{kl} divergence) but a better overall bound.
}
For more complex \textsc{vae} and \textsc{gan} posteriors, there is 
a reduced
tradeoff between the terms.   In these cases, the variational family is far enough from the true posterior (in reverse \textsc{kl} divergence) that either \gls{GIWAE} or \gls{IWAE} %
critic functions can approach $\log K$ improvement 
without significantly lowering the \gls{BA} term.

\vspace{-.3cm}
\paragraph{MINE-AIS Bounds}  
We now discuss results for \textsc{mine-ais}, where we have used fixed standard Gaussian $p(\vz)$ as the base variational distribution, energy-based training for $\giwaeT(\vx,\vz)$, and Multi-Sample \gls{AIS} evaluation of $\gls{IBAL}(p(\vz),\giwaeT)$.
We can see in \myfiga{different-bounds}{a} that \textsc{mine}-\gls{AIS} improves over \gls{BA} due to its flexible, energy-based variational family.
To evaluate the quality of the learned $\giwaeT(\vx,\vz)$, we also compare the \gls{IBAL} to the Multi-Sample \gls{AIS} lower bound, which assumes access to $p(\vx|\vz)$ and corresponds to the optimal critic (\myprop{mine-ais-optimal}).   We find that \textsc{mine-ais} 
underestimates the ground truth \gls{MI}
by $11\%$ and $24\%$ on \textsc{mnist}-\textsc{vae} and \textsc{mnist}-\textsc{gan}, respectively.

We also observe that \textsc{mine-ais} outperforms \gls{GIWAE} by $32\%$ and $50\%$ on \textsc{mnist}-\textsc{vae} and \textsc{mnist}-\textsc{gan}, respectively.  
We now investigate whether this improvement is due to a more costly Multi-Sample \gls{AIS} evaluation or our energy-based training scheme for the critic $\giwaeT$.   In \myfig{mine-ais} and \myfig{different-bounds}{a}, we use Multi-Sample \gls{AIS} to evaluate the $\gls{IBAL}$ corresponding to $(q_\theta, \giwaeT)$, which are learned by optimizing the \textsc{giwae} (with $K=100$) and \textsc{InfoNCE} (with $K=100$ and $q_\theta=p(\vz)$) lower bounds.   As argued in \myprop{mine-ais-limit}, the \gls{IBAL} corresponds to the limiting behavior of \gls{GIWAE} as $K\rightarrow \infty$.
We observe that the \gls{AIS} evaluation of the \gls{IBAL}, corresponding to \gls{GIWAE} or \textsc{InfoNCE} critic functions, only marginally improves upon 
evaluation of the original \gls{GIWAE} or \textsc{InfoNCE} lower bounds.
This indicates that the improvement of \textsc{mine-ais} over \gls{GIWAE} or \textsc{InfoNCE} can be primarily attributed to learning a better critic function using energy-based training.

Recall that the \gls{IWAE} and Structured \textsc{InfoNCE} critics use the true log importance weights $T^{*}(\vx,\vz) = \log \frac{p(\vx,\vz)}{\qzx} + c(\vx)$ (\mycor{giwae}), and with this optimal critic, $\gls{IBAL}(q_\theta, T^*) = \Ixz$ regardless of $\qzx$ (\myprop{mine-ais-optimal}).    We would thus expect Multi-Sample \gls{AIS} evaluation of the  $\gls{IBAL}(q_\theta, \giwaeT)$ to tightly bound the true \gls{MI} if the critic were optimal.
We instead find that \textsc{ais} evaluation 
for a learned \gls{GIWAE} critic falls short of the
true \gls{MI} by $31\%$ and $46\%$ on \textsc{mnist}-\textsc{vae} and \textsc{mnist}-\textsc{gan}, respectively.
This is despite the fact that the \gls{GIWAE} critic came close to matching 
the performance of \gls{IWAE} and saturating the contrastive term at $\log K$ in the original objective.
These observations highlight a further shortcoming of the contrastive bounds used in \gls{GIWAE} and \textsc{InfoNCE}: 
beyond their $\log K$ limitations for evaluation (\mycor{giwae_logk} and \mycor{infonce}),
these bounds may not be conducive to learning the true log importance weights. 

\paragraph{Validation of Approximate Reverse Annealing for IBAL Lower Bound} 
As discussed in \mysec{mine_ais_eval}, obtaining a lower bound on the \gls{IBAL} is intractable due to the need for exact samples from $\pi_{\theta,\phi}(\vz|\vx)$.   In \myfig{mine-ais}, we confirm that our approximate reverse annealing procedure described in \myapp{ais_eval_ibal} 
underestimates
the \gls{IBAL} for all \textsc{mine-ais}, \textsc{giwae} and \textsc{InfoNCE} experiments, although we cannot mathematically guarantee this procedure provides a lower bound in general. 
With detailed discussion in  \myapp{ais_eval_ibal}, we
derive a sufficient condition under which our approximate reverse annealing procedure is guaranteed to provide a lower bound.

We can use our Multi-Sample \gls{AIS} upper bound on the \gls{IBAL} to validate the convergence of our estimation procedure, despite the fact that this quantity does not lower or upper bound the true \gls{MI} in general.  In other words, we can conclude that the true $\gls{IBAL}(q_{\theta},T_{\phi})$ has been obtained when its lower and upper bounds converge to the same estimate for a large number of intermediate distributions $T$.

\paragraph{MINE-DV and MINE-F} 
We do not include results for \textsc{mine-dv} and \textsc{mine-f}, as they are highly unstable 
in large \textsc{mi} settings 
due to the difficulty of direct Monte Carlo estimation of $\Exp{p(\vx)p(\vz)}{e^{\giwaeT(\vx,\vz)}}$ \citep{mcallester2020formal}.
We expect Generalized \gls{MINE} to suffer from the same challenges and instead recommend using \textsc{mine-ais}. 
\section{Conclusion}\label{sec:conclusion}

In this work, we have provided a unifying view of mutual information estimation from the perspective of importance sampling.
We derived probabilistic interpretations of each bound, which shed light on the limitations of existing estimators
and motivated our novel \gls{GIWAE},  Multi-Sample \gls{AIS}, and \textsc{mine}-\gls{AIS} bounds.
When the conditional is not known, our \gls{GIWAE} bounds highlight how variational bounds can complement contrastive learning to improve lower bounds on \gls{MI} beyond known $\log K$ limitations.
When the full joint distribution is known, we show that our Multi-Sample \gls{AIS} bounds can tightly estimate large values of \gls{MI} without exponential sample complexity, and thus should be considered the gold standard for \gls{MI} estimation in these settings.
Finally, \textsc{mine}-\gls{AIS} extends Multi-Sample \gls{AIS} evaluation to unknown conditional densities, and
can be viewed as the infinite-sample behavior of \gls{GIWAE} and existing contrastive bounds.
Our \textsc{mine}-\gls{AIS} and Multi-Sample \gls{AIS} methods highlight how \gls{MCMC} techniques can be used to improve \gls{MI} estimation when a single analytic marginal or conditional density is available.

\clearpage
\vspace*{-.4cm}
\section*{Acknowledgements}

The authors thank Shengyang Sun, Guodong Zhang, Vaden Masrani, and Umang Gupta for helpful comments on drafts of this work.    We also thank the anonymous reviewers whose comments greatly improved the presentation and encouraged us to derive several additional propositions.  
MG, RG and AM acknowledge support from the Canada CIFAR AI Chairs program.

\bibliography{iclr2022_conference}

\begin{thebibliography}{47}
\providecommand{\natexlab}[1]{#1}
\providecommand{\url}[1]{\texttt{#1}}
\expandafter\ifx\csname urlstyle\endcsname\relax
  \providecommand{\doi}[1]{doi: #1}\else
  \providecommand{\doi}{doi: \begingroup \urlstyle{rm}\Url}\fi

\bibitem[Alemi et~al.(2018)Alemi, Poole, Fischer, Dillon, Saurous, and
  Murphy]{alemi2018fixing}
Alexander Alemi, Ben Poole, Ian Fischer, Joshua Dillon, Rif~A Saurous, and
  Kevin Murphy.
\newblock Fixing a broken {ELBO}.
\newblock In \emph{International Conference on Machine Learning}, pp.\
  159--168, 2018.

\bibitem[Alemi \& Fischer(2018)Alemi and Fischer]{alemi2018gilbo}
Alexander~A Alemi and Ian Fischer.
\newblock {GILBO}: One metric to measure them all.
\newblock In \emph{Advances in Neural Information Processing Systems}, 2018.

\bibitem[Alemi et~al.(2017)Alemi, Fischer, Dillon, and Murphy]{alemi2016deep}
Alexander~A Alemi, Ian Fischer, Joshua~V Dillon, and Kevin Murphy.
\newblock Deep variational information bottleneck.
\newblock In \emph{International Conference on Learning Representations}, 2017.

\bibitem[Amari(2009)]{amari2009alpha}
Shun-Ichi Amari.
\newblock Alpha-divergence is unique, belonging to both f-divergence and
  {Bregman} divergence classes.
\newblock \emph{IEEE Transactions on Information Theory}, 55\penalty0
  (11):\penalty0 4925--4931, 2009.

\bibitem[Arbel et~al.(2021)Arbel, Zhou, and Gretton]{arbel2020generalized}
Michael Arbel, Liang Zhou, and Arthur Gretton.
\newblock Generalized energy based models.
\newblock In \emph{International Conference on Learning Representations}, 2021.

\bibitem[Barber \& Agakov(2003)Barber and Agakov]{barber2003algorithm}
David Barber and Felix Agakov.
\newblock The {IM} algorithm: A variational approach to information
  maximization.
\newblock In \emph{Proceedings of the 16th International Conference on Neural
  Information Processing Systems}, pp.\  201--208, 2003.

\bibitem[Belghazi et~al.(2018)Belghazi, Baratin, Rajeshwar, Ozair, Bengio,
  Courville, and Hjelm]{belghazi2018mutual}
Mohamed~Ishmael Belghazi, Aristide Baratin, Sai Rajeshwar, Sherjil Ozair,
  Yoshua Bengio, Aaron Courville, and Devon Hjelm.
\newblock Mutual information neural estimation.
\newblock In \emph{International Conference on Machine Learning}, pp.\
  531--540. PMLR, 2018.

\bibitem[Burda et~al.(2015)Burda, Grosse, and Salakhutdinov]{burda2015accurate}
Yuri Burda, Roger Grosse, and Ruslan Salakhutdinov.
\newblock Accurate and conservative estimates of {MRF} log-likelihood using
  reverse annealing.
\newblock In \emph{Artificial Intelligence and Statistics}, pp.\  102--110.
  PMLR, 2015.

\bibitem[Burda et~al.(2016)Burda, Grosse, and
  Salakhutdinov]{burda2016importance}
Yuri Burda, Roger~B Grosse, and Ruslan Salakhutdinov.
\newblock Importance weighted autoencoders.
\newblock In \emph{International Conference on Learning Representations}, 2016.

\bibitem[Chatterjee et~al.(2018)Chatterjee, Diaconis,
  et~al.]{chatterjee2018sample}
Sourav Chatterjee, Persi Diaconis, et~al.
\newblock The sample size required in importance sampling.
\newblock \emph{The Annals of Applied Probability}, 28\penalty0 (2):\penalty0
  1099--1135, 2018.

\bibitem[Chen et~al.(2016)Chen, Duan, Houthooft, Schulman, Sutskever, and
  Abbeel]{chen2016infogan}
Xi~Chen, Yan Duan, Rein Houthooft, John Schulman, Ilya Sutskever, and Pieter
  Abbeel.
\newblock {InfoGAN}: Interpretable representation learning by information
  maximizing generative adversarial nets.
\newblock In \emph{Proceedings of the 30th International Conference on Neural
  Information Processing Systems}, pp.\  2180--2188, 2016.

\bibitem[Cichocki \& Amari(2010)Cichocki and Amari]{cichocki2010families}
Andrzej Cichocki and Shun-ichi Amari.
\newblock Families of alpha-beta-and gamma-divergences: Flexible and robust
  measures of similarities.
\newblock \emph{Entropy}, 12\penalty0 (6):\penalty0 1532--1568, 2010.

\bibitem[Cranmer et~al.(2020)Cranmer, Brehmer, and Louppe]{cranmer2020frontier}
Kyle Cranmer, Johann Brehmer, and Gilles Louppe.
\newblock The frontier of simulation-based inference.
\newblock \emph{Proceedings of the National Academy of Sciences}, 117\penalty0
  (48):\penalty0 30055--30062, 2020.

\bibitem[Cremer et~al.(2017)Cremer, Morris, and
  Duvenaud]{cremer2017reinterpreting}
Chris Cremer, Quaid Morris, and David Duvenaud.
\newblock Reinterpreting importance-weighted autoencoders.
\newblock \emph{arXiv preprint arXiv:1704.02916}, 2017.

\bibitem[Dembo \& Zeitouni(2009)Dembo and Zeitouni]{DZ10}
Amir Dembo and Ofer Zeitouni.
\newblock \emph{Large deviations techniques and applications}.
\newblock Springer, 2009.

\bibitem[Domke \& Sheldon(2018)Domke and Sheldon]{domke2018importance}
Justin Domke and Daniel~R Sheldon.
\newblock Importance weighting and variational inference.
\newblock In \emph{Advances in neural information processing systems}, pp.\
  4470--4479, 2018.

\bibitem[Doucet et~al.(2022)Doucet, Grathwohl, Matthews, and
  Strathmann]{doucet2022annealed}
Arnaud Doucet, Will~Sussman Grathwohl, Alexander G de~G Matthews, and Heiko
  Strathmann.
\newblock Annealed importance sampling meets score matching.
\newblock In \emph{ICLR Workshop on Deep Generative Models for Highly
  Structured Data}, 2022.

\bibitem[Finke(2015)]{finke2015extended}
Axel Finke.
\newblock \emph{On extended state-space constructions for Monte Carlo methods}.
\newblock PhD thesis, University of Warwick, 2015.

\bibitem[Grosse et~al.(2013)Grosse, Maddison, and
  Salakhutdinov]{grosse2013annealing}
Roger~B Grosse, Chris~J Maddison, and Russ~R Salakhutdinov.
\newblock Annealing between distributions by averaging moments.
\newblock \emph{Advances in Neural Information Processing Systems}, 26, 2013.

\bibitem[Grosse et~al.(2015)Grosse, Ghahramani, and
  Adams]{grosse2015sandwiching}
Roger~B Grosse, Zoubin Ghahramani, and Ryan~P Adams.
\newblock Sandwiching the marginal likelihood using bidirectional monte carlo.
\newblock \emph{arXiv preprint arXiv:1511.02543}, 2015.

\bibitem[Grosse et~al.(2016)Grosse, Ancha, and Roy]{grosse2016measuring}
Roger~B Grosse, Siddharth Ancha, and Daniel~M Roy.
\newblock Measuring the reliability of {MCMC} inference with bidirectional
  {Monte Carlo}.
\newblock In \emph{Proceedings of the 30th International Conference on Neural
  Information Processing Systems}, pp.\  2459--2467, 2016.

\bibitem[Hinton(2002)]{hinton2002training}
Geoffrey~E Hinton.
\newblock Training products of experts by minimizing contrastive divergence.
\newblock \emph{Neural computation}, 14\penalty0 (8):\penalty0 1771--1800,
  2002.

\bibitem[Huang et~al.(2020)Huang, Makhzani, Cao, and
  Grosse]{huang2020evaluating}
Sicong Huang, Alireza Makhzani, Yanshuai Cao, and Roger Grosse.
\newblock Evaluating lossy compression rates of deep generative models.
\newblock In \emph{International Conference on Machine Learning}. PMLR, 2020.

\bibitem[Kingma \& Ba(2014)Kingma and Ba]{adam}
Diederik~P Kingma and Jimmy Ba.
\newblock Adam: A method for stochastic optimization.
\newblock \emph{arXiv preprint arXiv:1412.6980}, 2014.

\bibitem[Krizhevsky(2009)]{krizhevsky2009learning}
Alex Krizhevsky.
\newblock Learning multiple layers of features from tiny images.
\newblock 2009.

\bibitem[Lawson et~al.(2019)Lawson, Tucker, Dai, and
  Ranganath]{lawson2019energy}
Dieterich Lawson, George Tucker, Bo~Dai, and Rajesh Ranganath.
\newblock Energy-inspired models: learning with sampler-induced distributions.
\newblock In \emph{Proceedings of the 33rd International Conference on Neural
  Information Processing Systems}, pp.\  8501--8513, 2019.

\bibitem[LeCun et~al.(1998)LeCun, Bottou, Bengio, Haffner,
  et~al.]{lecun1998gradient}
Yann LeCun, L{\'e}on Bottou, Yoshua Bengio, Patrick Haffner, et~al.
\newblock Gradient-based learning applied to document recognition.
\newblock \emph{Proceedings of the IEEE}, 86\penalty0 (11):\penalty0
  2278--2324, 1998.

\bibitem[Maddison et~al.(2017)Maddison, Lawson, Tucker, Heess, Norouzi, Mnih,
  Doucet, and Teh]{maddison2017filtering}
Chris~J Maddison, Dieterich Lawson, George Tucker, Nicolas Heess, Mohammad
  Norouzi, Andriy Mnih, Arnaud Doucet, and Yee~Whye Teh.
\newblock Filtering variational objectives.
\newblock In \emph{Proceedings of the 31st International Conference on Neural
  Information Processing Systems}, pp.\  6576--6586, 2017.

\bibitem[Makhzani et~al.(2015)Makhzani, Shlens, Jaitly, Goodfellow, and
  Frey]{makhzani2015adversarial}
Alireza Makhzani, Jonathon Shlens, Navdeep Jaitly, Ian Goodfellow, and Brendan
  Frey.
\newblock Adversarial autoencoders.
\newblock \emph{arXiv preprint arXiv:1511.05644}, 2015.

\bibitem[McAllester \& Stratos(2020)McAllester and
  Stratos]{mcallester2020formal}
David McAllester and Karl Stratos.
\newblock Formal limitations on the measurement of mutual information.
\newblock In \emph{International Conference on Artificial Intelligence and
  Statistics}, pp.\  875--884. PMLR, 2020.

\bibitem[Mohamed \& Rezende(2015)Mohamed and Rezende]{mohamed2015variational}
Shakir Mohamed and Danilo~Jimenez Rezende.
\newblock Variational information maximisation for intrinsically motivated
  reinforcement learning.
\newblock \emph{Advances in Neural Information Processing Systems},
  28:\penalty0 2125--2133, 2015.

\bibitem[Neal(2001)]{neal2001annealed}
Radford~M Neal.
\newblock Annealed importance sampling.
\newblock \emph{Statistics and computing}, 11\penalty0 (2):\penalty0 125--139,
  2001.

\bibitem[Neal(2011)]{neal2011mcmc}
Radford~M Neal.
\newblock {MCMC} using hamiltonian dynamics.
\newblock \emph{Handbook of Markov Chain Monte Carlo}, 2011.

\bibitem[Nguyen et~al.(2010)Nguyen, Wainwright, and
  Jordan]{nguyen2010estimating}
XuanLong Nguyen, Martin~J Wainwright, and Michael~I Jordan.
\newblock Estimating divergence functionals and the likelihood ratio by convex
  risk minimization.
\newblock \emph{IEEE Transactions on Information Theory}, 56\penalty0
  (11):\penalty0 5847--5861, 2010.

\bibitem[Nowozin et~al.(2016)Nowozin, Cseke, and Tomioka]{nowozin2016f}
Sebastian Nowozin, Botond Cseke, and Ryota Tomioka.
\newblock f-{GAN}: Training generative neural samplers using variational
  divergence minimization.
\newblock \emph{Advances in Neural Information Processing Systems}, 29, 2016.

\bibitem[Polyanskiy \& Wu(2022)Polyanskiy and Wu]{polyanskiy2022information}
Yury Polyanskiy and Yihong Wu.
\newblock Information theory: From coding to learning, 2022.

\bibitem[Poole et~al.(2019)Poole, Ozair, Van Den~Oord, Alemi, and
  Tucker]{poole2019variational}
Ben Poole, Sherjil Ozair, Aaron Van Den~Oord, Alex Alemi, and George Tucker.
\newblock On variational bounds of mutual information.
\newblock In \emph{International Conference on Machine Learning}, 2019.

\bibitem[Radford et~al.(2015)Radford, Metz, and
  Chintala]{radford2015unsupervised}
Alec Radford, Luke Metz, and Soumith Chintala.
\newblock Unsupervised representation learning with deep convolutional
  generative adversarial networks.
\newblock \emph{arXiv preprint arXiv:1511.06434}, 2015.

\bibitem[Rassoul-Agha \& Sepp{\"a}l{\"a}inen(2015)Rassoul-Agha and
  Sepp{\"a}l{\"a}inen]{RS15}
Firas Rassoul-Agha and Timo Sepp{\"a}l{\"a}inen.
\newblock \emph{A course on large deviations with an introduction to Gibbs
  measures}, volume 162.
\newblock American Mathematical Soc., 2015.

\bibitem[Salimans et~al.(2016)Salimans, Goodfellow, Zaremba, Cheung, Radford,
  and Chen]{salimans2016improved}
Tim Salimans, Ian Goodfellow, Wojciech Zaremba, Vicki Cheung, Alec Radford, and
  Xi~Chen.
\newblock Improved techniques for training {GANs}.
\newblock In \emph{Advances in Neural Information Processing Systems}, pp.\
  2234--2242, 2016.

\bibitem[Sobolev(2019)]{sobolevblog}
Artem Sobolev.
\newblock Thoughts on mutual information estimation: More estimators.
\newblock \emph{Blog post}, 2019.
\newblock URL
  \url{http://artem.sobolev.name/posts/2019-08-10-thoughts-on-mutual-information-more-estimators.html}.

\bibitem[Sobolev \& Vetrov(2019)Sobolev and Vetrov]{sobolev2019hierarchical}
Artem Sobolev and Dmitry~P Vetrov.
\newblock Importance weighted hierarchical variational inference.
\newblock In \emph{Advances in Neural Information Processing Systems},
  volume~32, 2019.

\bibitem[Song \& Ermon(2019)Song and Ermon]{song2019understanding}
Jiaming Song and Stefano Ermon.
\newblock Understanding the limitations of variational mutual information
  estimators.
\newblock In \emph{International Conference on Learning Representations}, 2019.

\bibitem[Tishby et~al.(2000)Tishby, Pereira, and Bialek]{tishby2000information}
Naftali Tishby, Fernando~C Pereira, and William Bialek.
\newblock The information bottleneck method.
\newblock \emph{arXiv preprint physics/0004057}, 2000.

\bibitem[van~den Oord et~al.(2018)van~den Oord, Li, and
  Vinyals]{oord2018representation}
Aaron van~den Oord, Yazhe Li, and Oriol Vinyals.
\newblock Representation learning with contrastive predictive coding.
\newblock \emph{arXiv preprint arXiv:1807.03748}, 2018.

\bibitem[Wu et~al.(2017)Wu, Burda, Salakhutdinov, and
  Grosse]{wu2016quantitative}
Yuhuai Wu, Yuri Burda, Ruslan Salakhutdinov, and Roger Grosse.
\newblock On the quantitative analysis of decoder-based generative models.
\newblock In \emph{International Conference on Learning Representations}, 2017.

\bibitem[Zhao et~al.(2018)Zhao, Song, and Ermon]{zhao2018information}
Shengjia Zhao, Jiaming Song, and Stefano Ermon.
\newblock The information autoencoding family: A {Lagrangian} perspective on
  latent variable generative models.
\newblock In \emph{Proc. 34th Conference on Uncertainty in Artificial
  Intelligence}, 2018.

\end{thebibliography}
\bibliographystyle{iclr2022_conference}

\appendix
\clearpage
\vspace*{-3cm}
\part{Appendix} %
\parttoc
\section{A General Approach for Deriving Extended State Space Bounds on Log Partition Functions}\label{app:multi-sample}\label{app:general} %
In this section, we give a short proof that the gap in our general extended state space bounds from \mysec{general} corresponds to a forward or reverse \kl divergence.   We derive various upper and lower bounds on $\log p(\vx)$ using this approach throughout the paper and appendix, and we provide a visual summary in \myfig{multi-sample-ais}.

First, we consider an extended state space target $\ptgt{}(\args)$ and proposal $\qprop{}(\args)$ distributions.  For all cases discussed in this work, we will choose our target and proposal distributions such that $\log \frac{\mathcal{Z}_\textsc{tgt}(\vx)}{\mathcal{Z}_\textsc{prop}(\vx)} = \log p(\vx)$.   For example, a common construction is to have $\mathcal{Z}_\textsc{tgt}(\vx) = \int \ptgt{}(\args) d\argsZ = p(\vx)$ and $\mathcal{Z}_\textsc{prop}(\vx) = \int \qprop{}(\args) d\argsZ = 1$.  Our `reverse' importance sampling bounds \myapp{rev_iwae}-\ref{app:reverse_ms_ais} construct target and proposal such that $\mathcal{Z}_\textsc{tgt}(\vx) = p(\vx)^K$ and $\mathcal{Z}_\textsc{prop}(\vx) = p(\vx)^{K-1}$, which still yields $\mathcal{Z}_\textsc{tgt}(\vx)/\mathcal{Z}_\textsc{prop}(\vx) = p(\vx)$.   

Each pair of extended state-space proposal and target distributions provides both an upper and lower bound on the log partition function.
Taking the expected log ratio of unnormalized densities under either the proposal or target distribution, we have
\begin{align}\label{eq:general}
   \Exp{q_{\textsc{prop}(\argscond)}}{\log \frac{p_\textsc{tgt}(\args)}{q_{\textsc{prop}}(\args)} } \leq \log \frac{\mathcal{Z}_\textsc{tgt}(\vx)}{\mathcal{Z}_\textsc{prop}(\vx)} \leq \Exp{p_\textsc{tgt}(\argscond)}{\log \frac{p_\textsc{tgt}(\args)}{q_{\textsc{prop}}(\args)} } . %
\end{align}
To confirm that these are indeed lower and upper bounds for any $\qprop{}$ and $\ptgt{}$, we can show that the gap in the lower bound in \myeq{general} is the forward KL divergence $\DKL[\qprop{} \|\ptgt{}]$, and the gap in the upper bound is the reverse KL divergence, $\DKL[\ptgt{}\|\qprop{}]$
\small
\begin{align}
    \Exp{\qprop{}(\argscond)}{\frac{p_\textsc{tgt}(\args)}{q_{\textsc{prop}}(\args)}}
    = \myunderbrace{\log \frac{\mathcal{Z}_\textsc{tgt}(\vx)}{\mathcal{Z}_\textsc{prop}(\vx)} - \DKL[\qprop{}(\argscond) \| \ptgt{}(\argscond) ]}{\gls{ELBO}(\vx;\qprop{}, \ptgt{})} \leq \log \frac{\mathcal{Z}_\textsc{tgt}(\vx)}{\mathcal{Z}_\textsc{prop}(\vx)} \label{eq:fwd_kl_expand}
\end{align}
\begin{align}
    \Exp{\ptgt{}(\argscond)}{\frac{p_\textsc{tgt}(\args)}{q_{\textsc{prop}}(\args)}}
    = \myunderbrace{\log \frac{\mathcal{Z}_\textsc{tgt}(\vx)}{\mathcal{Z}_\textsc{prop}(\vx)} + \DKL[\ptgt{}(\argscond) \| \qprop{}(\argscond) ]}{\gls{EUBO}(\vx;\qprop{}, \ptgt{})} \geq \log \frac{\mathcal{Z}_\textsc{tgt}(\vx)}{\mathcal{Z}_\textsc{prop}(\vx)} \label{eq:rev_kl_expand} .
\end{align}
\normalsize
Thus, the bounds in \myeq{general} directly generalize the standard \gls{ELBO} ($\log p(\vx) - \DKL[\qzx\|\pzx]$) and \gls{EUBO} ($\log p(\vx) + \DKL[\pzx\|\qzx]$), which appear as special cases when $K=1$, $T=1$, the proposal distribution is $\qprop{}(\vz|\vx) = \qzx$, and the target distribution is $\ptgt{}(\vz|\vx) = \pzx$. 
In what follows, our extended state space proposal or target distributions may include $\qzx$ as an initial or base variational distribution, with the posterior $\pzx$ often appearing within target distributions $\ptgt{}(\vz_{\ext}|\vx)$.

In \myfig{multi-sample-ais}, we summarize various extended state space proposal (third column) and target distributions (fourth column).
We emphasize that the base variational distribution $\qzx$ (blue circles) and posterior distribution $\pzx$ (red circles) may be used multiple times, in either or both of the extended state space proposal and target distributions.  Similarly, forward \gls{AIS} chains (blue circles) starting from the initial distribution and the backward \gls{AIS} chains (shown in red circles) starting from the posterior
may be used repeatedly in both the proposal or target.
In the next sections, we proceed to derive each of the bounds in \myfig{multi-sample-ais} as special cases of this general approach, thus interpreting each importance sampling bound in terms of probabilistic inference in an extended state space.

\begin{figure}[!t]
\centering
\includegraphics[scale=0.6]{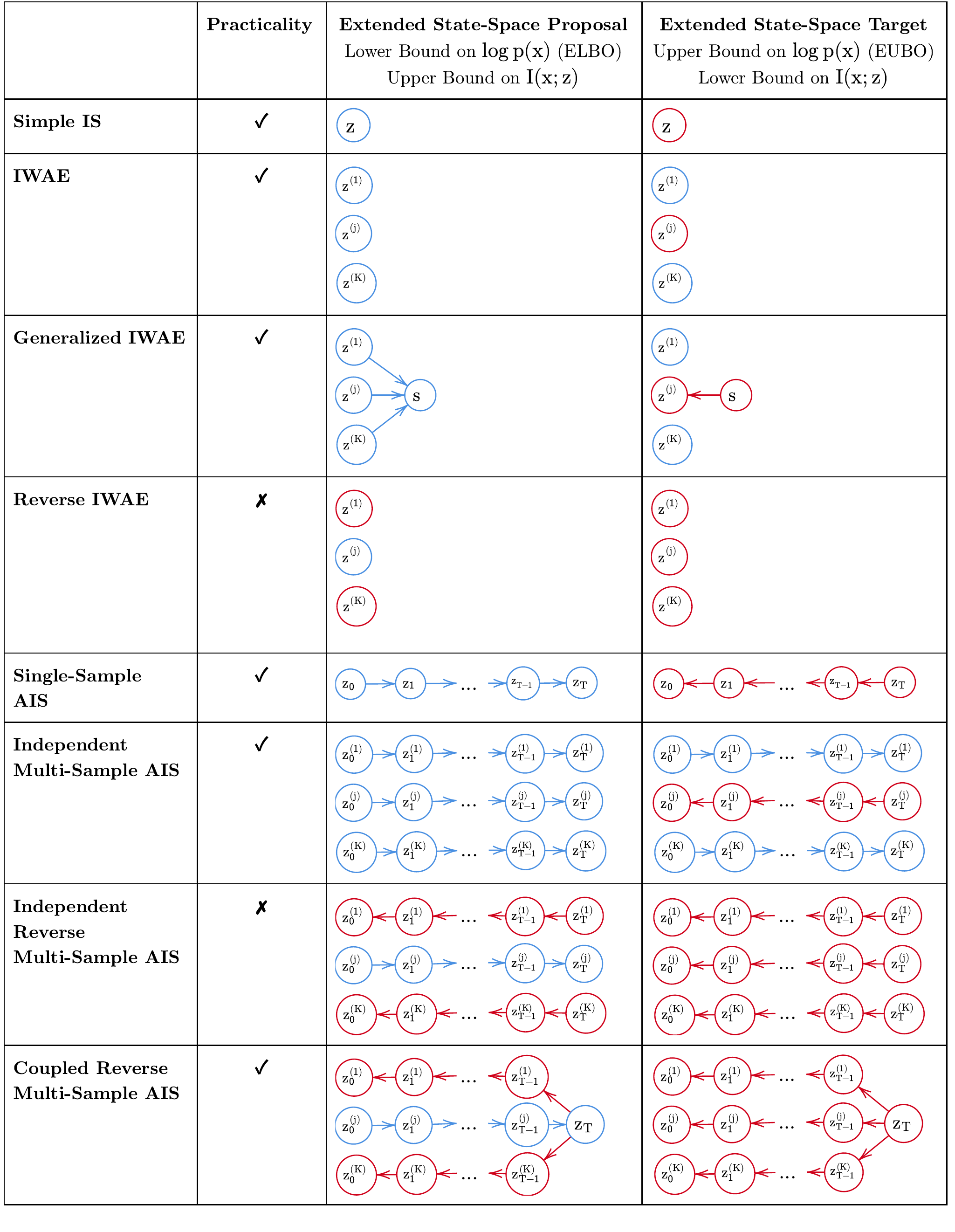}%
\vspace{-.3cm}
\caption{\label{fig:multi-sample-ais} Comparison of the probabilistic extended state-space interpretations of different multi-sample bounds. Forward chains of \textsc{ais} and variational distributions in \textsc{is} / \textsc{iwae} are colored in blue. Backward chains in \textsc{ais} or posterior distributions in \textsc{is} / \textsc{iwae} are colored in red. Note that a lower bound on $\log p(\vx)$, translates to an upper bound on \textsc{mi}, and vice versa.}
\end{figure}

\section{Importance Weighted Autoencoder (IWAE)}\label{app:iwae}
\subsection{Probabilistic Interpretation and Bounds}\label{app:iwae_prob}
Consider a $K$-sample proposal distribution $ \qprop{iwae}(\setofz)$ using independent draws from an initial distribution $q_{\theta}(\vz|\vx)$.  The target distribution
is defined as a uniform mixture of $K$ components, where each component replaces the initial sample $q_{\theta}(\vz^{(k)}|\vx)$ in index $k$ with a sample from the target distribution $\pzx$.  %
\begin{align}
\begin{aligned}
    \qprop{iwae}(\setofz| \vx) = \prod \limits_{\kk=1}^{\K} q_{\theta}(\vz^{(\kk)}|\vx) \label{eq:iwae_fwdrev},
\end{aligned}\hspace*{1cm}
\begin{aligned}
    \ptgt{iwae}(\vx, \setofz) = \frac{1}{\K} \sum \limits_{\kk=1}^{\K} \, \p(\vx, \vz^{(s)}) \, \prod  \limits_{\substack{k =1 \\ k \neq s}}^{K} q_{\theta}(\vz^{(s)}|\vx)  \, . %
\end{aligned}
\end{align}
Note that the normalizing constant of $\ptgt{iwae}(\vx, \setofz)$, or $\int \ptgt{iwae}(\vx, \setofz) d\setofz$, equals $p(\vx)$ since $q_{\theta}(\vz|\vx)$ is normalized and $\int \p(\vx, \vz) d\vz= p(\vx)$.
To connect this with the \gls{IWAE} bound, we show that the $\log$ importance ratio reduces to
\begin{align} 
 \log \frac{\ptgt{iwae}(\vx, \setofz) }{\qprop{iwae}(\setofz|\vx)}  &= \log \frac{ \dfrac{1}{K} \sum \limits_{s=1}^K p(\vx, \vz^{(s)}) \prod \limits_{\myoverset{k=1}{k \neq s}}^K q_{\theta}(\vz^{(k)}|\vx) }{ \prod \limits_{k=1}^K q_{\theta}(\vz^{(k)}|\vx) } \nonumber \\
    &= \log \frac{1}{K} \sum \limits_{k=1}^K \dfrac{ p(\vx, \vz^{(k)})}{ q_{\theta}(\vz^{(k)}|\vx)} . \label{eq:iwae_importance}
\end{align}
As in \cref{eq:general}, we can obtain lower and upper bounds on $\log p(\vx)$ by taking expectations under the proposal and target distributions, respectively.

\paragraph{Alternative Probabilistic Interpretation}
We now present an alternative probabilistic interpretation of \gls{IWAE}, which is similar to \citet{domke2018importance} and will be used as the foundation for our \gls{GIWAE} bounds in \myapp{giwae}.   Consider the following extended state space target distribution,
\begin{align}
\ptgt{iwae}(\vx, \vz^{(1:K)}, \idx) = \frac1K \, p(\vx,\vz^{(\idx)}) \prod \limits_{\substack{k=1 \\ k \neq \idx}}^K q(\vz^{(k)}|\vx) \, . \label{eq:giwae_reverse}
\end{align}
Note that
marginalization over $s$ leads to the \gls{IWAE} mixture target in \myeq{iwae_fwdrev2} or \myeq{iwae_fwdrev}.    
We consider the extended state space proposal
\small
\begin{align}
  \qprop{iwae}(\vz^{(1:K)}, \idx |\vx) &= \bigg( \prod \limits_{k=1}^K q_{\theta}(\vz^{(k)}|\vx) \bigg) \qprop{iwae}(\idx |\vz^{(1:K)}, \vx)\,,
  \label{eq:iwae_s_forward_app}
\end{align}
\normalsize
where we have defined
$\qprop{iwae}(\idx |\vz^{(1:K)}, \vx )= 
{\frac{p(\vx,\vz^{(s)})}{q_{\theta}(\vz^{(s)}|\vx)}} / { \sum \limits_{k=1}^K \frac{p(\vx,\vz^{(k)})}{q_{\theta}(\vz^{(k)}|\vx)}}$.

As desired, the log importance weight match \cref{eq:iwae_importance}
\footnotesize
\begin{align} 
 \log \frac{\ptgt{iwae}(\vx, \setofz, \idx) }{\qprop{iwae}(\setofz, \idx|\vx)}  &= \log \frac{ \dfrac{1}{K} p(\vx, \vz^{(s)}) \prod \limits_{\myoverset{k=1}{k \neq s}}^K q_{\theta}(\vz^{(k)}|\vx) }{ \prod \limits_{k=1}^K q_{\theta}(\vz^{(k)}|\vx)  \frac{\frac{p(\vx,\vz^{(s)})}{q_{\theta}(\vz^{(s)}|\vx)}}{ \sum \limits_{k=1}^K \frac{p(\vx,\vz^{(k)})}{q_{\theta}(\vz^{(k)}|\vx)} }}  
    = \log \dfrac{1}{K} \sum \limits_{k=1}^K \frac{p(\vx,\vz^{(k)})}{q_{\theta}(\vz^{(k)}|\vx)}  \, \label{eq:iwae_importance2}.
\end{align}
\normalsize
\paragraph{Lower Bound on $\log \px$ and Upper Bound on MI}
Using the general approach in \myapp{general}, taking expectations under $\qprop{iwae}$ leads to a lower bound on $\log \px$ 
\begin{align}
    \gls{ELBO}_{\gls{IWAE}}(\vx; q_\theta, K)\coloneqq \mathbb{E}_{\prod \limits_{k=1}^K q_{\theta}(\vz^{(k)}|\vx)} \bigg[\log \frac{1}{K} \sum \limits_{k=1}^K \dfrac{ p(\vx, \vz^{(k)})}{ q_{\theta}(\vz^{(k)}|\vx)} \bigg] \,.
\end{align}
This corresponds to the following upper bound on \gls{MI} for known $p(\vx|\vz)$
\small
\begin{align}
    \Ixz \leq I_{\gls{IWAE}_U}(q_\theta, K) &= \Exp{p(\vx,\vz)}{\log p(\vx|\vz)} - \mathbb{E}_{p(\vx)}\big[\gls{ELBO}_{\gls{IWAE}}(\vx; q_\theta, K) \big]\\
    &= \Exp{p(\vx,\vz)}{\log p(\vx|\vz)} - \mathbb{E}_{p(\vx) \prod \limits_{k=1}^K q_{\theta}(\vz^{(k)}|\vx)} \bigg[\log \frac{1}{K} \sum \limits_{k=1}^K \dfrac{ p(\vx, \vz^{(k)})}{ q_{\theta}(\vz^{(k)}|\vx)} \bigg]
\end{align}
\normalsize

\paragraph{Upper Bound on $\log \px$ and Lower Bound on MI}
Similarly, with expectations under $\ptgt{iwae}( \setofz | \vx)$, we obtain an upper bound on $\log \px$.
Since, for independent draws, the uniform mixture $\ptgt{iwae}$ is invariant to index permutations, we may choose the target sample to be $\vz^{(1)}$ and obtain the upper bound of \citet{sobolev2019hierarchical}:
\begin{align}
 \gls{EUBO}_{\gls{IWAE}}(\vx; q_\theta, K) \coloneqq\Exp{p(\vz^{(1)}|\vx)\qzxi{2:K}}{\log \frac{1}{K} \sum_{i=1}^K \frac{p(\vx,\vz^{(k)})}{q(\vz^{(k)}|\vx)}}   \geq \log\px .
\end{align}
Translating this to a lower bound on \gls{MI} with known $p(\vx|\vz)$,
\small
\begin{align}
    \Ixz \geq I_{\gls{IWAE}_L}&(q_\theta, K) = \Exp{p(\vx,\vz)}{\log p(\vx|\vz)} - \mathbb{E}_{p(\vx)}\big[\gls{EUBO}_{\gls{IWAE}}(\vx; q_\theta, K)\big] \\
    &= \Exp{p(\vx,\vz)}{\log p(\vx|\vz)} - \mathbb{E}_{p(\vx) {p(\vz^{(1)}|\vx)\qzxi{2:K}}} \bigg[\log \frac{1}{K} \sum \limits_{k=1}^K \dfrac{ p(\vx, \vz^{(k)})}{ q_{\theta}(\vz^{(k)}|\vx)} \bigg]\label{eq:iwae_ub_app} .
\end{align}
\normalsize

\subsection{Proof of Logarithmic Improvement in $K$ for IWAE EUBO}\label{app:pf_iwae} %
We first recall results that \gls{IWAE} bounds tighten with increasing $K$. %
\begin{restatable}[]{proposition}{iwaek}\label{prop:iwaek}
For given $\qzx$,  $\gls{ELBO}_{\gls{IWAE}}(\vx; q_\theta, K+1) \geq \gls{ELBO}_{\gls{IWAE}}(\vx; q_\theta, K)$ \citep{burda2016importance} and $\gls{EUBO}_{\gls{IWAE}}(\vx; q_\theta, K+1) \leq \gls{EUBO}_{\gls{IWAE}}(\vx; 
q_\theta, K)$ \citep{sobolev2019hierarchical}.
\end{restatable}

\iwaeelboeubo*
\begin{proof}
We first note the single-sample $\gls{ELBO}(\vx; q_\theta)$ and $\gls{EUBO}(\vx; q_\theta)$ are special cases of 
$\gls{ELBO}_{\gls{GIWAE}}(\vx; q_\theta, \giwaeT, K)$ and $\gls{EUBO}_{\gls{GIWAE}}(\vx; q_\theta,\giwaeT, K)$ 
(\mylemma{elbo_eubo_giwae}) with $\giwaeT = \text{const}$.   
As a result, the gap between $\gls{ELBO}(\vx; q_\theta)$ and $\gls{ELBO}_{\gls{IWAE}}(\vx; q_\theta, K)$, for example, follows as a special case of the gap between $\gls{ELBO}_{\gls{GIWAE}}(\vx; q_\theta, \giwaeT, K)$ and $\gls{ELBO}_{\gls{IWAE}}(\vx; q_\theta, K)$, which we characterize in
\cref{lemma:giwae_elbo_eubo} (\myapp{pf_iwaegiwae}).  The result in \myprop{iwae_elbo_eubo} follows directly.

We now justify the range of the \kl divergences in \cref{eq:iwae_over_elbo} and \cref{eq:iwae_over_eubo} referenced in the underbraces.

\textit{Improvement of $\gls{EUBO}_{\gls{IWAE}}$}:
Note $ \DKL[\indexposterior(\idx|\vz^{(1:K)}, \vx) \| \mathcal{U}(s) ] = \log K - H( \indexposterior(\idx|\vz^{(1:K)}, \vx)) $ is bounded above by $\log K$ since the entropy of a discrete random variable is nonnegative.  

Thus, the improvement of $\gls{EUBO}_{\gls{IWAE}}(\vx; q_\theta, K)$ over $\gls{EUBO}(\vx; q_\theta)$  is limited to $\log K$.   We prove similar results in 
\myprop{iwae_giwae} and \mycor{giwae_logk}  (\myapp{pf_iwaegiwae}),
\myprop{ais_elbo_eubo}
(\myapp{ais_logarithmic_pf}) and 
\myprop{c_r_multi_sample_ais} (\myapp{coupled_pf}).

\textit{Improvement of $\gls{ELBO}_{\gls{IWAE}}$}:
On the other hand, the \kl divergence $\DKL[\mathcal{U}(s)\| \indexposterior(\idx|\vz^{(1:K)}, \vx) ]$ is not limited by $\log K$.   However, we do know that the improvement of $\gls{ELBO}_{\gls{IWAE}}(\vx; q_\theta, K)$ over $\gls{ELBO}(\vx; q_\theta)$ will be limited by $\DKL[\qzx\|\pzx]$, the gap of $\gls{ELBO}(\vx; q_\theta)$.

\end{proof}

\myprop{iwae_elbo_eubo} shows that both \gls{IWAE} upper and lower bounds on $\log p(\vx)$ improve upon their single-sample counterparts, with the improvement of $\gls{EUBO}_{\gls{IWAE}}(\vx; q_{\theta}, K)$ over $\gls{EUBO}(\vx; q_\theta)$ limited by $\log K$.
\mycor{iwae_mi_short} follows directly by translating these results to bounds on \gls{MI}.

\iwaemishort*
\begin{proof}
Recall from \mysec{setting} and \mysec{ba} that
$I_{\gls{BA}_L}(q_{\theta}) = \mathbb{E}_{p(\vx, \vz)}[\log p(\vx|\vz)] - \mathbb{E}_{p(\vx)}[\gls{EUBO}(\vx;q_{\theta})]$, %
and $I_{\gls{IWAE}_L}(q_{\theta}, K) = \mathbb{E}_{p(\vx, \vz)}[\log p(\vx|\vz)] - \mathbb{E}_{p(\vx)}[\gls{EUBO}_{\gls{IWAE}}(\vx; q_{\theta}, K)]$. %
Using \myprop{iwae_elbo_eubo}, the fact that $\gls{EUBO}(\vx; q_{\theta}) - \log K \leq \gls{EUBO}_{\gls{IWAE}}(\vx; q_{\theta}, K)$ for any $\vx$ implies that 
\begin{align}
    I_{\gls{IWAE}_L}(q_{\theta}, K) - I_{\gls{BA}_L}(q_{\theta}) =  \mathbb{E}_{p(\vx)}[\gls{EUBO}(\vx; q_{\theta}) -\gls{EUBO}_{\gls{IWAE}}(\vx; q_{\theta}, K)] \leq \log K ,
\end{align}
which results in $I_{\gls{IWAE}_L}(q_{\theta}, K) \leq  I_{\gls{BA}_L}(q_{\theta}) + \log K$, as desired.  

$I_{\gls{BA}_L}(q_{\theta}) \leq I_{\gls{IWAE}_L}(q_{\theta}, K)$ and $ I_{\gls{IWAE}_U}(q_{\theta}, K) \leq I_{\gls{BA}_U}(q_{\theta})$ follow from the fact that \gls{IWAE} bounds tighten with increasing $K$ in \myprop{iwaek}.
\end{proof}

\subsection{Experimental Results showing Logarithmic Improvement for $I_{\gls{IWAE}_L}(q_{\theta}, K)$} \label{app:iwae_log_improvement}  %
In \mysec{gen_iwae}, we showed that \gls{IWAE} is a special case of \gls{GIWAE}, which decomposes into the sum of a variational \gls{BA} lower bound and a $K$-sample contrastive term.
This suggests that the \gls{IWAE} lower bound on \gls{MI}, which arises from an upper bound on $\log p(\vx)$, may be written as
\small
\begin{align}%
I_{\textsc{IWAE}_{L}}(q_{\theta},K)=
\underbrace{\Exp{p(\vx,\vz)}{\log \frac{\qzx}{p(\vz)}}
\vphantom{\Exp{p(\vx) p(\vz^{(1)}|\vx)\qzxi{2:K}}{\log
\frac{\frac{p(\vx,\vz^{(1)})}{q_{\theta}(\vz^{(1)}|\vx)}}{\frac{1}{K} \sum_{i=1}^K \frac{p(\vx,\vz^{(k)})}{q_{\theta}(\vz^{(k)}|\vx)}}}}
}_{I_\textsc{BA}(q)} +
\underbrace{\Exp{p(\vx) p(\vz^{(1)}|\vx)\qzxi{2:K}}{\log
\frac{\frac{p(\vx,\vz^{(1)})}{q_{\theta}(\vz^{(1)}|\vx)}}{\frac{1}{K} \sum_{i=1}^K \frac{p(\vx,\vz^{(k)})}{q_{\theta}(\vz^{(k)}|\vx)}}}}_{\hspace{.5cm} 0\leq \, \text{contrastive term} \, \leq \log K} .
 \label{eq:iwae_two_terms}
\end{align}
\normalsize
This way of writing $I_{\textsc{IWAE}_L}(q_\theta, K)$ provides additional intuition for the result in \myprop{iwae_elbo_eubo} and \mycor{iwae_mi_short}.   In particular, the improvement of $I_{\textsc{IWAE}_L}(q_\theta, K)$ over $I_{\gls{BA}_L}(q_\theta)$ is simply the contrastive term, which is limited to $\log K$.

\begin{table}[h!]
\resizebox{8cm}{!}{
\begin{tabularx}{\textwidth}{c|c|c|ccc|ccc}%
\cmidrule{1-9}
\textbf{\small Method}&\boldmath$\log K$&\small \textbf{Proposal} &\textbf{VAE2}&\textbf{VAE10}&\textbf{VAE100}&\textbf{GAN2}&\textbf{GAN10}&\textbf{GAN100}\\%
\cmidrule{1-9}%
\multirow{2}{*}{\shortstack[c]{\textbf{IWAE} \\ (K=1)}}&\multirow{2}{*}{$0$}&\textbf{$ p(\vz)$}&$0+0=0$&$0+0=0$&$0+0=0$&$0+0=0$&$0+0=0$&$0+0=0$\\%

\cmidrule{3-9}%
&&\textbf{$q(\vz|\vx)$}& $8.63+0=8.63$&$25.20+0=25.20$&$44.54+0=44.54$&$8.83+0=8.83$&$4.23+0=4.23$&$3.23+0=3.23$\\%
\cmidrule{1-9}%

\multirow{2}{*}{\shortstack[c]{\textbf{IWAE} \\ (K=1K)}}&\multirow{2}{*}{$6.91$} &\textbf{$ p(\vz)$}&$0+6.81=6.81$&$0+6.91=6.91$&$0+6.91=6.91$&$0+6.88=6.88$&$0+6.91=6.91$&$0+6.91=6.91$\\%
\cmidrule{3-9}%
&&\textbf{$q(\vz|\vx)$}&$7.29+1.8=9.09$&$25.20+6.49=31.69$&$44.54+6.90=51.44$&$7.82+2.92=10.74$&$4.23+6.91=11.14$&$3.23+6.91=10.14$\\%
\cmidrule{1-9}%

\multirow{2}{*}{\shortstack[c]{\textbf{IWAE} \\ (K=1M)}}&\multirow{2}{*}{$13.82$}&\textbf{$ p(\vz)$}&  $0+9.09=9.09$&$0+13.82=13.82$&$0+13.82=13.82$&$0+10.76=10.76$&$0+13.81=13.81$&$0+13.82=13.82$\\%

\cmidrule{3-9}%
&&\textbf{$q(\vz|\vx)$}&$3.78+5.31=9.09$&$25.20+8.90=34.10$&$44.54+13.81=58.35$&$6.02+4.79=10.81$&$4.23+13.53=17.76$&$3.17+13.81=16.98$\\%
\cmidrule{1-9}%

\end{tabularx} 
}
\caption{
Decomposition of $I_{\gls{IWAE}_L}(q_\theta, K)$ into \gls{BA} term and contrastive term ($<\log K$) on \textsc{mi} estimation for \textsc{vae} and \textsc{gan} models trained on \textsc{mnist}.}
\label{table:de-mnist}
\end{table}

\begin{table}[h]
\centering
\resizebox{9.5cm}{!}{
\begin{tabular}{c|c|c|ccc}%
\toprule
\textbf{\small Method}&\boldmath$\log K$&\textbf{\small \small Proposal}&\textbf{GAN5}&\textbf{GAN10}&\textbf{GAN100} \\%
\cmidrule{1-6}
\multirow[c]{2}{*}{\shortstack[c]{\textbf{IWAE} \\K=1}}&\multirow[c]{2}{*}{$0$}&{\textbf{$ p(\vz)$}}&$0+0=0$&$0+0=0$&$0+0=0$\\%
\cmidrule{3-6}%
&&{\textbf{$q(\vz|\vx)$}}&$14.53+0=14.53$&$17.45+0=17.45$&$20.00+0=20.00$\\%

\cmidrule{1-6}%
\multirow[c]{2}{*}{\shortstack[c]{\textbf{IWAE} \\K=1k}}&\multirow[c]{2}{*}{$6.91$}&{\textbf{$ p(\vz)$}}&$0+6.91=6.91$&$0+6.91=6.91$&$0+6.91=6.91$\\%
\cmidrule{3-6}%
&&{\textbf{$q(\vz|\vx)$}}&$14.53+6.90=21.43$&$16.68+6.9=23.58$&$20.07 +6.91=26.98$\\%

\cmidrule{1-6}%
\multirow[c]{2}{*}{\shortstack[c]{\textbf{IWAE} \\K=1M}}&\multirow[c]{2}{*}{$13.82$}&{\textbf{$ p(\vz)$}}&$0+13.82=13.82$&$0+13.82=13.82$&$0+13.82=13.82$\\%
\cmidrule{3-6}%
&&{\textbf{$q(\vz|\vx)$}}&$14.53+13.81=28.34$&$16.92+13.81=30.73$&$20.00+13.81=33.81$\\%
   \cmidrule{1-6}
\end{tabular}
}
\caption{
Decomposition of $I_{\gls{IWAE}_L}(q_\theta, K)$ into \gls{BA} term and contrastive term ($<\log K$) on \textsc{mi} estimation for \textsc{gan} models trained on \textsc{cifar}-10.}
\label{table:de-cifar}
\end{table}

\mytable{de-mnist} and \mytable{de-cifar} show the \gls{IWAE} objective decomposition to the \gls{BA} term and the contrastive term, on \textsc{vae}s and \textsc{gan}s trained \textsc{mnist} and \textsc{cifar}-10 dataset.  We can see that in all the experiments, the contribution of the contrastive term is always less than or equal to $\log K$. 

For \textsc{mnist} \textsc{vae}s and \textsc{gan}s with two dimensional latent spaces, the contrastive term may contribute notably less than $\log K$.  In these cases, even the optimal critic function $T^{*}(\vx,\vz)=\log \frac{p(\vx,\vz)}{\qzx} + c(\vx)$, as used in \gls{IWAE}, has difficulty distinguishing posterior and variational samples.   

 However, the contribution of the contrastive term is almost exactly $\log K$ for higher dimensional \textsc{vae} and \textsc{gan} models, where the posterior $\pzx$ is more complex and is more easily distinguishable from the variational $\qzx$. These results highlight the inherent exponential sample complexity of the \gls{IWAE} lower bound on \gls{MI}.

\subsection{Bias Reduction in $K$ for IWAE Lower Bound on $\log \px$ / Upper Bound on MI}\label{app:iwae_importance}
We have seen in \myapp{pf_iwae} \myprop{iwae_elbo_eubo}  
that the improvement of $\gls{EUBO}_{\gls{IWAE}}(\vx; q_\theta, K)$ and $I_{\gls{IWAE}_L}(\vx; q_\theta, K)$ over the single-sample $\gls{EUBO}(\vx; q_\theta)$ and $I_{\gls{BA}_L}(\vx; q_\theta)$ is limited by $\log K$.   This suggests that exponential sample complexity in the gap of the \gls{EUBO}, $K \propto \exp\{\DKL[p(\vz|\vx)\|\qzx] \}$, is required to obtain a tight lower bound.   

The quantity $\DKL[\mathcal{U}(s) \| \ptgt{iwae}(s|\vz^{(1:K)},\vx) ] $, which measures  the improvement of the \gls{IWAE} lower bound on $\log p(\vx)$ over the \gls{ELBO}, is not explicitly limited by $\log K$.   
However, \citet{chatterjee2018sample} suggests that the same exponential sample complexity, $K \propto \exp\{\DKL[p(\vz|\vx)\|\qzx] \}$, is required for accurate importance sampling estimation with the proposal $\qzx$.   \citet{maddison2017filtering, domke2018importance} find that the bias of the \gls{IWAE} lower bound in the limit of $K\rightarrow \infty$ reduces at the rate of $\mathcal{O}(\frac{1}{2K} \text{Var}[\frac{p(\vz|\vx)}{\qzx}])$, although the $\text{Var}[\frac{p(\vx, \vz)}{\qzx}]$ term is at least exponential in  $\DKL[p(\vz|\vx)\|\qzx]$ \citep{song2019understanding}. 

This exponential sample complexity for exact estimation of $\log p(\vx)$ or $I(\vx,\vz)$ is usually impractical for complex target distributions and limited variational families, where $\DKL[p(\vz|\vx)\|\qzx ]$ may be large.
This motivates our improved, Multi-Sample \gls{AIS} proposals in \mysec{multi-sample-ais}, which achieve more favorable (linear) bias reduction by introducing \gls{MCMC} transition kernels to bridge between $\qzx$ and $p(\vz|\vx)$.

\subsection{Relationship with Structured InfoNCE}\label{app:structured}
We can recognize the Structured \textsc{InfoNCE} upper and lower bounds for known $\pxgz$ (\citet{poole2019variational} Sec. 2.5) as simply applying the standard \gls{IWAE} bounds, using the marginal $p(\vz)$ in place of the variational $q_{\theta}(\vz|\vx)$
\footnotesize
\begin{align*}
\mathbb{E}_{p(\vx,\vz^{(1)}) \prod \limits_{k=2}^K p(\vz^{(k)})} \bigg[ \log \frac{p(\vx|\vz^{(1)})}{\frac{1}{K} \sum \limits_{k=1}^{K} p(\vx|\vz^{(k)})} \bigg] \leq \Ixz \leq  \mathbb{E}_{\px \prod \limits_{k=1}^K p(\vz^{(k)})}\bigg[\log \frac{1} {\frac{1}{K} \sum \limits_{k=1}^{K} p(\vx|\vz^{(k)})} \bigg] - \Hxz .
\end{align*}
\normalsize
We refer to the lower bound as $I_{\textsc{S-InfoNCE}_L}(K)$ and the upper bound as $I_{\textsc{S-InfoNCE}_U}(K)$.
From \mycor{iwae_mi_short}, we obtain an alternative proof that the Structured \textsc{InfoNCE} lower bound is upper bounded by $\log K$.   Since $I_{\textsc{ba}_L}(p(\vz)) \leq I_{\gls{IWAE}_L}(p(\vz), K) \leq I_{\textsc{ba}_L}(p(\vz)) + \log K$ and the \gls{BA} bound with a prior proposal equals $0$ from \myeq{simple_importance_sampling2}, we have that $0 \leq I_{\textsc{S-InfoNCE}_L}(K) \leq \log K$.

\section{Generalized IWAE}\label{app:giwae}
\subsection{Probabilistic Interpretation and Bounds}\label{app:giwae_prob}
To derive a probabilistic interpretation for \gls{GIWAE}, 
our starting point is to further extend the state space of the \gls{IWAE} target distribution in \myeq{iwae_fwdrev2}, using a uniform index variable $p(s) = \frac1K \, \forall s$ that specifies which sample $\vz^{(k)}$ is drawn from the posterior $\pzx$.  This is shown in \myfig{multi-sample-ais}, and leads to a joint distribution over $(\vx, \vz^{(1:K)}, s)$ as
\begin{align}
\ptgt{giwae}(\vx, \vz^{(1:K)}, \idx) = \frac1K \, p(\vx,\vz^{(\idx)}) \prod \limits_{\myoverset{k=1}{k \neq \idx}}^K q(\vz^{(k)}|\vx) \, , \label{eq:giwae_reverse_app}
\end{align}
with marginalization over $s$ leading to the mixture in \myeq{iwae_fwdrev2}.
The posterior over the index variable $s$, which infers the `positive' sample drawn from  $\pzx$ given a set of samples $\vz^{(1:K)}$, corresponds to the normalized importance weights
\begin{align}
    \ptgt{giwae}( \idx |\vx, \vz^{(1:K)}) =  \frac{\frac{p(\vx,\vz^{(\idx)})}{q_{\theta}(\vz^{(\idx)}|\vx)}}{\sum_{k=1}^K \frac{p(\vx,\vz^{(k)})}{q_{\theta}(\vz^{(k)}|\vx)}} , \label{eq:giwae_posterior}
\end{align}
which 
can be derived from
$\ptgt{giwae}( \idx |\vx, \vz^{(1:K)}) = \frac{\ptgt{giwae}( \vx, \vz^{(1:K)}, \idx)}{\ptgt{giwae}( \vx, \vz^{(1:K)})} = \frac{\frac{1}{K} p(\vx,\vz^{(\idx)}) \prod_{k\neq \idx} q_\phi(\vz^{(k)}|\vx)}{\frac{1}{K}\sum_{j=1}^K p(\vx,\vz^{(j)}) \prod_{k\neq j} q_\phi(\vz^{(k)}|\vx) }$. 

For the \gls{GIWAE} extended state space proposal distribution, we consider a categorical index variable $\qprop{giwae}(\idx | \vz^{(1:K)}, \vx)$ drawn according \gls{SNIS}, with weights calculated using the critic function $\giwaeT$.
\small
\begin{align}
    \hspace*{-.1cm} \qprop{giwae}(\vz^{(1:K)}, \idx |\vx) &= \left( \prod \limits_{k=1}^K q_{\theta}(\vz^{(k)}|\vx) \right) \qprop{giwae}(\idx |\vz^{(1:K)}, \vx ), \,\,  \label{eq:giwae_forward_app} \\ \text{where} \quad \qprop{giwae}(\idx |\vz^{(1:K)}, \vx ) &= \dfrac{e^{\giwaeT(\vx,\vz^{(\idx)})}}{ \sum \limits_{k=1}^K e^{\giwaeT(\vx,\vz^{(k)})}}. \label{eq:giwae_snis}
\end{align}
\normalsize
Similarly to \citet{lawson2019energy}, the variational \gls{SNIS} distribution $\qprop{giwae}(\idx |\vz^{(1:K)}, \vx )$ is approximating
the true \gls{SNIS} distribution $\ptgt{giwae}( \idx |\vx, \vz^{(1:K)})$.
As we show in \myapp{pf_iwaegiwae},  the optimal critic function is $\giwaeT(\vx,\vz) = \log \frac{p(\vx,\vz)}{q(\vz|\vx)} + c(\vx)$, which recovers the \gls{IWAE} probabilistic interpretation (see \myapp{iwae_prob}). %

\paragraph{Log Importance Ratio}
To derive bounds on the log partition function, we first calculate the log unnormalized density ratio 
\begin{align}
    \log \frac{\ptgt{giwae}(\vz^{(1:K)}, \idx, \vx)}{\qprop{giwae}(\vz^{(1:K)}, \idx |\vx) }& = \log \frac{\frac{1}{K} p(\vx, \vz^{(\idx)})  \prod \limits_{\myoverset{k=1}{k \neq \idx}}^{K} q(\vz^{(k)}|\vx) }{ q(\idx |\vz^{(1:K)}, \vx ) \prod \limits_{k=1}^K q(\vz^{(k)}|\vx) }  \\
    &= \log \frac{1}{K} \dfrac{\sum_k e^{T(\vx,\vz^{(k)})}}{e^{T(\vx,\vz^{(\idx)})}} \frac{p(\vx, \vz^{(\idx)})}{q(\vz^{(\idx)}|\vx)} \\
    &= \log \frac{p(\vx, \vz^{(\idx)})}{q(\vz^{(\idx)}|\vx)}  -  T(\vx, \vz^{(\idx)})  + \log \frac{1}{K} \sum \limits_{k=1}^K e^{T(\vx,\vz^{(k)})}\,.
\end{align}

Taking the expectation of the log unnormalized density ratio under the proposal or target distribution yields a lower or upper bound, respectively, on $\log \px$
\begin{align}
\myunderbrace{\Exp{\qprop{giwae}}{\log \frac{\ptgt{giwae}(\vz^{(1:K)},  \idx,\vx)}{\qprop{giwae}(\vz^{(1:K)}, \idx |\vx)}}}{\gls{ELBO}_{\gls{GIWAE}}(\vx;q_\theta, \giwaeT, K) } \leq \log \px \leq \myunderbrace{\Exp{\ptgt{giwae}}{\log \frac{\ptgt{giwae}(\vz^{(1:K)}, \idx,\vx)}{\qprop{giwae}(\vz^{(1:K)}, \idx |\vx)}}}{\gls{EUBO}_{\gls{GIWAE}}(\vx;q_\theta, \giwaeT, K)}. \label{eq:giwae_px}
\end{align}
As in \mysec{general} and \myapp{general}, the gap in the lower and upper bounds 
can be derived as \textsc{kl} divergences in the extended state space.
\paragraph{Upper Bound on $\log \px$ and Lower Bound on MI}
To derive an explicit form for the \gls{GIWAE} upper bound on $\log \px$, we write
\begin{align}
    \gls{EUBO}_{\gls{GIWAE}}(&q_\theta, \giwaeT, K) = \Exp{\ptgt{giwae}(\vz^{(1:K)}, s|\vx)}{\log \frac{p(\vx, \vz^{(\idx)})}{q(\vz^{(\idx)}|\vx)}  -  T(\vx, \vz^{(\idx)})  + \log \frac{1}{K} \sum \limits_{k=1}^K e^{T(\vx,\vz^{(k)})}} \nonumber \\
    &=
    \Exp{p(\vz|\vx)}{\log \frac{p(\vx, \vz)}{q(\vz|\vx)}} - 
    \Exp{p(\vz^{(1)}|\vx) \prod \limits_{k =2} q(\vz^{(k)}|\vx)}{ \log \frac{e^{T(\vx, \vz^{(1)})} }{\frac{1}{K} \sum \limits_{k=1}^K e^{T(\vx,\vz^{(k)})}}} . \label{eq:giwae_ub_logpx}
\end{align}
where, in the first term of the second line, we note that $\ptgt{giwae}(\vz^{(1:K)}, s|\vx)$ specifies that $\vz^{(\idx)}\sim p(\vz|\vx)$.   Since $s \sim p(s) = \frac{1}{K}$ is sampled uniformly, we can assume $s=1$ in the second term due to permutation invariance.

Translating this into a lower bound on \gls{MI}, we consider $\Ixz = -\Exp{p(\vx)}{\log p(\vx)} - \Hxz \geq - \mathbb{E}_{p(\vx)}[\gls{EUBO}_{\gls{GIWAE}}(q_\theta, \giwaeT, K)] -\Hxz$.   Writing the conditional entropy term over the index $s$,
\begin{align}
\Ixz & \geq \mathbb{E}_{p(\vx,\vz)}[\log p(\vx|\vz)] \\
&\phantom{\mathbb{E}_{p}}- \bigg( \mathbb{E}_{p(\vx)p(\vz|\vx)}\left[\log \frac{p(\vx, \vz)}{q(\vz|\vx)}\right]  -  \mathbb{E}_{p(\vx) p(\vz^{(1)}|\vx) \prod \limits_{k=2}^K q(\vz^{(k)}|\vx)} \bigg[  \log \frac{e^{\giwaeT(\vx, \vz^{(1)})}}{ \frac{1}{K} \sum \limits_{k=1}^K e^{\giwaeT(\vx,\vz^{(k)})}}  \bigg] \bigg)  \nonumber \\[1.5ex]
&= \Exp{p(\vx,\vz)}{\log \frac{q(\vz|\vx)}{p(\vz)}} +
\Exp{p(\vx) p(\vz^{(1)}|\vx)\qzxi{2:K}}{ \log 
\frac{e^{\giwaeT(\vx,\vz^{(1)})}}{\frac{1}{K} \sum_{i=1}^K e^{\giwaeT(\vx,\vz^{(k)})}}}
 \,, \label{eq:giwae_info_lb}
\end{align}
which matches \myeq{giwae_lb_mi_main} from the main text.

\subsection{GIWAE Upper Bound on MI Does Not Provide Benefit Over IWAE}\label{app:giwae_ub}
To derive an explicit form for the \gls{GIWAE} lower bound on $\log \px$,
\footnotesize
\begin{align}
&\gls{ELBO}_{\gls{GIWAE}}(\vx; q_\theta, \giwaeT, K) = \Exp{\qprop{giwae}(\vz^{(1:K)}, s|\vx)}
    {\log \frac{p(\vx, \vz^{(\idx)})}{q(\vz^{(\idx)}|\vx)}  -  \giwaeT(\vx, \vz^{(\idx)})  + \log \frac{1}{K} \sum \limits_{k=1}^K e^{\giwaeT(\vx,\vz^{(k)})}} \label{eq:giwae_lb_logpx}
\end{align}
\normalsize
Translating this to an upper bound on \gls{MI} yields 
\footnotesize
\begin{align}
\Ixz & \leq \mathbb{E}_{p(\vz^{(s)}|\vx)}\bigg[ \log p(\vx|\vz^{(s)}) \bigg] - \bigg( \mathbb{E}_{\prod \limits_{k=1}^K q(\vz^{(k)}|\vx)} \bigg[ \sum \limits_{s=1}^K  \frac{e^{\giwaeT(\vx,\vz^{(s)})} }{ \sum \limits_{k=1}^K e^{\giwaeT(\vx,\vz^{(k)})} } \log \frac{p(\vx, \vz^{(s)})}{q(\vz^{(s)}|\vx)} \bigg] \label{eq:giwae_info_ub} \\
&\phantom{\geq \mathbb{E}_{p(\vz^{(s)}|\vx)} [\log p(\vx|\vz^{(s)})] -  \mathbb{E}} -  \mathbb{E}_
{\qprop{giwae}(\vz^{(1:K)}, s|\vx)}
\bigg[  \log \frac{e^{T(\vx, \vz^{(s)})}}{ \frac{1}{K} \sum \limits_{k=1}^K e^{T(\vx,\vz^{(k)})}}  \bigg] \bigg)  
\nonumber
\end{align}
\normalsize
Thus, in \cref{eq:giwae_info_ub}, knowledge of the full joint density is required to evaluate both the conditional entropy and $\log \frac{p(\vx,\vz^{(s)})}{q(\vz^{(s)}|\vx)}$ terms. 
 If both $p(\vz)$ and $p(\vx|\vz)$ are known, then we will show in \mycor{giwae} below that the optimal critic or negative energy function in \gls{GIWAE} yields the true importance weights, and the resulting \gls{MI} or $\log p(\vx)$ bounds matches the \gls{IWAE} bounds. 

We thus conclude that the \gls{GIWAE} upper bound on \gls{MI} and 
lower bound on $\log p(\vx)$
does not provide any benefit over \gls{IWAE} in practice.   However, 
$\gls{ELBO}_{\gls{GIWAE}}(\vx; q_\theta, \giwaeT, K)$ is still useful for analysis, as our proof of \mylemma{elbo_eubo_giwae} below allows us characterize the gap between $\gls{ELBO}_{\gls{IWAE}}(\vx; q_\theta, K)$ and $\gls{ELBO}(\vx; q_\theta)$ in \myprop{iwae_elbo_eubo}.

\subsection{ELBO and EUBO are Special Cases of GIWAE Log Partition Function Bounds}
\label{app:elbo_eubo_giwae}
\begin{restatable}{lemma}{elbo_eubo_giwae}\label{lemma:elbo_eubo_giwae}
The single-sample \gls{ELBO} and \gls{EUBO} are special cases of \gls{GIWAE}, with
\begin{align}
    \gls{ELBO}(\vx;q_\theta)&= \gls{ELBO}_{\gls{GIWAE}}(\vx;q_\theta, T_{\phi_0} = \text{const}, K) \nonumber , \\
    \qquad
    \gls{EUBO}(\vx;q_\theta)&= \gls{EUBO}_{\gls{GIWAE}}(\vx;q_\theta, T_{\phi_0} = \text{const}, K) \nonumber .
\end{align}
\normalsize
In both cases, the \gls{SNIS} sampling distribution (\cref{eq:giwae_snis}) is uniform $\qprop{giwae}(\vz^{(1:K)}, \idx | \vx) = \frac{1}{K} = \mathcal{U}(s)$.
\end{restatable}
\begin{proof}

We consider the \gls{GIWAE} probabilistic interpretation (\myapp{giwae_prob}) for $T_{\phi_0} = \text{const}$.    We refer to this extended state space proposal as $\qprop{ba}$, since it leads to $I_{\gls{BA}_L}(q_\theta)$ and $I_{\gls{BA}_U}(q_\theta)$ bounds on \gls{MI}.
\begin{align}
      \qprop{ba}(\vz^{(1:K)}, \idx | \vx) &= \bigg( \prod \limits_{k=1}^K q_\theta(\vz^{(k)}|\vx) \bigg) \qprop{ba}(s|\vz^{(1:K)}, \vx) \,, \label{eq:qba_proposal} \\
      \text{where} \,\, \qprop{ba}(s|\vz^{(1:K)}, \vx) &= \frac{e^{T_{\phi_0}(\vx,\vz^{(s)})}}{\sum_{k=1}^K e^{T_{\phi_0}(\vx,\vz^{(k)})}} = \mathcal{U}(s)=  \frac{1}{K} \,. \label{eq:qba}
\end{align}
Note that the \gls{SNIS} sampling distribution $\qprop{ba}(s|\vz^{(1:K)}, \vx)$ will be uniform, which matches $p(s) = \frac{1}{K}$ in the \gls{GIWAE} extended state space target distribution  
 \begin{align}
     \ptgt{giwae} (\vz^{(1:K)},\idx|\vx) = \frac{1}{K} p(\vz^{(\idx)}|\vx) \prod \limits_{\myoverset{k=1}{k\neq \idx}}^K q_\theta(\vz^{(k)}|\vx)
 \end{align}
Now, taking the log unnormalized importance weights, we obtain
\begin{align}
   \log \frac{\ptgt{giwae} (\vz^{(1:K)},\idx|\vx)}{\qprop{ba}(\vz^{(1:K)}, \idx | \vx)} = \log \frac{\frac{1}{K} p(\vz^{(\idx)}|\vx) \prod \limits_{\myoverset{k=1}{k\neq \idx}}^K q_\theta(\vz^{(k)}|\vx)}{\frac{1}{K}\prod \limits_{k=1}^K q_\theta(\vz^{(k)}|\vx) } 
   = \log \frac{p(\vx, \vz^{(\idx)})}{q_\theta(\vz^{(\idx)}|\vx)} 
\end{align}
Taking expectations with respect to $\qprop{ba}(\vz^{(1:K)}, \idx | \vx)$ or $\ptgt{giwae}(\vz^{(1:K)},\idx|\vx)$ leads to $\gls{ELBO}(\vx;q_\theta)$ and $\gls{EUBO}(\vx;q_\theta)$ respectively, as in \mysec{background}-\ref{sec:ba}.
\end{proof}

\subsection{Proof of Relationship between IWAE and GIWAE Probabilistic Interpretations (\myprop{iwae_giwae}) }\label{app:pf_iwaegiwae} %

We first prove a lemma which relates both the \gls{GIWAE} lower and upper bounds on $\log p(\vx)$ to the respective \gls{IWAE} bounds.    From this lemma, \myprop{iwae_elbo_eubo} follows directly and relates the \gls{ELBO} and \gls{EUBO} (which are special cases of $\gls{ELBO}_{\gls{GIWAE}}$ and $\gls{EUBO}_{\gls{GIWAE}}$) to $\gls{ELBO}_{\gls{IWAE}}$ and $\gls{EUBO}_{\gls{IWAE}}$.

\begin{restatable}{lemma}{giwae_elbo_eubo}\label{lemma:giwae_elbo_eubo}
We can characterize the difference between \gls{IWAE} and \gls{GIWAE} bounds on $\log p(\vx)$ using \kl divergences between their respective \gls{SNIS} distributions.
\scriptsize
\begin{align}%
\hspace*{-.15cm} \gls{ELBO}_{\textsc{IWAE}}(\vx; q_\theta,K)
&= \gls{ELBO}_{\textsc{GIWAE}}(\vx; q_\theta, \giwaeT, K)+  
{\mathbb{E}_{\qprop{giwae}(\vz^{(1:K)}|\vx)} \bigg[\DKL[\, \qprop{giwae}(\idx |\vz^{(1:K)}, \vx  )\| \ptgt{iwae}(\idx |\vz^{(1:K)}, \vx  ) ] \vphantom{\frac{1}{2} }  \,  \bigg]}
\label{eq:iwae_over_giwaelb} \\
\gls{EUBO}_{\textsc{IWAE}}(\vx; q_\theta,K) 
&= \gls{EUBO}_{\textsc{GIWAE}}(\vx; q_\theta, \giwaeT, K) - 
 {\mathbb{E}_{\ptgt{iwae}(\vz^{(1:K)}|\vx)}\bigg[\DKL[\ptgt{iwae}(\idx |\vz^{(1:K)}, \vx  ) \, \| \, \qprop{giwae}(\idx |\vz^{(1:K)}, \vx  ) ] \vphantom{\frac{1}{2} }  \,  \bigg]} \label{eq:iwae_over_giwaeub}
\end{align}
\end{restatable}
\begin{proof}
Recall from \mysec{general} or \myapp{general}
that the gap of the lower bound $\gls{ELBO}_{\gls{GIWAE}}(\vx; q_\theta, \giwaeT, K)$ is 
${\DKL[\qprop{giwae}(\vz^{(1:K)}, \idx |\vx) \| \ptgt{giwae}(\vz^{(1:K)}, \idx,\vx)]}$,
while the gap of the upper bound 
$\gls{EUBO}_{\gls{GIWAE}}(\vx; q_\theta, \giwaeT, K)$ 
is 
${\DKL[\ptgt{giwae}(\vz^{(1:K)}, \idx|\vx)\|\qprop{giwae}(\vz^{(1:K)}, \idx |\vx)]}.$
We will expand these \kl divergences to reveal the relationship between the \gls{GIWAE} bounds and  \gls{IWAE} bounds. 

First, recall from \cref{eq:giwae_posterior} that the posterior over the index variable $\idx$, or target \gls{SNIS} distribution, is 
\small
\begin{align}
    \ptgt{giwae} (\idx|\vz^{(1:K)}, \vx) &= \frac{\ptgt{giwae} (\vz^{(1:K)}, \idx, \vx)}{\sum \limits_{\idx=1}^K \ptgt{giwae}(\vz^{(1:K)}, s,\vx)} = \dfrac{\frac{1}{K} p(\vx,\vz^{(\idx)}) \prod_{\myoverset{k=1}{k \neq \idx}}^K q_{\theta}(\vz^{(k)}|\vx)  }{ \frac{1}{K} \sum \limits_{\idx=1}^K p(\vx,\vz^{(\idx)}) \prod_{\myoverset{k=1}{k \neq \idx}}^K q_{\theta}(\vz^{(k)}|\vx) } = \dfrac{\frac{p(\vx,\vz^{(\idx)})}{q(\vz^{(\idx)}|\vx)}}{\sum \limits_{\idx=1}^K \frac{p(\vx,\vz^{(\idx)})}{q(\vz^{(\idx)}|\vx)}} .\nonumber
\end{align}
\normalsize
The joint distribution then factorizes as $\ptgt{giwae}(\vz^{(1:K)}, \idx, \vx) =\ptgt{giwae}(\vz^{(1:K)}, \vx) \cdot \indexposterior(\idx|\vz^{(1:K)}, \vx) $.  

\textit{ELBO Case: \quad}
Using this factorization of $\ptgt{giwae}(\vz^{(1:K)}, \idx, \vx)$, we can rewrite the gap of   $\gls{ELBO}_{\gls{GIWAE}}(\vx; q_\theta, \giwaeT, K)$ as follows
\scriptsize
\begin{align}
    D&_{\text{KL}}[\qprop{giwae}(\vz^{(1:K)}, \idx| \vx) \| \ptgt{giwae}(\vz^{(1:K)}, \idx| \vx) ] \label{eq:elbo_pf1} \\[1.5ex]
    &= \mathbb{E}_{\qprop{giwae}(\vz^{(1:K)}, \idx| \vx)} \left[ \log \frac{\qprop{giwae}(\vz^{(1:K)}| \vx)}{ \ptgt{giwae}(\vz^{(1:K)}| \vx)}
    \frac{\qprop{giwae}(\idx|\vz^{(1:K)}, \vx)}{\ptgt{giwae}(\idx|\vz^{(1:K)}, \vx)} \right] \nonumber \\[1.5ex]
    &= \text{$\scriptsize \DKL\bigg[ \prod \limits_{k=1}^K q_{\theta}(\vz^{(k)}|\vx) \| \frac{1}{K} \sum \limits_{\idx=1}^K p(\vz^{(s)}|\vx) \prod \limits_{\myoverset{k=1}{k\neq \idx}}^K q_{\theta}(\vz^{(k)}|\vx) \bigg] + \Exp{\qprop{giwae}}{\DKL[ \qprop{giwae}(\idx|\vz^{(1:K)}, \vx) \| \ptgt{giwae}(\idx|\vz^{(1:K)}, \vx) ]}$} \nonumber \\[1.5ex]
    &= \DKL\bigg[ \qprop{iwae}(\vz^{(1:K)}|\vx) \| \ptgt{iwae}(\vz^{(1:K)}|\vx) \bigg] 
    + \Exp{\qprop{giwae}}{\DKL[ \qprop{giwae}(\idx|\vz^{(1:K)}, \vx) \| \indexposterior(\idx|\vz^{(1:K)}, \vx) ]} \label{eq:kl_decomp_lb} ,
\end{align}
\normalsize
where we can recognize the first term as the gap in $\gls{ELBO}_{\gls{IWAE}}(\vx; q_\theta,  K)$.
Noting that 
\small $ \gls{ELBO}_{\gls{IWAE}}(\vx; q_\theta,  K) - \gls{ELBO}_{\gls{GIWAE}}(\vx; q_\theta, \giwaeT, K) = 
\DKL[\qprop{giwae}(\vz^{(1:K)}, \idx| \vx) \| \ptgt{giwae}(\vz^{(1:K)}, \idx| \vx) ] - \DKL[\qprop{iwae}(\vz^{(1:K)}|\vx) \| \ptgt{iwae}(\vz^{(1:K)}|\vx)]$ \normalsize, we obtain \cref{eq:iwae_over_giwaelb}, as desired    

\begin{equation}
\resizebox{\textwidth}{!}{$
    \gls{ELBO}_{\gls{IWAE}}(\vx; q_\theta,  K) = \gls{ELBO}_{\gls{GIWAE}}(\vx; q_\theta, \giwaeT, K) + \mathbb{E}_{\qprop{giwae}(\vz^{(1:K)}|\vx)} \bigg[ \DKL[ \qprop{giwae}(\idx|\vz^{(1:K)}, \vx) \| \indexposterior(\idx|\vz^{(1:K)}, \vx) ] \bigg] \nonumber . 
    $}
\end{equation}

\normalsize

\textit{EUBO Case: \quad}
For $\gls{ELBO}_{\gls{GIWAE}}(\vx; q_\theta, \giwaeT, K)$, the derivations follow in a similar fashion to \cref{eq:elbo_pf1}-\cref{eq:kl_decomp_lb}, but using the reverse \kl divergence and expectations under $\ptgt{giwae}(\vz^{(1:K)}, \idx| \vx)$.
\end{proof}
Translating \mylemma{giwae_elbo_eubo} to a statement relating \gls{IWAE} and \gls{GIWAE} bounds on \gls{MI}, we obtain the following proposition.
\iwaegiwae*
\begin{proof}
The result follows directly from \mylemma{giwae_elbo_eubo}.  First, note that  $\ptgt{iwae}(\vx, \vz^{(1:K)}, \idx)=\ptgt{giwae}(\vx, \vz^{(1:K)}, \idx)$.   Then, using the upper bounds on $\log p(\vx)$ and taking outer expectations with respect to $p(\vx)$, we have 
\small
\begin{align}
    \hspace*{-.2cm} I_{\gls{IWAE}_L}(q_\theta, K) -  I_{\gls{GIWAE}_L}(q_\theta, \giwaeT&, K) = -\Hxz - \Exp{p(\vx)}{\gls{EUBO}_{\textsc{IWAE}}(\vx; q_\theta,K)} \\
     &\phantom{=====}+ \Hxz + \Exp{p(\vx)}{\gls{EUBO}_{\textsc{GIWAE}}(\vx; q_\theta, \giwaeT, K)} \nonumber \\[1.5ex]
     = & \Exp{p(\vx)\ptgt{iwae}(\vz^{(1:K)}|\vx  )}{\DKL[\ptgt{iwae}(\idx |\vz^{(1:K)}, \vx  ) \, \| \, \qprop{giwae}(\idx |\vz^{(1:K)}, \vx  ) ] } . 
\end{align}
\normalsize
\end{proof}

\subsection{Proof of GIWAE Optimal Critic Function and Logarithmic Improvement (\mycor{giwae} and \mycor{giwae_logk})}\label{app:pf_giwae}
We now prove \mycor{giwae} and \mycor{giwae_logk} from the main text, with a corollary stating the results for the special case of \textsc{InfoNCE} in \myapp{infonce}.
\giwae*
\begin{proof}
 Using \myprop{iwae_giwae}, we can see that the gap in the \gls{GIWAE} and \gls{IWAE} bounds, which corresponds to the posterior \textsc{kl} divergence over the index variable $s$, will equal zero iff
\begin{align}
    \qprop{giwae}(\idx |\vz^{(1:K)}, \vx ) = \ptgt{giwae}(\idx|\vz^{(1:K)}, \vx) \implies& \dfrac{e^{T(\vx,\vz^{(\idx)})}}{\sum \limits_{k=1}^K e^{T(\vx,\vz^{(k)})}} = \dfrac{\frac{p(\vx, \vz^{(\idx)})}{q_\theta(\vz^{(\idx)}|\vx)}}{\sum \limits_{\idx=1}^K \frac{p(\vx, \vz^{(\idx)})}{q_\theta(\vz^{(\idx)}|\vx)}} \label{eq:condition} 
\end{align}
This condition also ensures that the overall \gls{GIWAE} proposal $\ptgt{giwae}(\vz^{(1:K)},s|\vx)$  (\cref{eq:giwae_reverse_app}) and target $\qprop{giwae}(\vz^{(1:K)},s|\vx)$ (\cref{eq:giwae_forward_app}) distributions  match.
We will show that any $T(\vx,\vz)$ which satisfies \cref{eq:condition} has the form
\begin{align}
    T^*(\vx,\vz) = \log \frac{p(\vx, \vz)}{q_\theta(\vz|\vx)} + c(\vx) . \label{eq:result_tstar}
\end{align}
Let $f(\vx,\vz) = \log \frac{p(\vx, \vz)}{q_\theta(\vz|\vx)} + g(\vx,\vz)$, which represents an arbitrary choice of critic function.  We will show that $g(\vx,\vz)$ must be constant with respect to $\vz$
\begin{align*}
    \dfrac{e^{\log \frac{p(\vx, \vz^{(\idx)})}{q_\theta(\vz^{(\idx)}|\vx)} + g(\vx,\vz^{(\idx)})}}{\sum \limits_{k=1}^K e^{\log \frac{p(\vx, \vz^{(\idx)})}{q_\theta(\vz^{(\idx)}|\vx)} + g(\vx,\vz^{(\idx)})}} &= \dfrac{\frac{p(\vx, \vz^{(\idx)})}{q_\theta(\vz^{(\idx)}|\vx)}}{\sum \limits_{k=1}^K \frac{p(\vx, \vz^{(k)})}{q_\theta(\vz^{(k)}|\vx)}} \\
    \implies   e^{\log \frac{p(\vx, \vz^{(\idx)})}{q_\theta(\vz^{(\idx)}|\vx)}} \cdot e^{g(\vx,\vz^{(\idx)})}  \cdot \sum \limits_{k=1}^K \frac{p(\vx, \vz^{(k)})}{q_\theta(\vz^{(k)}|\vx)}  &= \frac{p(\vx, \vz^{(\idx)})}{q_\theta(\vz^{(\idx)}|\vx)} \cdot \sum \limits_{k=1}^K e^{\log \frac{p(\vx, \vz^{(k)})}{q_\theta(\vz^{(k)}|\vx)} + g(\vx,\vz^{(k)})} \\
    \sum \limits_{k=1}^K \frac{p(\vx, \vz^{(k)})}{q_\theta(\vz^{(k)}|\vx)} \frac{p(\vx, \vz^{(\idx)})}{q_\theta(\vz^{(\idx)}|\vx)} e^{g(\vx,\vz^{(s)})}  &= \sum \limits_{k=1}^K \frac{p(\vx, \vz^{(k)})}{q_\theta(\vz^{(k)}|\vx)} \frac{p(\vx, \vz^{(\idx)})}{q_\theta(\vz^{(\idx)}|\vx)} e^{g(\vx,\vz^{(k)})} \\
    \implies g(\vx,\vz) &= c(\vx)
\end{align*}
where $g(\vx,\vz) = c(\vx)$ is required in order to ensure that $g(\vx,\vz^{(\idx)}) = g(\vx,\vz^{(k)})$ for arbitrary choices of $\vz$ samples.

This form for the optimal critic function  $T^*(\vx,\vz)$ in \cref{eq:result_tstar} 
 implies that learning $\giwaeT(\vx,\vz)$ in \gls{GIWAE} becomes unnecessary when the density ratio $\frac{p(\vx,\vz)}{\qzx}$ is available in closed form, as is assumed in the \gls{IWAE} bound.

For this choice of $T^{*}(\vx,\vz)$, 
the value of the \gls{GIWAE} objective matches the \gls{IWAE} lower bound on \gls{MI}.
\small
\begin{align}
I_{\textsc{giwae}}(q_\theta, &T^{*},K) = \Exp{p(\vx,\vz)}{\log \frac{q(\vz|\vx)}{p(\vz)}} +
\Exp{ p(\vx,\vz^{(1)})\prod \limits_{\myoverset{k=2}{}}^{K} q_{\theta}(\vz^{(k)}|\vx)}{ \log 
\frac{e^{\log \frac{p(\vx,\vz^{(1)})}{q(\vz^{(1)}|\vx)} }\cdot {e^{c(\vx)}} }{\frac{1}{K} \sum \limits_{k=1}^K e^{\log \frac{p(\vx, \vz^{(k)})}{q_{\theta}(\vz^{(k)}|\vx)} }  \cdot {e^{c(\vx)}} }} \nonumber \\ 
&= \Exp{p(\vx,\vz)}{\log \frac{\cancel{q(\vz|\vx)}}{p(\vz)}  +
 \log \frac{p(\vx,\vz)}{\cancel{q(\vz|\vx)}} }  -  \Exp{p(\vx,\vz^{(1)})\prod \limits_{\myoverset{k=2}{}}^{K} q_{\theta}(\vz^{(k)}|\vx)}{ \log \frac{1}{K} \sum_{i=1}^K \frac{p(\vx,\vz^{(k)})}{q_{\theta}(\vz^{(k)}|\vx)} }  \label{eq:iwae_second_to_last}\\
&= -\Hxz + \left(- \Exp{p(\vx) p(\vz^{(s)}|\vx)\prod \limits_{\myoverset{k=2}{}}^{K} q_{\theta}(\vz^{(k)}|\vx)}{ \log \frac{1}{K} \sum_{i=1}^K \frac{p(\vx,\vz^{(k)})}{q_{\theta}(\vz^{(k)}|\vx)} } \right) \nonumber \\
&= I_{\textsc{IWAE}_L}(q_{\theta}, K) . \nonumber
\end{align}\normalsize
The second term in \cref{eq:iwae_second_to_last} is exactly the negative of the \gls{IWAE} upper bound on $\log p(\vx)$ in \cref{eq:iwae_ub_app}.
Combined with the entropy $-\Hxz$, we obtain the $I_{\textsc{IWAE}_L}(q_\theta, K)$ lower bound on \gls{MI} as desired.   
\end{proof}

\giwaelogk*
\begin{proof}
We begin by showing that $I_{\textsc{BA}_L}(q_\theta)  \leq I_{\textsc{GIWAE}_L}(q_{\theta},T_{\phi^*},K)$. This follows from the assumption that there exists $\phi_0$ in the parameter space of the neural network $T_{\phi}$ such that $T_{\phi_0}=\text{const}$. With this $\phi_0$, we would have $I_{\textsc{BA}_L}(q_\theta) = I_{\textsc{GIWAE}_L}(q_{\theta},T_{\phi_0},K)$ as in \mylemma{elbo_eubo_giwae}.   Thus, the optimal ${\phi^*}$ in the parameter space can only improve upon the \gls{BA} bound.  

Next, for any given ${\phi}$ (including ${\phi^*}$), we have $I_{\textsc{GIWAE}_L}(q_{\theta},T_{\phi},K) \leq I_{\textsc{IWAE}_L}(q_{\theta},K)$, since \gls{IWAE} uses the (unconstrained) optimal critic function $T^{*} = \log \frac{p(\vx,\vz)}{\qzx} + c(\vx)$ (\mycor{giwae}). 
The final inequality follows from \myprop{iwae_elbo_eubo}, which shows that $I_{\textsc{IWAE}_L}(q_{\theta},K)$ improves by at most $\log K$ over $I_{\textsc{BA}_L}(q_\theta)$.  
These relationships are visualized in \myfiga{different-bounds}{b}.
\end{proof}

\subsection{Properties of InfoNCE}\label{app:infonce}
\begin{corollary}
Using the prior $\qzx=p(\vz)$ as in \textsc{InfoNCE}, 
\begin{enumerate}[label=(\alph*)]
\item For the optimal critic function, \mycor{giwae} implies $T^{*}(\vx,\vz) = \log \pxgz + c(\vx)$, and
\begin{align}
 I_{\textsc{GIWAE}_L}( p(\vz), T^{*}, K) = I_{\textsc{InfoNCE}_L}(T^*,K) = I_{\textsc{S-InfoNCE}_L}(K) .
\end{align}
\item For an arbitrary critic function $T_{\phi}(\vx,\vz)$, \mycor{giwae_logk} translates to 
\begin{align}
    0 \leq I_{\textsc{InfoNCE}_L}(T_{\phi},K) \leq  I_{\textsc{S-InfoNCE}_L}(K)  \leq \log K .
\end{align}
\end{enumerate}
\label{cor:infonce}
\end{corollary}
For \textsc{InfoNCE}, note that using the prior as the proposal does allow the critic to admit an efficient bi-linear implementation $\giwaeT(\vx,\vz)= f_{\phi_\vx}(\vx)^T  f_{\phi_\vz}(\vz)$,
which requires only $N+K$ forward passes instead of $NK$ for \gls{GIWAE}, where $N$ is the batch size and $K$ is the total number of positive and negative samples.

\section{Single-Sample AIS}\label{app:single-sample-ais}
\subsection{Proof of \myprop{ais} (Complexity in $T$ for Single-Sample AIS)}\label{app:ais_pf}
In this section, we prove \myprop{ais}, which relates the sum of the gaps in the single-sample \gls{AIS} upper and lower bounds to the symmetrized \kl divergence between the endpoint distributions.   We extend this result to \textsc{im-ais} in \mycor{im_ais_complexity_t} and \textsc{cr-ais} in \mycor{cr_ais_complexity_t}.
\aisprop*
Note that $\gls{EUBO}_{\gls{AIS}}(\vx; \pi_0, T) - \gls{ELBO}_{\gls{AIS}}(\vx; \pi_0, T)$ on the left hand side also corresponds to the sum of the gaps $\DKL[ \qprop{ais}(\vz_{0:T}|\vx) \| \ptgt{ais}(\vz_{0:T}|\vx)] + \DKL[ \ptgt{ais}(\vz_{0:T}|\vx)\|\qprop{ais}(\vz_{0:T}|\vx) ]$.   This result translates to mutual information bounds as in \mysec{setting}.

In contrast to \citet{grosse2013annealing} Thm. 1, our 
linear bias reduction result holds for finite $T$.  
\begin{proof}
For linear scheduling and perfect transitions, we simplify the difference in the single-sample upper and lower bounds as
\small
\begin{align}
 \delta_L^{T,K=1} &+ \delta_U^{T,K=1} =  \gls{EUBO}_{\gls{AIS}}(\vx; \pi_0, T) - \gls{ELBO}_{\gls{AIS}}(\vx;\pi_0, T) \nonumber \\
 &= \Exp{\vz_{0:T} \sim \ptgt{ais}}{  \log  \frac{\ptgt{ais}(\vx,\vz_{0:T})}{\qprop{ais}(\vz_{0:T}|\vx)} } -  \Exp{\vz_{0:T} \sim \qprop{ais}}{ \log  \frac{\ptgt{ais}(\vx,\vz_{0:T})}{\qprop{ais}(\vz_{0:T}|\vx)} } \nonumber \\
 &= \Exp{\vz_{0:T} \sim \ptgt{ais}}{\log \frac{\pi_T( \vx,\vz_T ) \prod \limits_{t=1}^{T} \trev(\vz_{t-1} | \vz_{t})}{\prop(\vz_0|\vx) \prod \limits_{t=1}^{T} \tfwd(\vz_{t} | \vz_{t-1})}} - \Exp{\vz_{0:T} \sim \qprop{ais}}{ \log \frac{\pi_T( \vx,\vz_T ) \prod \limits_{t=1}^{T} \trev(\vz_{t-1} | \vz_{t})}{\prop(\vz_0|\vx) \prod \limits_{t=1}^{T} \tfwd(\vz_{t} | \vz_{t-1}) }} \nonumber \\
 &= \Exp{\vz_{0:T} \sim \ptgt{ais}}{\log \prod \limits_{t=1}^{T} \bigg( \dfrac{{\pi}_T( \vx, \vz_{\h})}{{\pi}_0(\vz_{\h}|\vx)} \bigg)^{\beta_{\h}-\beta_{\h-1}}} - \Exp{\vz_{0:T} \sim \qprop{ais}}{ \log \prod \limits_{t=1}^{T} \bigg( \dfrac{{\pi}_T( \vx, \vz_{\h})}{{\pi}_0(\vz_{\h}|\vx)} \bigg)^{\beta_{\h}-\beta_{\h-1}}} \nonumber \\
  &= \Exp{\vz_{0:T} \sim \ptgt{ais}}{\sum \limits_{\h=1}^{T} (\beta_{\h}-\beta_{\h-1}) \log \dfrac{{\pi}_T( \vx, \vz_{\h})}{{\pi}_0(\vz_{\h}|\vx)}} - \Exp{\vz_{0:T} \sim \qprop{ais}}{ \sum \limits_{\h=1}^{T} (\beta_{\h}-\beta_{\h-1}) \log \dfrac{{\pi}_T( \vx, \vz_{\h})}{{\pi}_0(\vz_{\h}|\vx)}} . \label{eq:tvo_t} \\
  &\overset{(1)}{=} \sum \limits_{\h=1}^{T} \Exp{\pi_{\beta_{\h}}(\vz)}{ (\beta_{\h}-\beta_{\h-1}) \log \dfrac{{\pi}_T( \vx, \vz)}{{\pi}_0(\vz|\vx)}} -  \sum \limits_{\h=1}^{T} \Exp{\pi_{\beta_{\h-1}}(\vz)}{ (\beta_{\h}-\beta_{\h-1}) \log \dfrac{{\pi}_T( \vx, \vz)}{{\pi}_0(\vz|\vx)}} \nonumber \\
 & \overset{(2)}{=} \frac{1}{T} \Exp{\pi_{T}(\vz)}{ \log \dfrac{{\pi}_T( \vx, \vz)}{{\pi}_0(\vz|\vx)}} - \frac{1}{T} \Exp{\pi_{0}(\vz)}{ \log \dfrac{{\pi}_T( \vx, \vz)}{{\pi}_0(\vz|\vx)}} \nonumber \\ 
 & = {1 \over T} \big( \DKL({\pi_0}\|{\pi_T}) + \DKL({\pi_T}\|{\pi_0})\big) \, ,\nonumber
\end{align}
\normalsize
where in $(2)$, we use the linear annealing schedule $\beta_{t}-\beta_{t-1} = \frac{1}{T} \, \, \forall t$ and note that intermediate terms cancel in telescoping fashion.  
In $(1)$, we have used the assumption of perfect transitions (\textsc{pt}), which is common in analysis of \gls{AIS} \citep{neal2001annealed, grosse2013annealing}.  
In this case, the \gls{AIS} proposal and target distributions have the following factorial form
\begin{align}
   \vz_{0:T} \sim \qprop{ais}(\vz^{(1:K)}_{0:T}|\vx) &\overset{\textsc{(pt)}}{=} 
   \pi_0(\vz_0) \prod \limits_{t=1}^T \pi_{\beta_{t-1}}(\vz_t) ,\label{eq:perfect_fwd} \\
      \vz_{0:T} \sim \ptgt{ais}(\vz^{(1:K)}_{0:T}|\vx) 
&\overset{(\textsc{pt})}{=} 
   \pi_T(\vz_T) \prod \limits_{t=1}^T \pi_{\beta_{t-1}}(\vz_{t-1}) .  \label{eq:perfect_rev}
\end{align}
In other words,  for $1 \leq t \leq T$, perfect transitions results in independent, exact samples from $\vz_t \sim \pi_{\beta_{t-1}}(\vz)$ in the forward direction, and $\vz_{t} \sim \pi_{\beta_t}(\vz)$ in the reverse direction.  Using the factorized structure of \myeq{perfect_fwd} and \myeq{perfect_rev}, the expectations over the extended state space simplify to a sum of expectations at each $\vz_t$.

The above proves the proposition for the case of single sample \gls{AIS}, but should also hold for our tighter multi-sample \gls{AIS} bounds.   We extend this result to Independent Multi-Sample \gls{AIS} in \myapp{im_ais_pf} below, and extend the result to Coupled Reverse Multi-Sample \gls{AIS} in \myapp{coupled_t_pf}.
\end{proof}

\section{Independent Multi-Sample \gls{AIS}}\label{app:ind-multi-sample-ais}

\subsection{Probabilistic Interpretation and Bounds}\label{app:probabilistic_im_ais}
Our Independent Multi-Sample \gls{AIS} (\textsc{im-ais}) are identical to the standard \gls{IWAE} bounds in \myapp{iwae}, but using \gls{AIS} forward $\qprop{ais}(\vz_{0:T}|\vx)$ and backward $\ptgt{ais}(\vz_{0:T}|\vx)$ chains of length $T$ instead of the endpoint distributions only $q_{\theta}(\vz|\vx)$ and $p(\vz|\vx)$.

In particular, we construct an extended state space proposal by running $K$ independent \gls{AIS} forward chains $\vz_{0:T}^{(k)} \sim \qprop{ais}$ in parallel.
As in the \gls{IWAE} upper bound (\myeq{iwae_ublb}), the extended state space target involves selecting a single index $s$ uniformly at random, 
 and running a backward \gls{AIS} chain $\vz_{0:T}^{(s)} \sim \ptgt{ais}$ starting from a true posterior sample $\vz_T \sim p(\vz|\vx)$.  The remaining $K-1$ samples are obtained by running forward \gls{AIS} chains, as visualized in \myfig{ais}
\small
\begin{align}
    \qprop{im-ais}(\vz^{(1:K)}_{0:T}| \vx) &\coloneqq
    \prod \limits_{k=1}^{K} \qprop{ais}(\vz^{(k)}_{0:T}|\vx), %
    \quad \, \ptgt{im-ais}(\vx, \vz^{(1:K)}_{0:T}) \coloneqq
    \frac{1}{K} \sum \limits_{s=1}^K  \ptgt{ais}(\vx, \vz^{(s)}_{0:T}) \prod \limits_{\myoverset{k=1}{k\neq s}}^{K}  \qprop{ais}(\vz^{(k)}_{0:T}| \vx)   . \nonumber 
\end{align}
\normalsize
where  $\qprop{ais}$ and $\ptgt{ais}$ were defined in \myeq{ais_fwd_rev}.
Expanding the log unnormalized density ratio,
\begin{align}
    \log \frac{\ptgt{{im-ais}}(\vx,\vz^{(1:K)}_{0:T})}{\qprop{\textsc{im-ais}}(\vz^{(1:K)}_{0:T}|\vx)} &= \log \frac{\frac{1}{K} \sum \limits_{s=1}^K  \ptgt{ais}(\vx, \vz^{(s)}_{0:T}) \prod \limits_{\myoverset{k=1}{k\neq s}}^{K}  \qprop{ais}(\vz^{(k)}_{0:T}| \vx)}{ \prod \limits_{k=1}^{K} \qprop{ais}(\vz^{(k)}_{0:T}|\vx)} \\
    &= \log \frac{1}{K} \sum \limits_{k=1}^K \frac{\ptgt{ais}(\vx, \vz^{(k)}_{0:T})}{\qprop{ais}(\vz^{(k)}_{0:T}|\vx)}
\end{align}
which is similar to the \gls{IWAE} ratio but involves \gls{AIS} chains.
Taking the expectation under the proposal and target as in \myapp{general} recovers the  lower and upper bounds in \cref{eq:im_ais_ublb}.   The gap in the lower bound is $\DKL[\qprop{\textsc{im-ais}}(\vz^{(1:K)}_{0:T}|\vx) \|\ptgt{{im-ais}}(\vz^{(1:K)}_{0:T}|\vx))]$ and the gap in the upper bound is $\DKL[\ptgt{{im-ais}}(\vz^{(1:K)}_{0:T}|\vx))\|\qprop{\textsc{im-ais}}(\vz^{(1:K)}_{0:T}|\vx)]$.

\subsection{Proof of Linear Bias Reduction in $T$ for IM-AIS }\label{app:im_ais_pf}

\begin{restatable}[Complexity in $T$ for Independent Multi-Sample \gls{AIS} Bound]{corollary}{aiscorr}\label{cor:im_ais_complexity_t}
Assuming perfect transitions and a 
geometric annealing path with linearly-spaced $\{\beta_t\}_{t=0}^T$, 
the sum of the gaps in the 
Independent Multi-Sample
\gls{AIS} sandwich bounds on \gls{MI},  $I_{\textsc{im-ais}_U}(\pi_0, T) - I_{\textsc{im-ais}_L}(\pi_0, T)$,
reduces \textit{linearly} with 
increasing 
$T$.
\end{restatable}
\begin{proof}
    Using identical proof techniques as for \gls{IWAE} (\citet{burda2016importance,sobolev2019hierarchical}, \myprop{iwaek}), we can show that our Independent Multi-Sample \gls{AIS} bounds $\gls{ELBO}_{\textsc{IM-AIS}}(\vx; \pi_0, T, K-1) \leq \gls{ELBO}_{\textsc{IM-AIS}}(\vx; \pi_0, T, K)$ and $\gls{EUBO}_{\textsc{IM-AIS}}(\vx; \pi_0, T, K) \leq \gls{EUBO}_{\textsc{IM-AIS}}(\vx; \pi_0, T, K-1)$ improve with increasing $K$.   Thus, the bias of our multi-sample bounds is less than the bias of the single-sample bounds, so the inequality in \cref{eq:symm_kl_ineq} and linear bias reduction in \myprop{ais} also hold for Independent Multi-Sample \gls{AIS}.  We characterize this improvement in \myprop{ais_elbo_eubo} below.
\end{proof}

\subsection{Proof of Logarithmic Improvement of IM-AIS EUBO}\label{app:ais_logarithmic_pf}
\begin{restatable}[{Improvement of Independent Multi-Sample \gls{AIS} over Single-Sample \gls{AIS}}]{proposition}{aiselboeubo}\label{prop:ais_elbo_eubo}
Let $\ptgt{im-ais}( \idx |\vx, \vz_{0:T}^{(1:K)})$
$= \frac{\ptgt{ais}(\vx,\vz_{0:T}^{(\idx)})}{\qprop{ais}(\vz_{0:T}^{(\idx)}|\vx)}/\sum_{k=1}^K \frac{\ptgt{ais}(\vx,\vz_{0:T}^{(k)})}{\qprop{ais}(\vz_{0:T}^{(k)}|\vx)}$
denote the normalized importance weights over \gls{AIS} chains, and let $\mathcal{U}(s)$ indicate the uniform distribution over $K$ discrete values.  
Then, we can characterize the improvement of the Independent Multi-Sample \gls{AIS} bounds on $\log p(\vx)$, $\gls{ELBO}_{\textsc{IM-AIS}}(\vx; \pi_0, T,  K)$ and $\gls{EUBO}_{\textsc{IM-AIS}}(\vx; \pi_0, T, K)$, over the single-sample \gls{AIS} bounds $\gls{ELBO}_{\gls{AIS}}(\vx; \pi_0, T)$ and $\gls{EUBO}_{\gls{AIS}}(\vx;\pi_0, T)$ using \kl divergences, as follows
\small
\begin{align}%
\hspace*{-.15cm} \gls{ELBO}_{\textsc{IM-AIS}}(\vx; \pi_0, T,  K)
&= \gls{ELBO}_{\gls{AIS}}(\vx; \pi_0, T) +  \underbrace{\mathbb{E}_{\qprop{im-ais}(\vz_{0:T}^{(1:K)}|\vx)} \bigg[\DKL[\, \mathcal{U}(s) \| \ptgt{im-ais}(\idx |\vz_{0:T}^{(1:K)}, \vx  ) ] \vphantom{\frac{1}{2} }  \,  \bigg]}
_{\mathclap{\text{\scriptsize $0 \leq   
 \textsc{kl} \text{ of uniform from  \gls{SNIS} weights}
\leq \DKL[\qprop{ais}(\vz_{0:T}|\vx)\| \ptgt{ais}(\vz_{0:T}|\vx) ]$}}} \, , 
\label{eq:imais_over_elbo} 
\\[2ex]
\gls{EUBO}_{\textsc{IM-AIS}}(\vx; \pi_0, T,  K)
&= \gls{EUBO}_{\gls{AIS}}(\vx; \pi_0, T) -  \underbrace{\mathbb{E}_{\ptgt{im-ais}(\vz_{0:T}^{(1:K)}|\vx)} \bigg[\DKL[\, \ptgt{im-ais}(\idx |\vz_{0:T}^{(1:K)}, \vx  ) \| \mathcal{U}(s)  ] \vphantom{\frac{1}{2} }  \,  \bigg]}_{\mathclap{\text{\footnotesize $0 \leq  \textsc{kl}  \text{ of \gls{SNIS} weights from uniform} \leq \log K$}}} . \label{eq:imais_over_eubo2}
\end{align}
\normalsize
\end{restatable}
\begin{proof}
The result follows directly from \myapp{pf_iwaegiwae} \myprop{iwae_giwae} by viewing Independent Multi-Sample \gls{AIS} as \gls{IWAE} with an 
\gls{AIS} proposal as in \myapp{probabilistic_im_ais}.
\end{proof}

\section{Reverse IWAE}\label{app:rev_iwae}
In this section, we propose Reverse \gls{IWAE} (\textsc{riwae}), which is an impractical alternative to standard \gls{IWAE}.  However, we use this as the basis for our Independent Reverse (\myapp{reverse_ms_ais}) and Coupled Reverse Multi-Sample \gls{AIS} bounds (\mysec{coupled_rev} and  \myapp{coupled}).
\subsection{Probabilistic Interpretation and Bounds}
Similarly to simple reverse importance sampling, 
$\frac{1}{p(\vx)} = \mathbb{E}_{\pzx}\big[\frac{\qzx}{p(\vx,\vz)} \big]$ 
, we consider $K$ independent posterior samples in an extended state space target distribution.  The proposal distribution for our importance sampling scheme is a mixture of $K-1$ posterior distributions and one variational distribution,  %
\small
\begin{align}
    \qprop{riwae}(\vx,\setofz) = \frac{1}{\K} \sum \limits_{\kk=1}^{\K} \, q_{\theta}( \vz^{(\kk)}|\vx) \, \prod  \limits_{\myoverset{\jj =1}{\jj \neq \kk}}^{\K} p(\vx, \vz^{(\jj)})  \,, \qquad
   \ptgt{riwae}(\vx,\setofz) =  %
   \prod \limits_{k=1}^{\K} p(\vx,\vz^{(k)}) \, . \label{eq:riwae_fwdrev2}
\end{align}
\normalsize
Note that we have normalization constants of $ \int \ptgt{riwae}(\vx,\vz^{(1:K)}) d\vz^{(1:K)} =  p(\vx)^{K}$ and $ \int \qprop{riwae}(\vx,\vz^{(1:K)})  d\vz^{(1:K)} = p(\vx)^{K-1}$ since the transition kernels of the reverse chains do not change the normalization.

We visualize this sampling scheme in \myfig{multi-sample-ais}.   
Similarly to \myeq{iwae_importance}, the log unnormalized density ratio simplifies to
\begin{align}
    \log \frac{\ptgt{riwae}(\vx,\setofz)}{\qprop{riwae}(\vx,\setofz) }  &= \log \frac{ \prod \limits_{k=1}^K p(\vx, \vz^{(k)}) }{ \dfrac{1}{K} \sum \limits_{s=1}^K  q_{\theta}(\vz^{(s)}|\vx) \prod \limits_{\myoverset{k=1}{k \neq s}}^K p(\vx, \vz^{(k)})  } \label{eq:riwae_deriv} \\
    &= -\log \frac{1}{K} \sum \limits_{k=1}^K \dfrac{ q_{\theta}(\vz^{(k)}|\vx)}{ p(\vx, \vz^{(k)})} \label{eq:rev_logw} .
\end{align}
Taking the expectation under the proposal and target distributions yield lower and upper bounds on $\log \px$.   These translate to upper and lower bounds on $\log \px$ which are different that standard \gls{IWAE} bounds in \myeq{iwae_ublb},
  \small
\begin{align*}%
 \Exp{q_{\theta}(\vz^{(1)}|\vx) \prod \limits_{k=2}^K p(\vz^{(k)}|\vx)}{-\log \frac{1}{K} \sum \limits_{k=1}^K \frac{q_{\theta}(\vz^{(k)}|\vx)}{p(\vx,\vz^{(k)})}}
\leq \log\px \leq
\Exp{\prod \limits_{k=1}^K p(\vz^{(k)}|\vx)}{- \log \frac{1}{K} \sum \limits_{k=1}^K \frac{q_{\theta}(\vz^{(k)}|\vx)}{p(\vx,\vz^{(k)})}} .
\end{align*}
\normalsize
Note these bounds are \emph{impractical} as they would require more than one true posterior sample from $p(\vz|\vx)$. However, we use them in \mysec{coupled_rev} and \myapp{coupled} to derive practical multi-sample reverse \textsc{ais} bounds. 

\subsection{Improvement of RIWAE over ELBO and EUBO}
\begin{restatable}[{Improvement of Reverse \gls{IWAE} over \gls{ELBO} and \gls{EUBO}}]{proposition}{riwaeelboeubo}\label{prop:riwae}
Let $ \qprop{riwae}( \idx |\vx, \vz^{(1:K)})= \frac{q_{\theta}(\vz^{(\idx)}|\vx)}{p(\vx,\vz^{(\idx)}|\vx)}/{\sum \limits_{k=1}^K \frac{q_{\theta}(\vz^{(k)}|\vx)}{p(\vx,\vz^{(k)}|\vx)} }$
denote the normalized reverse importance sampling weights, or the posterior over the index variable in \cref{eq:riwae_fwdrev2}.  Let $\mathcal{U}(s)$ indicate the uniform distribution over $K$ discrete values.  
Then, we can characterize the improvement of the Reverse \gls{IWAE} bounds on $\log p(\vx)$, $\gls{ELBO}_{\textsc{riwae}}(\vx; q_\theta, K)$ and $\gls{EUBO}_{\textsc{riwae}}(\vx; q_\theta, K)$, over the single-sample 
bounds $\gls{ELBO}(\vx; q_\theta)$ and $\gls{EUBO}(\vx; q_\theta)$ using \kl divergences, as follows
\small
\begin{align}%
\hspace*{-.15cm} \gls{ELBO}_{\textsc{riwae}}(\vx; q_\theta, K)
&= \gls{ELBO}(\vx; q_\theta) +  \underbrace{\mathbb{E}_{\qprop{riwae}(\vz^{(1:K)}|\vx)} \bigg[\DKL[ \qprop{riwae}( \idx |\vx, \vz^{(1:K)}) \| \mathcal{U}(s) ] \vphantom{\frac{1}{2} }  \,  \bigg]} \nonumber
_{\mathclap{\text{\footnotesize $0 \leq  \textsc{kl}  \text{ of uniform from  \gls{SNIS} weights} \leq \log K$}}} \, , \label{eq:imais_over_elbo} \\[2ex]
\gls{EUBO}_{\textsc{riwae}}(\vx; q_\theta, K)
&= \gls{EUBO}(\vx; q_\theta) - \underbrace{\mathbb{E}_{\ptgt{riwae}(\vz^{(1:K)}|\vx)} \bigg[\DKL[ \mathcal{U}(s) \| \qprop{riwae}( \idx |\vx, \vz^{(1:K)})   ] \vphantom{\frac{1}{2} }  \,  \bigg]}_{\mathclap{\text{\footnotesize $0 \leq  \textsc{kl}  \text{ of \gls{SNIS} weights from uniform} \leq \DKL[ p(\vz|\vx) \| q_{\theta}(\vz|\vx) ] $}}} \, . \nonumber 
\end{align}
\normalsize
\end{restatable}
\begin{proof}
The proof follows similarly as \myprop{iwae_elbo_eubo} or \cref{lemma:giwae_elbo_eubo}.
\end{proof}

\section{Independent Reverse Multi-Sample AIS}\label{app:reverse_ms_ais}

\subsection{Probabilistic Interpretation and Bounds}
In similar fashion to Reverse \gls{IWAE}, we now use $K$ independent reverse \gls{AIS} chains to form an extended state space target distribution $\ptgt{ir-ais}(\vz^{(1:K)}_{0:T}, \vx)$, and a mixture proposal distribution $\qprop{ir-ais}(\vz^{(1:K)}_{0:T}, \vx)$ which includes a single forward \gls{AIS} chain (see \myfig{ais}) 
\begin{align}
    \qprop{ir-ais}(\vx, \vz^{(1:K)}_{0:T}) &=  \frac{1}{K} \sum \limits_{s=1}^K \qprop{ais}(\vz^{(s)}_{0:T}|\vx) \prod \limits_{\myoverset{k=1}{k\neq s}}^{K}  \ptgt{ais}( \vx, \vz^{(k)}_{0:T}) , \nonumber
    \\
    \ptgt{ir-ais}( \vx, \vz^{(1:K)}_{0:T}) &=  \prod \limits_{k=1}^{K}  \ptgt{ais}(\vx, \vz^{(k)}_{0:T}).
\end{align}
Similarly to \myeq{riwae_deriv}-(\ref{eq:rev_logw}), the log unnormalized density ratio becomes
\begin{align}
    \log \frac{\ptgt{ir-ais}(\vz^{(1:K)}_{0:T}, \vx)}{\qprop{ir-ais}(\vz^{(1:K)}_{0:T}, \vx) } &= 
    \log \frac{\prod \limits_{k=1}^{K}  \ptgt{ais}(\vx,\vz^{(k)}_{0:T})}{\frac{1}{K} \sum \limits_{k=1}^K \qprop{ais}(\vz^{(k)}_{0:T}|\vx) \prod \limits_{t=1}^{T}  \ptgt{ais}(\vx,\vz^{(k)}_{0:T})} \\
    &= -\log \frac{1}{K} \sum \limits_{k=1}^K \dfrac{ \qprop{ais}(\vz^{(k)}_{0:T}, \vx)}{\ptgt{ais}(\vx,\vz^{(k)}_{0:T})} \label{eq:rev_ais_logw} .
\end{align}
Taking expectations of the log unnormalized density ratio under the proposal and target, respectively, yield lower and upper bounds on $\log p(\vx)$
\begin{align}
      \myunderbrace{
      \mathbb{E}_{\myoverset{\vz^{(1)}_{0:T} \sim \qprop{ais}}{\vz^{(2:K)}_{0:T} \sim \ptgt{ais}}}
        \bigg[ -  \log  \frac{1}{K} \sum \limits_{k =1}^{K}
            \frac{\qprop{ais}(\vz^{(k)}_{0:T}|\vx)}
                 {\ptgt{ais}(\vx,\vz^{(k)}_{0:T})}
        \bigg] }
        {\gls{ELBO}_{\textsc{ir-ais}}(\vx;\pi_0,K,T)}  \leq  \log p(\vx)\leq
 \myunderbrace{
    \mathbb{E}_{\myoverset{\vz^{(1:K)}_{0:T} \sim \ptgt{ais}}{\phantom{\vz^{(1:K)}_{0:T} \sim \ptgt{ais}}} } 
        \bigg[
            -  \log  \frac{1}{K} \sum \limits_{k=1}^{K}
            \frac{\qprop{ais}(\vz^{(k)}_{0:T}|\vx)}
                 {\ptgt{ais}(\vx,\vz^{(k)}_{0:T})}
        \bigg]
}
{\gls{EUBO}_{\textsc{ir-ais}}(\vx; \pi_0,K,T)} . \label{eq:ri_ais_ublb2}
\end{align}

However, Independent Reverse Multi-Sample \gls{AIS} may be impractical in common settings, 
since it is infeasible to have access to more than one true posterior sample.
\subsection{Proof of Logarithmic Improvement in $K$ for Independent Reverse AIS} 

\begin{restatable}[{Improvement of Independent Reverse Multi-Sample \gls{AIS} over Single-Sample \gls{AIS}}]{proposition}{iraiselboeubo}\label{prop:i_r_multi_ais_ub}
Let $ \qprop{ir-ais}( \idx |\vx, \vz_{0:T}^{(1:K)})= \frac{\qprop{ais}( \vz_{0:T}^{(\idx)}|\vx)}{\ptgt{ais}(\vx,\vz_{0:T}^{(\idx)})}/{\sum \limits_{k=1}^K \frac{\qprop{ais}(\vz_{0:T}^{(k)}|\vx)}{\ptgt{ais}(\vx,\vz_{0:T}^{(k)})} }$
denote the normalized reverse importance sampling weights over \gls{AIS} chains, and let $\mathcal{U}(s)$ indicate the uniform distribution over $K$ discrete values.  
Then, we can characterize the improvement of the Independent Reverse Multi-Sample \gls{AIS} bounds on $\log p(\vx)$, $\gls{ELBO}_{\textsc{ir-ais}}(\vx; \pi_0, K,T)$ and $\gls{EUBO}_{\textsc{ir-ais}}(\vx; \pi_0, K,T)$, over the single-sample $\gls{ELBO}_{\textsc{ais}}(\vx; \pi_0, T)$ and $\gls{ELBO}_{\textsc{ais}}(\vx; \pi_0,T)$  using \kl divergences, as follows
\small
\begin{align}%
\hspace*{-.2cm} 
\gls{ELBO}_{\textsc{ir-ais}}(\vx;\pi_0, K,T)
&= \gls{ELBO}_{\textsc{ais}}(\vx;\pi_0,T) +  \underbrace{\mathbb{E}_{\qprop{ir-ais}(\vz_{0:T}^{(1:K)}|\vx)} \bigg[\DKL[ \qprop{ir-ais}( \idx |\vx, \vz_{0:T}^{(1:K)}) \| \mathcal{U}(s) ] \vphantom{\frac{1}{2} }  \,  \bigg]} \nonumber
_{\mathclap{\text{\footnotesize $0 \leq  \textsc{kl}  \text{ of uniform from  \gls{SNIS} weights} \leq \log K$}}} \, ,  \\[2ex]
\gls{EUBO}_{\textsc{ir-ais}}(\vx;\pi_0, K,T)
&= \gls{EUBO}_{\textsc{ais}}(\vx;\pi_0,T) -  \underbrace{\mathbb{E}_{\ptgt{ir-ais}(\vz_{0:T}^{(1:K)}|\vx)} \bigg[\DKL[ \mathcal{U}(s) \| \qprop{ir-ais}( \idx |\vx, \vz_{0:T}^{(1:K)})   ] \vphantom{\frac{1}{2} }  \,  \bigg]}_{\mathclap{\text{\scriptsize $0 \leq \textsc{kl}  \text{ of \gls{SNIS} weights from uniform} \leq \DKL[ \ptgt{ais}(\vz_{0:T}|\vx) \| \qprop{ais}(\vz_{0:T}|\vx) ] $}}} \, . \nonumber 
\end{align}
\normalsize
\end{restatable}
\begin{proof}
The proof follows similarly as \myprop{iwae_elbo_eubo} or \cref{lemma:giwae_elbo_eubo}.   
\end{proof}

\section{Coupled Reverse Multi-Sample AIS}\label{app:coupled}

\subsection{Probabilistic Interpretation and Bounds}\label{app:prob_interp_cr_ais}
The Coupled Reverse Multi-Sample \gls{AIS} extended state space target distribution in \cref{fig:ais} initializes $K$ backward chains from a \textit{single} target sample $\vz_T \sim \pi_T(\vz|\vx)$, which makes the bound useful in practical situations.
We denote the remaining transitions as $\ptgt{ais}(\vz_{0:T-1}|\vz_T, \vx)$, since they are identical to standard \gls{AIS} in \myeq{ais_fwd_rev}.  Thus, the \textsc{cr-ais} extended state space target distribution is
\begin{align}
     \ptgt{cr-ais}(\vx, \vz^{(1:K)}_{0:T-1}, \vz_T) &=  \pi_T(\vx, \vz_T) \prod \limits_{k=1}^{K}  \ptgt{ais}(\vz^{(k)}_{0:T-1}| \vz_T, \vx).
     \label{eq:cr_ais_tgt}
\end{align}

 The extended state space proposal is obtained by selecting an index $s$ uniformly at random and running a single forward \gls{AIS} chain.  We then run $K-1$ backward chains, all starting from the last state of the selected forward chain,
\begin{align}
    \qprop{cr-ais}(\vz^{(1:K)}_{0:T-1},\vz_T|\vx) 
    &\coloneqq 
    \frac{1}{K} \sum \limits_{s=1}^K \qprop{ais}(\vz^{(s)}_{0:T-1}, \vz_{T}|\vx) \prod \limits_
    {\myoverset{k=1}{k\neq s}}^{K}  
    \ptgt{ais}( \vz^{(k)}_{0:T-1} | \vz_T, \vx) . \label{eq:cr_ais_prop}
\end{align}
See \myfig{ais} or \myfig{multi-sample-ais} for a graphical model description.

We can construct log-partition bounds using the log ratio of unnormalized densities,
\begin{align}
\log \frac{\ptgt{cr-ais}(\vz^{(1:K)}_{0:T-1}, \vz_T, \vx)}{\qprop{cr-ais}(\vz^{(1:K)}_{0:T-1}, \vz_T| \vx)}&= 
\log \frac{\pi_T(\vz_T, \vx) \prod \limits_{k=1}^K \ptgt{ais}(\vz^{(k)}_{0:T-1}|\vz_T, \vx) }{\frac{1}{K} \sum \limits_{\idx=1}^K \qprop{ais}(\vz_{0:T-1}^{(\idx)},\vz_T|\vx) \prod \limits_{\myoverset{k=1}{k\neq \idx}}^K \ptgt{ais}(\vz^{(k)}_{0:T-1}|\vz_T) } \\
&= - \log  \frac{1}{K} \sum \limits_{k=1}^K  \dfrac{\qprop{ais}(\vz_{0:T-1}^{(k)},\vz_T|\vx)}{ \ptgt{ais}(\vz^{(k)}_{0:T-1}|\vz_T) \pi_T(\vz_T,\vx) } . %
\end{align}

Taking the expected log ratio under the proposal and target yields lower and upper bounds on $\log p(\vx)$,
\small
\begin{align}
   \gls{ELBO}_{\textsc{cr-ais}}(\vx;\pi_0,K,T) &\coloneqq
   - \mathbb{E}_{\myoverset{\vz^{(1)}_{0:T-1}, \vz_T \sim \qprop{ais}(\vz_{0:T}|\vx)}{\vz^{(2:K)}_{0:T-1} \sim \ptgt{ais}(\vz_{0:T-1}|\vz_T, \vx)}} \bigg[ \log  \frac{1}{K} \sum \limits_{k =1}^{K}  \frac{\qprop{ais}(\vz^{(k)}_{0:T-1},\vz_{T}|\vx)}{\ptgt{ais}(\vx,\vz^{(k)}_{0:T-1},\vz_{T})} \bigg] %
   \leq  \log p(\vx) 
   \nonumber \\
     \gls{EUBO}_{\textsc{cr-ais}}(\vx;\pi_0,K,T)&\coloneqq -\mathbb{E}_{\myoverset{\vz_{T} \sim \pi_T(\vz_T|\vx)}{\vz^{(1:K)}_{0:T-1} \sim \ptgt{ais}(\vz_{0:T-1}|\vz_{T}, \vx)}}
     \bigg[ \log  \frac{1}{K} \sum \limits_{k =1}^{K}  \frac{\qprop{ais}(\vz^{(k)}_{0:T-1},\vz_{T}|\vx)}{\ptgt{ais}(\vx,\vz^{(k)}_{0:T-1},\vz_{T})} \bigg] \geq \log \px .
     \nonumber
\end{align}
\normalsize
 \subsection{Proof that CR-AIS Bounds Tighten with Increasing $K$}\label{app:coupled_k_pf}
 In this section, we prove that Coupled Reverse Multi-Sample \gls{AIS} bounds get tighter with increasing $K$.
 Our proof provides an alternative perspective to \citet{burda2016importance,sobolev2019hierarchical} for showing the monotonic improvement of \gls{IWAE} or Independent Multi-Sample \gls{AIS} with $K$.
 We will also characterize the improvement of \textsc{cr-ais} bounds over single-sample \gls{AIS} bounds in \myprop{c_r_multi_sample_ais}, as a direct consequence of this lemma.
 
 \begin{restatable}[{\textsc{CR-AIS} Bounds Tighten with Increasing $K$}]{lemma}{craiswithk}\label{prop:cr_ais_k}
 Coupled Reverse Multi-Sample \gls{AIS} bounds get tighter with increasing number of samples $K$.   In other words, for any $K>1$, 
 \footnotesize
 \begin{align}
     \gls{ELBO}_{\textsc{cr-ais}}(\vx; \pi_0, T, K-1) \leq \gls{ELBO}_{\textsc{cr-ais}}(\vx; \pi_0, T, K) \leq \log p(\vx) ,\\
     \gls{EUBO}_{\textsc{cr-ais}}(\vx; \pi_0, T, K-1) \geq \gls{EUBO}_{\textsc{cr-ais}}(\vx; \pi_0, T, K) \geq \log p(\vx) .
 \end{align}
 \normalsize
 \end{restatable}
 \begin{proof}
Our proof will proceed by introducing an additional set of $M$ index variables $s_{1:M}$ into the $K$-sample 
probabilistic interpretation of \textsc{cr-ais} in \cref{eq:cr_ais_tgt}-(\ref{eq:cr_ais_prop}), with $M < K$.   We will show that the \textsc{kl} divergence in this  \textit{joint} state space (including $s_{1:M}$) is equal to the gap of the $M$-sample $\gls{ELBO}_{\textsc{cr-ais}}(\vx; \pi_0, T, M)$ or $ \gls{EUBO}_{\textsc{cr-ais}}(\vx; \pi_0, T, M)$.   We then show that marginalizing over $s_{1:M}$ yields the gap of the $K$-sample $\gls{ELBO}_{\textsc{cr-ais}}(\vx; \pi_0, T, K)$ or $ \gls{EUBO}_{\textsc{cr-ais}}(\vx; \pi_0, T, K)$.    Since marginalization cannot increase the \textsc{kl} divergence, we will have shown that for any $M < K$,
$D_{KL}[\qprop{cr-ais}(\vz^{(1:K)}_{0:T-1}, \vz_T|\vx) \| \ptgt{cr-ais}(\vz^{(1:K)}_{0:T-1}, \vz_T|\vx)] ] \leq D_{KL}[\qprop{cr-ais}(\vz^{(1:M)}_{0:T-1}, \vz_T|\vx) \| \ptgt{cr-ais}(\vz^{(1:M)}_{0:T-1}, \vz_T|\vx)] ]$ and thus
$\gls{ELBO}_{\textsc{cr-ais}}(\vx; \pi_0, T, M) \leq \gls{ELBO}_{\textsc{cr-ais}}(\vx; \pi_0, T, K)$.   Identical reasoning holds for the $\gls{EUBO}_{\textsc{cr-ais}}$.

\paragraph{Sub-Sampling Probabilistic Interpretation}
Let $\udist(s_1, ... ,s_M)$ indicate the probability of drawing $M < K$ sample indices uniformly without replacement (i.e. each $s_m$ is distinct).   
In the \textsc{cr-ais} target distribution, there is no distinction between the indices $\{ s_{1:M}\}$ and $k \not \in \{ s_{1:M}\}$ after drawing $\vz_T \sim \pi_T(\vx,\vz)$.  We can write
\footnotesize
\begin{align}
    \ptgt{cr-ais}(\vx, s_{1:M}, \vz^{(1:K)}_{0:T-1}, \vz_T)  &=  \udist(s_1, ... ,s_M)\cdot \pi_T(\vx, \vz_T)  \prod \limits_{k=1}^{K}  \ptgt{ais}(\vz^{(k)}_{0:T-1}| \vz_T, \vx) \label{eq:cr_ais_tgt_extended}
\end{align}
\normalsize
which, after marginalization over $s_{1:M}$, clearly matches $\ptgt{cr-ais}(\vx, \vz^{(1:K)}_{0:T-1}, \vz_T)$.

 We also draw $s_{1:M} \sim \udist(s_1, \ldots ,s_M)$ for the extended state space proposal.  Next, we select an index $m$ uniformly at random from $\{{1,\ldots,M}\}$, which is used to 
to specify
which chain $\vz_{0:T-1}^{(s_m)}$ is run in the forward direction to obtain $\vz_T$, as in \cref{eq:cr_ais_prop}.  After marginalizing over $m$, we obtain the following mixture proposal distribution
\footnotesize
\begin{align}
    \qprop{cr-ais}(s_{1:M}, \vz^{(1:K)}_{0:T-1}, \vz_T|\vx) 
    \coloneqq 
    \udist(s_1, \ldots ,s_M)  
    \cdot \bigg( \frac{1}{M} &\sum \limits_{m=1}^M \qprop{ais}(\vz^{(s_m)}_{0:T-1}, \vz_{T}|\vx) \prod \limits_{\myoverset{j=1}{j \neq m}}^{M} \ptgt{ais}( \vz^{(s_j)}_{0:T-1} | \vz_T, \vx) \bigg) \nonumber  \\
    &\cdot \prod \limits_{\myoverset{k=1}{k \not\in \{s_{1:M} \}}}^{K}  \ptgt{ais}(\vz^{(k)}_{0:T-1}| \vz_T, \vx)\,.
    \label{eq:cr_ais_prop_extended}
\end{align}
\normalsize
We will consider marginalizing over $s_{1:M}$ below, but first write the \textsc{kl} divergence in the extended state space which includes $s_{1:M}$.
For example, the forward \textsc{kl} divergence matches the gap of the $M$-sample $\gls{ELBO}_{\textsc{cr-ais}}(\vx; \pi_0, T, M)$,

\begin{flalign}
   & \DKL[ \qprop{cr-ais}(s_{1:M}, \vz^{(1:K)}_{0:T-1}, \vz_T|\vx)\| \ptgt{cr-ais}(s_{1:M}, \vz^{(1:K)}_{0:T-1}, \vz_T|\vx)] &
\end{flalign}
\tiny
\begin{flalign}
&= \mathbb{E}_{\qprop{cr-ais}}
   \left[ \log \frac{\cancel{\udist(s_{1:M})  }
    \cdot \Bigg( \frac{1}{M} \sum \limits_{m=1}^{M} \qprop{ais}(\vz^{(s_m)}_{0:T-1}, \vz_{T}|\vx) \prod \limits_{\myoverset{j=1}{j \neq m}}^{M} \ptgt{ais}( \vz^{(s_j)}_{0:T-1} | \vz_T, \vx) \bigg) \cdot \cancel{ \prod \limits_{\myoverset{k=1}{k \not\in \{s_{1:M} \}}}^{K}  \ptgt{ais}(\vz^{(k)}_{0:T-1}| \vz_T, \vx)}  }{\cancel{\udist(s_{1:M})} \cdot \pi_T(\vx, \vz_T)  \prod \limits_{k \in \{s_{1:M}\}}  \ptgt{ais}(\vz^{(k)}_{0:T-1}| \vz_T, \vx)
    \cancel{ \prod \limits_{\myoverset{k=1}{k \not\in \{s_{1:M}\}}}^{K}  \ptgt{ais}(\vz^{(k)}_{0:T-1}| \vz_T, \vx) }} \right] \nonumber &
\end{flalign}
\normalsize
\begin{flalign}
&= \DKL\big[ \qprop{cr-ais}( \vz^{{(1:M)}}_{0:T-1}, \vz_T|\vx)\| \ptgt{cr-ais}(\vz^{{(1:M)}}_{0:T-1}, \vz_T|\vx) \big] \,,&\label{eq:m-sample-kl}
\end{flalign}

which matches the $M$-sample probabilistic interpretation of \textsc{cr-ais} from \myapp{prob_interp_cr_ais}.  Identical reasoning holds for the case of $\gls{EUBO}_{\textsc{cr-ais}}(\vx; \pi_0, T, K) \leq \gls{EUBO}_{\textsc{cr-ais}}(\vx; \pi_0, T, M)$ using the reverse \textsc{kl} divergence.

\paragraph{Marginalization over $s_{1:M}$}
We have already seen from \cref{eq:cr_ais_tgt_extended} that the marginal {$\sum_{s_{1:M}} \ptgt{cr-ais}(\vx, s_{1:M}, \vz^{(1:K)}_{0:T-1}, \vz_T)$} matches the $K$-sample target distribution $\ptgt{cr-ais}(\vx, \vz^{(1:K)}_{0:T-1}, \vz_T)$.

We would now like to marginalize over $s_{1:M}$ in $\qprop{cr-ais}(s_{1:M}, \vz^{(1:K)}_{0:T-1}, \vz_T|\vx)$.
Combining the two product terms in \cref{eq:cr_ais_prop_extended},
\footnotesize
\begin{align}
   \sum \limits_{s_{1:M}}  \qprop{cr-ais}(s_{1:M}, \vz^{(1:K)}_{0:T-1}, \vz_T|\vx) 
   &=
    \mathbb{E}_{\udist(s_{1:M})}
    \bigg[ \frac{1}{M} \sum \limits_{m=1}^M \qprop{ais}(\vz^{(s_m)}_{0:T-1}, \vz_{T}|\vx) \prod \limits_{\myoverset{k=1}{k \neq s_m} }^{K} \ptgt{ais}( \vz^{(k)}_{0:T-1} | \vz_T, \vx) \bigg] 
    \label{eq:cr_ais_prop_extended2} \\
    &= \frac{1}{M} \sum \limits_{m=1}^M \mathbb{E}_{\udist(s_{1:M})} \bigg[  \qprop{ais}(\vz^{(s_m)}_{0:T-1}, \vz_{T}|\vx) \prod \limits_{\myoverset{k=1}{k \neq s_m} }^{K} \ptgt{ais}( \vz^{(k)}_{0:T-1} | \vz_T, \vx) \bigg] \\
    &\overset{(1)}{=} \frac{1}{M} \sum \limits_{m=1}^M \mathbb{E}_{\udist(s_{m})} \bigg[  \qprop{ais}(\vz^{(s_m)}_{0:T-1}, \vz_{T}|\vx)  \prod \limits_{\myoverset{k=1}{k \neq s_m} }^{K} \ptgt{ais}( \vz^{(k)}_{0:T-1} | \vz_T, \vx) \bigg] \\
    &\overset{(2)}{=} \frac{1}{M} \sum \limits_{m=1}^M \frac{1}{K} \sum \limits_{j=1}^K \bigg[  \qprop{ais}(\vz^{(j)}_{0:T-1}, \vz_{T}|\vx) \prod \limits_{\myoverset{k=1}{k \neq j} }^{K} \ptgt{ais}( \vz^{(k)}_{0:T-1} | \vz_T, \vx) \bigg] \\
    &= \frac{1}{K} \sum \limits_{j=1}^K \bigg[  \qprop{ais}(\vz^{(j)}_{0:T-1}, \vz_{T}|\vx) \prod \limits_{\myoverset{k=1}{k \neq j} }^{K} \ptgt{ais}( \vz^{(k)}_{0:T-1} | \vz_T, \vx) \bigg] 
\end{align}
\normalsize
where in $(1)$, we can write the marginal $\udist(s_m)$ since the terms inside expectation do not explicitly depend on the other indices in $s_{1:M}$.    In $(2)$, we use the fact that the marginal over any $s_m$ is uniform.

Thus, we have shown that the marginal distributions match the standard $K$-sample target and proposal distributions of \textsc{cr-ais}, with
\tiny
\begin{align}
\DKL \big[ \sum \limits_{s_{1:M}} \qprop{cr-ais}(s_{1:M}, \vz^{(1:K)}_{0:T-1}, \vz_T|\vx)\| \sum \limits_{s_{1:M}} \ptgt{cr-ais}(s_{1:M}, \vz^{(1:K)}_{0:T-1}, \vz_T|\vx) \big] = \underbrace{ \DKL[ \qprop{cr-ais}(\vz^{(1:K)}_{0:T-1}, \vz_T|\vx)\| \ptgt{cr-ais}(\vz^{(1:K)}_{0:T-1}, \vz_T|\vx)]}_{\text{gap of $\gls{ELBO}_{\textsc{cr-ais}}(\vx; \pi_0, T, K)$}} \nonumber
\end{align}
\normalsize
and similar reasoning for the \gls{EUBO} using the reverse \textsc{kl} divergence.

\paragraph{Conclusion of Proof}
Since marginalization can not increase the \textsc{kl} divergence, we have
\scriptsize
\begin{align}
   \underbrace{ \DKL[ \qprop{cr-ais}(\vz^{(1:K)}_{0:T-1}, \vz_T|\vx)\| \ptgt{cr-ais}(\vz^{(1:K)}_{0:T-1}, \vz_T|\vx)]}_{\text{gap of $\gls{ELBO}_{\textsc{cr-ais}}(\vx; \pi_0, T, K)$}} 
   \leq \underbrace{ \DKL[ \qprop{cr-ais}(s_{1:M}, \vz^{(1:K)}_{0:T-1}, \vz_T|\vx)\| \ptgt{cr-ais}(s_{1:M}, \vz^{(1:K)}_{0:T-1}, \vz_T|\vx)]}_{\text{gap of $\gls{ELBO}_{\textsc{cr-ais}}(\vx; \pi_0, T, M)$ (\cref{eq:m-sample-kl})} } . \label{eq:to_show_m}
\end{align}
\normalsize
This shows that $\gls{ELBO}_{\textsc{cr-ais}}(\vx; \pi_0, T, K)$ is tighter than $\gls{ELBO}_{\textsc{cr-ais}}(\vx; \pi_0, T, M)$, with $\gls{ELBO}_{\textsc{cr-ais}}(\vx; \pi_0, T, K) \geq \gls{ELBO}_{\textsc{cr-ais}}(\vx; \pi_0, T, M)$.
Identical reasoning holds for the case of $\gls{EUBO}_{\textsc{cr-ais}}(\vx; \pi_0, T, K) \leq \gls{EUBO}_{\textsc{cr-ais}}(\vx; \pi_0, T, M)$ using the reverse \textsc{kl} divergence.   Choosing $M = K-1$ proves that the bounds tighten as $K$ increases, as desired.
\end{proof}

\paragraph{Characterizing the Improvement of $K$-Sample over $M$-Sample CR-AIS} Note that we can explicitly write the gap of the inequality \cref{eq:to_show_m} using the chain rule for joint probability, for example $\qprop{cr-ais}(s_{1:M},\vz^{(1:K)}_{0:T-1}, \vz_T|\vx)= \qprop{cr-ais}(\vz^{(1:K)}_{0:T-1}, \vz_T|\vx) \qprop{cr-ais}(s_{1:M}|\vz^{(1:K)}_{0:T-1}, \vz_T,\vx)$.  

The gap in \cref{eq:to_show_m} also corresponds to the gap between $\gls{ELBO}_{\textsc{cr-ais}}(\vx; \pi_0, T, K)$ and $\gls{ELBO}_{\textsc{cr-ais}}(\vx; \pi_0, T, M)$, so we can write
\begin{align}
&\gls{ELBO}_{\textsc{cr-ais}}(\vx; \pi_0, T, K) - \gls{ELBO}_{\textsc{cr-ais}}(\vx; \pi_0, T, M) \nonumber \\
&= \DKL[\qprop{cr-ais}(s_{1:M}|\vz^{(1:K)}_{0:T-1}, \vz_T,\vx) \|\ptgt{cr-ais}(s_{1:M}|\vz^{(1:K)}_{0:T-1}, \vz_T,\vx)] . \label{eq:conditional_kl}
\end{align}
Below, we analyze the posterior over the index variable for $M=1$ as a special case.   This allows us to characterize the gap between $\gls{ELBO}_{\textsc{cr-ais}}(\vx; \pi_0, T, K)$ and the single-sample $\gls{ELBO}_{\textsc{ais}}(\vx; \pi_0, T) = \gls{ELBO}_{\textsc{cr-ais}}(\vx; \pi_0, T, K=1)$.

 \subsection{Proof of Logarithmic Improvement in $K$ for CR-AIS ELBO }\label{app:coupled_pf} %

\begin{restatable}[{Improvement of Coupled Reverse Multi-Sample \gls{AIS} over Single-Sample \gls{AIS}}]{proposition}{craiselboeubo}\label{prop:cr_ais_elbo_eubo}\label{prop:c_r_multi_sample_ais}
Let
\begin{align}
\qprop{cr-ais}( \idx |\vx,\vz_{0:T-1}^{(1:K)}, \vz_T)= \dfrac{\dfrac{\qprop{ais}(\vz_{0:T-1}^{(\idx)}, \vz_T|\vx)}{\ptgt{ais}(\vx,\vz_{0:T-1}^{(\idx)}|\vz_T) }}{\sum \limits_{k=1}^K \dfrac{\qprop{ais}(\vz_{0:T-1}^{(k)}, \vz_T|\vx)}{\ptgt{ais}(\vx,\vz_{0:T-1}^{(k)}|\vz_T) }}
\end{align}
denote the normalized importance weights over the \gls{AIS} chains used in Coupled Reverse Multi-Sample \gls{AIS}, and let $\mathcal{U}(s)$ indicate the uniform distribution over $K$ discrete values.  
Then, we can characterize the improvement of the Coupled Reverse Multi-Sample \gls{AIS} bounds on $\log p(\vx)$, $\gls{ELBO}_{\textsc{cr-ais}}(\vx;\pi_0, T,  K)$ and $\gls{EUBO}_{\textsc{cr-ais}}(\vx;\pi_0, T, K)$, over the single-sample \gls{AIS} bounds $\gls{ELBO}_{\gls{AIS}}(\vx;\pi_0, T)$ and $\gls{EUBO}_{\gls{AIS}}(\vx;\pi_0, T)$ using \kl divergences, as follows
\footnotesize
\begin{align}%
\hspace*{-.15cm} \gls{ELBO}_{\textsc{CR-AIS}}(\vx;\pi_0, T,  K)
&= \gls{ELBO}_{\gls{AIS}}(\vx;\pi_0, T) +  \underbrace{\mathbb{E}_{\qprop{cr-ais}(\vz_{0:T-1}^{(1:K)}, \vz_T|\vx)} \bigg[\DKL[ \qprop{cr-ais}( \idx |\vx, \vz_{0:T-1}^{(1:K)}, \vz_T) \| \mathcal{U}(s) ] \vphantom{\frac{1}{2} }  \,  \bigg]} \nonumber
_{\mathclap{\text{\footnotesize $0 \leq  \textsc{kl}  \text{ of uniform from  \gls{SNIS} weights} \leq \log K$}}} \, , \label{eq:crais_over_elbo} \\
\gls{EUBO}_{\textsc{CR-AIS}}(\vx;\pi_0, T,  K)
&= \gls{EUBO}_{\gls{AIS}}(\vx;\pi_0, T) -  \underbrace{\mathbb{E}_{\ptgt{cr-ais}(\vz_{0:T-1}^{(1:K)}, \vz_T|\vx)} \bigg[\DKL[ \mathcal{U}(s) \| \qprop{cr-ais}( \idx |\vx, \vz_{0:T-1}^{(1:K)}, \vz_T)   ] \vphantom{\frac{1}{2} }  \,  \bigg]}_{\mathclap{\text{\scriptsize $0 \leq  \textsc{kl}  \text{ of \gls{SNIS} weights from uniform} \leq \DKL[ \ptgt{ais}(\vz_{0:T}|\vx) \| \qprop{ais}(\vz_{0:T}|\vx) ] $}}} \, . \nonumber 
\end{align}
\normalsize
\end{restatable}
\begin{proof}
Using $M=1$ in \cref{eq:conditional_kl} above, we obtain 
\begin{align}
    \gls{ELBO}_{\textsc{cr-ais}}&(\vx; \pi_0, T, K) - \gls{ELBO}_{\textsc{ais}}(\vx; \pi_0, T, M=1) \nonumber \\
&= \DKL[\qprop{cr-ais}(s|\vz^{(1:K)}_{0:T-1}, \vz_T,\vx) \|\ptgt{cr-ais}(s|\vz^{(1:K)}_{0:T-1}, \vz_T,\vx)] 
\end{align}
From \cref{eq:cr_ais_prop_extended}, we can write the posterior over the index variable as
\begin{align}
    \qprop{cr-ais}(s|\vz^{(1:K)}_{0:T-1}, \vz_T,\vx) = \frac{\frac{\qprop{ais}(\vz^{(s)}_{0:T-1}, \vz_{T}|\vx)}{\ptgt{ais}(\vx, \vz^{(s)}_{0:T-1}| \vz_{T})} }{\sum \limits_{k=1}^K \frac{\qprop{ais}(\vz^{(k)}_{0:T-1}, \vz_{T}|\vx)}{\ptgt{ais}(\vx, \vz^{(k)}_{0:T-1}| \vz_{T})}}
\end{align}
For the target distribution in \cref{eq:cr_ais_tgt_extended}, all $K$ samples are drawn from $\ptgt{ais}(\vz^{(k)}_{0:T-1}|\vx, \vz_{T})\pi(\vx,\vz_T)$ and the posterior $\ptgt{cr-ais}(s|\vz^{(1:K)}_{0:T-1}, \vz_T,\vx)= \mathcal{U}(s)$ is uniform.

Thus, we can characterize the improvement of the $K$-sample \textsc{cr-ais} \gls{ELBO} over its single-chain \gls{AIS} 
\small
\begin{align}
 \gls{ELBO}_{\textsc{cr-ais}}&(\vx; \pi_0, T, K) - \gls{ELBO}_{\textsc{ais}}(\vx; \pi_0, T) = \mathbb{E}_{\qprop{cr-ais}(\vz_{0:T-1}^{(1:K)}, \vz_T|\vx)} \bigg[\DKL[ \qprop{cr-ais}( \idx |\vx, \vz_{0:T-1}^{(1:K)}, \vz_T) \| \mathcal{U}(s) ] \bigg]    
\end{align}
\normalsize
The proof follows identically for the \gls{EUBO}.

For the \gls{ELBO}, the \textsc{kl} divergence with the uniform distribution in the second argument is bounded above by $\log K$.   For the \gls{EUBO}, the improvement of \textsc{cr-ais} is limited by the gap in the single-chain $\gls{EUBO}_{\gls{AIS}}(\vx;\pi_0, T)$, which corresponds to $\DKL[ \ptgt{ais}(\vz_{0:T}|\vx) \| \qprop{ais}(\vz_{0:T}|\vx) ]$.
\end{proof}

\subsection{Linear Bias Reduction in $T$ for CR-AIS}\label{app:coupled_t_pf}

\begin{restatable}[Complexity in $T$ for Coupled Reverse Multi-Sample \gls{AIS} Bound]{corollary}{craiscorr}\label{cor:cr_ais_complexity_t}
Assuming perfect transitions and a 
geometric annealing path with linearly-spaced $\{\beta_t\}_{t=0}^T$, 
the sum of the gaps in the 
Coupled Reverse Multi-Sample
\gls{AIS} sandwich bounds on \gls{MI},  $I_{\textsc{cr-ais}_U}(\pi_0, K, T) - I_{\textsc{cr-ais}_L}(\pi_0, K, T)$,
reduces \textit{linearly} with 
increasing 
$T$.
\end{restatable}
\begin{proof} 
Since we have shown in \myprop{cr_ais_elbo_eubo} that $\gls{ELBO}_{\textsc{cr-ais}}(\pi_0, K, T)$ and $\gls{EUBO}_{\textsc{cr-ais}}(\pi_0, K, T)$ 
are tighter than $\gls{ELBO}_{\gls{AIS}}(\pi_0, T)$ and $\gls{EUBO}_{\gls{AIS}}(\pi_0, T)$, the  bias in the \textsc{cr-ais} sandwich bounds  $\gls{EUBO}_{\textsc{cr-ais}}(\pi_0, K, T) -\gls{ELBO}_{\textsc{cr-ais}}(\pi_0, K, T)$ will be less than that of single-sample \gls{AIS}.   Thus, \cref{eq:symm_kl_ineq} will hold with an inequality, and \textsc{cr-ais} will inherit the linear bias reduction under perfect transitions and linear scheduling from \myprop{ais} (see \myapp{ais_pf}).
\end{proof}
\section{Comparison of Multi-Sample AIS Bounds}\label{app:comparison}\label{app:multi-sample-ais-comparison}
In this section, we provide more extensive results for the experiments of \mysec{experiment_multi_comparison}.
In \myfig{multiple-ais-plot}, we compare the performance of our various Multi-Sample \gls{AIS} bounds for different values of $K$ and $T$.
We consider \gls{MI} estimation of Linear \textsc{vae}, \textsc{mnist-vae}, and \textsc{mnist-gan}.

 \begin{figure}[t]
\centering
\subcaptionbox{Linear VAE ($K=1$)}{
\includegraphics[scale=.27]{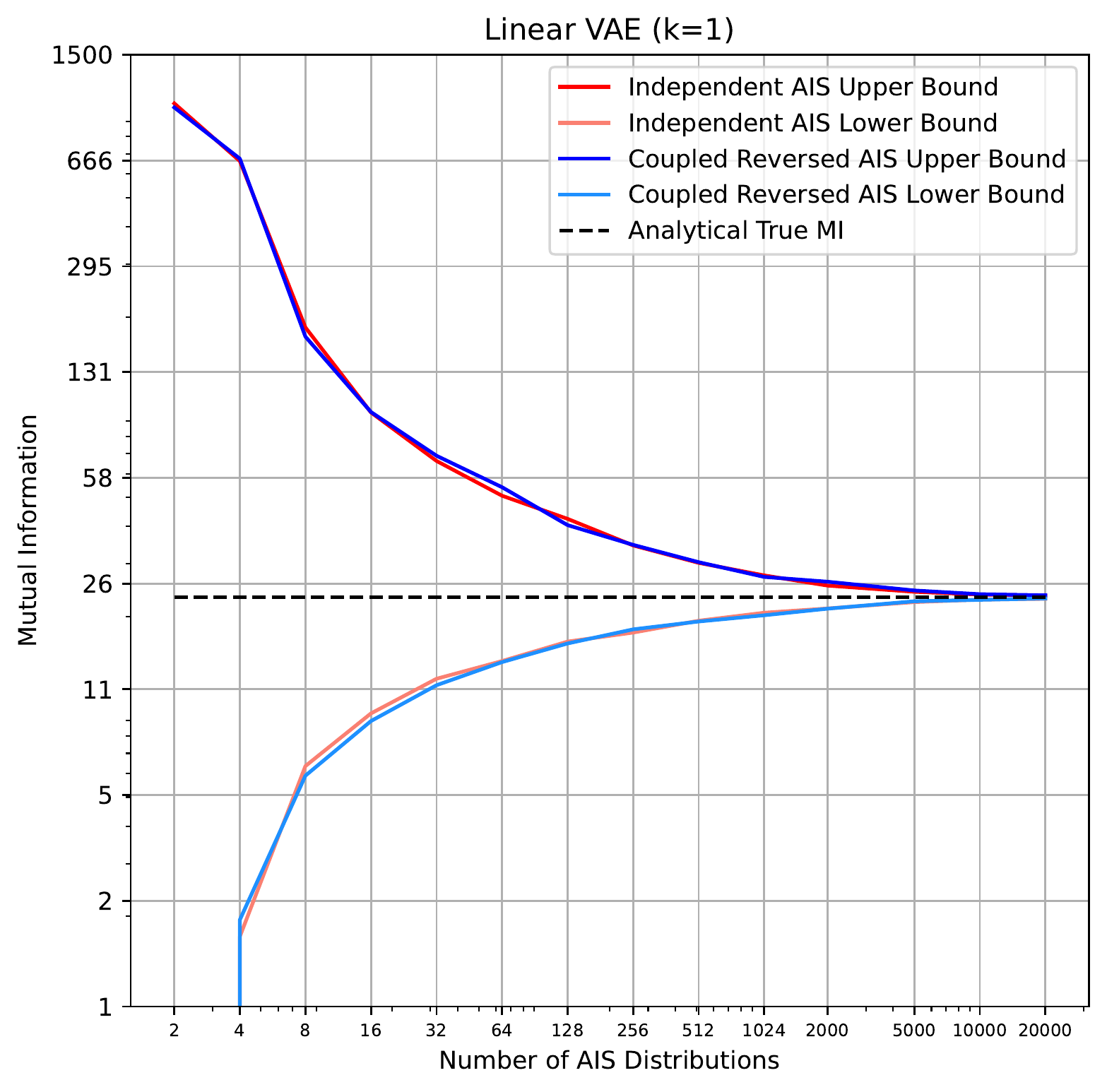}}
\hspace{.1cm}
\subcaptionbox{Linear VAE ($K=100$)}{
\includegraphics[scale=.27]{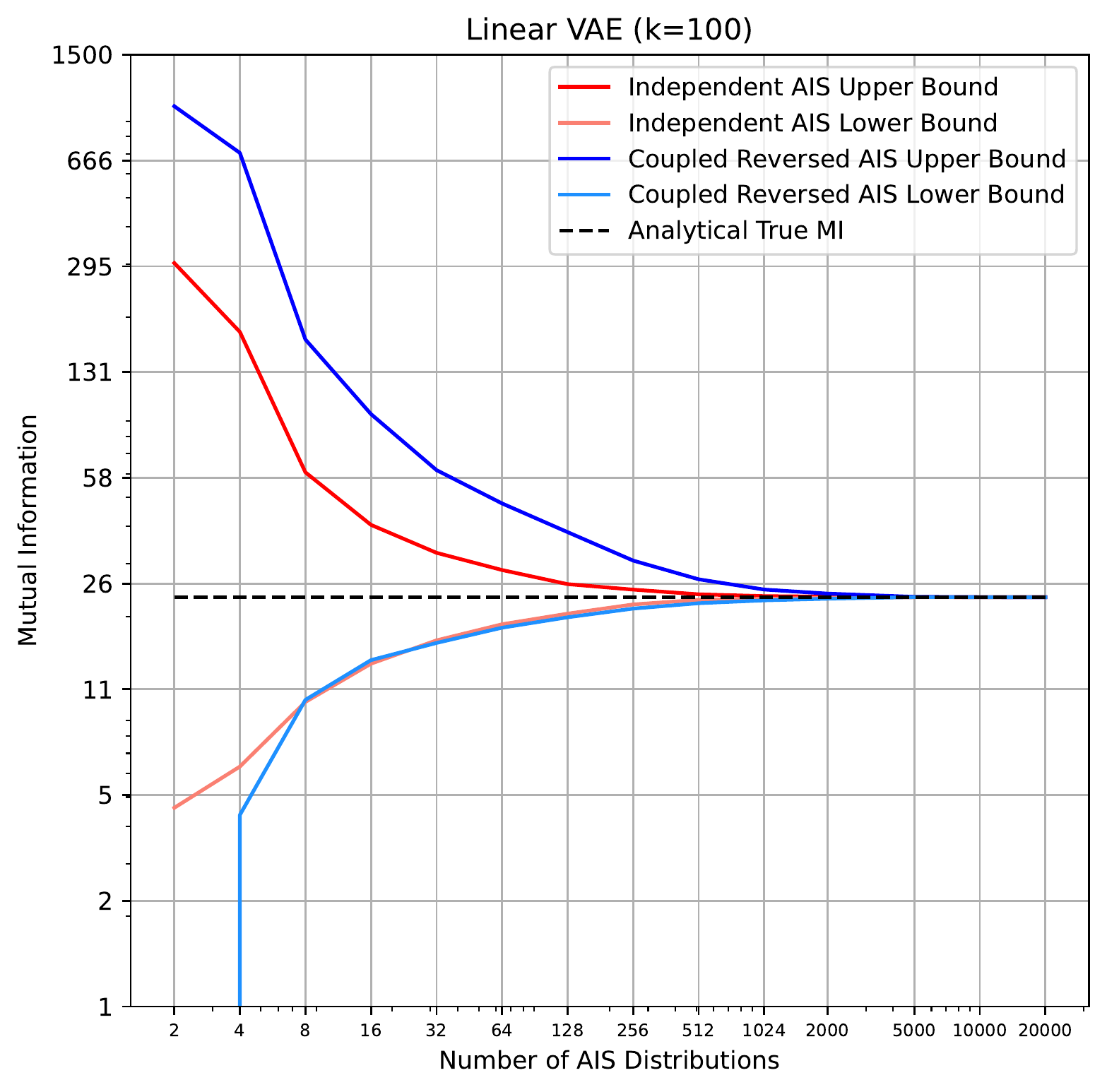}}
\hspace{.1cm}
\subcaptionbox{Linear VAE ($K=10000$)}{
\includegraphics[scale=.27]{fig/fig-multiple-ais-iclr/LinearVAE_10000.pdf}}
\subcaptionbox{VAE ($K=1$)}{
\includegraphics[scale=.27]{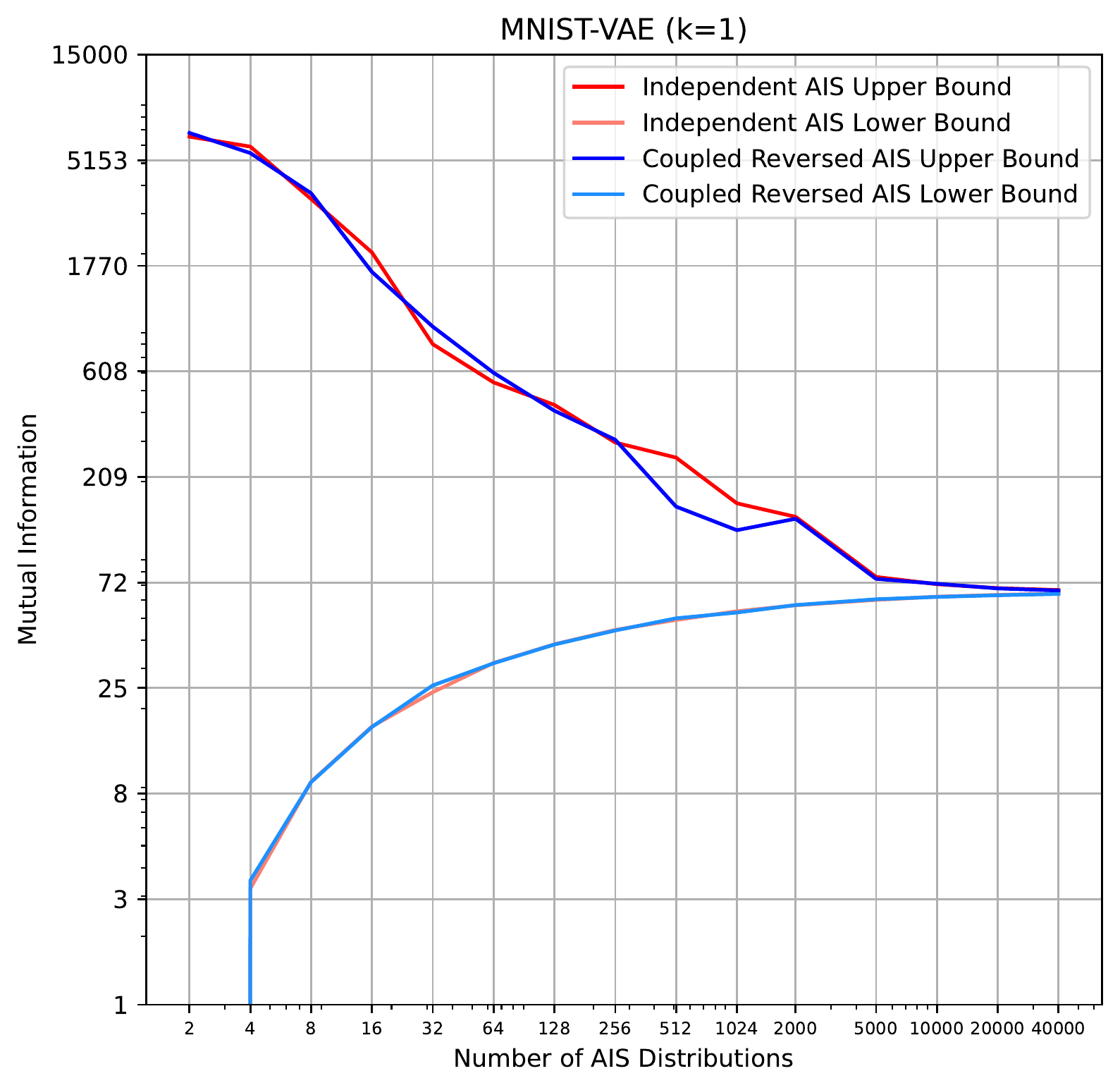}}
\hspace{.1cm}
\subcaptionbox{VAE ($K=100$)}{
\includegraphics[scale=.27]{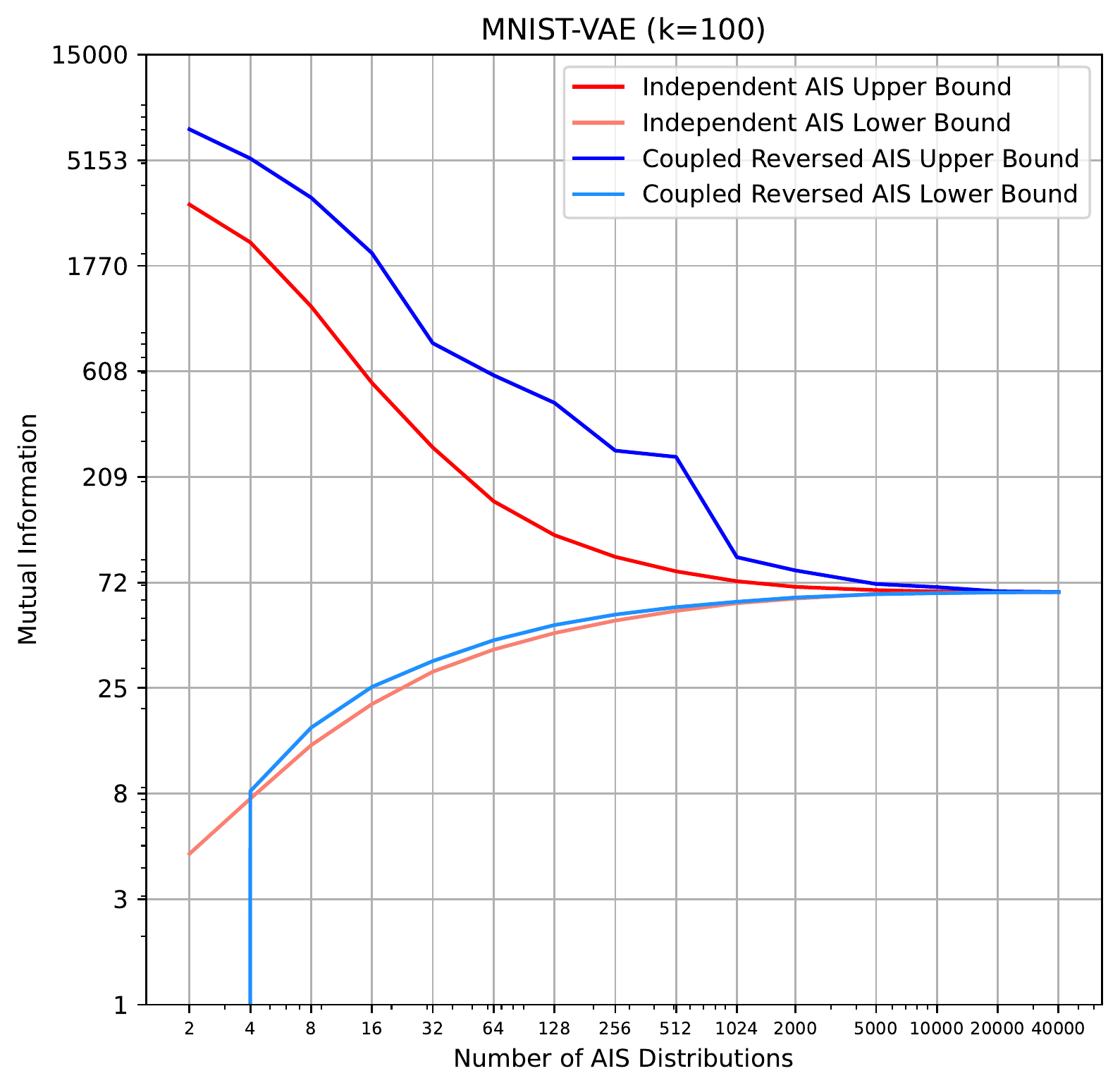}}
\hspace{.1cm}
\subcaptionbox{VAE ($K=1000$)}{
\includegraphics[scale=.27]{fig/fig-multiple-ais-iclr/VAE_1000.pdf}}
\subcaptionbox{GAN ($K=1$)}{
\includegraphics[scale=.27]{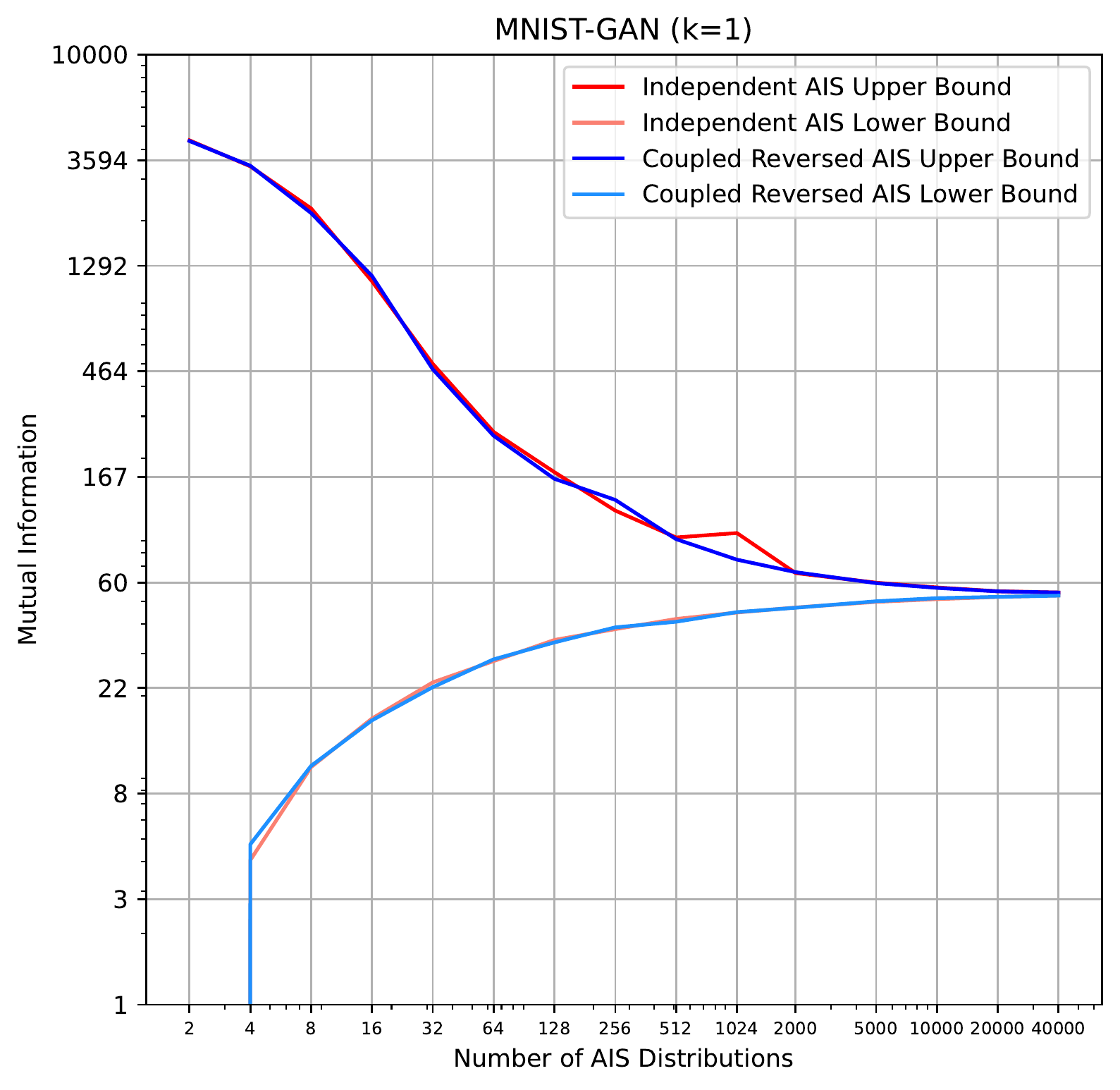}}
\hspace{.1cm}
\subcaptionbox{GAN ($K=100$)}{
\includegraphics[scale=.27]{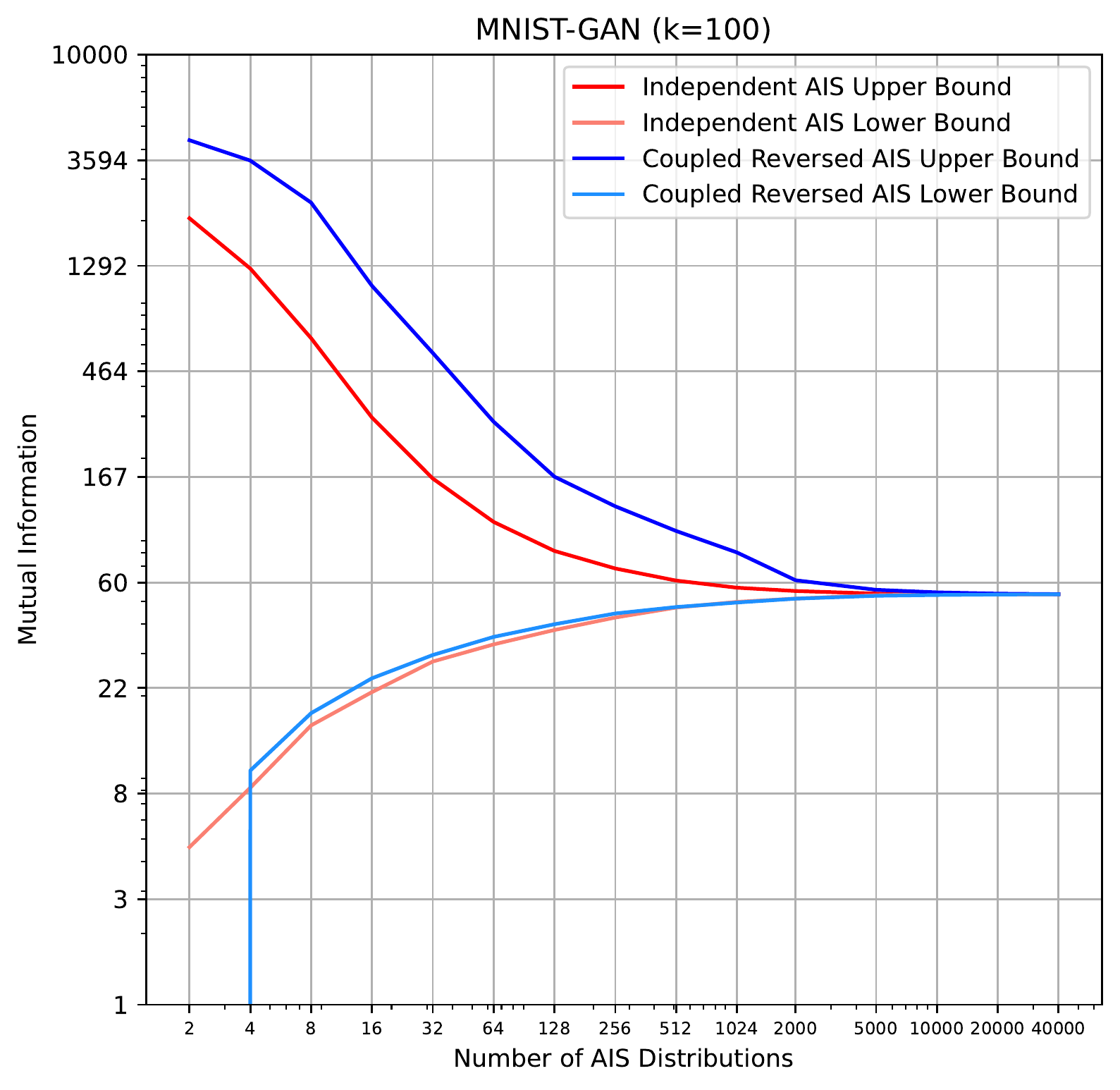}}
 \hspace{.1cm}
\subcaptionbox{GAN ($K=1000$)}{
\includegraphics[scale=.27]{fig/fig-multiple-ais-iclr/GAN_1000.pdf}}
\caption{\label{fig:multiple-ais-plot}
Comparing Multi-Sample AIS sandwich bounds and evaluating the effect of $K$ and $T$ for \gls{MI} estimation in deep generative models.}
\end{figure}
\section{Generalized Mutual Information Neural Estimation (\textsc{gmine})}\label{app:mine}
In this section, we provide probabilistic interpretations for the \gls{MINE} lower bounds on \gls{MI} from \citet{belghazi2018mutual}, which allows us to derive novel \textit{Generalized} \textsc{mine} bounds.   In a similar spirit to \textsc{giwae}, Generalized \gls{MINE} uses a base variational distribution $\qzx$ to tighten the \textsc{mine-dv} or \textsc{mine-f} bound and can be evaluated when $p(\vz)$ is available.     

See \cref{fig:gmine_bounds} for a summary of Generalized \gls{MINE} bounds and their relationships.   We discuss probabilistic interpretations in this section, and provide complementary interpretations in terms of conjugate dual representations of the \kl divergence in \myapp{conjugate_duality}.

We begin by deriving a probabilistic interpretation for the \gls{IBAL} lower bound in \mysec{mine-ais}, which we will show is closely related to the probabilistic interpretations of \gls{MINE}.   Our \textsc{mine-ais} method is designed to optimize and evaluate the \gls{IBAL} lower bound on \gls{MI}, with detailed discussion in \mysec{mine-ais} and \myapp{mine_ais}.

\subsection{Probabilistic Interpretation of IBAL}\label{app:mine-ais_deriv}
For an energy-based posterior approximation 
\begin{align}
    \pi_{\theta,\phi}(\vz|\vx) \coloneqq \frac{1}{\mathcal{Z}_{\theta,\phi}(\vx)} \qzx e^{\giwaeT(\vx,\vz)}\,, \quad \text{where} \quad \mathcal{Z}_{\theta,\phi}(\vx) = \mathbb{E}_{\qzx} \left[ e^{\giwaeT(\vx, \vz)} \right] \,, \label{eq:energy_pi2}
\end{align}
we consider the \gls{BA} lower bound on \gls{MI}, 
\begin{align}
I_{\gls{BA}_L}(\pi_{\theta,\phi}) &= I(\vx,\vz)- \Exp{p(\vx)}{\DKL[\pzx\|\minevar_{\theta,\phi}(\vz|\vx)]} \label{eq:ibal_gap} \\
&= \Exp{p(\vx,\vz)}{\log \frac{p(\vx,\vz)}{p(\vx)p(\vz)}} - \Exp{p(\vx,\vz)}{\log p(\vz|\vx)} + \Exp{p(\vx,\vz)}{T(\vx,\vz)} - \Exp{p(\vx,\vz)}{\log \mathcal{Z}_{\theta, \phi}} \nonumber \\
&= \Exp{p(\vx,\vz)}{\log \frac{\qzx}{p(\vz)}} + \Exp{p(\vx,\vz)}{\giwaeT(\vx,\vz)} - \Exp{p(\vx)}{\log \mathcal{Z}_{\theta, \phi}(\vx)} \nonumber \\
&= \myunderbrace{\Exp{p(\vx,\vz)}{ \log \frac{\qzx}{p(\vz)}} \vphantom{ \Exp{p(\vx,\vz)}{\vphantom{\frac{1}{2}}  \log \frac{e^{\giwaeT(\vx,\vz)}}{\Exp{\qzx}{e^{\giwaeT(\vx,\vz)}}}} } }{I_{\gls{BA}_L}(q_\theta)} + \myunderbrace{\Exp{p(\vx,\vz)}{\vphantom{\frac{1}{2}}  \log \frac{e^{\giwaeT(\vx,\vz)}}{\Exp{\qzx}{e^{\giwaeT(\vx,\vz)}}}}}{\text{contrastive term}} \nonumber \\[1.5ex]
&\eqqcolon \gls{IBAL}(q_{\theta}, \giwaeT),  \nonumber
\end{align}
where the gap in the mutual information bound is $\mathbb{E}_{p(\vx)} \big[ \DKL[\pzx\|\minevar_{\theta,\phi}(\vz|\vx)] \big]$.   Note that $\gls{IBAL}(q_{\theta}, \giwaeT)$ includes $I_{\gls{BA}_L}(q_\theta)$ as one of its terms, where we refer to $\qzx$ as the base variational distribution.   We visualize relationships between various energy-based bounds in \myfiga{different-bounds}{b} and \myfig{gmine_bounds}.

\mineaisopt*
See \myapp{mine-ais-properties} for the proof.

\begin{figure}[t]
    \centering
    \includegraphics[width=\textwidth]{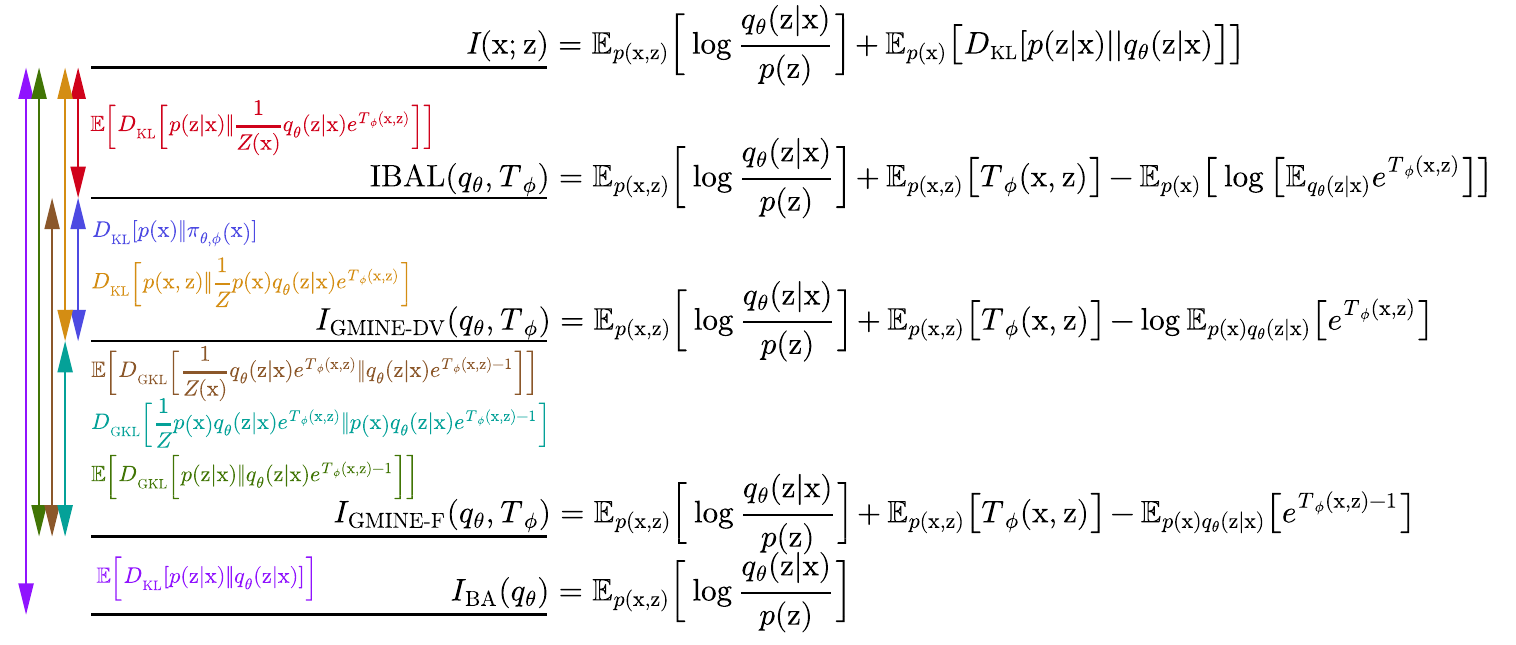}
    \caption{Generalized Energy Based Bounds.   Arrows indicate the gaps in each \gls{MI} lower bound or its relationship to other lower bounds. $D_\textsc{gkl}[\cdot\|\cdot]$ represents the \emph{generalized} KL divergence between two unnormalized densities (see \myapp{mine_f_prob}). All bounds are written in terms of a base variational distribution $\qzx$, which may be chosen to be the marginal $p(\vz)$ as in \textsc{mine-dv} and \textsc{mine-f}.}
    \label{fig:gmine_bounds}
\end{figure}

\subsection{Probabilistic Interpretation of Generalized \textsc{mine-dv}}\label{app:mine_dv_prob}
We can interpret the \textsc{mine-dv} bound of \citet{belghazi2018mutual} as arising from an energy-based variational approximation $\pi_{\theta,\phi}(\vx,\vz)$ of the \textit{full joint} distribution $p(\vx,\vz)$ as follows
\begin{align}
    \pi_{\theta,\phi}(\vx,\vz) \coloneqq \frac{1}{\mathcal{Z}} p(\vx) \qzx e^{\Txz}\,, \quad \text{where} \quad \mathcal{Z} = \mathbb{E}_{p(\vx)\qzx} \left[ e^{T_{\minetparam}(\vx, \vz)} \right] \,, \label{eq:gmine_dv}
\end{align}
$\qzx$ is a base variational distribution, and $T_{\minetparam}$ is a critic or negative energy function. 

Note that for the induced marginal 
$\pi_{\theta,\phi}(\vx) \coloneqq \int \pi_{\theta,\phi}(\vx,\vz) d\vz \neq p(\vx)$ due to the contribution of the critic function.   Instead, we have
\begin{align}
    \pi_{\theta,\phi}(\vx) = \frac{1}{\mathcal{Z}} p(\vx) \mathcal{Z}(\vx) , \label{eq:pi_marginal}
\end{align}
where $\mathcal{Z}(\vx)=\int q(\vz|\vx)e^{\giwaeT(\vx,\vz)} d\vz$ and $\mathcal{Z}=\int p(\vx)q(\vz|\vx)e^{\giwaeT(\vx,\vz)} d\vx d\vz = \mathbb{E}_{p(\vx)}[\mathcal{Z}(\vx)]$.

Subtracting the joint \kl divergence $\DKL[p(\vx,\vz) \| \pi_{\theta, \minetparam}(\vx,\vz)]$ from $\Ixz$, we obtain the \textit{Generalized} \textsc{mine-dv} lower bound on \gls{MI}
\begin{align}
   \hspace*{-.3cm} \Ixz &\geq \Ixz - \DKL \Big[p(\vx,\vz) \Big\| \frac{1}{\mathcal{Z}} p(\vx) \qzx e^{\Txz} \Big] \label{eq:gen_minedv_gap} \\[1.5ex]
    &= \Exp{p(\vx,\vz)}{\log \frac{\qzx}{p(\vz)}} + \Exp{p(\vx,\vz)}{\giwaeT(\vx,\vz)} - \log \Exp{p(\vx)\qzx}{ e^{T_{\minetparam}(\vx, \vz)} } \, \label{eq:gen_minedv} \\
    &\eqqcolon I_{\textsc{gmine-dv}}(q_{\theta}, T_\phi) \nonumber .
\end{align}
By construction, \cref{eq:gen_minedv_gap} shows that the gap in Generalized \textsc{mine-dv} is $\DKL[p(\vx,\vz) \| \frac{1}{\mathcal{Z}} p(\vx) \qzx e^{\Txz}]$, which has a probabilistic interpretation in terms of the approximate joint distribution $\pi_{\theta, \minetparam}(\vx,\vz)$.
\paragraph{Relationship with \textsc{mine-dv}}  
For $\qzx = p(\vz)$, we obtain the \textsc{mine-dv} bound \citep{belghazi2018mutual, poole2019variational} as a special case with $I_{\textsc{mine-dv}}(\giwaeT)=I_{\textsc{gmine-dv}}(\pz, \giwaeT)$.   In particular, the joint base distribution in \cref{eq:gmine_dv} corresponds to the product of marginals $p(\vx)p(\vz)$.
Our probabilistic interpretation shows that the gap in \textsc{mine-dv} corresponds to $\DKL[p(\vx,\vz) \| \frac{1}{\mathcal{Z}} p(\vx) p(\vz) e^{\Txz}]$.

We expect that the Generalized \gls{MINE}, with a learned variational distribution $\qzx$, to obtain tighter bounds than \textsc{mine-dv}.   We can guarantee that  $I_\textsc{gmine-dv}(q^*_{\theta}, T_{\phi}) \geq I_\textsc{mine-dv}(T_{\phi})$ for the optimal $q_{\theta}^*$ with a given $\giwaeT$, so long as $p(\vz)$ is in the variational family (i.e., $\exists \, \theta_0$ such that $q_{\theta_0}(\vz|\vx) = p(\vz)$).

\paragraph{Relationship with BA Bound}  Choosing a constant critic function $T_{\phi_0}(\vx,\vz)=\text{const}$, we can see that 
$\pi_{\theta,\phi_0}(\vx,\vz)= p(\vx) \qzx$ and $I_{\textsc{gmine-dv}}(q_{\theta}, T_{\phi_0}=\text{const}) = I_{\textsc{BA}_L}(q_{\theta})$ for a given $\qzx$.

\paragraph{Optimal Critic Function} For a given $\qzx$, the optimal critic function of Generalized \textsc{mine-dv} corresponds to the log importance weight between the target $p(\vx,\vz)$ and the joint base distribution $p(\vx)\qzx$, plus a constant
\begin{align}
    T^{*}(\vx,\vz) = \log \frac{p(\vx,\vz)}{p(\vx)\qzx} + c \, . \label{eq:optimal_minedv}
\end{align}
For the optimal critic function associated with a given $\qzx$, we have $ I_{\textsc{mine-dv}}(q_\theta,T^{*}) = \Ixz$.

\paragraph{Relationship with IBAL} 
We can use our probabilistic interpretation to show that the \gls{IBAL} is tighter than Generalized \textsc{mine-dv}, with $\textsc{ibal}(q_\theta,\giwaeT) \geq I_\textsc{gmine-dv}(q_\theta,\giwaeT)$.   Subtracting the gaps in the bounds in \cref{eq:ibal_gap} and \cref{eq:gen_minedv_gap}, 
\begin{align}
\textsc{ibal}(q_\theta,\giwaeT) &= I_\textsc{gmine-dv}(q_\theta,\giwaeT) + \underbrace{\DKL[p(\vx,\vz) \| \pi_{\theta, \minetparam}(\vx,\vz)]}_{\text{gap in $\textsc{gmine-dv}(q_\theta,\giwaeT)$ } } - \underbrace{\Exp{p(\vx)}{\DKL[\pzx\|\minevar_{\theta,\phi}(\vz|\vx)]}}_{\text{gap in $\textsc{ibal}(q_\theta,\giwaeT)$ } } \nonumber \\
&= I_\textsc{gmine-dv}(q_\theta,\giwaeT) + \mathbb{E}_{p(\vx)}\Big[\log \frac{p(\vx)}{\minevar_{\theta,\phi}(\vx)} \Big] \\
&=  I_\textsc{gmine-dv}(q_\theta,\giwaeT) +  \KL{p(\vx)}{\pi_{\theta,\phi}(\vx)} \\
& \geq I_\textsc{gmine-dv}(q_\theta, \giwaeT).
\end{align}
We see that the difference between $\textsc{ibal}(q_\theta,\giwaeT)$ and $I_\textsc{gmine-dv}(q_\theta, \giwaeT)$ corresponds to a marginal \kl divergence in the $\vx$ space,
where $\pi_{\theta,\phi}(\vx)$ is defined in \cref{eq:pi_marginal}.   

 As in \citet{poole2019variational}, we can also interpret the gap between the $\textsc{ibal}(q_\theta,\giwaeT)$ and $I_\textsc{gmine-dv}(q_\theta,\giwaeT)$ as an application of Jensen's inequality, with 
 \begin{align}
     \KL{p(\vx)}{\pi_{\theta,\phi}(\vx)} &= \DKL \Big[ p(\vx) \Big\| \frac{1}{\mathcal{Z}} p(\vx)\mathcal{Z}(\vx) \Big] \\
     &=\log \mathcal{Z} - \mathbb{E}_{p(\vx)}[ \log \mathcal{Z}(\vx)]\\
     &= \log \mathbb{E}_{\px}[\mathcal{Z}(\vx)] - \mathbb{E}_{p(\vx)}[ \log \mathcal{Z}(\vx)],
 \end{align}
since
$\log \mathbb{E}_{\px}[\mathcal{Z}(\vx)] \geq \mathbb{E}_{p(\vx)}[ \log \mathcal{Z}(\vx)] $.
\subsection{Probabilistic Interpretation of Generalized \textsc{mine-f}}\label{app:mine_f_prob}
The \gls{BA} lower bound and its corresponding \gls{EUBO} upper bound on $\log p(\vx)$ are derived using a \textsc{kl} divergence between the true and approximate posterior, which are normalized conditional distributions over $\vz$ given $\vx$.   In this section, we interpret the \textsc{mine-f} bound \citep{belghazi2018mutual} as arising from a generalized notion of the \textsc{kl} divergence between possibly \textit{unnormalized} density functions.   In particular, our probabilistic interpretation will involve an unnormalized $\tqzx$ which seeks to approximate the true (normalized) posterior $p(\vz|\vx)$. %

\parhead{Generalized \textsc{KL} Divergence (GKL)} First, we state the definition of the \emph{generalized} \textsc{kl} divergence, which takes unnormalized measures as input arguments and corresponds to the limiting behavior of the $\alpha$-divergence (\citet{ cichocki2010families})%
\begin{align}
    \GKL[\tilde{r}(\vz) \|  \tilde{s}(\vz)] = \int \tilde{r}(\vz) \log \frac{\tilde{r}(\vz)}{\tilde{s}(\vz)} d\vz - \int \tilde{r}(\vz)  d\vz + \int \tilde{s}(\vz)  d\vz \, .
\end{align}
As in the case of the standard \textsc{kl} divergence, this quantity is nonnegative, convex in either argument, and vanishes for $\tilde{r}(\vz) = \tilde{s}(\vz)$.    It is also a member of both the family of Bregman divergences and $f$-divergences \citep{amari2009alpha}.  If $r(\vz)$ and $s(\vz)$ are normalized, then $\GKL[r(\vz) \| s(\vz)] = \KL{r(\vz)}{s(\vz)}$.  Finally, if $r(\vz)$ is normalized and $\tilde{s}(\vz)$ is unnormalized, one can easily confirm that
\begin{align} 
\GKL[r(\vz) \| \tilde{s}(\vz)] &= \KL{r(\vz)}{s(\vz)}+\GKL[s(\vz) \| \tilde{s}(\vz)]  \label{eq:pythagorean} \\
&\geq \KL{r(\vz)}{s(\vz)}  \, . \nonumber
\end{align}

\parhead{Generalized \textsc{EUBO} (GEUBO)}
Since the Generalized \textsc{kl} divergence is always nonnegative, we define a Generalized \gls{EUBO} by adding the Generalized \textsc{kl} divergence between the normalized true posterior and an unnormalized approximate posterior $\tqzxnoparam$.
\begin{align}
    \textsc{geubo}(\vx;\tqzxshort) &\coloneqq  \log p(\vx) + \GKL[p(\vz|\vx) \|  \tilde{\pi}(\vz|\vx)] \\
    &= \Exp{p(\vz|\vx)}{\log \frac{p(\vx,\vz)}{\tqzxnoparam}} - 1 + \normq,
\end{align}
where we define $\normq \coloneqq \int \tqzxnoparam d\vz$. 
In general, we have $\textsc{eubo}(\vx; \pi) \leq \textsc{geubo}(\vx;\tqzxshort)$ where $\pi$ is the normalized distribution of $\tilde{\pi}$, with equality if $\tqzxshort$ is normalized.

\parhead{Generalized \textsc{BA} (GBA)}
Using the Generalized \textsc{EUBO} in place of $\log p(\vx)$, 
we obtain the following lower bound on \gls{MI}, which we denote as the Generalized \gls{BA} lower bound
\begin{align}
    \Ixz &\geq \Ixz - \Exp{p(\vx)}{\GKL[p(\vz|\vx) \| \tqzxnoparam]} \nonumber \\
    &= \Exp{p(\vx,\vz)}{\log \frac{\tqzxnoparam}{p(\vz)}} + 1 - \Exp{p(\vx)}{\normq} \label{eq:ba_unnorm}\\
    &\eqqcolon I_{\textsc{gba}_L}(\tilde{\pi}). \nonumber
\end{align}

\parhead{Generalized MINE-F (GMINE-F)}
We now consider an unnormalized, energy-based approximation to the true posterior, involving a base variational distribution $\qzx$ and a learned critic function $\Txz$
\begin{align}
    \tqzx \coloneqq \baseminef e^{\Txz - 1} \,. \label{eq:unnorm_energy}
\end{align}
Using this unnormalized approximate posterior in the \textsc{gba} lower bound in \cref{eq:ba_unnorm}, we obtain the \textit{Generalized} \textsc{mine-f} lower bound
\small
\begin{align}
    \hspace*{-.2cm}  \Ixz &\geq \Ixz - \Exp{p(\vx)}{\GKL[p(\vz|\vx) \| \tqzx]} \label{eq:gap_minef} \\
    &= \Exp{p(\vx,\vz)}{\log \frac{\baseminef}{p(\vz)}} +  \Exp{p(\vx,\vz)}{\Txz} - \Exp{p(\vx)\baseminef}{ e^{\Txz-1}}  \label{eq:gen_minef}\\
    &\eqqcolon I_{\textsc{gmine-f}}(q_{\theta}, T_{\phi}).
\end{align}
\normalsize

By construction, we can see that the gap in $I_{\textsc{gmine-f}}(q_{\theta}, T_{\phi})$ is equal to the Generalized \textsc{kl} divergence $\Exp{p(\vx)}{\GKL[p(\vz|\vx) \|  \tqzx]}$ in \cref{eq:gap_minef}.  
As in the case of \textsc{mine-dv}, we obtain the standard \textsc{mine-f} lower bound $I_{\textsc{mine-f}}(T_{\phi}) = I_{\textsc{gmine-f}}(p(\vz), T_{\phi})$ when using the marginal $p(\vz)$ as the proposal.

\paragraph{Optimal Critic Function} The optimal critic function of Generalized \textsc{mine-f} is $T^{*}(\vx,\vz) = 1 +\log \frac{p(\vx,\vz)}{p(\vx)q(\vz|\vx)}$ \citep{poole2019variational}.   In this case, we obtain $I_{\textsc{gmine-f}}(q_{\theta}, T^{*}) = \Ixz$ in \cref{eq:gen_minef}.

\paragraph{Relationship with IBAL}%
We would now like 
to relate the gap in $I_{\textsc{gmine-f}}(q_{\theta}, T_{\phi})$ to the gap in $I_{\gls{IBAL}}(q_{\theta}, T_{\phi})$. 
First, note that the normalized distribution corresponding to $\tqzx = \qzx e^{\Txz -1}$ matches the the energy-based posterior in the \gls{IBAL}, $\tqzxnorm = \frac{1}{\mathcal{Z}(\vx)} \qzx e^{\Txz}$.
Using \cref{eq:pythagorean}, we have
\small
\begin{align}
\Exp{p(\vx)}{\GKL[p(\vz|\vx) \|  \tqzx]} = \, \mathbb{E}_{p(\vx)}&\left[ \DKL[p(\vz|\vx)\|\tqzxnorm] \right] + \Exp{p(\vx)}{\GKL[\tqzxnorm \| \tqzx]}  \nonumber .
\end{align}
\normalsize
and substituting in the definition of each term
\footnotesize
\begin{align}
\underbrace{\Exp{p(\vx)}{\GKL\Big[p(\vz|\vx) \Big\|  \qzx e^{\Txz -1}\Big] \vphantom{\frac{1}{\mathcal{Z}(\vx)}}}}
_{\text{\small  gap of \textsc{gmine-f}}} &=
\underbrace{\Exp{p(\vx)}{\DKL \Big[p(\vz|\vx)\Big\|\frac{1}{\mathcal{Z}(\vx)}\qzx e^{\Txz} \Big]}}
_{\text{\small  gap of \gls{IBAL}}}  \\
&\phantom{=}+\underbrace{\Exp{p(\vx)}{\GKL\Big[ \frac{1}{\mathcal{Z}(\vx)}\qzx e^{\Txz} \Big\| \qzx e^{\Txz -1}\Big]}}
_{\text{\small $\geq 0$}}, \nonumber
\end{align}
\normalsize
where the final term  $\GKL[\tqzxnorm \| \tqzx] \geq 0 $ is nonnegative because it is a generalized \kl divergence. 
Thus, we have
\begin{align}
    \textsc{ibal}(q_\theta,\giwaeT) &= I_\textsc{gmine-f}(q_\theta,\giwaeT) + \Exp{p(\vx)}{\GKL[\tqzxnorm \| \tqzx]}   \geq I_\textsc{gmine-f}(q_\theta, \giwaeT) . 
\end{align}
Alternatively, \citet{poole2019variational} use the inequality $\log u \leq u -1$ to show that
\begin{align}
\GKL[\tqzxnorm \| \tqzx] = \Exp{p(\vx)}{\mathcal{Z}(\vx)- 1 - \log \mathcal{Z}(\vx)} \geq 0\,. \nonumber
\end{align}

We visualize the relationship between \gls{IBAL} and Generalized \textsc{mine-f} in \cref{fig:gmine_bounds}.

\paragraph{Relationship with Generalized MINE-DV}
To characterize the gap between Generalized \textsc{mine-dv} and Generalized \textsc{mine-f}, we can again use the equality from \cref{eq:pythagorean}, but this time using divergences over joint distributions.  
\begin{align}
\GKL \Big[p(\vx,\vz) \Big\| \tilde{\pi}_{\theta,\phi}(\vx,\vz) \Big] &= \DKL \Big[p(\vx,\vz)\Big\|\pi_{\theta,\phi}(\vx,\vz)\Big] + \GKL\Big[\pi_{\theta,\phi}(\vx,\vz) \Big\| \tilde{\pi}_{\theta,\phi}(\vx,\vz) \Big] \nonumber \\
\underbrace{\GKL\Big[p(\vx,\vz) \Big\| p(\vx)\qzx e^{\Txz-1}\Big] \vphantom{ \KL{p(\vx,\vz)}{\frac{1}{\mathcal{Z}}p(\vx)\qzx e^{\Txz}}} }
_{\text{\small gap in \textsc{gmine-f}}}  &= 
\underbrace{\DKL \Big[ p(\vx,\vz) \Big\| \frac{1}{\mathcal{Z}}p(\vx)\qzx e^{\Txz}\Big]}
_{\text{\small gap in \textsc{gmine-dv}}} \\
&\phantom{=}+\underbrace{\GKL \Big[\frac{1}{\mathcal{Z}}p(\vx)\qzx e^{\Txz} \Big\| p(\vx)\qzx e^{\Txz-1}\Big]}_{\text{\small $\geq 0$}} . \nonumber
\end{align}
In this case, we have $\GKL[\frac{1}{\mathcal{Z}}p(\vx)\qzx e^{\Txz} \| p(\vx)\qzx e^{\Txz-1}] \geq 0$.   Thus, Generalized \textsc{mine-dv} is tighter than Generalized \textsc{mine-f} (see \cref{fig:gmine_bounds}), which generalizes the finding in \citet{poole2019variational}
that standard \textsc{mine-dv} is tighter than standard \textsc{mine-f}.

\newcommand{\Lone}{L^{1}(X,dx)}
\newcommand{\Linf}{L^{\infty}(X,dx)}
\newcommand{\DEKL}{D_{\textsc{ekl}}}

\section{Conjugate Duality Interpretations}\label{app:conjugate_duality}
In this section, we interpret the energy-based \gls{MI} lower bounds in \textsc{mine-ais}, Generalized \textsc{mine-dv}, Generalized \textsc{mine-f}, \gls{GIWAE}, \gls{IWAE}, and \textsc{InfoNCE} from the perspective of conjugate duality.   In particular, we highlight that the critic or negative energy function in the above bounds arises as a dual variable in the convex conjugate representation of the \textsc{kl} divergence.   
In all cases, the \textsc{kl} divergence of interest corresponds to the gap in the \gls{BA} lower bound
\begin{align}
    \Ixz = \Exp{p(\vx,\vz)}{\log \frac{\qzx}{p(\vz)}} + \Exp{p(\vx)}{\DKL[\pzx\|\qzx]} \,. \label{eq:ba_plus_gap}
\end{align}
For $\qzx = p(\vz)$, the \gls{BA} lower bound term is $0$ and our derivations correspond to taking dual representation of \gls{MI} directly, e.g. $\Exp{p(\vx)}{\DKL[\pzx\|p(\vz)]}$, as in \citet{belghazi2018mutual}.

Our conjugate duality interpretations are complementary to our probabilistic interpretations, with either approach equally valid for deriving lower bounds and characterizing their gaps.  

\subsection{Convex Analysis Background}
Suppose $(X,\Sigma)$ is a measurable space, $B(\Sigma)$ is the space of bounded measurable functions $u:X\to \mathbb{R}$, and $\mathcal{M}(X),\mathcal{M}_+(X),\mathcal{P}(X)$ are the space of finite signed measures, finite positive measures, and probability measures, respectively. Consider the following dual pairing between $\mathcal{M}(X)$ and $B(\Sigma)$
\begin{equation}
    \left\langle \pi, u\right\rangle \coloneqq \int_{X}u d\pi\,, 
    \qquad
    \forall \pi\in \mathcal{M}(X) \,, \quad \forall u\in B(\Sigma) \,.
\end{equation}
The conjugate function of $\Omega:\mathcal{M}(X) \to (-\infty,+\infty]$, denoted by $\Omega^*:B(\Sigma) \to (-\infty,+\infty]$, and the biconjugate function, denoted by $\Omega^{**}:\mathcal{M}(X) \to (-\infty,+\infty]$, are obtained by
\begin{align}
    \Omega^*\left(u \right) &\coloneqq \sup_{\pi \in \mathcal{M}(X)}\left\{ \left\langle \pi , u \right\rangle -\Omega(\pi) \right\}\,, 
    \qquad
    \forall  u \in B(\Sigma) \,, \label{eq:conjugate1} \\
    \Omega^{**}\left(\pi \right) &\coloneqq \sup_{u \in B(\Sigma)}\left\{ \left\langle \pi , u \right\rangle -\Omega^*(u) \right\}\,, 
    \qquad
    \forall \pi \in \mathcal{M}(X) \,. \label{eq:biconjugate}
\end{align}
If $\Omega$ is a convex and lower-semicontinuous function, we have $\Omega=\Omega^{**}$. 
The pair of $(\pi,u)$ are in dual correspondence iff $u \in \partial \Omega(\pi)$, where $\partial$ denotes the subdifferential mapping. Note that the duality mapping between $(\pi,u)$ could generally be a multi-valued mapping.   
However, it can be shown that $(\pi,u)$ are in dual correspondence or $u \in \partial \Omega(\pi)$ iff  $\Omega(\pi) + \Omega^*(u)  = \langle \pi, u \rangle$ (see \cref{eq:fy_gap} below).

Suppose $\pi_1,\pi_2 \in \mathcal{M}(X)$, and $(\pi_1,u_1)$ are in dual correspondence. Using the subdifferential $u_1$, we can define the Bregman divergence generated by a convex function $\Omega$ as follows
\begin{align}
D_{\Omega}^{u_{1}}\left(\pi_{2},\pi_{1}\right)&=\Omega(\pi_2) -\Omega(\pi_1) -\left\langle \pi_{2}-\pi_{1}, u_{1} \right\rangle \,.
\end{align}
This Bregman divergence can be used to characterize the gap of Fenchel-Young inequality
\begin{align}
\Omega\left(\pi_{2}\right)+\Omega^*\left(u_{1}\right)&=\left\langle \pi_{2}, u_{1}\right\rangle +D_{\Omega}^{u_{1}}\left(\pi_{2},\pi_{1}\right) \label{eq:fy_gap} \\
&\geq \left\langle \pi_{2}, u_{1}\right\rangle \,, 
\qquad \qquad \qquad \qquad
\forall \pi_1,\pi_2 \in \mathcal{M}(X),\forall u_1\in\partial \Omega\left(\pi_1\right)\,. \nonumber 
\end{align}
The Fenchel-Young inequality is tight iff $(\pi_2,u_1)$ are in dual correspondence.

\subsubsection{Dual of KL Divergence}\label{app:dual_kl_norm}
The \textsc{kl} divergence from a reference measure $m\in \mathcal{P}(X)$ is denoted by $\Omega_m=\DKL[\cdot \| m]: \mathcal{M}(X) \to(-\infty,+\infty]$ and defined as
\begin{align}\label{eq:relative_entropy}
\Omega^*_m(\pi)=\DKL[\pi \| m] & \coloneqq\begin{cases}
\left\langle \pi, \log\frac{d\pi}{dm} \right\rangle  & \ensuremath{ \pi \ll m},\,\pi \in \mathcal{P}(X) \,, \\ 
+\infty & \text{otherwise} \,.
\end{cases}
\end{align}
The conjugate function of the \textsc{kl} divergence, $\Omega^*_m : B(\Sigma) \to(-\infty,+\infty]$ is 
\begin{align}\label{eq:conj_logsumexp}
\Omega^*_m(u) = \log \int e^{u}dm\,.
\end{align}

See \citet[Lemma 6.2.13]{DZ10} or \citet[Theorem 5.4]{RS15} for proof.
Furthermore, the pair of $(\pi, u)$ are in dual correspondence if  
\begin{align}\label{eq:duality_logsumexp}
\pi=\frac{m e^{u}}{\int e^{u}dm} \quad	\Longleftrightarrow \quad u=\log\frac{d\pi}{dm}+c \,,
\end{align}
where $c$ is an arbitrary constant.

Suppose $\pi_1,\pi_2 \in \mathcal{M}(X)$, and that $(\pi_1,u_1)$ and $(\pi_2,u_2)$ are in dual correspondence. Then, the Bregman divergence generated by the \kl divergence is the \kl divergence:
\begin{align}\label{eq:bregman_logsumexp}
D_{\DKL[\cdot \| m]}^{u_1}\left(\pi_{2},\pi_{1}\right) &= \left\langle \pi_{2}, \log\frac{d\pi_{2}}{d\pi_{1}} \right\rangle \\
&= \DKL[\pi_2 \| \pi_1] \,.
\end{align}
This Bregman divergence characterizes the gap of Fenchel-Young inequality as follows
\begin{align}\label{eq:fenchel_logsumexp}
\DKL[\pi_2 \| m] &=\left\langle \pi_{2}, u_{1} \right\rangle -\log \int e^{u_{1}}dm +\DKL[\pi_2 \| \pi_1] \nonumber  \\
&\geq \left\langle  \pi_{2}, u_{1} \right\rangle -\log \int e^{u_{1}}dm \,.
\end{align}

\subsubsection{Dual of Generalized KL Divergence}
The Generalized \kl (\textsc{gkl}) divergence from a reference measure $m\in \mathcal{M}_+(X)$, is denoted by $\Omega_m=\GKL[\cdot \| m]: \mathcal{M}(X) \to(-\infty,+\infty]$, and defined as
\begin{align}\label{eq:extended_relative_entropy}
\Omega_m(\pi)=\GKL[\pi \| m] & \coloneqq\begin{cases}
\left\langle \pi, \log\frac{d\pi}{dm}  \right\rangle - \pi(X) + m(X) & \ensuremath{\pi \ll m}, \pi \in \mathcal{M}_+(X) \,, \\
+\infty & \text{otherwise} \,.
\end{cases}
\end{align}
Note that unlike the \kl divergence, the \textsc{gkl} divergence could output finite values for non-probability measures.  However, for probability measures, it is equal to the KL divergence.

The conjugate function of the Generalized \kl divergence, $\Omega^*_m : B(\Sigma) \to(-\infty,+\infty]$ is
\begin{align}\label{eq:conj_sumexp}
\Omega^*_m(u) = \int e^{u}dm -m(X) \,.
\end{align}

See \citet[Theorem 7.24]{polyanskiy2022information} for proof.
Furthermore, the pair of $(\pi,u)$ are in dual correspondence if  
\begin{align}\label{eq:duality_sumexp}
d\pi= e^{u}dm	\quad	\Longleftrightarrow \quad u=\log\frac{d\pi}{dm} \,.
\end{align}

Suppose $\pi_1,\pi_2 \in \mathcal{M}(X)$, and that $(\pi_1,u_1)$ and $(\pi_2,u_2)$ are in dual correspondence. Then, the Bregman divergence generated by the \textsc{gkl} divergence is the \textsc{gkl} divergence,
\begin{align}\label{eq:bregman_sumexp}
D_{\GKL[\cdot \| m]}^{u_1}\left(\pi_{2},\pi_{1}\right) &= \left\langle \pi_{2}, \log\frac{d\pi_{2}}{d\pi_{1}}\right\rangle - \pi_2(X)+\pi_1(X)\\
&= \GKL[\pi_2 \| \pi_1] \,.
\end{align}
This Bregman divergence characterizes the gap of Fenchel-Young inequality as follows
\begin{align}\label{eq:fenchel_sumexp}
\GKL[\pi_2 \| m] &=\left\langle \pi_{2}, u_{1} \right\rangle -\int e^{u_{1}}dm + m(X) +\GKL[\pi_2 \| \pi_1] \nonumber \\
&\geq \left\langle \pi_{2}, u_{1}\right\rangle  -\int e^{u_{1}}dm + m(X) \,.
\end{align}
\newcommand{\Tpi}{T_{\pi}}
\newcommand{\Ttpi}{T_{\tpi}}
\newcommand{\piT}{\pi_{T}}
\newcommand{\Topt}{T_p}
\newcommand{\tpiT}{\tpi_{T}}

\newcommand{\pii}{\pi(\vz|\vx)}
\newcommand{\qzz}{\qzx}
\newcommand{\tpii}{\tpi(\vz|\vx)}
\newcommand{\tqzz}{\tilde{q}_\theta(\vz|\vx)}
\newcommand{\TT}{T(\vx,\vz)}
\newcommand{\logratiodom}{B(\Sigma)}%

\subsection{Conjugate Duality Interpretation of IBAL}\label{app:mine_ais_dual}
To obtain an alternative derivation of $\ibal(q_\theta, \giwaeT)$, 
we consider the
\textit{conditional} \textsc{kl} divergence function from a reference probability density $\qzx$ (with respect to the Lebesgue measure),
\small 
\begin{align}
\Omega(\pi) = \DKL[\pii \| \qzz]  \label{eq:cond_kl}\,,
\end{align}
\normalsize
which is a convex, lower semi-continuous function of $\pi$.  To derive its conjugate $\Omega^{*}(T)$, note that our optimization is implicitly restricted to normalized distributions by the definition in \cref{eq:relative_entropy}
\begin{align}
\Omega^*(T) &\coloneqq \sup \limits_{\pii} \int \pii  \TT d\mathbf{z} - \Omega(\pi)  
= \log \int \qzz e^{\TT} d\mathbf{z}
\label{eq:cond_kl_conj}\,.
\end{align}
As in \cref{eq:duality_logsumexp}, the dual correspondence between $\pi_T$ and $T_{\pi}$ is given by
\begin{align}
    \pi_T(\vz|\vx) = \frac{1}{\mathcal{Z}(\vx;T)} \qzx e^{\TT} \quad \Longleftrightarrow   \quad  T_{\pi}(\vx,\vz) \coloneqq \log \frac{\pii}{\qzz} + c(\vx) \,,
\end{align}
where $c(\vx)$ is an arbitrary constant function in $\vx$.  

We leverage this duality to estimate the \kl divergence $\Omega(\pzx) = \DKL[\pzx\|\qzx]$
from $\qzx$ to the true posterior $\pzx$.   In particular, plugging into \cref{eq:biconjugate} for the 
(convex, lower-semicontinuous) 
\kl divergence suggests the following variational representation
\begin{align}
        \hspace*{-.25cm} 
        {\vphantom{\frac{1}{2}} \DKL[\pzx\|\qzx]} &=
        \sup \limits_{\TT} \int p(\vz|\vx)  T(\vx,\vz) d\mathbf{z}  - \log \int \qzz e^{\TT} d\mathbf{z} \,. \label{eq:ibal_dual_rep}
\end{align}
\normalsize
For a suboptimal $T_\pi(\vx,\vz)$, which is in dual correspondence with $\pi_T(\vz|\vx)$ instead of the desired posterior $p(\vz|\vx)$, we can use \cref{eq:fenchel_logsumexp} to obtain a lower bound on  $\DKL[\pzx\|\qzx]$.   
To characterize the gap in this inequality, one can confirm using \cref{eq:bregman_logsumexp} that the Bregman divergence generated by the \kl divergence 
is also the \kl divergence $D_{\DKL[\cdot \| q]}[ p , \pi]=\DKL[p\|\pi]$. Thus, we have
\small
\begin{align}
 \hspace*{-.25cm} 
 \DKL[\pzx\|\qzx] = 
\int p(\vz|\vx)\, T(\vx,\vz) d\vz 
-  
\log \mathbb{E}_{\qzz}\left[e^{\TT}\right]
+ \DKL[\pzx\|\pi_T(\vz|\vx)] \label{eq:mine-ais-vrep}\,.
\end{align}
\normalsize
Finally, the $\gls{IBAL}(q_{\theta}, \giwaeT)$ uses this variational representation 
of the gap in the \gls{BA} lower bound, $\Exp{p(\vx)}{\DKL[p(\vz|\vx)\| \qzx]}$, to obtain a tighter bound on \gls{MI}.
In particular, for any learned critic function $T(\vx,\vz)$, we can use \cref{eq:mine-ais-vrep} to derive the \gls{IBAL} and its gap,
\scriptsize
\begin{align}
       \hspace*{-.2cm} \Ixz &= \underbrace{\Exp{p(\vx,\vz)}{\log \frac{\qzx}{p(\vz)}}}_{I_{\gls{BA}_{\textsc{l}}}(q_{\theta})} + \Exp{p(\vx)}{\DKL[p(\vz|\vx)\| \qzx]}\\[1.5ex]
        &= \underbrace{\Exp{p(\vx,\vz)}{\log \frac{\qzx}{p(\vz)}} + \Exp{p(\vx)}{\Exp{p(\vz|\vx)}{ T(\vx,\vz)} - \log \Exp{\qzx}{e^{T(\vx,\vz)}}}}_{\gls{IBAL}(q_{\theta}, \giwaeT)} + \Exp{p(\vx)}{\DKL[\pzx\|\pi_T(\vz|\vx)]} \,.
        \nonumber
\end{align}
\normalsize
The optimal critic function $T^*(\vx,\vz)$ provides the maximizing argument in \cref{eq:ibal_dual_rep} and is in dual correspondence with the true posterior $\pzx$.  In particular, we have $\pi_{T^*}(\vz|\vx) = \pzx$, resulting in $\Ixz=\gls{IBAL}(q_{\theta}, T^*)$.

\subsection{Conjugate Duality Interpretation of Generalized MINE-DV}\label{app:mine_dv_dual}
To obtain a conjugate duality interpretation of (Generalized) \textsc{mine-dv}, we consider the dual representation of the \textsc{kl} divergence over \textit{joint} distributions.  Choosing $\pi_0(\vx,\vz) = p(\vx) \qzx$ as the reference density, the \kl divergence is a convex function of the first argument
\begin{align}
\Omega(\pi) = \DKL[ \pi(\vx, \vz) \| p(\vx) \qzx] \,.
\end{align}   
This \kl divergence matches the gap in the \gls{BA} bound $\mathbb{E}_{p(\vx)}[\DKL[p(\vz|\vx)\| \qzx]]=\DKL[p(\vx,\vz)\| p(\vx)\qzx]$ after noting the marginal distribution of both $p(\vx,\vz)$ and $p(\vx)\qzx$ is $p(\vx)$. 
However, the duality associated with $\Omega(\cdot) = \DKL[\cdot\| \pi_0(\vx,\vz)]$ holds for arbitrary joint distributions, and we will see that using this divergence leads to looser bound on \gls{MI} than in \myapp{mine_ais_dual}.

As in \cref{eq:conj_logsumexp}, the conjugate function is 
\begin{align}
\Omega^*(T) &\coloneqq \sup \limits_{\pi(\vx,\vz)} \int \pi(\vx, \vz) T(\vx,\vz) d\vx d\vz - \Omega(\pi(\vx,\vz))  = \log 
\mathbb{E}_{p(\vx)\qzx}\left[ e^{T(\vx,\vz)} \right] \,,
\label{eq:joint_kl_conj} 
\end{align}
\normalsize
where we have used $p(\vx)\qzx$ as the reference density in the definition \cref{eq:relative_entropy}.
Note that the expectation over $p(\vx)$ now appears inside the $\log$ in $\Omega^*(T)$, compared with the conjugate for the conditional \kl divergence in \cref{eq:cond_kl_conj}.
We have the following dual correspondence
\begin{align}
    \pi_T(\vx,\vz) &\coloneqq \frac{1}{\mathcal{Z}(T)} p(\vx) \qzx e^{T(\vx,\vz)} \quad \Longleftrightarrow   \quad T_\pi(\vx,\vz) = \log \frac{\pi(\vx,\vz)}{p(\vx)\qzx} + c\,,
\end{align}
where $c$ is an arbitrary constant. 

Finally, we use \myeq{biconjugate} to write the dual representation of joint \kl divergence as
\begin{align}
   \hspace*{-.3cm} \DKL[p(\vx,\vz)\|p(\vx)\qzx] = \sup \limits_{T(\vx,\vz)} \int p(\vx, \vz) T(\vx,\vz) d\vx d\vz  - \log \mathbb{E}_{p(\vx)\qzx}\left[ e^{T(\vx,\vz)}\right] \label{eq:dual_of_kl}.
\end{align}
As in the previous section and \cref{eq:fenchel_logsumexp},  a suboptimal $T(\vx,\vz)$, which is in dual correspondence with  $\pi_T(\vz|\vx)$ instead of the desired posterior $p(\vz|\vx)$, yields a lower bound on  $\DKL[p(\vx,\vz)\|p(\vx)\qzx]$, with the gap equal to the Bregman divergence (or \kl divergence)
\footnotesize
\begin{align}
\hspace*{-.2cm} \DKL[p(\vx,\vz)\|p(\vx)\qzx] = \int p(\vx, \vz) T(\vx,\vz) d\vx d\vz - \log  \mathbb{E}_{p(\vx)\qzx}\left[e^{T(\vx,\vz)}\right] + \DKL[p(\vx,\vz) \| \pi_T(\vx,\vz) ]\,. \label{eq:minedv_vrep}
\end{align}
\normalsize
The Generalized \textsc{mine-dv} bound $I_{\textsc{gmine-dv}}(q_{\theta}, T_{\phi})$ uses this variational representation of the gap in the \gls{BA} lower bound to obtain a tighter bound on \gls{MI}.   For a learned critic $T(\vx,\vz)$, we can use \cref{eq:minedv_vrep} to write
\scriptsize
\begin{align}
       \hspace*{-.2cm} \Ixz &= \underbrace{\Exp{p(\vx,\vz)}{\log \frac{\qzx}{p(\vz)}}}_{I_{\gls{BA}_{\textsc{l}}}(q_{\theta})} + \Exp{p(\vx)}{\DKL[p(\vz|\vx)\| \qzx]}\\[1.5ex]
        &= \underbrace{\Exp{p(\vx,\vz)}{\log \frac{\qzx}{p(\vz)}} + \Exp{p(\vx)p(\vz|\vx)}{ T(\vx,\vz)} - \log \Exp{p(\vx)\qzx}{e^{ T(\vx,\vz)}}}_{I_{\textsc{gmine-dv}}(q_{\theta}, T_{\phi})} + {\DKL[p(\vx) \pzx\|\pi_T(\vx,\vz)]} 
        \nonumber
\end{align}
\normalsize
As noted in \myapp{mine_dv_prob}, the Generalized \textsc{mine-dv} bound is looser than the \gls{IBAL}.

\subsection{Conjugate Duality Interpretation of Generalized MINE-F}\label{app:mine_f_dual}
\citet{nguyen2010estimating} consider the conjugate duality associated with the family of $f$-divergences, of which the \kl divergence is a special case.   
We will show that this dual representation corresponds to taking the conjugate function of the Generalized \kl divergence, which can take unnormalized densities as input (see the definition in \cref{eq:extended_relative_entropy}),
\begin{align}
\Omega(\tpi) = \GKL[ \tpii \| \tqzz] .
\end{align}
To calculate the conjugate, we do not need to restrict to normalized distributions and obtain, as in \cref{eq:conj_sumexp},
\begin{align}
\Omega^{*}(T^{\prime})&\coloneqq \sup_{\tilde{\pi}(\vz|\vx)} \int \tilde{\pi}(\vz|\vx) T^{\prime}(\vx,\vz) d\vz - \GKL[\tilde{\pi}(\vz|\vx)\| \tilde{q}_\theta(\vz|\vx)] \nonumber \\
&=  \int \tilde{q}_\theta(\vz|\vx) e^{T^{\prime}(\vx,\vz)} d\vz - \int \tilde{q}_\theta(\vz|\vx) d\vz  \label{eq:dual_minef}
\end{align} 
Solving for the optimizing argument or writing the dual correspondence in \cref{eq:duality_sumexp}, we obtain
\begin{align}
    \tpi_{T^\prime}(\vz|\vx) =  \tilde{q}_\theta(\vz|\vx) e^{T^{\prime}(\vx,\vz)} \quad \Longleftrightarrow   \quad T^{\prime}_{\tilde{\pi}}(\vx,\vz) = \log \frac{ \tilde{\pi}(\vz|\vx)}{\tilde{q}_\theta(\vz|\vx)} . \label{eq:minef_primal_dual}
\end{align}
The dual representation of the Generalized \kl divergence now matches \citet{nguyen2010estimating}, %
\small  
\begin{align}
  \GKL[\tilde{p}(\vz|\vx)\| \tilde{q}(\vz|\vx)] = \sup \limits_{T^{\prime}(\vx,\vz)} \int \tilde{p}(\vz|\vx) T^{\prime}(\vx,\vz) d\vz - \int \tilde{q}(\vz|\vx) e^{T^{\prime}(\vx,\vz)} d\vz + \int \tilde{q}(\vz|\vx) d\vz \,. \label{eq:gkl_dual}
\end{align}
\normalsize
We now consider the reparameterization $T^{\prime}(\vx,\vz) = T(\vx,\vz)- 1$. Assuming a normalized $\tilde{q}(\vz|\vx)=q(\vz|\vx)$ and $\tilde{p}(\vz|\vx)=p(\vz|\vx)$, and noting that $\GKL[p(\vz|\vx)\| q_{\theta}(\vz|\vx)]=\DKL[p(\vz|\vx)\| q_{\theta}(\vz|\vx)]$ for normalized distributions, we obtain
\begin{align}
    \DKL[p(\vz|\vx)\| q_{\theta}(\vz|\vx)] = \sup \limits_{T(\vx,\vz)} \int p(\vz|\vx) T(\vx,\vz) d\vz -  \int \qzx e^{T(\vx,\vz) - 1} d\vz  \,, \label{eq:gkl_dual_reparam}
\end{align}
which matches dual representation of the \kl divergence found in \citet{belghazi2018mutual, nowozin2016f}.  

For a suboptimal $T_\pi(\vx,\vz)$ which is in dual correspondence with $\pi_T(\vz|\vx)$ instead of the desired posterior $p(\vz|\vx)$, we can use \cref{eq:fenchel_sumexp} to obtain a lower bound on  $\GKL[\pzx\|\qzx]=\DKL[p(\vz|\vx)\| q_{\theta}(\vz|\vx)]$.   
To characterize the gap of this lower bound, note from \cref{eq:bregman_sumexp} that the Bregman divergence generated by the Generalized \kl divergence is also the Generalized \kl divergence $D_{\GKL[\cdot \| \tilde{q}]}[ \tilde{p} \| \tpi]=\GKL[\tilde{p}\|\tpi]$.   Thus, we have
\small
\begin{align}
   \hspace*{-.2cm} \DKL[p(\vz|\vx)\| q_{\theta}(\vz|\vx)] = \int p(\vz|\vx) \, T(\vx,\vz) d\vz - \int \qzx e^{T(\vx,\vz) - 1} d\vz + \GKL[\pzx\|\qzx e^{T(\vx,\vz)-1}] \,.  \label{eq:dual_opt_mine_f}
\end{align}
\normalsize
Finally, we obtain the  Generalized \textsc{mine-f} bound by evaluating the dual representation of the Generalized \kl divergence for normalized $p(\vz|\vx)$ and $q_\theta(\vz|\vx)$.
In particular, $I_{\textsc{gmine-f}}(q_{\theta}, T_{\phi})$ uses the critic function $T_{\phi}$ to tighten the gap in $I_{\gls{BA}_{\textsc{l}}}(q_{\theta})$ via the dual optimization in \cref{eq:dual_opt_mine_f},
\scriptsize
\begin{align}
        \Ixz &= \underbrace{\Exp{p(\vx,\vz)}{\log \frac{\qzx}{p(\vz)}}}_{I_{\gls{BA}_{\textsc{l}}}(q_{\theta})} + \Exp{p(\vx)}{\DKL[p(\vz|\vx)\| \qzx]}\\[1.5ex]
        &= \underbrace{\Exp{p(\vx,\vz)}{\log \frac{\qzx}{p(\vz)}} + \Exp{p(\vx)}{\Exp{p(\vz|\vx)}{ T(\vx,\vz)} - \Exp{\qzx}{e^{T(\vx,\vz)-1}}}}_{I_{\textsc{gmine-f}}(q_{\theta}, T_{\phi})} + \Exp{p(\vx)}{\GKL[\pzx\|\tpi_{T-1}(\vz|\vx)]}\,.
\end{align}
\normalsize
The optimal critic function is the dual variable of the true posterior $\pzx$, which can be found using \cref{eq:minef_primal_dual} as $T(\vx,\vz) = 1 + \log \frac{{p}(\vz|\vx)}{{q}_{\theta}(\vz|\vx)}$. With this optimal critic, the Generalized \textsc{mine-f} bound is tight.

\newcommand{\mst}{\mathcal{T}} 
\newcommand{\TK}{\mst(\vx,\vz^{(1:K)}, \idx)} 
\newcommand{\mstlim}{\mathcal{T}_{\textsc{giwae}}}
\newcommand{\TKlim}{\mstlim(\vx,\vz^{(1:K)}, \idx)}
\newcommand{\TKlimK}{\mstlim(\vx,\vz^{(1:K)}, \idx;T_\phi,K)}
\newcommand{\pks}{ p_{\textsc{tgt}}( \vz^{(1:K)}, \idx |\vx) }
\newcommand{\qks}{ q_{\textsc{prop}}( \vz^{(1:K)}, \idx |\vx) }

\subsection{Conjugate Duality Interpretation of GIWAE, IWAE, and InfoNCE}\label{app:giwae_dual}
In this section, we use conjugate duality to derive the \gls{GIWAE}, \gls{IWAE}, and \textsc{Info-NCE} bounds on mutual information and characterize their gaps.  Our approach extends that of \citet{poole2019variational}, where the \textsc{mine-f} dual representation (\myapp{mine_f_dual}) 
was used to derive \textsc{Info-NCE}.    
We provide an alternative derivation using the dual representation associated with the conditional \kl divergence and \gls{IBAL} in \myapp{mine_ais_dual}.
For either dual representation, \gls{GIWAE}, \gls{IWAE}, and \textsc{Info-NCE} arise from limiting the family of the critic functions $\giwaeT$ in order to 
eliminate the intractable log partition function term.

We start from the decomposition of $\Ixz$ into $I_{\gls{BA}_L}(q_\theta)$ and its gap
\begin{align}
\Ixz = \underbrace{\Exp{p(\vx,\vz)}{\log \frac{\qzx}{p(\vz)}}}_{I_{\gls{BA}_L}(q_\theta)} + \Exp{p(\vx)}{\DKL[\pzx\|\qzx]} \, . \label{eq:mi_as_ba_infonce}
\end{align} 
We will focus on the dual representation of the normalized, conditional \kl divergence 
$\Omega(\cdot) = \DKL[\cdot \| \qzx]$, as 
in \myapp{mine_ais_dual}.  

\paragraph{Multi-Sample IBAL}
Consider extending the state space of the posterior or target distribution $\pzx$, by using an additional $K-1$ samples 
from a base variational distribution $\qzx$ to construct 
\begin{align}
    \pks &\coloneqq \udist(s) p(\vz^{(s)}|\vx) \prod \limits_{\myoverset{k\neq s}{k=1}}^K q_{\theta}(\vz^{(k)}|\vx) , 
\end{align}
where $\idx$ is an index variable $\idx \sim \udist(s)$ drawn uniformly at random from $1 \leq s \leq K$, which specifies the index of the posterior sample $\pzx$.
We similarly expand the state space of the base variational distribution to write
\begin{align}
    \qks &\coloneqq \udist(s) \prod \limits_{k=1}^K q_{\theta}(\vz^{(k)}|\vx) .
\end{align}
It can be easily verified that this construction does not change the value of the \kl divergence 
\small
\begin{align}
\DKL[\pks\|\qks]
&=\mathbb{E}_{\pks}
\left[\log 
\frac{\pks}{\qks}
\right] \nonumber \\
&=\mathbb{E}_{\pks} 
\left[\log \frac{\cancel{\udist(\idx)}p(\vz^{(\idx)}|\vx)\prod_{k\neq\idx}q_{\theta}(\vz^{(k)}\vert \vx)}{\cancel{\mathcal{U}(\idx)} \prod \limits_{k=1}^K q_{\theta}(\vz^{(k)}|\vx)}
\right] \nonumber \\
&=\mathbb{E}_{\udist(\idx)p(
\vz^{(\idx)}|\vx)}\left[\log \frac{p(\vz^{(\idx)}|\vx)}{q_{\theta}(\vz^{(\idx)} \vert \vx)}\right] \nonumber \\
&=\mathbb{E}_{\pzx}\left[\log \frac{\pzx}{\qzx}\right] \nonumber \\
&= \DKL[\pzx\|\qzx] \nonumber \,.
\end{align}
\normalsize
Consider the convex function 
\begin{align}
   \Omega(\cdot) \coloneqq \DKL\big[ \cdot \| \qks \big],
\end{align}
where the primal variable is a distribution in the extended-state space of $\left( \vz^{(1:K)}, s\right)$, and the dual variable is a critic function $\TK$. 
We derive a conjugate optimization in similar fashion to \myapp{mine_ais_dual}, but now over the extended state space.  
For this $\Omega(\cdot)$, the conjugate function $\Omega^{*}(\mst)$  takes a log-mean-exp form analogous to \cref{eq:conj_logsumexp} \footnote{We consider the conjugate with restriction to normalized distributions as in \myapp{dual_kl_norm}.}. 
We can write the variational representation of $\Omega(p_{\textsc{tgt}})$ as 
\footnotesize
\begin{align}  
 \hspace*{-.2cm}
 \Omega(p_{\textsc{tgt}}) =  \sup \limits_{\mst} \int \sum \limits_{s=1}^K 
\pks 
 \TK d\vz^{(1:K)}
 -  \log \Exp{\qks}{e^{\TK}}.
\label{eq:ksample_dual}
\end{align}
\normalsize
For a particular choice of $\TK$, we can use \cref{eq:ksample_dual} to obtain a lower bound on the \kl divergence $\DKL[\pzx\|\qzx]$ (as in \cref{eq:fenchel_logsumexp}).  This lower bound translates to the 
\textit{Multi-Sample \gls{IBAL}} lower bound on \gls{MI}, $I_{\textsc{ms-ibal}}(q_\theta, \mst)$ via \cref{eq:mi_as_ba_infonce} . 
\scriptsize
\begin{align}
    \Ixz &\geq \Exp{p(\vx,\vz)}{\log \frac{\qzx}{p(\vz)}} + \mathbb{E}_{p(\vx)}\Bigg[\mathbb{E}_{\pks}\Big[ \TK \Big] - \log \underbrace{\mathbb{E}_{\qks}\Big[ e^{\TK} \Big]}_{\mathcal{Z}(\vx; \mst)} \Bigg]\nonumber \\[1.5ex]
    &\eqqcolon \text{\small $I_{\textsc{ms-ibal}}(q_\theta, \mst)$},
\end{align}
\normalsize
where $\mathcal{Z}(\vx; \mst)$ is the normalization constant of the dual distribution $\pi_{\mst}(\vz^{(1:K)}, \idx|\vx)$ corresponding to $\mst$,
\begin{align}
   \pi_{\mst}(\vz^{(1:K)}, \idx|\vx) = \frac{1}{\mathcal{Z}(\vx; \mst) }\qks e^{\TK} . 
\end{align}
As in \cref{eq:fenchel_logsumexp}, we can write the gap of $I_{\textsc{ms-ibal}}(q_\theta, \mst)$ as a Bregman divergence or \kl divergence 
\begin{align}
    \Ixz = I_{\textsc{ms-ibal}}(q_\theta, \mst) + \mathbb{E}_{p(\vx)}\Big[ \DKL\big[\pks \|  \pi_{\mst}(\vz^{(1:K)}, \idx|\vx)\big] \Big] \, . \label{eq:gap_in_ext_ibal}
\end{align}
The optimal $K$-sample energy function in \cref{eq:ksample_dual} should 
result in $\pi_{\mst^*}(\vz^{(1:K)}, \idx|\vx) = \pks$.   Using similar reasoning as in 
\myapp{pf_giwae} or \myapp{mine-ais-properties}, this occurs for $\mst^{*}(\vx,\vz^{(1:K)}, \idx) = \log \frac{p(\vz^{(\idx)}|\vx)}{q_\theta(\vz^{(\idx})|\vx)} + c(\vx)$, for which we have $\Ixz = I_{\textsc{ms-ibal}}(q_\theta, \mst^*)$.

\paragraph{GIWAE is a Multi-Sample IBAL with a Restricted Function Family} 
Although extending the state space did not change the value of the \kl divergence or alter the optimal critic function, it does allow us to consider a restricted class of multi-sample energy functions that yield tractable, low variance estimators.  
In particular, \gls{GIWAE} and \textsc{Info-NCE} arise from choosing a restricted family of functions $\TKlim$, under which the problematic $\log \mathcal{Z}(\vx;\mst)$ term evaluates to $0$.
This function family is defined as
\vspace{-.2cm}
\begin{align}
    \TKlim &\coloneqq \log \frac{e^{\giwaeT(\vx,\vz^{(\idx)})}}{\frac{1}{K} \sum \limits_{k=1}^K  e^{\giwaeT(\vx,\vz^{(k)})}} \label{eq:restricted_t} ,
\end{align}
where $\TKlim$ is specified by an arbitrary single-sample critic function $\giwaeT(\vx,\vz)$. 
We can now see that $\log \mathcal{Z}(\vx;\mstlim)=0$,
\scriptsize
\begin{align}
\log \mathcal{Z}(\vx;\mstlim) = {\log \Exp{\qks}{e^{\TKlim}} } 
&=  
    \log \mathbb{E}_{\udist(s) \prod \limits_{k=1}^K q_{\theta}(\vz^{(k)}|\vx)}\left[{ \frac{e^{\giwaeT(\vx,\vz^{(\idx)})}}{\frac{1}{K} \sum \limits_{k=1}^K  e^{\giwaeT(\vx,\vz^{(k)})}}}\right] %
    =0\,, \nonumber 
\end{align}
\normalsize
using the fact that $\vz^{(1:K)}\sim \prod \limits_{k=1}^K q_{\theta}(\vz^{(k)}|\vx)$ is invariant to re-indexing.
With this simplification, $I_{\textsc{ms-ibal}}(q_\theta, \mstlim)$ recovers the \gls{GIWAE} lower bound on \gls{MI}
\begin{align}
    \Ixz &\geq I_{\textsc{ms-ibal}}(q_\theta, \mstlim) \\
    &= \Exp{p(\vx,\vz)}{\log \frac{\qzx}{p(\vz)}} +  \mathbb{E}_{\frac{1}{K} p(\vx)p(\vz^{(\idx)}|\vx) \prod \limits_{\myoverset{k\neq s}{k=1}}^K q_{\theta}(\vz^{(k)}|\vx)} \left[\log  \frac{e^{\giwaeT(\vx,\vz^{(\idx)})}}{\frac{1}{K} \sum \limits_{k=1}^K  e^{\giwaeT(\vx,\vz^{(k)})}} \right] \\
    &= \Exp{p(\vx,\vz)}{\log \frac{\qzx}{p(\vz)}} +  \mathbb{E}_{ p(\vx)p(\vz^{(1)}|\vx) \prod \limits_{{k=2}}^K q_{\theta}(\vz^{(k)}|\vx)} \left[\log  \frac{e^{\giwaeT(\vx,\vz^{(1)})}}{\frac{1}{K} \sum \limits_{k=1}^K  e^{\giwaeT(\vx,\vz^{(k)})}} \right] \\
    &=  I_{\gls{GIWAE}}(q_\theta, \giwaeT, K) . \nonumber
\end{align}

Finally, we can use \cref{eq:gap_in_ext_ibal} to recover the probabilistic interpretation of \gls{GIWAE} and the gap in the lower bound on \gls{MI}.  
As we saw above, the dual distribution $\pi_{\mstlim}(\vz^{(1:K)},s|\vx)$ is normalized with $\mathcal{Z}(\vx; \mstlim)=1$.  In particular, we can write 
\begin{align}
   \pi_{\mstlim}(\vz^{(1:K)}, \idx|\vx) &= \frac{1}{\mathcal{Z}(\vx; \mstlim) }\qks e^{\TKlim} \\
&=\frac{1}{\mathcal{Z}(\vx; \mstlim) }  \cancel{\udist(s)} \prod \limits_{k=1}^K q_{\theta}(\vz^{(k)}|\vx) \frac{e^{\giwaeT(\vx,\vz^{(\idx)})}}{\cancel{ \frac{1}{K}} \sum \limits_{k=1}^K  e^{\giwaeT(\vx,\vz^{(k)})}} \\
&= \prod \limits_{k=1}^K q_{\theta}(\vz^{(k)}|\vx) \frac{e^{\giwaeT(\vx,\vz^{(\idx)})}}{\sum \limits_{k=1}^K  e^{\giwaeT(\vx,\vz^{(k)})}}
\end{align}
which recovers $\qprop{giwae}(\vz^{(1:K)}, \idx|\vx)$ from the probabilistic interpretation of \gls{GIWAE} in \cref{eq:giwae_forward_app} or \myapp{giwae_prob}.
We can write the gap of $I_{\gls{GIWAE}}(q_\theta, \giwaeT, K) = I_{\textsc{ms-ibal}}(q_\theta, \mst_{\textsc{giwae}})$ as the Bregman divergence or \kl divergence as in \cref{eq:fenchel_logsumexp}
\begin{align}
    \Ixz = I_{\textsc{ms-ibal}}(q_\theta, \mst_{\textsc{giwae}}) + \Exp{p(\vx)}{\DKL\Big[\pks \|  \pi_{\mstlim}(\vz^{(1:K)}, \idx|\vx)\Big]} \, , \label{eq:gap_in_ext_ibal2}
\end{align}
which matches the reverse \kl divergence $\Exp{p(\vx)}{\DKL[\ptgt{giwae}(\vz^{(1:K)},s|\vx)\| \qprop{giwae}(\vz^{(1:K)},s|\vx)]}$ derived from the probabilistic approach in \myapp{giwae_prob}.
Recall that \textsc{Info-NCE} is a special case of \gls{GIWAE} with $\qzx = p(\vz)$ (\mysec{gen_iwae}).

\paragraph{Conjugate Duality Interpretation of \gls{IWAE}} 
We can gain alternative perspective on \gls{IWAE} from this conjugate duality interpretation.   In particular, \gls{IWAE} is a special case of \gls{GIWAE}, where the optimal single-sample critic function $T^{*}(\vx,\vz) = \log \frac{p(\vx,\vz)}{\qzx} + c(\vx)$  (see \mysec{gen_iwae}) is used in \cref{eq:restricted_t} to construct the optimal multi-sample function $\mathcal{T}_{\textsc{iwae}}$ from within the \textsc{giwae} restricted multi-sample function family. Thus, we have $I_{\textsc{ms-ibal}}(q_\theta, \mathcal{T}_{\textsc{iwae}})= I_{\gls{IWAE}}(q_\theta, K)$.

Although \gls{IWAE} uses the optimal critic function, the restriction to the function family in \cref{eq:restricted_t} is necessary to obtain a tractable bound on the 
\kl divergence $\DKL[\pzx\|\qzx]$ and mutual information. 
Without the restricted function family, the intractable log partition term in \cref{eq:ksample_dual} would require \gls{MCMC} methods such as \gls{AIS} for accurate estimation, as we saw for the single-sample \gls{IBAL} in \mysec{mine-ais}.  
\section{Properties of the IBAL}\label{app:ibal}
Our \textsc{mine-ais} method in \mysec{mine-ais} and \myapp{mine_ais} optimizes the \textit{Implicit Barber Agakov} lower bound (\gls{IBAL}) on \gls{MI} from \cref{eq:mine-ais}.   We first recall the probabilistic interpretation of the \gls{IBAL} bound from \myapp{mine-ais_deriv}.
For a posterior approximation with a learned negative energy function $\giwaeT$ and base variational distribution $\qzx$,
\begin{align}
    \pi_{\theta,\phi}(\vz|\vx) = \frac{1}{\mathcal{Z}_{\theta,\phi}(\vx)} \qzx e^{\giwaeT(\vx,\vz)}, \quad \text{where} \quad \mathcal{Z}_{\theta,\phi}(\vx) = \mathbb{E}_{\qzx} \left[ e^{\giwaeT(\vx, \vz)} \right] \, , \label{eq:energy_pi2}
\end{align}
we consider the \gls{BA} lower bound on \gls{MI}, 
\small
\begin{align}
I(\vx,\vz) &\geq  I_{\gls{BA}_L}(\pi_{\theta,\phi}) \\
&= I(\vx,\vz)- \Exp{p(\vx)}{\DKL[\pzx\|\minevar_{\theta,\phi}(\vz|\vx)]} \\
&= \myunderbrace{\Exp{p(\vx,\vz)}{ \log \frac{\qzx}{p(\vz)}} \vphantom{ \Exp{p(\vx,\vz)}{\vphantom{\frac{1}{2}}  \log \frac{e^{\giwaeT(\vx,\vz)}}{\Exp{\qzx}{e^{\giwaeT(\vx,\vz)}}}} } }{\text{\small $I_{\gls{BA}_L}(q_\theta)$}} + \myunderbrace{\Exp{p(\vx,\vz)}{\vphantom{\frac{1}{2}}  \log \frac{e^{\giwaeT(\vx,\vz)}}{\Exp{\qzx}{e^{\giwaeT(\vx,\vz)}}}}}{\text{\small contrastive term}} \label{eq:ibal_appendix_expression} \\
&\eqqcolon \gls{IBAL}(q_{\theta}, \giwaeT) .
\end{align}
\normalsize
The gap of this lower bound on mutual information is $\mathbb{E}_{p(\vx)} \big[ \DKL[\pzx\|\minevar_{\theta,\phi}(\vz|\vx)] \big]$, as in \mysec{ba}.   

\subsection{Proofs for IBAL Optimal Critic Function (\myprop{mine-ais-optimal} and \ref{prop:mineaisproperties})}\label{app:mine-ais-properties}
\mineaisopt*
\begin{proof}  Recall 
that the gap in $\IBAL(q_\theta, T_{\phi^{*}})$ is $\mathbb{E}_{p(\vx)} \big[ \DKL[\pzx\|\minevar_{\theta,\phi}(\vz|\vx)] \big] = \Ixz - \IBAL(q_\theta, T_{\phi^{*}})$.   This implies that the bound will be tight iff $p(\vz|\vx) = \minevar_{\theta,\phi}(\vz|\vx) \propto p(\vx,\vz)$.   We can easily show that the true log importance weights (plus a constant)  satisfy this property 
\begin{align}
    \minevar_{\theta,\phi}(\vz|\vx) = \frac{1}{\mathcal{Z}_{\theta,\phi}(\vx)} \qzx e^{\log \frac{p(\vx,\vz)}{\qzx} + c(\vx)}=  \frac{e^{c(\vx)}}{\mathcal{Z}_{\theta,\phi}(\vx)} p(\vx,\vz) = p(\vz|\vx) ,
\end{align}
where $e^{c(\vx)}$ is absorbed into the normalization constant.  
Conversely, using $g(\vx,\vz)$ which depends on $\vz$ in $T(\vx,\vz) = \log \frac{p(\vx,\vz)}{\qzx} + g(\vx, \vz)$ would change the density over $\vz$ to no longer match $p(\vz|\vx)$.
At this value, the \gls{IBAL} is exactly equal to $\Ixz$
\scriptsize
\begin{align*}
    \ibal(q_\theta,T^{*}) &= \Exp{p(\vx,\vz)}{\log \frac{\qzx}{p(\vz)}} + \Exp{p(\vx,\vz)}{\giwaeT(\vx,\vz)} - \Exp{p(\vx)}{\log \mathcal{Z}_{\theta, \phi}(\vx)} \\
    &=\Exp{p(\vx,\vz)}{\log \frac{\qzx}{p(\vz)}} + \Exp{p(\vx,\vz)}{\log \frac{p(\vx,\vz)}{\qzx} + \log c(\vx)} 
    - \Exp{p(\vx)}{\log \mathbb{E}_{\qzx} \frac{p(\vx,\vz)}{\qzx} + \log c(\vx)} \\
    &= \Exp{p(\vx,\vz)}{\log \frac{p(\vx,\vz)}{p(\vz)p(\vx)}} = \Ixz .
\end{align*}
\normalsize
\end{proof}

\begin{restatable}{proposition}{mineaisproperties}
\label{prop:mineaisproperties}
Suppose the critic function $T_{\phi}(\vx,\vz)$ is parameterized by $\phi$, and that $\exists \, \phi_0 \, \, s.t. \,    T_{\phi_0}(\vx,\vz) = \text{const}$. For a given $\qzx$,
let $T_{\phi^*}(\vx,\vz)$ denote the critic function that maximizes $\ibal(q_\theta,T_{\phi})$. %
Then,
\begin{align}
    I_{\textsc{BA}_L}(q_{\theta})  \leq \ibal(q_\theta,T_{\phi^*}) \leq \Ixz =  I_{\textsc{BA}_L}(q_{\theta})  + \Exp{\px}{\DKL[p(\vz|\vx)\|\qzx]}.
\end{align}
In particular, the contrastive term in \myeq{mine-ais} is upper bounded by $\Exp{\px}{\DKL[p(\vz|\vx)\|\qzx]}$.  
\end{restatable}
\begin{proof} 
The \gls{BA} bound is a special case of the $\ibal(q_\theta,T_{\phi})$ with constant $T_{\phi_0}=c$ and the contrastive term equal to $0$.  Since $\phi_0$ is a possible parameterization, we can only improve upon $I_{\textsc{BA}_L}(q_{\theta})$ by learning $\giwaeT$.   The parameterized family of $\giwaeT$ may not be expressive enough to match the true log importance weights, so  $\ibal(q_\theta,T_{\phi^{*}}) \leq \ibal(q_\theta,T^{*}) =\Ixz$ using \myprop{mine-ais-optimal}. 
\end{proof}

\subsection{Proof of IBAL as Limiting Behavior of the GIWAE Objective as $K\rightarrow \infty$ (\myprop{mine-ais-limit})}\label{app:mine-ais-limit-pf}
\begin{restatable}[\gls{IBAL} as Limiting Behavior of \gls{GIWAE}]{proposition}{mineaislimit}
\label{prop:mine-ais-limit}
For given $\qzx$ and $T_{\phi}(\vx,\vz)$, we have
\begin{align}
    \lim \limits_{K\rightarrow \infty} I_{\textsc{GIWAE}_L}( q_{\theta}, T_{\phi}, K) =  \ibal(q_\theta,T_{\phi}) \, .
\end{align}
\end{restatable}
\begin{proof}
Comparing the form of $I_{\gls{GIWAE}_L}(q_\theta, \giwaeT, K)$ to the $\gls{IBAL}(q_\theta, \giwaeT)$ for a fixed $q_\theta$ and $\giwaeT$, 
\small
\begin{align}
I_{\gls{GIWAE}}(q_{\theta}, \giwaeT, K)  &= \Exp{p(\vx,\vz)}{\log \frac{q(\vz|\vx)}{p(\vz)}} +
\Exp{p(\vx, \vz^{(1)})\qzxi{2:K}}{ \log 
\frac{e^{\giwaeT(\vx,\vz^{(1)})}}{\frac{1}{K} \sum_{i=1}^K e^{\giwaeT(\vx,\vz^{(k)})}}} \nonumber \\
\gls{IBAL}(q_{\theta}, \giwaeT) &= 
{\Exp{p(\vx,\vz)}{ \log \frac{\qzx}{p(\vz)}} \vphantom{ \Exp{p(\vx,\vz)}{\vphantom{\frac{1}{2}}  \log \frac{e^{\giwaeT(\vx,\vz)}}{\Exp{\qzx}{e^{\giwaeT(\vx,\vz)}}}} } }
+ 
{\Exp{p(\vx,\vz)}{\vphantom{\frac{1}{2}}  \log \frac{e^{\giwaeT(\vx,\vz)}}{\Exp{\qzx}{e^{\giwaeT(\vx,\vz)}}}}} .
\nonumber
\end{align}
\normalsize
we can see that  both bounds include the same first term $I_{\gls{BA}_L}(q_\theta)$ and the same numerator of the contrastive term. 
To prove the proposition, we can thus focus on characterizing the limiting behavior of the denominator
\begin{align}
\forall \vx : \, \, \,  \lim \limits_{K\rightarrow \infty} \underbrace{  \Exp{p(\vz^{(1)}|\vx) \prod \limits_{k=2}^K q_{\theta}(\vz^{(k)}|\vx)}{\log \frac{1}{K} \sum \limits_{k=1}^K e^{\giwaeT(\vx,\vz)} } }_{ \text{\normalsize $= \logZgiwae$ } } = \underbrace{\vphantom{\prod \limits_{k=2}^K q_{\theta}(\vz^{(k)}|\vx)} \log \Exp{\qzx \vphantom{\prod \limits_{k=2}^K}}{e^{\giwaeT(\vx,\vz)}} \vphantom{\left[ \log \frac{1}{K} \sum \limits_{k=1}^K e^{\giwaeT(\vx,\vz)}\right]}}_{\text{\normalsize $= \logZmine$}} , \, \,  \label{eq:giwae_to_mine_partition}
\end{align}
where we introduce the notation $\logZgiwae$ for convenience and the right hand side is the log partition function for the \gls{IBAL} energy-based posterior $\pi_{\theta,\phi}(\vz|\vx)$.   
Intuitively, we expect \myeq{giwae_to_mine_partition} to hold since the contribution of the single posterior sample $p(\vz|\vx)$ in the \gls{GIWAE} expectation will vanish as $K \rightarrow \infty$.  

More formally, we consider the sequence of values $\logZgiwae$ as a function of $K$.   
We derive lower and upper bounds on the value of $\logZgiwae$ for fixed $K$ and show that each of these sequences of lower and upper bounds converge to $\logZmine$ in the limit as $K \rightarrow \infty$.
Using the squeeze theorem for sequences, this is sufficient to demonstrate the claim that $\lim_{K\rightarrow \infty} \logZgiwae= \logZmine$ in \myeq{giwae_to_mine_partition}.   Since the other terms in $I_{\gls{GIWAE}_L}(q_\theta, \giwaeT, K)$ and $\gls{IBAL}(q_\theta, \giwaeT)$ are identical, this is sufficient to prove the proposition.

\textit{Lower Bound on $\logZgiwae$}:
We rely on the fact that the exponential function $e^{\giwaeT(\vx,\vz)} \geq 0$ to simply ignore the contribution of the $p(\vz|\vx)$ term.
\begin{align}
   \logZgiwae &= \Exp{p(\vz^{(1)}|\vx)\prod \limits_{k=2}^K q_{\theta}(\vz^{(k)}|\vx)}{\log \sum \limits_{k=1}^K e^{\giwaeT(\vx,\vz^{(k)})} }- \log K \\
   &\overset{(1)}{\geq} \Exp{\prod \limits_{k=2}^K q_{\theta}(\vz^{(k)}|\vx)}{\log \sum \limits_{k=2}^K e^{\giwaeT(\vx,\vz^{(k)})} } - \log K \\
    &= \Exp{\prod \limits_{k=2}^K q_{\theta}(\vz^{(k)}|\vx)}{\log \frac{1}{K-1} \sum \limits_{k=2}^K e^{\giwaeT(\vx,\vz^{(k)})}} + \log \frac{K-1}{K} \\
    &\eqqcolon \logZgiwaelb ,
\end{align}
where in $(1)$, we obtain a lower bound by ignoring the contribution of the positive sample from $p(\vz|\vx)$.
The first term 
is the $k$-sample \gls{IWAE} lower bound with the proposal $q_{\theta}(\vz)$ and target $\pi_{\theta,\phi}(\vx,\vz)$:  $\mathbb{E}_{q_{\theta}(\vz^{1:K})}[\log \frac{1}{K} \sum_k \frac{\pi_{\theta,\phi}(\vx,\vz^{(k)})}{q_{\theta}(\vz^{(k)}|\vx)} ]$, and thus converges to $\log \mathcal{Z}_{\theta,\phi}$ as $K\rightarrow \infty$.  As $K\rightarrow \infty$, we also have $\log \frac{K-1}{K} \rightarrow 0$, so that the limit of their sum is 
\begin{align}
   \lim \limits_{K \rightarrow \infty} \logZgiwaelb &= \log \Exp{q_{\theta}(\vz|\vx)}{e^{\giwaeT(\vx,\vz)}} = \logZmine .
\end{align}

\textit{Upper Bound on $\logZgiwae$}:
To upper bound $\logZgiwae$, we separately consider terms arising from $p(\vz|\vx)$ samples and $\qzx$ samples.   Noting that $\ptgt{giwae}(\vz^{(1:K)}, s|\vx)$ in \cref{eq:giwae_reverse} is invariant to the index $s$, we can assume $\vz^{(1)} \sim p(\vz|\vx)$ and write
\small
\begin{align}
    \logZgiwae &= \Exp{p(\vz^{(1)}|\vx)\prod \limits_{k=2}^K q_{\theta}(\vz^{(k)}|\vx)}{\log \left( \frac{1}{K} e^{\giwaeT(\vx,\vz^{(1)})} + \frac{1}{K} \sum  \limits_{k=2}^K e^{\giwaeT(\vx,\vz^{(k)})} \right) }  \nonumber \\
    &= \Exp{p(\vz^{(1)}|\vx)\prod \limits_{k=2}^K q_{\theta}(\vz^{(k)}|\vx)}{\log \left( \frac{1}{K} e^{\giwaeT(\vx,\vz^{(1)})} + \frac{K-1}{K} \cdot \frac{1}{K-1} \sum  \limits_{k=2}^K e^{\giwaeT(\vx,\vz^{(k)})} \right) } \nonumber  \\
    &\leq \log 
    \left( \frac{1}{K} \Exp{p(\vz^{(1)}|\vx)}{ e^{\giwaeT(\vx,\vz^{(1)})}} +  \Exp{\prod \limits_{k=2}^K q_{\theta}(\vz^{(k)}|\vx)}{\frac{K-1}{K} \cdot \frac{1}{K-1} \sum  \limits_{k=2}^K e^{\giwaeT(\vx,\vz^{(k)})}} \right) \nonumber \\
        &= \log 
    \left( \frac{1}{K}  \Exp{p(\vz^{(1)}|\vx)}{ e^{\giwaeT(\vx,\vz^{(1)})}} + \frac{K-1}{K} \cdot 
    \frac{1}{K-1} \sum \limits_{k=2}^{K} \Exp{q(\vz^{(k)}|\vx)}{
    e^{\giwaeT(\vx,\vz^{(k)})}} \right) \nonumber \\
     &= \log  \left( \frac{1}{K} \Exp{p(\vz^{(1)}|\vx)}{ e^{\giwaeT(\vx,\vz^{(1)})}} + \frac{K-1}{K}  \mathcal{Z}_{\theta,\phi}(\vx)  \right) \label{eq:196}\\
     &\eqqcolon \logZgiwaeub.
\end{align}
\normalsize
Since $\log(u)$ is a continuous function, we know that $\lim_{K\rightarrow \infty} \log(f(K)) = \log \lim_{K\rightarrow \infty} f(K)$.   We thus reason about the limiting behavior of the terms inside the logarithm in \cref{eq:196}.
 As $K \rightarrow \infty$, we have $\frac{1}{K} \rightarrow 0$, and thus the first term inside the $\log$ goes to $0$.  For the second term, we also have $\frac{K-1}{K} \rightarrow 1$. Thus, we have
\begin{align}
   \lim \limits_{K \rightarrow \infty} \logZgiwaeub &= 
   \log \Exp{q_{\theta}(\vz|\vx)}{e^{\giwaeT(\vx,\vz)}} = \logZmine .
\end{align}
As reasoned above, the convergence of the sequence of both upper and lower bounds to $\logZmine$ implies that $ \lim_{K \rightarrow \infty} \logZgiwae = \logZmine$.  By the reasoning surrounding \cref{eq:giwae_to_mine_partition}, this implies $\lim_{K \rightarrow \infty} I_{\gls{GIWAE}_L}(q_{\theta}, \giwaeT) = \gls{IBAL}(q_\theta, \giwaeT)$ as desired.
\end{proof}

\newcommand{\giwaesnis}{{\qprop{giwae}}(\vz|\vx; K)}%
\subsection{Convergence of GIWAE SNIS Distribution to IBAL Energy-Based Posterior}\label{app:convergence}
In this section, we consider the marginal \gls{SNIS} distribution of \gls{GIWAE}, which is induced by sampling $K$ times from $\qzx$ and returning the sample in index $s$ with probability $\qprop{giwae}(\idx|\vx,\vz^{(1:K)}) \propto e^{\giwaeT(\vx,\vz^{(\idx)})}$.   As $K \rightarrow \infty$, we show that this distribution converges to the single-sample energy-based posterior approximation underlying the \gls{IBAL} and \textsc{mine-ais}.   A similar observation is made in Sec. 3.2 of \citet{lawson2019energy}.
This result regarding the probabilistic interpretations of \gls{GIWAE} and the \textsc{ibal} is complementary to the result in \myprop{mine-ais-limit} regarding the limiting behavior of the bounds.

\begin{proposition} \label{prop:giwae_mine_distribution}
Define the marginal \gls{SNIS} distribution of \gls{GIWAE}, $\giwaesnis$, using the following sampling procedure
\begin{enumerate}
    \item Sample from $\idx, \vz^{(1:K)} \sim \qprop{giwae}(\idx, \vz^{(1:K)} | \vx)$ according to \cref{eq:giwae_forward_app} 
    \item Return $\vz = \vz^{(s)}$.
\end{enumerate}
Then, as $K \rightarrow \infty$, the \textsc{kl} divergence (in either direction) between the marginal \gls{SNIS} distribution of \gls{GIWAE} and the energy-based variational distribution of \textsc{ibal}, $\pi_{\theta, \phi}(\vz|\vx) \propto \qzx e^{\giwaeT(\vx,\vz)}$, goes to zero.
\begin{align*}
\hspace*{-.2cm} \lim_{K \rightarrow \infty} &\KL{\giwaesnis}{\pi_{\theta, \phi}(\vz|\vx)} = 0 \\
\, \, \text{and} \, \, \lim_{K \rightarrow \infty} &\KL{\pi_{\theta, \phi}(\vz|\vx)}{\giwaesnis} = 0 .
\end{align*}
\end{proposition}

\begin{proof}
To prove the proposition, we consider a mixture target distribution similar to the case of \gls{IWAE}. However, in this case, we take a single sample from $\pi_{\theta,\phi}(\vz|\vx)$ instead of $p(\vz|\vx)$
\begin{align}
    \ptgtgiwae(\idx, \vz^{(1:K)}, \vx) = \frac{1}{K} \pi_{\theta,\phi}(\vz^{(\idx)}|\vx) \prod \limits_{\myoverset{k=1}{k\neq \idx}}^K q_{\theta}(\vz^{(\idx)}|\vx) .
\end{align}
Note that the probabilistic interpretations $\qprop{giwae}(\idx, \vz^{(1:K)} | \vx)$ and $\ptgtgiwae(\idx, \vz^{(1:K)}, \vx)$ exactly match the 
\gls{IWAE} proposal and target distributions in \myapp{iwae_prob} where, in the target distribution, the posterior $\pzx$ has been replaced by $\pi_{\theta,\phi}(\vz|\vx)$.  
Thus, we can analyze the 
\kl divergences $\DKL[ \qprop{giwae}(\idx, \vz^{(1:K)}| \vx) \| \ptgtgiwae(\idx, \vz^{(1:K)}| \vx) ]$ and $\DKL[ \ptgtgiwae(\idx, \vz^{(1:K)}| \vx) \| \qprop{giwae}(\idx, \vz^{(1:K)}| \vx)  ]$ using techniques from previous work on \gls{IWAE}.
Following similar arguments as in \citet{domke2018importance} (Thm 2) and \citet{cremer2017reinterpreting}, these extended state space \kl divergences upper bound the \kl divergence between the marginal \gls{SNIS} distribution and energy-based target, e.g. $\DKL[\giwaesnis \| \pi_{\theta,\phi}(\vz|\vx)] \leq \DKL[ \qprop{giwae}(\idx, \vz^{(1:K)}| \vx) \| \ptgtgiwae(\idx, \vz^{(1:K)}| \vx) ]$.   
As $K \rightarrow \infty$, $\DKL[ \qprop{giwae}(\idx, \vz^{(1:K)}| \vx) \|\ptgtgiwae(\idx, \vz^{(1:K)}| \vx) ] \rightarrow 0$ since it is an instance of the \gls{IWAE} gap.
Thus, the \kl divergence $\DKL[\giwaesnis \| \pi_{\theta,\phi}(\vz|\vx)]$ also vanishes.
Similar reasoning applies for the reverse \kl divergence.
\end{proof}

\section{MINE-AIS}\label{app:mine_ais}

\subsection{Multi-Sample AIS Evaluation of the IBAL}\label{app:ais_eval_ibal}\label{app:reverse_annealing_ais}\label{app:approximate} %

After training the variational base distribution $\qzx$ and the critic function $\giwaeT(\vx,\vz)$ using the \textsc{mine-ais} training procedure above, 
we still need to evaluate the $\ibal$ lower bound on \gls{MI}, $\gls{IBAL}(q_{\theta},\giwaeT)$.  We can easily upper bound $\gls{IBAL}(q_{\theta},\giwaeT)$ using a Multi-Sample \gls{AIS} lower bound on $\logZmine$ with expectations under the forward sampling procedure $\qprop{ais}(\vz_{0:T}|\vx)$.   However, an upper bound on $\gls{IBAL}(q_{\theta},\giwaeT)$ is not guaranteed to preserve a lower bound on \gls{MI}.

In order to obtain a lower bound on the \gls{IBAL}, we would need to obtain an upper bound on $\logZmine = \log \Exp{\qzx}{e^{\giwaeT(\vx,\vz)}}$, the log partition function of $\pi_{\theta,\phi}(\vz|\vx)$.  However, considering this to be the target distribution $\pi_T(\vz|\vx)$ in Multi-Sample \gls{AIS}, we would require exact samples from $\pi_{\theta,\phi}(\vz|\vx)$ to guarantee an upper bound on $\logZmine$.
Since these samples are unavailable, we demonstrate conditions under which we can preserve an upper bound on $\logZmine$ (and lower bound on  $\ibal(q_\theta,T_\phi)$) by sampling from $\pzx$ instead of $\pi_{\theta,\phi}(\vz|\vx)$ to initialize our backward annealing chains in \myprop{mine-ais-sampling-ais} below.

Using the single sample \gls{AIS} bounds to estimate $\logZmine$, we have the following extended state space proposal and target distributions, 
\begin{align}
   \hspace*{-.2cm} \ptgtenergy(\vz_{0:T}|\vx) \coloneqq \pi_{\theta,\phi}(\vz_T|\vx) \prod \limits_{t=1}^T \trevphi(\vz_{t-1}|\vz_t)\,, \quad \, \,  \qpropenergy(\vz_{0:T}|\vx) \coloneqq q_{\theta}(\vz_0|\vx) \prod \limits_{t=1}^T \tfwdphi(\vz_t|\vz_{t-1})\,. \label{eq:energy_prop_tgt}
\end{align}
We emphasize that in $\ptgtenergy, \qpropenergy$, the transition kernels and intermediate densities are based on the critic $\giwaeT(\vx,\vz)$ and energy-based posterior $\pi_{\theta,\phi}(\vz|\vx)$ whose log partition function we seek to estimate. 
Recall from \mysec{general} that taking the expected log importance weights under $ \ptgtenergy(\vz_{0:T}|\vx)$ yields an upper bound on $\logZmine$. 
However, since it is difficult to draw exact samples from $\pi_{\theta,\phi}(\vz|\vx)$ to initialize backward annealing chains and sample from $\ptgtenergy(\vz_{0:T}|\vx)$, we instead consider sampling from the posterior $\pzx$.  
Using the same transition kernels $\trevphi$ as above, we first define the conditional distribution $\ptgtenergy(\vz_{0:T-1}|\vx, \vz_T)$ of the backward chain for a given $\vz_T$
\begin{align}
\ptgtenergy(\vz_{0:T-1}|\vx, \vz_T) \coloneqq \prod_{t=1}^T  \trevphi(\vz_{t-1}|\vz_t)\,. \label{eq:conditional_backward}
\end{align}
Using this notation, the target distribution in \cref{eq:energy_prop_tgt} can also be rewritten as $\ptgtenergy(\vz_{0:T}|\vx) = \pi_{\theta,\phi}(\vz_T|\vx) \ptgtenergy(\vz_{0:T-1}|\vx, \vz_T)$.

We now define an \emph{approximate} extended state space target distribution in which backward chains are initialized with samples from $\vz_T \sim \pzx$,
\begin{align}
    \ptgtpost(\vz_{0:T}|\vx) \coloneqq p(\vz_T|\vx) \prod \limits_{t=1}^T \trevphi(\vz_{t-1}|\vz_t) = p(\vz_T|\vx) \ptgtenergy(\vz_{0:T-1}|\vx, \vz_T) .\label{eq:ptgt_post}
\end{align}

In the following proposition, we characterize the conditions under which sampling from $\ptgtpost(\vz_{0:T}|\vx)$ preserves an upper bound on $\logZmine$.

\begin{proposition} \label{prop:mine-ais-sampling-ais} 
Define the 
\gls{AIS} marginal distribution $\qpropenergy(\vz_{T}|\vx)$ over the final state 
in the extended state space proposal as $\qpropenergy(\vz_T|\vx) \coloneqq \int \qpropenergy(\vz_{0:T}|\vx) d\vz_{0:T-1}$.  If we have
\begin{align}
\DKL[p(\vz_T|\vx) \| \qpropenergy(\vz_T|\vx)] \geq \DKL[p(\vz_T|\vx) \| \pi_{\theta,\phi}(\vz_T|\vx)]\,,
\end{align}
then initializing the backward \gls{AIS} chain using $\vz_T \sim \pzx$ (i.e., sampling under $\ptgtpost(\argsais|\vx)$), yields an upper bound on $\logZmine$,
\begin{align}
    \Exp{\ptgtpost(\argsais|\vx)}{\log \frac{\ptgtenergy(\vx,\argsais)}{\qpropenergy(\argsais|\vx)}} \geq \logZmine \label{eq:approx-ub-ais}.
\end{align}
\end{proposition}
\begin{proof}
We begin by writing several definitions, which will allow us to factorize the extended state space proposal $\qpropenergy(\vz_{0:T}|\vx)$ in the time-reversed direction.   This factorization will include the final \gls{AIS} marginal $\qpropenergy(\vz_T|\vx)$.   

Starting from \cref{eq:energy_prop_tgt}, the forward transitions $\qpropenergy(\vz_{0:T}|\vx) = q_{\theta}(\vz_0|\vx) \prod_{t=1}^T \tfwdphi(\vz_t|\vz_{t-1})$ induce the marginal distributions $\qpropenergy(\vz_{t}|\vx)$ at each step.
The forward transitions and marginals induce a \textit{posterior} kernel $\tqrev(\vz_{t-1}|\vz_t)$, which allows us to rewrite  $\qpropenergy(\vz_{0:T}|\vx)$ using a reverse factorization %
\begin{align}
    \qpropenergy(\vz_{0:T}|\vx) 
    = \qpropenergy(\vz_T|\vx) \prod \limits_{t=1}^T &\tqrev(\vz_{t-1}|\vz_t) , \label{eq:qprop_rev}
\end{align}
\begin{align}
        \text{where} \quad &\tqrev(\vz_{t-1}|\vz_t) = \frac{\qpropenergy(\vz_{t-1}|\vx) \tfwdphi(\vz_t|\vz_{t-1})}{\qpropenergy(\vz_{t}|\vx)}. \label{eq:q_rev_kernel}
\end{align}
The posterior reverse transitions $\tqrev(\vz_{t-1}|\vz_t)$ are intractable in practice, and cannot be simplified to match the kernels in the target distribution $\ptgtenergy(\vz_{0:T-1}|\vx, \vz_T) = \prod_{t=1}^T  \trevphi(\vz_{t-1}|\vz_t)$ using the invariance or detailed balance conditions.  \citet{doucet2022annealed} provide a promising approach using score matching to approximate these posterior transitions $\tqrev(\vz_{t-1}|\vz_t)$. 

Finally, we write the posterior reverse process conditioned on a particular $\vz_T$ as 
\begin{align}
    \qpropenergy(\vz_{0:T-1}|\vx,\vz_T) \coloneqq \prod_{t=1}^T \tqrev(\vz_{t-1}|\vz_t) .
\end{align}

With the goal of upper bounding $\log \mathcal{Z}_{\theta,\phi}(\vx) = \int \qzx e^{\giwaeT(\vx,\vz)}d\vz$, we consider the log importance weights with expectations under the target distribution $\ptgtpost(\argsais|\vx) = p(\vz_T|\vx) \ptgtenergy(\vz_{0:T-1}|\vx, \vz_T)$, as in \cref{eq:approx-ub-ais}.

Using the above notation, we have
\small
\begin{align}
   \hspace*{-.3cm} & \Exp{\ptgtpost(\argsais|\vx)}{\log \frac{\ptgtenergy(\vx,\argsais)}{\qpropenergy(\argsais|\vx)}} = \logZmine + \Exp{\ptgtpost(\argsais|\vx)}{\log \frac{\ptgtenergy(\argsais|\vx)}{\qpropenergy(\argsais|\vx)}}  \\[3ex]
    &\phantom{\frac{\ptgtenergy(\vx,\argsais)}{\qpropenergy(\argsais|\vx)}} = \logZmine + \Exp{\ptgtpost(\argsais|\vx)}{\log \frac{\pi_{\theta,\phi}(\vz_T|\vx) \ptgtenergy(\vz_{0:T-1}|\vx, \vz_T)}{\qpropenergy(\vz_T|\vx) \qpropenergy(\vz_{0:T-1}|\vx,\vz_T)}\frac{p(\vz_T|\vx)}{p(\vz_T|\vx)} } \\[2ex]
  &\phantom{\frac{\ptgtenergy(\vx,\argsais)}{\qpropenergy(\argsais|\vx)}}= \logZmine + \DKL\big[p(\vz_T|\vx)\|\qpropenergy(\vz_T|\vx) \big] - \DKL\big[p(\vz_T|\vx)\|\pi_{\theta,\phi}(\vz_T|\vx) \big]  \\
  &\phantom{\frac{\ptgtenergy(\vx,\argsais)}{\qpropenergy(\argsais|\vx)}=  \logZmine} 
  + \Exp{p(\vz_T|\vx)}{\DKL[\ptgtenergy(\vz_{0:T-1}|\vx, \vz_T) \|  \qpropenergy(\vz_{0:T-1}|\vx,\vz_T)]}  \nonumber 
\end{align}
\normalsize
The intractable \kl divergence $\DKL[\ptgtenergy(\vz_{0:T-1}|\vx, \vz_T) \|  \qpropenergy(\vz_{0:T-1}|\vx,\vz_T)]$ compares the reverse kernels in the target distribution $\trevphi(\vz_{t-1}|\vz_t)$ against the posterior $\tqrev(\vz_{t-1}|\vz_t)$.   Ignoring this nonnegative term, we can lower bound the expectation on the \textsc{lhs}
\footnotesize
\begin{align}
  \Exp{\ptgtpost}{\log \frac{\ptgtenergy(\vx,\argsais)}{\qpropenergy(\argsais|\vx)}} &\geq  \logZmine + \DKL\big[p(\vz_T|\vx)\|\qpropenergy(\vz_T|\vx) \big] - \DKL\big[p(\vz_T|\vx)\|\pi_{\theta,\phi}(\vz_T|\vx) \big]  . \nonumber
  \end{align}
    \normalsize
{Finally, under the assumption of the proposition that $ \DKL\big[p(\vz_T|\vx)\|\qpropenergy(\vz_T|\vx) \big] \geq  \DKL\big[p(\vz_T|\vx)\|\pi_{\theta,\phi}(\vz_T|\vx) \big]$, we have the desired result.}
\end{proof}
\textbf{KL Divergence Condition}
In \myprop{mine-ais-sampling-ais}, we have shown that we can preserve an upper bound on $\logZmine$ by initializing the reverse chain using a posterior sample, under a condition on the \kl divergence, 
\begin{align}
    \DKL\big[p(\vz_T|\vx)\|\qpropenergy(\vz_T|\vx) \big] \geq \DKL\big[p(\vz_T|\vx)\|\pi_{\theta,\phi}(\vz_T|\vx) \big]. \label{eq:kl_cond}
\end{align}
While we cannot guarantee this condition, we intuitively expect \cref{eq:kl_cond} to hold in practice since $\pi_{\theta,\phi}(\vz|\vx)$ has been directly trained to match $\pzx$.  
By contrast, $\qpropenergy(\vz_T|\vx)$ is the final state of an \gls{AIS} procedure, which approximates $\pi_{\theta,\phi}(\vz|\vx)$ and does not have access to information about $\pzx$.
\citet{burda2015accurate} use a similar approach for lower bounding the log likelihood in \textsc{ebm}s, but give an example of a \gls{RBM} model (in their Sec. 5) in which \cref{eq:kl_cond} does not hold.

As desired, we find in our experiments in 
\myfig{mine-ais}
that our approximate reverse annealing procedure underestimates the \gls{IBAL} in all of the \textsc{mine-ais}, \textsc{giwae} and \textsc{InfoNCE} experiments, for all numbers of intermediate distributions $T$.

\paragraph{$\mathbf{T=1}$ Special Case}  
For $T=1$, we have $\qpropenergy(\vz_1|\vx) = q_{\theta}(\vz_1|\vx)$. %
In particular, \myprop{mine-ais-sampling-ais} will provide an upper bound on $\logZmine$ if 
\begin{align}
\KL{\pzx}{q_{\theta}(\vz|\vx)} \geq \KL{\pzx}{\pi_{\theta,\phi}(\vz|\vx)}   \label{eq:kl_condition_sis} \,.
\end{align}
This condition is guaranteed under the assumptions of \myprop{mineaisproperties}, where $\gls{IBAL}(q_{\theta}, T_{\phi^*}) = I_{\gls{BA}_L}(\pi_{\theta, \phi^*})$ improves
upon $I_{\gls{BA}_L}(q_{\theta})$ for an energy function $T_{\phi^*}(\vx,\vz)$ which has been trained to maximize the \gls{IBAL}.  
In \cref{eq:kl_condition_sis} the \kl divergence on the left-hand side corresponds to the gap in the \gls{BA} lower bound, which is larger than the gap in the \gls{IBAL} on the right-hand side.
We can further show that the lower bound on $\gls{IBAL}(q_\theta, T_{\phi})$ resulting from \myprop{mine-ais-sampling-ais} with $T=1$ is the \gls{BA} lower bound,
\begin{align*}
    \gls{IBAL}(q_\theta, \giwaeT) &= \mathbb{E}_{p(\vx,\vz)}\left[ \log \frac{\qzx}{p(\vz)} \right] + \mathbb{E}_{p(\vx,\vz)}[ T_{\phi}(\vx,\vz) ] - \mathbb{E}_{p(\vx)}[ \log \mathcal{Z}_{\theta,\phi}(\vx) ]\\
    &\geq\mathbb{E}_{p(\vx,\vz)}\left[ \log \frac{\qzx}{p(\vz)} \right] + \mathbb{E}_{p(\vx,\vz)}[ T_{\phi}(\vx,\vz) ] - \Exp{\ptgtpost}{\log \frac{\ptgtenergy(\vx,\argsais)}{\qpropenergy(\argsais|\vx)}} \nonumber  \\
    &=\mathbb{E}_{p(\vx,\vz)}\left[ \log \frac{\qzx}{p(\vz)} \right] + \mathbb{E}_{p(\vx,\vz)}[ T_{\phi}(\vx,\vz) ] - \Exp{p(\vx,\vz)}{\log \frac{\qzx e^{T_{\phi}(\vx,\vz)}}{\qzx}} \nonumber \\
    &= \mathbb{E}_{p(\vx,\vz)}\left[ \log \frac{\qzx}{p(\vz)} \right] + \cancel{\mathbb{E}_{p(\vx,\vz)}[ T_{\phi}(\vx,\vz) ] } -  \cancel{\mathbb{E}_{p(\vx,\vz)}[ T_{\phi}(\vx,\vz) ] } \nonumber \\
    &= I_{\gls{BA}_L}(q_\theta)\,.
\end{align*} 
We can confirm this in \myfig{mine-ais}, where for $T=1$, the approximate lower bounds on $\gls{IBAL}(q_\theta, \giwaeT)$ begin from the appropriate \gls{BA} lower bound.   For example, in the case of \gls{IBAL} evaluation for \gls{GIWAE} ($K=100$), the light green curve starts from the \gls{BA} lower bound term reported in the decomposition of the \gls{GIWAE} ($K=100$) lower bound in \myfiga{different-bounds}{a}.    %
Using more intermediate distributions ($T > 1$) for the \gls{AIS} approximate bound in \myprop{mine-ais-sampling-ais}, the estimates in  \myfig{mine-ais} approach the true value of $\gls{IBAL}(q_\theta, \giwaeT)$ in all cases.

\section{Applications to Mutual Information Estimation without Known Marginals}\label{app:representation}
\newcommand{\param}{\theta}
\newcommand{\map}{g_{\param}}
\newcommand{\paramenc}{\psi}
\newcommand{\qenc}{q_{\paramenc}(\vz|\vx)}
\newcommand{\iqenc}{q_{\paramenc}(\vx|\vz)}
\newcommand{\qdata}{q_{\text{data}}(\vx)}
\newcommand{\qdatay}{q_{\text{data}}(\vy)}
\newcommand{\mapenc}{g_{\paramenc}}
\newcommand{\igen}{I}
\newcommand{\ienc}{I}

While the focus in of our work is evaluating the mutual information $I(\vx;\vz)$ in settings where at least a single marginal is available,  we are often interested in estimating or optimizing the mutual information where no marginal distribution is available. A natural setting where no marginal distribution is available is representation learning where the goal is to maximize the mutual information between the data distribution $\qdata$ and the representation induced by a stochastic mapping $\qenc$ parameterized by $\paramenc$,
\begin{align}
    \ienc(\vx;\vz) = \Exp{\qdata \qenc}{\log \frac{q_\paramenc{(\vx,\vz)}}{\qdata q_\paramenc{(\vz)}}} 
    = \Exp{\qdata \qenc}{\log \frac{q_\paramenc{(\vx|\vz)}}{\qdata}}  \label{eq:ienc}
\end{align}
where $q_\paramenc{(\vz)} = \int \qdata \qenc d\vx$ represents the ``aggregated posterior'' \citep{makhzani2015adversarial} or induced marginal distribution over $\vz$.
\subsection{BA Lower Bound}
Using the notation of \myeq{ienc}, we can write the \textsc{ba} lower bound as
\begin{align}
    \ienc(\vx;\vz) &\geq 
    \ienc(\vx;\vz) -
    \underbrace{\vphantom{\Exp{\qdata \qenc}{ \log \frac{q_\paramenc(\vx|\vz)}{\qdata} } } \mathbb{E}_{q_{\paramenc}(\vz)}\bigg[ \KL{q_\paramenc(\vx|\vz)}{p_{\theta}(\vx|\vz)} \bigg]}_{\text{gap}} \nonumber \\
    &= \mathbb{E}_{\qdata} \mathbb{E}_{\qenc}\bigg[ \log \frac{p_{\theta}(\vx|\vz)}{\qdata} \bigg] \\
    &= \mathbb{E}_{\qdata} \underbrace{\mathbb{E}_{\qenc}\bigg[ \log p_{\theta}(\vx|\vz) \bigg] \vphantom{\bigg[ \frac{1}{2} \bigg]}}_{\text{negative reconstruction loss of } \vx} + \underbrace{H_{\text{data}}(\vx) \vphantom{\bigg[ \frac{1}{2} \bigg]}}_{\text{constant}} \label{eq:ba_encoding} \\
     & \eqqcolon I_{\gls{BA}_L}(p_{\theta}(\vx|\vz)) 
\end{align}
where $p_{\theta}(\vx|\vz)$ is a variational distribution, parameterized by $\theta$, that tries to match to the \emph{inverse encoding distribution} $\iqenc \propto \qdata \qenc$. The first term in \myeq{ba_encoding} can be interpreted as the reconstruction term, and the second term is the entropy of the data distribution, which is constant.

\paragraph{Evaluating MI up to a Constant} 
Note that the gradient of $I_{\gls{BA}_L}(p_{\theta}(\vx|\vz))$ with respect to the parameters $\theta$ does not depend on the marginal $\qdata$, and thus we may still optimize the variational distribution when the data distribution is unknown.   We can then use the resulting $p_{\theta}(\vx|\vz)$ to estimate the \gls{BA} lower bound on mutual information \textit{up to a constant}, ${H_{\text{data}}(\vx)}$.
This is useful in comparing the \gls{MI} induced by two different representations $q_{\paramenc_1}(\vz|\vx)$ and $q_{\paramenc_2}(\vz|\vx)$ of the same data distribution.

\paragraph{Optimizing MI with the BA Lower Bound}
The \gls{BA} lower bound is also amenable to backpropagation through the parameters of the encoding distribution $\qenc$, even when analytic marginal densities for $\qdata$ or $q_{\paramenc}(\vz)$ are not available. 
Maximization of the \gls{BA} lower bound appears in various settings, including in representation learning \citep{alemi2016deep, alemi2018fixing}, reinforcement learning \citep{mohamed2015variational}, improving interpretability in \textsc{gan}s \citep{chen2016infogan}, and variational information bottleneck methods \citep{tishby2000information, alemi2016deep, alemi2018fixing}.

\subsection{GIWAE Lower Bound}
In this section, we discuss the applicability of our \gls{GIWAE} lower bound in mutual information maximization settings.   
Rewriting the \gls{GIWAE} lower bound in \cref{eq:giwae_lb_mi_main} for the case of estimating $\ienc(\vx;\vz)$ in \cref{eq:ienc}, we have
\scriptsize
\begin{align}
    \ienc(\vx;\vz) &\geq  \underbrace{\mathbb{E}_{\qdata \qenc}\big[ \log p_{\theta}(\vx|\vz) \vphantom{\Exp{q_\paramenc(\vx,\vz) \prod \limits_{k=2}^K p_{\theta}(\vx^{(k)}| \vz)}{\log \frac{e^{\giwaeT(\vx,\vz)}}{\frac{1}{K} (e^{\giwaeT(\vx,\vz)}}}} \big] \vphantom{\frac{1}{2}}}_{\text{negative reconstruction loss of } \vx}  + \underbrace{H_{\text{data}}(\vx)  \vphantom{\Exp{q_\paramenc(\vx,\vz) \prod \limits_{k=2}^K p_{\theta}(\vx^{(k)}| \vz)}{\log \frac{e^{\giwaeT(\vx,\vz)}}{\frac{1}{K} (e^{\giwaeT(\vx,\vz)}}}}\vphantom{\frac{1}{2}}}_{\text{constant}} 
     + \underbrace{\Exp{q_\paramenc(\vx,\vz) \prod \limits_{k=2}^K p_{\theta}(\vx^{(k)}| \vz)}{\log \frac{e^{\giwaeT(\vx,\vz)}}{\frac{1}{K} (e^{\giwaeT(\vx,\vz)} + \sum_{k=2}^K e^{\giwaeT(\vx^{(k)},\vz)})}}}_{\small \text{contrastive term} \leq \log K} \label{eq:giwae_encoding_mi} 
    \\
    &\eqqcolon I_{\textsc{giwae}_L}(p_{\theta}(\vx|\vz), \giwaeT, K) \nonumber
\end{align}
\normalsize
Note that we can use a single joint sample from $\qdata \qenc = q_{\paramenc}(\vz) q_{\paramenc}(\vx|\vz)$ to obtain a positive sample from the inverse encoding distribution $\vx \sim q_{\paramenc}(\vx|\vz)$ for a particular $\vz \sim q_{\paramenc}(\vz)$, in a similar fashion to our ancestral sampling in \mysec{setting}.
For a given $\vz$, the negative samples correspond to $K-1$ samples from the stochastic decoder $\vx^{(2:K)} \sim p_{\theta}(\vx|\vz)$.

From the importance sampling perspective, we can view the \gls{GIWAE} lower bound as corresponding to an upper bound on the ``log partition function'' $\log q_\paramenc(\vz)$ for a particular $\vz \sim q_{\paramenc}(\vz)$.
The contrastive term in \cref{eq:giwae_encoding_mi} arises from \gls{SNIS} sampling of a single $\vx^{(s)}$ in the extended state space proposal, with $\qprop{giwae}(s|\vz, \vx^{(1:K)}) = e^{\giwaeT(\vx^{(s)},\vz)}/ \sum_{k=1}^K e^{\giwaeT(\vx^{(k)},\vz)}$.   
Finally, we note that optimization over the energy function $\giwaeT(\vx,\vz)$ improves upon $I_{\gls{BA}_L}(p_{\theta}(\vx|\vz))$ in \cref{eq:ba_encoding} by at most $\log K$ nats using the contrastive term.   
\paragraph{Evaluating MI up to a Constant}

Similar to the \gls{BA} bound, the \textsc{giwae} lower bound can be used to evaluate the mutual information up to the constant data entropy when $\qdata$ is unknown.   This may be useful in comparing the \gls{MI} induced by two different representations $q_{\paramenc_1}(\vz|\vx)$ and $q_{\paramenc_2}(\vz|\vx)$ of the same data distribution.   
Note that the reconstruction and contrastive terms in \cref{eq:giwae_encoding_mi} depend on samples from $\qdata$ and not the density.   Thus, our ability to take gradients with respect to the parameters of the variational distribution $p_{\theta}(\vx|\vz)$ or energy function $\giwaeT(\vx,\vz)$ are not affected by the fact that the marginal distribution is unknown.

\paragraph{Optimizing MI with the GIWAE Lower Bound}
The \gls{GIWAE} lower bound may also be used for mutual information maximization with respect to the parameters of $\qenc$, since each term in \cref{eq:giwae_encoding_mi} is amenable to backpropagation.  
Since \gls{GIWAE} generalizes both the \gls{BA} and \textsc{Info-NCE} lower bounds, it can be used as a drop-in replacement for either of these bounds for optimizing \textsc{mi} (e.g., see \citet{oord2018representation}).

\subsection{MINE-AIS / IBAL Lower Bound}
Recall that \textsc{mine-ais} estimation would involve an energy based variational approximation to the inverse encoding distribution $q_{\paramenc}(\vx|\vz)$,
\begin{align}
    \pi_{\theta,\phi}(\vx|\vz) = \frac{1}{\mathcal{Z}(\vz)} p_{\theta}(\vx|\vz) e^{\giwaeT(\vx,\vz)} \, . \label{eq:energy_enc}
\end{align}
Note that \gls{IBAL} lower bound on $\ienc(\vx;\vz)$ involves an intractable log partition function term $
\mathbb{E}_{q_{\paramenc}(\vz)}[\log  \mathcal{Z}(\vz)]=  \mathbb{E}_{q_{\paramenc}(\vz)}\big[ \log \Exp{p_\theta(\vx|\vz)}{e^{\giwaeT(\vx,\vz)}} \big]$. 
\footnotesize
\begin{align}
    \ienc(\vx;\vz) &\geq  \mathbb{E}_{\qdata} \underbrace{\mathbb{E}_{\qenc}\big[ \log p_{\theta}(\vx|\vz) \big] \vphantom{\left[\log \frac{e^{\giwaeT(\vx,\vz)}}{ \Exp{p_{\theta}(\vx| \vz)}{e^{\giwaeT(\vx,\vz)}} }\right]}}_{\text{negative reconstruction loss of } \vx}  + \underbrace{H_{\text{data}}(\vx) \vphantom{\left[\log \frac{e^{\giwaeT(\vx,\vz)}}{ \Exp{p_{\theta}(\vx| \vz)}{e^{\giwaeT(\vx,\vz)}}}\right]}}_{\text{constant}}   + \underbrace{\Exp{q_\paramenc(\vx,\vz)}{\log \frac{e^{\giwaeT(\vx,\vz)}}{ \Exp{p_{\theta}(\vx| \vz)}{e^{\giwaeT(\vx,\vz)}} }}}_{ \text{contrastive term} \leq \Exp{q_{\paramenc}(\vz)}{\DKL[q_{\paramenc}(\vx|\vz)\|p_{\theta}(\vx|\vz) ] }} \label{eq:mine_ais_encoding_mi} \\
    &\eqqcolon \gls{IBAL}(p_{\theta}(\vx|\vz), \giwaeT) \nonumber 
\end{align}
\normalsize
where the term in the denominator is the partition function $\mathcal{Z}(\vz)= \Exp{p_{\theta}(\vx| \vz)}{e^{\giwaeT(\vx,\vz)}}$.

When taking gradients as in \cref{eq:mine-asi-derv}-(\ref{eq:mine-asi-derv2}), we obtain 
\begin{align}
\fpartialf{\theta} \ibal(p_\theta,T_\phi) &= \Exp{\qdata \qenc}{\fpartialf{\theta} \log p_\theta(\vx|\vz)} - \Exp{q_{\paramenc}(\vz) \pi_{\theta,\phi}(\vx|\vz) }{\fpartial \log p_\theta(\vx|\vz)} ,   \label{eq:mine-asi-derv-app}\\
\fpartialf{\phi}  \ibal(p_\theta,T_\phi) &=  \Exp{\qdata \qenc}{\fpartialf{\phi} \Txz} - \Exp{q_{\paramenc}(\vz) \pi_{\theta,\phi}(\vx|\vz)}{\fpartialf{\phi} \Txz} .  
\label{eq:mine-asi-derv2-app}
\end{align}
To obtain approximate negative samples from $\pi_{\theta,\phi}(\vx|\vz)$ for a given $\vz$, we can use \gls{MCMC} 
transition kernels since the unnormalized target density $\tpi_{\theta,\phi}(\vx|\vz)= p_{\theta}(\vx|\vz) e^{\giwaeT(\vx,\vz)}$ is tractable.

\paragraph{Evaluating MI up to a Constant}
Since only samples from $\qdata$ are required for the \textsc{mine-ais} training procedure in \cref{eq:mine-asi-derv-app}-(\ref{eq:mine-asi-derv2-app}), we can learn the base variational distribution $p_{\theta}(\vx|\vz)$ and energy function $\giwaeT(\vx,\vz)$ in cases when the marginal distribution $\qdata$ is unknown.   

As in the case of \gls{GIWAE} and \gls{BA} lower bounds above, we can also evaluate the \gls{IBAL} lower bound in \cref{eq:mine_ais_encoding_mi} up to a constant.   As in the main text, this involves using multi-sample \gls{AIS} techniques to bound the intractable log partition function term $\log \mathcal{Z}(\vz) = \log \Exp{p_{\theta}(\vx| \vz)}{e^{\giwaeT(\vx,\vz)}}$.

\paragraph{Optimizing MI with the MINE-AIS Lower Bound}
If we are interested in optimizing the \textsc{mine-ais} lower bound with respect to the parameters of a stochastic encoder $\qenc$, 
we would need to backpropagate through 
the 
Multi-Sample \gls{AIS} procedure used to estimate the 
$\log \mathcal{Z}(\vz)$ term, 
but computing the
gradients are intractable.
We thus conclude that, among our proposed methods, \gls{GIWAE} is the most directly applicable in settings of \gls{MI} for representation learning.

\section{Experimental Details}\label{app:experiment_details}
\subsection{Experiment Details of \mysec{exp_ais}}
In this section, we provide the experiment details used in \mysec{exp_ais}. For more details, see the public Github repository, \href{https://github.com/huangsicong/ais_mi_estimation}{https://github.com/huangsicong/ais\_mi\_estimation}.
\subsubsection{Datasets and Models}
\label{app:models}
We used \textsc{mnist}~\citep{lecun1998gradient} and \textsc{cifar}-10~\citep{krizhevsky2009learning} datasets in our experiments.

\parhead{Real-Valued MNIST} For the \textsc{vae} experiments on the real-valued \textsc{mnist} dataset (\mytable{mnist}), the encoder's architecture is $784 - 1024 - 1024 - 1024 - z$, where $z$ is the latent code size shown in the row header of \mytable{mnist}. The decoder architecture is the reverse of the encoder architecture. The decoder variance is learned scalar. 

For the \textsc{gan} experiments on \textsc{mnist} (\mytable{mnist}), we used the same decoder architecture as our \textsc{vae}s. In order to stabilize the training dynamics, we used the gradient penalty (GP)~\citep{salimans2016improved}. 

The network was trained for 300 epochs with the learning rate of 0.0001 using the Adam optimizer~\citep{adam}, and the checkpoint with the best validation loss was used for the evaluation. 

\parhead{CIFAR-10} For the \textsc{cifar}-10 experiments (\mytable{mnist}), we experimented with a smaller version of \textsc{dcgan}~\citep{radford2015unsupervised} (see the public code). The number at the end of each model name in \mytable{mnist} indicates the latent code size.  

\subsubsection{Experiment Details of \mysec{exp_ais}}
\label{app:ais_settings}
For the \textsc{ais} temperature schedule, We used sigmoid schedules as used in~\citet{wu2016quantitative}. The step size of \textsc{hmc} was adaptively tuned to achieve an average acceptance probability of $65\%$ as suggested in~\citet{neal2001annealed}. For all \textsc{mnist} experiments in \mytable{mnist}, we evaluated on a single batch size of 128 simulated data. For all \textsc{cifar} experiments in \mytable{mnist} we used a single batch of 32 simulated data. All experiments are run on on Tesla P100 or Quadro RTX 6000 or 
Tesla T4 GPUs. 

\subsection{Runtime Comparison}
\label{app:runtime}

We benchmarked the runtime on Tesla P100 GPUs. For \textsc{mnist}, it took about 35 minutes to run \textsc{iwae} with $K=1M$, about 8 hours to run \textsc{ais} with $T=30K$. For \textsc{cifar}, it took about 45 minutes to run \textsc{iwae} with $K=1M$, and about 12 hours for the \textsc{ais} with $T=100K$. 

In \myfig{runtime_mnist}, we evaluate the tradeoff between runtime and bound tightness for evaluating the generative \gls{MI} for \textsc{vae} and \textsc{gan} models with 100-dimensional latent codes trained on the \textsc{mnist} dataset.   We compare \gls{IWAE} and multi-sample \gls{AIS} evaluation, with the same experimental settings as in \myapp{ais_settings} and an initial distribution $\qzx$.
We plot wall clock time on the $x$-axis, where increasing runtime reflects increasing $K$ for \gls{IWAE} and increasing $T$ for \gls{AIS}.

 \begin{figure}[h]
     \centering
     \subcaptionbox{\label{fig:aa} VAE-100}{
     \includegraphics[width=.45\textwidth]{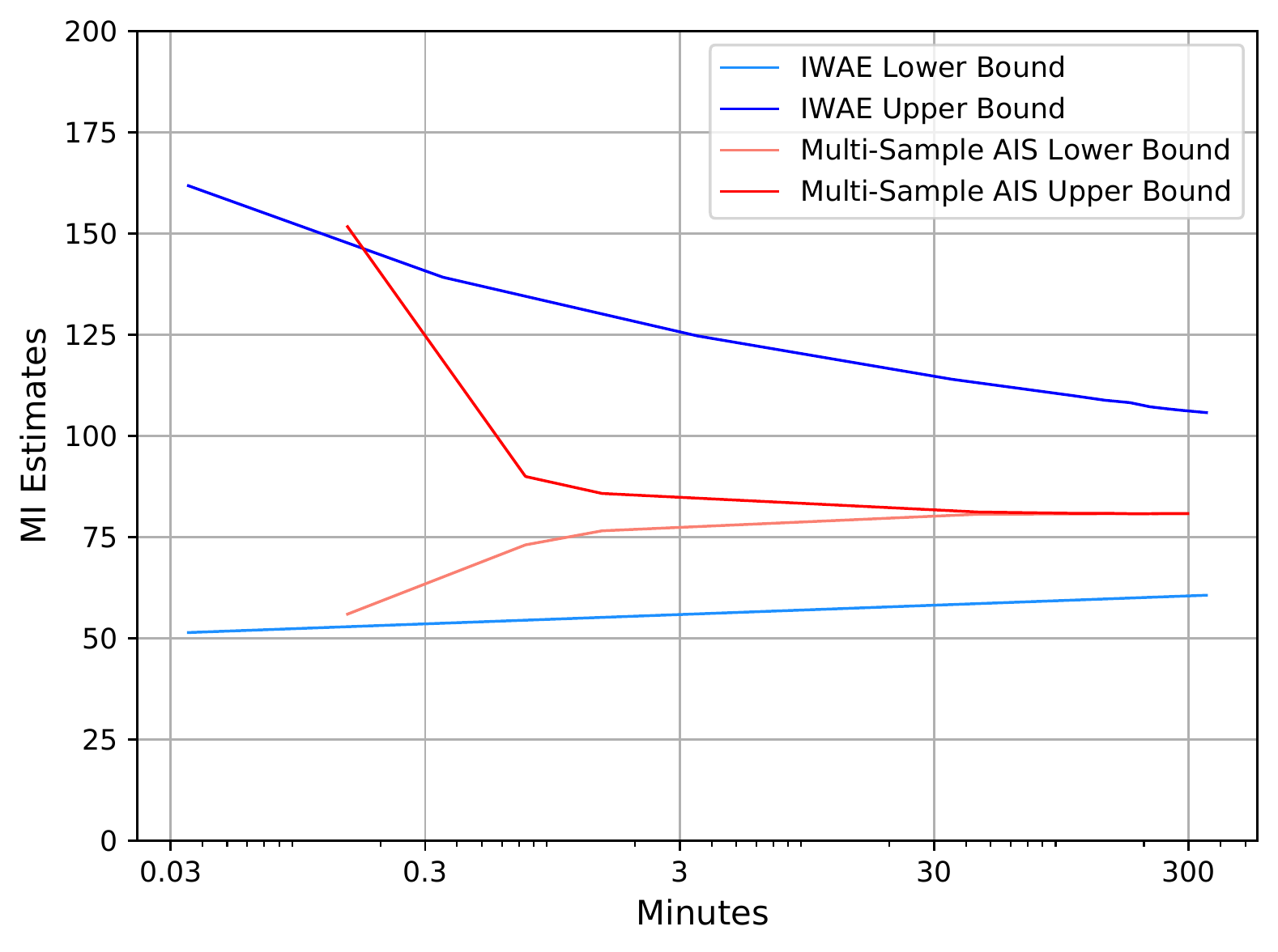}}
     \subcaptionbox{\label{fig:aa} GAN-100}{
     \includegraphics[width=.45\textwidth]{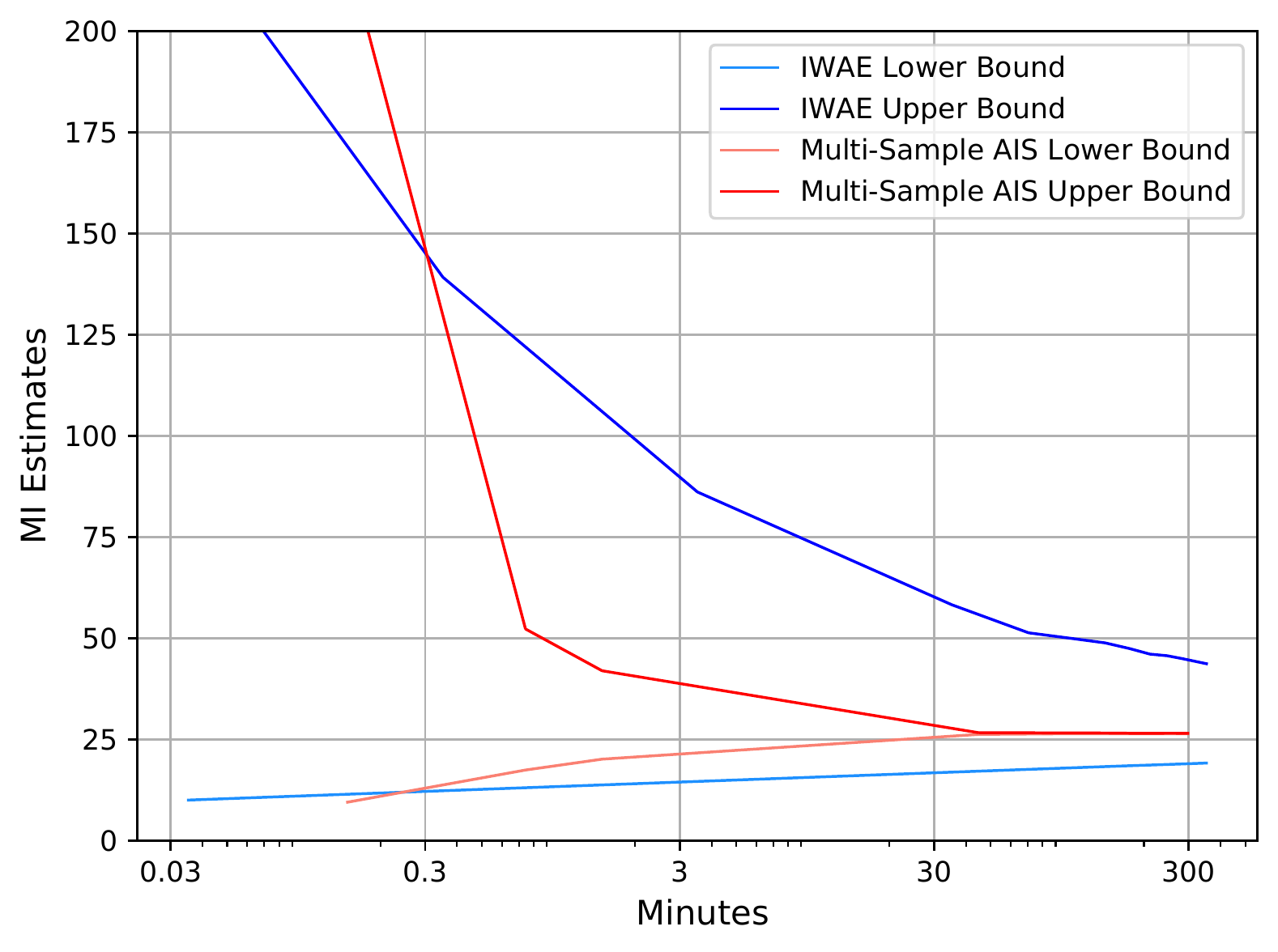}}
     \caption{
     {Runtime vs. Bound Tightness for \textsc{vae} (left) and \textsc{gan} (right) models on \textsc{mnist}. } }
     \label{fig:runtime_mnist}
 \end{figure}

\subsection{Experiment Details for Energy-Based Bounds (\mysec{exp_giwae})}

\parhead{Models and Data}The linear \textsc{vae} model has a Gaussian prior $\vz \sim \mathcal{N}(0,I)$ with the dimension of 10. The dimension of the output $\vx$ is $100$. The weights are sampled randomly from a Gaussian distribution with the standard deviation of $1$, and the standard deviation of the Gaussian observation noise at the output is $1$. The \textsc{mnist-vae}20, is a \textsc{vae} model that is trained on the real-valued MNIST dataset. It has a Gaussian prior $\vz \sim \mathcal{N}(0,I)$ with the dimension of 20. The decoder has one layer of ReLU non-linearity of size $1000$, followed by a linear layer that predicts the mean of a Gaussian observation model with the fixed standard deviation of $0.1$. \textsc{mnist-gan}20 decoder uses one layer of ReLU non-linearity of size $1000$ followed by a sigmoid layer that predicts the mean of a Gaussian observation model with the fixed standard deviation of $0.1$. In all the experiments, we use a fixed batch of $100$ data points.

\parhead{BA, IWAE and GIWAE}
For the \textsc{ba} bound, $\qzx$ is parameterized using a neural network that predicts the mean and log-std of a diagonal Gaussian distribution. The neural network has two layers of ReLU non-linearity of size $2000$. The \textsc{iwae} experiments used the same architecture for $\qzx$. The \textsc{giwae} experiments, in addition to $\qzx$ with the same architecture, used a critic function $\giwaeT(\vx,\vz)$ that is parameterized by a neural network that concatenates $(\vx,\vz)$ and pass them through two layers of ReLU non-linearity of size $2000$, followed by a linear layer that outputs a scalar value. The parameters of $\qzx$ and $\giwaeT(\vx,\vz)$ are trained jointly.

\parhead{Multi-Sample AIS} For the Multi-Sample \textsc{ais} bounds, for all models, we used up to $K=1000$ chains (see \myfig{multiple-ais-plot}), and up to $T=50K$ intermediate distributions with linear schedule (see \myfig{multiple-ais-plot} and \myfig{mine-ais}). We used \textsc{hmc} as the \textsc{ais} kernel, with $\epsilon=0.02$ and $L=20$ leap frog steps.

\parhead{MINE-AIS} For the training of the \textsc{mine-ais}, we used a critic function $\giwaeT(\vx,\vz)$ that is parameterized by a neural network that concatenates $(\vx,\vz)$ and pass them through three layers of ReLU non-linearity of size $2000$, followed by a linear layer that outputs a scalar value. We chose the Gaussian prior $\mathcal{N}(0,I)$ as the base distribution $\qzx$. In order to take a sample from $\pi_{\theta,\phi}(\vz|\vx)$, we used the \textsc{hmc} method that is initialized with a true posterior $p(\vz|\vx)$ sample, with $M=10$ iterations each with $L=20$ leapfrog steps. For the step size of \textsc{hmc}, we used $\epsilon=0.05$ for the linear \textsc{vae} model and $\epsilon=0.02$ for the \textsc{mnist-vae}20 and \textsc{mnist-gan}20. For the evaluation of \textsc{mine-ais}, we used Multi-Sample \textsc{ais} with the same parameters as the \textsc{ais} evaluation experiments described above.

\subsection{Analytical Solution of the Mutual Information on the Linear \textsc{mnist-vae}} 
\label{app:analytical} 
In order to verify our implementations, we have derived the \textsc{mi} analytically for the linear VAEs and verified that it matches the \textsc{mi} estimated by \textsc{ais}. For simplicity, we assume a fixed identity covariance matrix $\bI$ at the output of the conditional likelihood of the linear \textsc{vae} decoder, i.e., the decoder of the \textsc{vae} is simply: $\vx=\bW \bz+\bbb+\boldsymbol{\epsilon}$, where $\vx$ is the observation, $\vz$ is the latent code vector  $\vz \sim \mathcal{N} (\bzero,\,\bI)$, $\bb{W}$ is the decoder weight matrix and $\bbb$ is the bias. The observation noise of the decoder is $\boldsymbol{\epsilon} \sim \mathcal{N} (\bzero,\,\bI)$. It is easily shown that the conditional likelihood is $p(\vx | \vz) = \mathcal{N} (\vx | \bW \vz + \bbb, \bI)$ and thus we can solve for the marginal  
\begin{align}
p(\bx) = \mathcal{N} (\bx | \bmu_\bx = \bbb, \bSigma_\bx = \bI + \bW \bWT) .
\label{eq:px}
\end{align}
The differential entropy of $\vx$ is: 
\begin{align}
H(\bx) &= \frac{k}{2} +\frac{k}{2} \log(2 \pi) + \frac{1}{2} \log(\det \bSigma_\bx),
\label{eq:hx}
\end{align}
where k is the dimension of the observation. The conditional entropy is  
\begin{align}
H(\bx | \bz) &= \frac{k}{2} +\frac{k}{2} \log(2 \pi) + \frac{1}{2} \log(\det \bI)\\  &= \frac{k}{2} +\frac{k}{2} \log(2 \pi) .
\label{eq:hxz}
\end{align}

Thus, the mutual information is
\begin{align}
I(\bx; \bz) &= H(\bx) - H(\bx | \bz)\\
&=  \frac{1}{2} \log(\det \bSigma_\bx).
\label{eq:hx}
\end{align}

\subsection{Confidence Intervals for Multi-Sample AIS Experiments}\label{app:experiment-ci}
\mytable{mnistconf} and \mytable{cifarconf} provides the $95\%$ confidence intervals for the \textsc{ais} results reported in \mytable{mnist}. The confidence intervals were computed over confidence interval over 8 batches each with 16 data points for \textsc{mnist}; and over 8 batches 4 data points for \textsc{cifar}.
\begin{center}
\begin{table}[!ht]
\resizebox{9cm}{!}{
\begin{tabularx}{.99\textwidth}{c|cc|ccc|ccc}%
\cmidrule{1-9}
\textbf{Method}&\textbf{Proposal}&\textbf{}&\textbf{VAE2}&\textbf{VAE10}&\textbf{VAE100}&\textbf{GAN2}&\textbf{GAN10}&\textbf{GAN100}\\%
\cmidrule{1-9}%
\multirow{4}{*}{\shortstack[c]{\textbf{AIS} \\ T=1}}&\multirow{2}{*}{\textbf{$ p(\vz)$}}&UB&$(213.54, 286.11)$&$(1849.03, 2010.66)$&$(5564.53, 6096.51)$&$(665.25, 787.30)$&$(745.68, 826.55)$&$(795.81, 926.94)$\\%
&&LB&$(0.00, 0.00)$&$(0.00, 0.00)$&$(0.00, 0.00)$&$(0.00, 0.00)$&$(0.00, 0.00)$&$(0.00, 0.00)$\\%
\cmidrule{2-9}%
&\multirow{2}{*}{\textbf{$q(\vz|\vx)$}}&UB&$(8.77, 9.72)$&$(59.34, 66.66)$&$(345.45, 378.80)$&$(5.32, 32.95)$&$(293.60, 335.83)$&$(482.40, 544.26)$\\%
&&LB&$(7.47, 7.71)$&$(60.09, 66.61)$&$(32.46, 36.52)$&$(7.01, 7.41)$&$(3.41, 3.94)$&$(2.40, 2.83)$\\%
\cmidrule{1-9}%

\multirow{4}{*}{\shortstack[c]{\textbf{AIS} \\ T=500}}&\multirow{2}{*}{\textbf{$ p(\vz)$}}&UB&$(8.78, 9.45)$&$(38.61, 39.57)$&$(92.32, 98.03)$&$(10.52, 11.13)$&$(21.91, 23.04)$&$(26.89, 28.22)$\\%
&&LB&$(8.31, 8.95)$&$(33.69, 34.41)$&$(77.77, 82.02)$&$(8.89, 9.53)$&$(21.18, 21.96)$&$(25.36, 26.36)$\\%
\cmidrule{2-9}%
&\multirow{2}{*}{\textbf{$q(\vz|\vx)$}}&UB&$(8.77, 9.41)$&$(33.94, 34.63)$&$(79.96, 84.71)$&$(10.51, 11.07)$&$(22.53, 23.58)$&$(28.82, 30.24)$\\%
&&LB&$(8.78, 9.40)$&$(33.80, 34.53)$&$(78.01, 82.37)$&$(10.40, 10.98)$&$(21.19, 22.00)$&$(25.09, 26.06)$\\%
\cmidrule{1-9}%

\multirow{4}{*}{\shortstack[c]{\textbf{AIS} \\ T=30K}}&\multirow{2}{*}{\textbf{$ p(\vz)$}}&UB&$(8.74, 9.39)$&$(33.82, 34.60)$&$(78.62, 83.07)$&$(10.51, 11.11)$&$(21.54, 22.50)$&$(25.97, 27.07)$\\%
&&LB&$(8.65, 9.28)$&$(33.86, 34.56)$&$(78.51, 83.05)$&$(10.23, 10.88)$&$(21.50, 22.45)$&$(25.96, 26.98)$\\%
\cmidrule{2-9}%
&\multirow{2}{*}{\textbf{$q(\vz|\vx)$}}&UB&$(8.77, 9.40)$&$(33.85, 34.58)$&$(78.56, 83.03)$&$(10.52, 11.10)$&$(21.54, 22.47)$&$(26.01, 27.07)$\\%
&&LB&$(8.77, 9.41)$&$(33.86, 34.56)$&$(78.52, 83.02)$&$(10.51, 11.09)$&$(21.55, 22.48)$&$(26.02, 27.03)$\\%
\cmidrule{1-9}\morecmidrules\cmidrule{1-9}%

\multirow{4}{*}{\shortstack[c]{\textbf{IWAE} \\ K=1}}&\multirow{2}{*}{\textbf{$ p(\vz)$}}&UB&$(735.94, 863.16)$&$(3628.86, 4026.29)$&$(10705.98, 12297.86)$&$(1543.54, 1732.65)$&$(1556.00, 1704.00)$&$(1680.50, 1800.28)$\\%
&&LB&$(0.00, 0.00)$&$(0.00, 0.00)$&$(0.00, 0.00)$&$(0.00, 0.00)$&$(0.00, 0.00)$&$(0.00, 0.00)$\\%
\cmidrule{2-9}%
&\multirow{2}{*}{\textbf{$q(\vz|\vx)$}}&UB&$(8.83, 9.56)$&$(34.82, 35.86)$&$(92.82, 98.44)$&$(5.71, 29.45)$&$(38.80, 76.14)$&$(243.34, 278.41)$\\%
&&LB&$(8.47, 8.79)$&$(24.54, 25.86)$&$(41.53, 47.54)$&$(8.57, 9.10)$&$(3.85, 4.61)$&$(2.94, 3.52)$\\%
\cmidrule{1-9}%

\multirow{4}{*}{\shortstack[c]{\textbf{IWAE} \\ K=1K}}&\multirow{2}{*}{\textbf{$ p(\vz)$}}&UB&$(24.32, 34.49)$&$(1132.54, 1262.97)$&$(3987.51, 4480.86)$&$(101.75, 142.03)$&$(430.40, 463.19)$&$(462.95, 526.51)$\\%
&&LB&$(6.74, 6.87)$&$(6.91, 6.91)$&$(6.91, 6.91)$&$(6.84, 6.93)$&$(6.91, 6.91)$&$(6.91, 6.91)$\\%
\cmidrule{2-9}%
&\multirow{2}{*}{\textbf{$q(\vz|\vx)$}}&UB&$(8.78, 9.41)$&$(33.87, 34.61)$&$(83.15, 87.46)$&$(10.06, 12.74)$&$(34.11, 71.36)$&$(182.82, 219.55)$\\%
&&LB&$(8.79, 9.39)$&$(31.08, 32.30)$&$(48.44, 54.45)$&$(10.46, 11.02)$&$(10.75, 11.52)$&$(9.84, 10.43)$\\%
\cmidrule{1-9}%

\multirow{4}{*}{\shortstack[c]{\textbf{IWAE} \\ K=1M}}&\multirow{2}{*}{\textbf{$ p(\vz)$}}&UB&$(8.76, 9.43)$&$(352.91, 400.87)$&$(2078.30, 2417.17)$&$(10.62, 11.36)$&$(71.93, 91.09)$&$(100.27, 127.75)$\\%
&&LB&$(8.77, 9.41)$&$(13.82, 13.82)$&$(13.82, 13.82)$&$(10.50, 11.03)$&$(13.79, 13.82)$&$(13.82, 13.82)$\\%
\cmidrule{2-9}%
&\multirow{2}{*}{\textbf{$q(\vz|\vx)$}}&UB&$(8.77, 9.40)$&$(33.87, 34.57)$&$(81.42, 85.36)$&$(10.53, 11.10)$&$(28.51, 33.26)$&$(52.41, 63.68)$\\%
&&LB&$(8.78, 9.41)$&$(33.67, 34.53)$&$(55.35, 61.36)$&$(10.53, 11.10)$&$(17.31, 18.21)$&$(16.63, 17.32)$\\%
\cmidrule{1-9}%

\end{tabularx} 
}
\hspace{-.5cm}
\caption{Confidence intervals of \textsc{ais} and \textsc{iwae} estimates of \textsc{mi} on \textsc{mnist}. UB stands for Upper Bound, and LB stands for Lower Bound.}
\label{table:mnistconf}
\end{table} 
\end{center}

\begin{table}[t!]
\begin{center}
\resizebox{10cm}{!}{
\hspace{-2.5cm}
\begin{tabularx}{\textwidth}{c|cc|ccc}%
\cmidrule{1-6}%
\textbf{Model}&\textbf{Proposal}&\textbf{}&\textbf{GAN5}&\textbf{GAN10}&\textbf{GAN100}\\%
\cmidrule{1-6}%
\multirow{4}{*}{\shortstack[c]{\textbf{AIS} \\ (T=1)}}&\multirow{2}{*}{\textbf{$ p(\vz)$}}&UB&$(2875853.34, 4326658.66)$&$(3593559.12, 4477712.38)$&$(4038603.09, 5668217.91)$\\%
&&LB&$(0.00, 0.00)$&$(0.00, 0.00)$&$(0.00, 0.00)$\\%
\cmidrule{2-6}%
&\multirow{2}{*}{\textbf{$q(\vz|\vx)$}}&UB&$(30896.74, 54151.34)$&$(241342.62, 566015.82)$&$(1874938.58, 2881576.42)$\\%
&&LB&$(13.24, 14.64)$&$(15.70, 18.89)$&$(18.41, 21.93)$\\%
\cmidrule{1-6}%

\multirow{4}{*}{\shortstack[c]{\textbf{AIS} \\ (T=500)}}&\multirow{2}{*}{\textbf{$ p(\vz)$}}&UB&$(205.11, 313.44)$&$(24683.10, 41496.70)$&$(25453.02, 101127.79)$\\%
&&LB&$(17.44, 19.30)$&$(28.41, 30.63)$&$(98.37, 110.65)$\\%
\cmidrule{2-6}%
&\multirow{2}{*}{\textbf{$q(\vz|\vx)$}}&UB&$(59.40, 79.68)$&$(116.02, 156.28)$&$(1603.13, 3969.92)$\\%
&&LB&$(31.63, 33.30)$&$(45.45, 50.88)$&$(134.35, 156.02)$\\%
\cmidrule{1-6}%

\multirow{4}{*}{\shortstack[c]{\textbf{AIS} \\ (T=100K)}}&\multirow{2}{*}{\textbf{$ p(\vz)$}}&UB&$(40.55, 41.57)$&$(72.17, 75.80)$&$(479.04, 497.10)$\\%
&&LB&$(39.01, 40.15)$&$(70.19, 73.55)$&$(470.05, 490.47)$\\%
\cmidrule{2-6}%
&\multirow{2}{*}{\textbf{$q(\vz|\vx)$}}&UB&$(39.55, 40.56)$&$(71.87, 75.22)$&$(475.07, 494.61)$\\%
&&LB&$(38.64, 39.80)$&$(71.35, 74.75)$&$(468.64, 489.89)$\\%
\cmidrule{1-6}\morecmidrules\cmidrule{1-6}%
\multirow{4}{*}{\shortstack[c]{\textbf{IWAE} \\ (K=1)}}&\multirow{2}{*}{\textbf{$ p(\vz)$}}&UB&$(6115557.66, 8075210.34)$&$(6019901.15, 9511489.86)$&$(7339411.95, 12492792.05)$\\%
&&LB&$(0.00, 0.00)$&$(0.00, 0.00)$&$(0.00, 0.00)$\\%
\cmidrule{2-6}%
&\multirow{2}{*}{\textbf{$q(\vz|\vx)$}}&UB&$(39.68, 40.95)$&$(75.00, 80.03)$&$(-1810.70,12504.41)$\\%
&&LB&$(14.15, 14.90)$&$(16.28, 18.62)$&$(18.39, 21.60)$\\%
\cmidrule{1-6}%

\multirow{4}{*}{\shortstack[c]{\textbf{IWAE} \\ (K=1K)}}&\multirow{2}{*}{\textbf{$ p(\vz)$}}&UB&$(921258.02, 1209846.48)$&$(1673333.03, 2415008.47)$&$(2181660.84, 3531768.16)$\\%
&&LB&$(6.91, 6.91)$&$(6.91, 6.91)$&$(6.91, 6.91)$\\%
\cmidrule{2-6}%
&\multirow{2}{*}{\textbf{$q(\vz|\vx)$}}&UB&$(39.40, 40.06)$&$(72.31, 75.68)$&$(-1872.12, 12438.39)$\\%
&&LB&$(21.05, 21.81)$&$(22.18, 24.99)$&$(25.32, 28.64)$\\%
\cmidrule{1-6}%

\multirow{4}{*}{\shortstack[c]{\textbf{IWAE} \\ (K=1M)}}&\multirow{2}{*}{\textbf{$ p(\vz)$}}&UB&$(76776.88, 116619.32)$&$(582880.62, 838142.63)$&$(1479829.44, 2327879.56)$\\%
&&LB&$(13.82, 13.82)$&$(13.82, 13.82)$&$(13.82, 13.82)$\\%
\cmidrule{2-6}%
&\multirow{2}{*}{\textbf{$q(\vz|\vx)$}}&UB&$(39.26, 40.17)$&$(71.63, 75.08)$&$(-1883.22
, 12426.34)$\\%
&&LB&$(27.96, 28.72)$&$(29.86, 31.60)$&$(32.21, 35.42)$\\%
\cmidrule{1-6}%
\end{tabularx}

}
\end{center}
\caption{Confidence intervals of \textsc{ais} and \textsc{iwae} estimates of \textsc{mi} on \textsc{cifar}. UB stands for Upper Bound, and LB stands for Lower Bound.
}

\label{table:cifarconf}
\end{table}

\end{document}